\documentclass[twoside]{article}

\usepackage[accepted]{aistats_template/aistats2023}
\usepackage{url}
\usepackage{hyperref}    
\usepackage{booktabs}        
\usepackage{tabularx}       
\usepackage{amsmath,amssymb,amsthm,amsfonts,mathtools,bbm}
\usepackage{nicefrac}    
\usepackage{enumerate}
\usepackage{graphicx}
\usepackage{subfigure}
\usepackage{mathtools}
\usepackage[ruled,vlined]{algorithm2e}
\usepackage{todonotes}
\usepackage{microtype}
\usepackage{minitoc}
\usepackage{comment,bm}
\usepackage[toc,page,header]{appendix}
\usepackage{textcomp}

\usepackage{tikz}

\newcommand*\circled[1]{\tikz[baseline=(char.base)]{\node[shape=circle,draw,inner sep=2pt] (char) {#1};}}

\newcommand{\aggregation}{F}

\newcommand{\cenna}{\ensuremath{\text{NNM}}}

\newcommand{\tmean}[1]{\text{CWTM}{\left(#1\right)}}

\DeclarePairedDelimiter\ceil{\lceil}{\rceil}
\DeclarePairedDelimiter\floor{\lfloor}{\rfloor}

\DeclareMathOperator*{\argmax}{arg\,max}
\DeclareMathOperator*{\argmin}{arg\,min}

\def\D{\mathcal D}
\def\R{\mathbb{R}}
\def\N{\mathbb{N}}

\def\P{\mathcal{P}}

\def\H{\mathcal{H}}

\def\Z{\mathcal{Z}}

\newtheorem{assumption}{Assumption}
\newtheorem{theorem}{\bfseries Theorem}
\newtheorem{lemma}{\bfseries Lemma}
\newtheorem{proposition}{Proposition}
\newtheorem{corollary}{Corollary}
\newtheorem{definition}{Definition}
\newtheorem{remark}{Remark}
\newtheorem{observation}{Observation}

\makeatletter
\newtheorem*{rep@theorem}{\rep@title}
\newcommand{\newreptheorem}[2]{%
\newenvironment{rep#1}[1]{%
 \def\rep@title{#2 \ref{##1}}%
 \begin{rep@theorem}}%
 {\end{rep@theorem}}}
\makeatother
\newreptheorem{theorem}{Theorem}

\makeatletter
\newtheorem*{rep@corollary}{\rep@title}
\newcommand{\newrepcorollary}[2]{%
\newenvironment{rep#1}[1]{%
 \def\rep@title{#2 \ref{##1}}%
 \begin{rep@corollary}}%
 {\end{rep@corollary}}}
\makeatother
\newreptheorem{corollary}{Corollary}

\makeatletter
\newtheorem*{rep@lemma}{\rep@title}
\newcommand{\newreplemma}[2]{%
\newenvironment{rep#1}[1]{%
 \def\rep@title{#2 \ref{##1}}%
 \begin{rep@lemma}}%
 {\end{rep@lemma}}}
\makeatother
\newreptheorem{lemma}{Lemma}

\makeatletter
\newtheorem*{rep@proposition}{\rep@title}
\newcommand{\newrepproposition}[2]{%
\newenvironment{rep#1}[1]{%
 \def\rep@title{#2 \ref{##1}}%
 \begin{rep@proposition}}%
 {\end{rep@proposition}}}
\makeatother
\newreptheorem{proposition}{Proposition}

\makeatletter
\newtheorem*{rep@definition}{\rep@title}
\newcommand{\newrepdefinition}[2]{%
\newenvironment{rep#1}[1]{%
 \def\rep@title{#2 \ref{##1}}%
 \begin{rep@definition}}%
 {\end{rep@definition}}}
\makeatother
\newreptheorem{definition}{Definition}

\definecolor{gainsboro}{rgb}{0.86, 0.86, 0.86}

\newcommand{\expect}[1]{\mathop{{}\mathbb{E}}\left[{#1}\right]}
\newcommand{\condexpect}[2]{\mathbb{E}_{#1}\left[{#2}\right]}

\newcommand{\suchthat}{\ensuremath{~\middle|~}}
\newcommand{\knowing}{\suchthat{}}
\newcommand{\card}[1]{\left\lvert{#1}\right\rvert}
\newcommand{\absv}[1]{\card{#1}}

\newcommand{\norm}[1]{\left\lVert{#1}\right\rVert}

\newcommand{\indexvar}[3]{\ensuremath{{{#3}^{\ifthenelse{\equal{#1}{}}{}{\left({#1}\right)}}_{#2}}}}
\newcommand{\indexvarNoPar}[3]{\ensuremath{{{#3}^{\ifthenelse{\equal{#1}{}}{}{\left{#1}\right}}_{#2}}}}

\newcommand{\params}[2]{\indexvarNoPar{#1}{#2}{\theta}}

\providecommand{\iprod}[2]{\ensuremath{\left\langle #1,\,#2  \right\rangle}}
\providecommand{\mnorm}[1]{\ensuremath{\left\lvert#1\right\rvert}}
\providecommand{\norm}[1]{\ensuremath{\left\lVert#1\right\rVert }}

\newcommand{\loss}{\mathcal{L}}
\newcommand{\weight}[1]{\params{}{#1}}

\newcommand{\gradient}[2]{\indexvar{#1}{#2}{g}}

\newcommand{\proba}[2]{\ensuremath{\text{P}\!\left({#1}\ifthenelse{\equal{#2}{}}{}{\knowing{}{#2}}\right)}}

\renewcommand{\paragraph}[1]{\textbf{#1}~}

\newcommand{\drift}[1]{\xi_{#1}}
\newcommand{\dev}[1]{\delta_{#1}}
\newcommand{\mmt}[2]{m^{(#1)}_{#2}}

\newcommand{\AvgMmt}[1]{\overline{m}_{#1}}

\newcommand{\bucketSize}{s}

\usepackage{xcolor}
\usepackage{soul}
\newcommand{\mathcolorbox}[2]{\colorbox{#1}{$\displaystyle #2$}}
\usepackage{fancybox}
\usepackage{mdframed}

\begin{document}

\runningauthor{Allouah, Farhadkhani, Guerraoui, Gupta, Pinot, Stephan}

\twocolumn[
\aistatstitle{Fixing by Mixing: A Recipe for Optimal Byzantine ML under Heterogeneity}
\aistatsauthor{Youssef Allouah$^{\dag *}$ \And Sadegh Farhadkhani$^\dag$ \And Rachid Guerraoui$^\dag$} \vspace{3pt}
\aistatsauthor{Nirupam Gupta$^\dag$\And Rafaël Pinot$^\dag$\And  John Stephan$^\dag$ }  
\vspace{25pt}
]

\begin{abstract}

    Byzantine machine learning (ML) aims to ensure the resilience of distributed learning algorithms to misbehaving (or \emph{Byzantine}) machines. Although this problem received significant attention, prior works often assume the data held by the machines to be {\em homogeneous}, which is seldom true in practical settings. Data \emph{heterogeneity} makes Byzantine ML considerably more challenging, since a Byzantine machine can hardly be distinguished from a non-Byzantine outlier. A few solutions have been proposed to tackle this issue, but these provide suboptimal probabilistic guarantees and fare poorly in practice.
    
    This paper closes the theoretical gap, achieving optimality and inducing good empirical results. 
    In fact, we show how to automatically adapt existing solutions for (homogeneous) Byzantine ML to the heterogeneous setting through a powerful mechanism, we call  {\em nearest neighbor mixing} (NNM), which boosts  any standard robust distributed gradient descent variant to yield optimal Byzantine resilience under heterogeneity.
    We obtain similar guarantees (in expectation) by plugging NNM in the distributed {\em stochastic} heavy ball method, a practical substitute to distributed gradient descent. We obtain empirical results that significantly outperform state-of-the-art Byzantine ML solutions.
\end{abstract}

\section{Introduction}

Distributed machine learning (ML), i.e., distributing the training process amongst several machines (or {\em workers}), has been pivotal to the development of large complex models with high accuracy guarantees~\cite{Tensorflow2015,konevcny2016federated,OpenProblemsinFed2021}. In the now standard {\em master-worker} architecture, distributed ML essentially consists in the workers sharing their local actions with the help of a master machine (a.k.a., {\em server}) to compute an accurate {\em global} model. Despite its rising popularity, distributed ML is arguably very fragile and not yet ready for real-world deployment. In particular, a handful of misbehaving (a.k.a.,~{\em Byzantine}~\cite{lamport82}) workers can highly compromise the efficacy of standard distributed ML algorithms such as distributed gradient descent (D-GD)~\cite{bertsekas2015parallel}, by sending erroneous information to the server, e.g., see~\cite{little, empire}. Such behaviors can be caused by software and hardware bugs, by poisoned data or by malicious players controlling part of the system. 
Besides, this vulnerability can lead to severe societal repercussions if the resulting ML models are deployed in sensitive data-oriented applications such as medicine.

The problem of Byzantine resilience in distributed ML (or {\em Byzantine ML}) has received significant attention in recent years~\cite{krum,chen17,yin2018,allen2020byzantine,collaborativeElMhamdi21,Karimireddy2021,karimireddy2022byzantinerobust,farhadkhani2022byzantine}. It consists in designing a distributed algorithm that delivers an accurate model despite the presence of a subset of Byzantine workers.
The standard approach
(in a {\em master-worker} architecture) 
consists in having the server compute a robust aggregation to merge the information sent by the workers, to discard outliers. Most prior works relies however upon the strong assumption of \emph{homogeneity}; i.e, the data sampled by the workers during the training process are assumed to be identically distributed. This assumption, although justifiable in a centralized setting~\cite{bottou2018optimization}, is impractical in a distributed environment~\cite{OpenProblemsinFed2021}. Indeed, as each worker only holds a small part of the entire training dataset, the data samples at the workers are usually {\em heterogeneous} and need not be an accurate representation of the entire population.
These differences in the workers' data samples can camouflage disruptive deviations of Byzantine machines from the prescribed algorithm,  making the problem of Byzantine ML under heterogeneity significantly more challenging than its homogeneous counterpart~\cite{collaborativeElMhamdi21, karimireddy2022byzantinerobust}.

Recent works have shown that   heterogeneity in workers' data inevitably prevents any distributed ML algorithm from delivering an arbitrarily accurate model in the presence of Byzantine workers~\cite{charikar2017learning, liu2021approximate}. 
More specifically, in the context of smooth loss functions, a Byzantine ML solution can only deliver an approximate stationary point with an error lower bounded by the fraction of Byzantine workers times the maximal dispersion in workers' gradients (stemming from data heterogeneity)~\cite{karimireddy2022byzantinerobust}.
Despite a few attempts~\cite{gupta2020fault,collaborativeElMhamdi21,karimireddy2022byzantinerobust}, no  solution has so far matched the above lower bound deterministically, i.e., achieved \emph{optimal Byzantine resilience}. 
The best solutions so far~\cite{karimireddy2022byzantinerobust} are  randomized and could only match the lower bound in \emph{expectation}
(see Section~\ref{sec:back}).

\paragraph{Main result.} We close the theoretical gap and reach optimal Byzantine resilience under heterogeneity. Specifically, 
we show how to automatically adapt existing solutions for homogeneous Byzantine ML to the heterogeneous setting, while ensuring optimality. 
We do so through \textit{nearest neighbor mixing} ($\cenna$), a pre-aggregation method that averages each input with a subset of their nearest neighbors. 
We show that enhancing D-GD using a composition of $\cenna$ and a standard robust aggregation 
directly yields the first Byzantine ML solution to achieve optimal resilience under heterogeneity. Our guarantee holds as long as less than half of the workers are Byzantine, which is optimal. 

\paragraph{Technical contributions.} To prove our guarantees, we introduce a novel robustness criterion called {\em $(f,\kappa)$-robustness}. This notion quantifies the ability of an aggregation rule to estimate the average of honest workers' inputs despite $f$ out of $n$ workers being Byzantine. 
Crucially, our criterion characterizes a class of aggregations rules (for which $\kappa = \mathcal{O}\left( f/n\right)$) that grant optimal Byzantine resilience to D-GD under heterogeneity. 
While many notable aggregation rules (e.g., {\em geometric median}~\cite{small1990survey}, {\em coordinate-wise median}~\cite{yin2018}, and {\em Krum}~\cite{krum}) satisfy $(f,\kappa)$-robustness, they fall short of optimal robustness, as their $\kappa$ is in $\Theta{\left( 1\right)}$. 
Our main technical contribution is showing that $\cenna$ overcomes this shortcoming.
Particularly, we prove that $\cenna$ deterministically reduces the variance of the honest inputs by a factor $\mathcal{O}{(f/n)}$, while sufficiently limiting their drift from the true average. 
Consequently, we show that composing a $(f, \mathcal{O}(1))$-robust aggregation with $\cenna$ 
enables the larger class of $(f, \mathcal{O}(1))$-robust aggregation rules to confer optimal Byzantine resilience to D-GD.

\paragraph{Empirical evaluation.} 
Although we prove D-GD enhanced by $\cenna$ to be optimal, workers still need to compute the {\em full} gradient of their local loss function at every step, and risk getting stuck at saddle points~\cite{du2017gradient,bottou2018optimization}. 
We go one step further and show that applying $\cenna$ to a stochastic variant of D-GD, namely the distributed stochastic heavy ball method (D-SHB)~\cite{polyak1964some},
matches the lower bound in expectation\footnote{The expectation is on the randomness due to data subsampling, unlike~\cite{karimireddy2022byzantinerobust} which features additional sources of randomness that hinder empirical performance. 
See Section~\ref{sec:dshb-analysis} for details.}.
We empirically show that the resulting scheme significantly improves over the state-of-the-art
when tested on standard classification tasks such as MNIST and CIFAR-10.
In short, our approach enables a modular practice of Byzantine resilient ML under heterogeneity, by first patching existing solutions with $\cenna$ and then deploying them with D-SHB.

\paragraph{Paper outline.}
Section~\ref{sec:back} formally presents the problem of Byzantine ML and discusses the related work. Section~\ref{sec:dgd} introduces the analysis framework for robust variants of D-GD. Section~\ref{sec:optimal} presents $\cenna$, our solution to render existing methods optimal.
Section~\ref{sec:stochastic} presents a practical stochastic extension of our approach.
Section~\ref{sec:exp} presents our experimental results. 
We defer all proofs to Appendix~\ref{app:res-averaging}-\ref{app:convergence-dshb}.

\section{Background \& Related Work}
\label{sec:back}
We consider a master-worker distributed architecture with $n$ workers $w_1, \dots ,w_n$ and a central server. The workers hold local datasets $\D_1, \dots,\D_n$, composed each of $m$ data points from an input space $\Z$. Specifically, for any $i \in [n]$, $\D_i \coloneqq \{z_1^{(i)}, \dots, z_m^{(i)}\} \subset \Z^m$. For a given model parameterized by vector $\theta \in \R^d$, each worker $w_i$ has a loss function
\begin{equation*}
    \loss_i{(\theta)} \coloneqq \frac{1}{m}\sum_{k=1}^m \ell (\theta, z_k^{(i)} ),
\end{equation*}
where $\ell: \mathbb{R}^d \times \Z \rightarrow \R$ is a point-wise loss function, which we assume to be differentiable with respect to $\theta$. 
Furthermore, we assume each loss function $\loss_i$ to be $L$-smooth; i.e., for all $\theta, \theta' \in \R^d$,
\begin{align*}
    \norm{\nabla \loss_{i}(\theta{}) - \nabla \loss_{i}(\theta{}')} \leq L \norm{\theta{} - \theta{}'}.
\end{align*}

\paragraph{Objective in a Byzantine-free setting.} 
When all the workers are honest, i.e., they follow the prescribed algorithm correctly, the goal of the server is to compute a stationary point of the global loss function\footnote{Ideally, the server seeks to find a global minimizer of $\loss$. However, as the loss could be non-convex (e.g., for neural networks), global minimization is NP-hard in general~\cite{boyd2004convex}.
Hence, we usually aim to find stationary points.} $\loss(\theta) \coloneqq \frac{1}{n} \sum_{i = 1}^n \loss_i{(\theta)}$.
Specifically, the server seeks to compute a stationary point $\weight{}^*$, i.e. $\nabla \loss(\weight{}^*) = 0$, where $\nabla{\loss}$ denotes the {\em gradient} of $\loss$.
Assuming that $\loss$ admits a stationary point, this goal can be achieved using the classical distributed (stochastic) gradient descent (D-GD/SGD) method~\cite{bertsekas2015parallel}, up to an approximation error.
In this method, the server maintains a model which is updated iteratively upon averaging the (stochastic) gradients computed by the workers on their loss functions.

\paragraph{Objective with Byzantine workers.} 
We consider an adversarial setting where $f$ workers, of a priori unknown identities, are Byzantine~\cite{lamport82}. These workers need not follow the prescribed protocol and may send arbitrary messages to the server. Here, the goal of the server is finding a stationary point of the average loss function of the honest (i.e., non-Byzantine) workers. We formally define Byzantine resilience in this context as follows.

\vspace{5pt}
\begin{definition}[{\bf $(f, \varepsilon)$-Byzantine resilience}]
\label{def:resilience}
A learning algorithm is said {\em $(f, \, \varepsilon)$-Byzantine resilient} if, despite the presence of $f$ Byzantine workers, it outputs $\hat{\theta}$ such that 
$$\norm{ \nabla \loss_{\H}(\hat{\theta}) }^2 \leq \varepsilon,$$
where $\H$ denotes the set of indices of honest workers and $\loss_{\H}{(\theta{})} \coloneqq \frac{1}{\absv{\H}} \sum_{i \in \H} \loss_i(\theta{})$.
\end{definition}

In words, a $(f, \varepsilon)$-Byzantine resilient algorithm finds an {\em $\varepsilon$-approximate} stationary point for the honest empirical loss even in the presence of $f$ Byzantine workers.
Note that $(f,\varepsilon)$-Byzantine resilience is impossible in general (for any $\varepsilon$) when $f \geq \nicefrac{n}{2}$, see~\cite{liu2021approximate}. Therefore, throughout the paper, we assume that $f < \nicefrac{n}{2}$.

\paragraph{Standard Byzantine ML solutions.} A typical strategy to obtain $(f,\varepsilon)$-Byzantine resilience is to enhance the D-GD (or D-SGD) method by replacing the simple averaging of the workers’ gradients at the server by a \emph{robust aggregation rule}. Basically, such a scheme aims to mitigate the negative impact of Byzantine gradients by accurately estimating the average of honest workers' gradients. Prominent aggregation rules include Krum~\cite{krum}, geometric median (GM)~\cite{small1990survey,pillutla2019robust,acharya2021}, coordinate-wise median (CWMed)~\cite{yin2018}, coordinate-wise trimmed mean (CWTM)~\cite{yin2018}, and minimum diameter averaging (MDA)~\cite{rousseeuw1985multivariate,collaborativeElMhamdi21}. Recent studies have also explored the idea of using the {\em history} of workers' gradients to strengthen the resilience guarantees of these solutions, e.g., by using distributed momentum~\cite{karimireddy2022byzantinerobust,farhadkhani2022byzantine} or by tracking the history of the workers' gradients~\cite{allen2020byzantine}. 

Nevertheless, all these works rely heavily on the assumption that the honest workers have {\em homogeneous} data, i.e., there exists a {\em ground-truth distribution} $\mathcal{D}$ such that $D_i \sim \mathcal{D}^m$ for all $i \in \H$. Crucially, the robustness of these methods deteriorates drastically when this assumption is violated~\cite{karimireddy2022byzantinerobust}.

\paragraph{The challenge of {\em heterogeneity} in Byzantine ML.} The data across the honest workers is arguably {\em heterogeneous} in real-world distributed settings. In the context of non-convex optimization, data heterogeneity can be modeled as stated per the following assumption~\cite{karimireddy2020scaffold,karimireddy2022byzantinerobust}.
\begin{assumption}[Bounded heterogeneity]
\label{asp:hetero}
There exists a real value $G$ such that for all $\theta{} \in \mathbb{R}^d$,
\begin{align*}
    \frac{1}{\card{\H}} \sum_{i \in \H}\norm{\nabla \loss_i(\theta{}) - \nabla \loss_{\H}(\theta{})}^2 \leq G^2.
\end{align*}
\end{assumption}
Such heterogeneity makes the problem of Byzantine ML much more challenging, as the server can confuse incorrect gradients (from a Byzantine worker) with correct gradients from an honest worker holding outlier data points.
Indeed, recent works show that there exists a lower bound on the error of any distributed algorithm in the presence of Byzantine workers~\cite{collaborativeElMhamdi21, karimireddy2022byzantinerobust}. 
Specifically, we have the following result, owing to Theorem~III in \cite{karimireddy2022byzantinerobust}.
\begin{proposition}
\label{prop:lower-bound}
If a learning algorithm $\mathcal{A}$ is $(f,\varepsilon)$-Byzantine resilient for every collection of smooth loss functions $\loss_1, \ldots, \loss_n$ satisfying Assumption~\ref{asp:hetero},
then $\varepsilon = \Omega \left(\nicefrac{f}{n}\, G^2 \right)$.
\end{proposition}

\paragraph{Brittleness of existing solutions for heterogeneity.} Only a handful of prior works have studied Byzantine ML under heterogeneity. The problem was first formally addressed in~\cite{li19} by proposing {\em RSA}, a variant of D-SGD built upon $\ell_p$ regularization. However, the analysis of \emph{RSA} relies on the assumption of strong convexity, which is rarely satisfied in modern-day ML. A subsequent work~\cite{gupta2020fault} introduced a novel aggregation rule, namely comparative gradient elimination (CGE). While this work provides some valuable insight on the general problem of heterogeneous Byzantine ML (notably on the need for redundancy), CGE 
fails to guarantee convergence
even in the homogeneous setting (see~\cite{farhadkhani2022byzantine}). The work~\cite{collaborativeElMhamdi21} considers a peer-to-peer setting with asynchronous communication and heterogeneous data. However, the proposed algorithm is only analyzed asymptotically and provides suboptimal probabilistic robustness.
In this peer-to-peer setting, another work used a nearest neighbor scheme similar to ours~\cite{farhadkhani2022making}, with honest nodes using their own honest vectors as pivots. However, in our master-worker setting, their technique cannot be used as the server does not have access to any reliable vector.

The closest relevant work to ours is~\cite{karimireddy2022byzantinerobust}, which also proposes a pre-aggregation step called \emph{Bucketing}, which is reminiscent of the {\em median-of-means} estimator~\cite{nemirovski1983problem,jerrum1986random,alon1996space, tu2021variance}. Essentially, Bucketing consists in randomly partitioning the inputs into {\em buckets}, and then feeding the average of the inputs of each bucket to the robust aggregation rule. The randomness of the partition process reduces the {\em empirical variance} of the honest inputs in expectation, but it highly compromises the worst-case robustness of the scheme (see Appendix~\ref{app:Bucketing} for a detailed explanation). Furthermore, our experiments expose the inability of these methods to defend against state-of-the-art Byzantine attacks~\cite{empire,little,allen2020byzantine, karimireddy2022byzantinerobust}.
\section{Robust Distributed Gradient Descent}
\label{sec:dgd}
In this section, we introduce a general framework for analyzing the convergence of robust variants of D-GD (or {\em robust D-GD}). We use this framework later in Section~\ref{sec:optimal} to showcase the benefits of our proposed pre-aggregation step $\cenna$. We start by presenting the skeleton of robust D-GD in Section~\ref{sec:RDGD} below. Then, we present the convergence analysis of robust D-GD in Section~\ref{sec:analysisRDGD} upon introducing the notion of $(f,\kappa)$-robustness.

\subsection{Description of Robust D-GD}
\label{sec:RDGD}

\begin{table*}[t!]
\centering
\def\arraystretch{2}
\begin{tabular}{c|cccc|c} 
\textbf{Aggregation}  & GM & CWTM  & CWMed  & Krum & \textbf{Lower bound}\\ 
\hline
\textbf{$\kappa$}& $\left(1+\frac{f}{n-2f}\right)^2$ & $\frac{f}{n-2f}\left(1+\frac{f}{n-2f}\right)$ & $\left(1+\frac{f}{n-2f}\right)^2$ & $1+\frac{f}{n-2f}$ & $\frac{f}{n-2f}$\\
\end{tabular}
\caption{Robustness coefficient $\kappa$ for different $(f,\kappa)$-robust aggregation rules, ignoring numerical constants.
The exact robustness coefficients are deferred to Appendix~\ref{app:gars}.
Note also that our analysis is tight and significantly improves over previous works~\cite{farhadkhani2022byzantine,karimireddy2022byzantinerobust}, e.g., eliminating the dependence on dimension for CWMed and CWTM.
}
\label{tab:kappa}
\end{table*}
Robust D-GD, summarized below in Algorithm~\ref{gd}, is an iterative algorithm that proceeds in $T$ steps.

\begin{algorithm}[ht!]
\caption{Robust D-GD}
\label{gd}
\textbf{Input:} 
Initial model $\theta_0$,
robust aggregation $F$,
learning rate $\gamma$, and number of steps $T$.\\

\For{$t=1 \dots T$}{
\textcolor{blue!80}{{\textbf{Server}}} broadcasts $\theta_{t-1}$ to all workers\;
\For{\textbf{\textup{every}} \textcolor{orange}{\textbf{honest worker}, $\mathbf{i}$} \textbf{\textup{in parallel}}}{

Compute and send gradient $g_t^{(i)} = \nabla{\loss_i{(\theta_{t-1})}}$\;
}
\noindent

\textcolor{blue!80}{{\textbf{Server}}} aggregates the gradients:
$R_t = F{(g_t^{(1)}, \ldots, g_t^{(n)})}$\;

\textcolor{blue!80}{{\textbf{Server}}} updates the model:
$\theta_{t} = \theta_{t-1} - \gamma R_t$\;
}

\textcolor{blue!80}{{\textbf{Server}}} finds $\tau \in \argmin\limits_{1 \leq t \leq T} \norm{R_t}$\, and sets $\hat{\theta} = \theta_{\tau-1}$.

\textbf{return} $\hat{\theta}$\;
\end{algorithm}

Essentially, the server starts by initializing a parameter vector $\theta_0$. At each step $t \in [T]$, the server broadcasts the model $\theta_{t-1}$ to all the workers. After receiving $\theta_{t-1}$, each honest worker sends back the gradient $g_t^{(i)} = \nabla{\loss_i{(\theta_{t-1})}}$, computed by evaluating $\theta_{t-1}$ on their local dataset $\D_i$. While the honest workers follow the algorithm correctly, a Byzantine worker $w_i$ may send any arbitrary value for $g_t^{(i)}$.
Upon receiving the gradients from all the workers, the server aggregates them using a robust aggregation rule $F \colon \R^{d \times n} \to \R^d$. Specifically, the server computes
\begin{equation*}
    R_t = F{\left(g_t^{(1)}, \ldots, g_t^{(n)}\right)}.
\end{equation*}
Finally, the server updates the model to 
$\theta_{t} = \theta_{t-1} - \gamma R_t,$ where $\gamma > 0$ is referred to as the learning rate.  
After $T$ steps, the server outputs the model $\theta_t$ for which the associated aggregate $R_t$ has the smallest norm. That is, the algorithm outputs 
$\hat{\theta}= \theta_{\tau-1},\text{ where }\tau \in \argmin\limits_{1 \leq t \leq T} \norm{R_t}.$

\subsection{Analysis of Robust D-GD}
\label{sec:analysisRDGD}

At the core of robust D-GD lies the aggregation rule $F$. To tightly analyze the utility of $F$, we introduce the notion of $(f,\kappa)$-robustness which unifies previous robustness criteria and is sufficiently fine-grained to obtain tight convergence guarantees.  In words, $(f,\kappa)$-robustness ensures that the error of an aggregation rule, in estimating the average of the honest inputs, is \emph{uniformly} bounded by $\kappa$ times the variance of honest inputs. Formally, it is defined as follows.

\begin{definition}[{\bf $(f, \kappa)$-robustness}]
\label{def:resaveraging}
Let $f < \nicefrac{n}{2}$ and $\kappa \geq 0$. An aggregation rule $\aggregation$ is said to be {\em $(f, \kappa)$-robust} if for any vectors $x_1, \ldots, \, x_n \in \R^d$, and any set $S \subseteq [n]$ of size $n-f$, 
\begin{align*}
    \norm{\aggregation(x_1, \ldots, \, x_n) - \overline{x}_S}^2 \leq \frac{\kappa}{\card{S}} \sum_{i \in S} \norm{x_i - \overline{x}_S}^2
\end{align*}
where $\overline{x}_S = \frac{1}{\card{S}} \sum_{i \in S} x_i$. We refer to $\kappa$ as the {\em robustness coefficient}.
\end{definition}
Our criterion unifies the existing robustness definitions including {\em $(f,\lambda)$-resilient averaging}~\cite{farhadkhani2022byzantine} and {\em $(\delta_{max},c)$-ARAgg}~\cite{karimireddy2022byzantinerobust}.
Specifically, 
when $\kappa = \mathcal{O}{(\nicefrac{f}{n})}$, our definition implies both $(f,\lambda)$-resilient averaging and $(\delta_{max},c)$-ARAgg\footnote{Although for satisfying $(f,\lambda)$-resilient averaging condition having $\kappa = \sqrt{\lambda}$ suffices, i.e., it need not be in $\mathcal{O}{(\nicefrac{f}{n})}$.}.
We prove this claim in Appendix~\ref{app:unification}. 
Furthermore, $(f, \kappa)$-robustness allows us to devise a general convergence analysis for robust D-GD when up to $f$ workers are Byzantine, under the standard heterogeneity assumption. We present our general convergence analysis of robust D-GD in Theorem~\ref{thm:conv-gd} below, assuming $F$ to be $(f,\kappa)$-robust. Recall that $\H$ denotes the set of indices for honest workers.

\begin{theorem}
\label{thm:conv-gd}
Let Assumption~\ref{asp:hetero} hold and recall that $\loss_\H$ is $L$-smooth.
Consider Algorithm~\ref{gd} with $T \geq 1$ and learning rate $\gamma = \nicefrac{1}{L}$.
If $F$ is $(f,\kappa)$-robust then
$$
   \norm{\nabla{\loss_{\H}{(\hat{\theta})}}}^2 
    \leq  4 \kappa G^2
    + \frac{4 L (\loss_\H{(\theta_0)}-\loss^*)}{T},
$$
where $\loss^* \coloneqq \inf_{\theta \in \R^d} \loss_{\H}(\theta)$.
\end{theorem}

According to Theorem~\ref{thm:conv-gd}, the asymptotic error for robust D-GD (when $T \to \infty$) is optimal, i.e., it matches the lower bound from Proposition~\ref{prop:lower-bound} if $F$ is $(f,\kappa)$-robust with $\kappa = \mathcal{O}\left( \nicefrac{f}{n}\right)$. In the homogeneous case, i.e., when $G=0$, robust D-GD can asymptotically reach a stationary point of the average honest loss function despite the presence of $f$ Byzantine workers, as long as $F$ is $(f,\kappa)$-robust with a bounded $\kappa$. Note that the convergence rate of $\nicefrac{1}{T}$ is standard for smooth non-convex loss functions when analyzing first-order methods such as gradient descent~\cite{ghadimi2016accelerated}. 

\paragraph{Suboptimality of existing aggregations.}
Several existing aggregation rules such as \textnormal{CWTM}, \textnormal{Krum}, \textnormal{GM}, and  \textnormal{CWMed} can be shown to be $(f, \kappa)$-robust. The robustness coefficients $\kappa$ for these rules are listed in Table~\ref{tab:kappa}, and the formal derivations are deferred to Appendix~\ref{app:gars}. Note also that for any $f < \nicefrac{n}{2}$, an aggregation rule cannot be $(f,\kappa)$-robust for $\kappa < \frac{f}{n-2f}$ (see Appendix~\ref{app:kappa-lb} for details). This lower bound means that, in general, a robust aggregation rule cannot provide an estimate that is arbitrarily close to the average of honest inputs. This also indicates that the values in Table~\ref{tab:kappa} for \textnormal{Krum}, \textnormal{GM}, and  \textnormal{CWMed} are suboptimal, as they do not match the lower bound. We show in the next section that $\cenna$ solves this issue by boosting the robustness of these aggregation rules and provides optimal convergence for robust D-GD.

\section{Fixing by Nearest Neighbor Mixing}
\label{sec:optimal}
In this section, we present a principled way of fixing the suboptimality of existing solutions in terms of Byzantine resilience. Specifically, we introduce a pre-aggregation algorithm called nearest neighbor mixing ($\cenna$), and prove optimal Byzantine robustness when it is embedded in D-GD.
We describe the $\cenna$ procedure in Section~\ref{sec:nnm-description} and demonstrate in Section~\ref{sec:nnm-analysis} that $\cenna$ amplifies robustness when applied prior to an aggregation rule.

\subsection{Description of $\cenna$}
\label{sec:nnm-description}

Given a set of $n$ input vectors $x_1,\ldots,x_n \in \R^d$, $\cenna$ replaces every vector with the average of its $n-f$ nearest neighbors (including itself). Formally, $\cenna$ outputs $(y_1, \ldots, y_n) = \cenna{(x_1,\ldots,x_n)}$ where for each $i \in [n]$,
\begin{equation}
        y_i = \frac{1}{n-f} \sum_{j=1}^{n-f} x_{i:j} \; ;
    \end{equation}
where $x_{i:j}$ is the $j^{th}$ nearest vector to $x_i$ in $(x_1,\ldots,x_n)$.
Intuitively, in the context of robust D-GD, applying $\cenna$ mixes the gradients artificially, hence making every \emph{mixed gradient} a better representation of $\nabla \loss_{\H}(\theta_{t-1})$. The overall procedure for $\cenna$ can be found in Algorithm~\ref{algo:cenna}.

\begin{algorithm}[ht!]
\label{algo:cenna}
\textbf{Input:} number of inputs $n$, number of Byzantine inputs $f < \nicefrac{n}{2}$, vectors $x_1, \ldots, x_n \in \R^d$.

\For{$i = 1 \ldots n$}{
    Sort inputs to get $(x_{i:1}, \dots, x_{i:n})$ such that
    \begin{equation*}
        \|x_{i:1} - x_i\| \leq \ldots \leq \|x_{i:n} - x_i\|;
    \end{equation*}
    Average the $n-f$ nearest neighbors of $x_i$, i.e.,
    \vspace{-5pt}
    \begin{equation*}
        y_i = \frac{1}{n-f} \sum_{j=1}^{n-f} x_{i:j} \; ;
    \end{equation*}
    \vspace{-10pt}
}

\textbf{return} $y_1, \ldots, y_n$\;
\caption{Nearest Neighbor Mixing ($\cenna$)}
\end{algorithm}

\begin{remark}
The computational cost of $\cenna$ is $\mathcal{O}{\left(d n^2\right)}$ in the worst case, which is due to the search of the $n-f$ nearest neighbors of each input.
Faster algorithms for approximate nearest neighbor search~\cite{hajebi2011fast,muja2014scalable} could be used for efficiency.
Nevertheless, we argue that the cost of NNM is comparable to (or even smaller than) several aggregation rules including Krum~\cite{krum}, Multi-Krum~\cite{krum}, and MDA~\cite{rousseeuw1985multivariate,brute_bulyan}. 
Below, we list the cost of prominent aggregation rules: Krum and Multi-Krum: $\mathcal{O}{(dn^2)}$, CWMed~\cite{yin2018} and MeaMed~\cite{meamed}: $\mathcal{O}{(dn)}$, CWTM~\cite{yin2018}: $\mathcal{O}{(dn \log{n})}$, $\epsilon$-approximate GM~\cite{acharya2021}: $\mathcal{O}{(dn + d \epsilon^{-2})}$,
MDA: $\mathcal{O}{(dn^2 + {n \choose f} n^2)}$. 
Finally, unlike spectral methods~\cite{shejwalkar2021manipulating}, we stress that our algorithm preserves linear dependency in d, which may be extremely large in modern-day ML (i.e., $d \gg n $).
\end{remark}

\subsection{Analysis of $\cenna$}
\label{sec:nnm-analysis}

We now present in Lemma~\ref{lem:cenna} the robustness amplification that $\cenna$ brings to a $(f, \kappa)$-robust aggregation rule. We then show in Corollary~\ref{cor:resilience} how it leads to optimal Byzantine resilience guarantees under heterogeneity.  

\begin{lemma}
\label{lem:cenna}
Let $f < \nicefrac{n}{2}$ and $F \colon \R^{d \times n} \to \R^d$.
If $F$ is $(f,\kappa)$-robust, then $F \circ \cenna$ is $(f,\kappa')$-robust with
$$\kappa' \leq \frac{8f}{n-f}(\kappa+1).$$
\end{lemma}

Essentially, Lemma~\ref{lem:cenna} means that the composition of $\cenna$ with any $(f,\kappa)$-robust aggregation rule improves the order of magnitude of the robustness coefficient 
to $\mathcal{O}{\left(\nicefrac{f}{n}(\kappa+1)\right)}$. Specifically, this means that if $F$ has a robustness coefficient $\kappa = \mathcal{O}{(1)}$, $\cenna$ renders the new robustness coefficient optimal. We believe that the condition $\kappa = \mathcal{O}\left(1 \right)$ is general enough to hold for any standard robust aggregation rule. In fact, assuming that there exists $\nu >0$ such that $n \geq (2+\nu)f$, all the aforementioned aggregation rules are $(f,\kappa)$-robust with $\kappa = \mathcal{O}\left( 1 \right)$. Thus, from a theoretical point of view, any aggregation rule from Table~\ref{tab:kappa} becomes a good candidate in Algorithm~\ref{gd} when combined with $\cenna$. Then, as stated in Corollary~\ref{cor:resilience} below, we obtain optimal Byzantine resilience under heterogeneity.

\begin{corollary}
\label{cor:resilience}
Let Assumption~\ref{asp:hetero} hold and recall that $\loss_\H$ is $L$-smooth.
Consider Algorithm~\ref{gd} with $\gamma = \nicefrac{1}{L}$ and aggregation $F \circ \cenna$.
If $F$ is $(f,\kappa)$-robust with $\kappa = \mathcal{O}(1)$, then Algorithm~\ref{gd} is $(f,\varepsilon)$-Byzantine resilient with
$$\varepsilon = \mathcal{O}{\left(\nicefrac{f}{n}G^2 + \nicefrac{1}{T}\right)}.$$
\end{corollary}

\begin{remark}
Note that CWTM achieves an order-optimal robustness coefficient $\kappa = \mathcal{O}{(\nicefrac{f}{n})}$ without the use of $\cenna$, whenever $n \geq (2+\nu)f$ for some constant $\nu >0$.
However, using $\cenna$ prior to CWTM significantly improves its empirical performance (see Section~\ref{sec:exp}).
We believe that an average-case analysis, instead of our worst-case approach, could capture this improvement in theory.
\end{remark}
\section{Stochastic Extension}
\label{sec:stochastic}
Despite its optimality, our solution for robust D-GD is computationally demanding, as it requires the honest workers to compute each gradient on the whole dataset. In practice, it is more common to consider stochastic variants of D-GD, where workers compute gradients on random mini-batches of datasets. To accommodate this, we show that $\cenna$ can also be used to enhance the performance of robust variants of distributed stochastic heavy ball (D-SHB) that has been recently proven to perform well in Byzantine ML in the homogeneous setting~\cite{el2021distributed,Karimireddy2021,farhadkhani2022byzantine}. 

\subsection{Description of Robust D-SHB}
Similarly to robust D-GD, the algorithm proceeds in $T$ iterations as follows. At every step $t \in [T]$, the server holds a model $\theta_{t-1}$  and each honest worker holds a local momentum $m_{t-1}^{(i)}$\footnote{$\theta_0$ is set by the server and $m_0^{(i)}=0$ for all honest workers.}. The server broadcasts the current model $\theta_{t-1}$ to all the workers for them to update their local momentum. To do so, each honest worker samples a mini-batch of data $S_t^{(i)}$ at random from $D_i$ and computes a \emph{stochastic} estimate $g_t^{(i)}$ of its gradient $\nabla \loss_i(\theta_{t-1})$, defined as
\begin{equation}
\label{eq:stochgradient}
      g_t^{(i)} = \frac{1}{\vert S_t^{(i)} \vert} \sum_{z \in S_t^{(i)}} \nabla \ell{\left( \theta_{t-1}, z \right)}.
      \vspace{-5pt}
\end{equation}
Then, each honest worker updates and sends to the server its local momentum
\begin{align}
\label{eq:momentum}
    m_{t}^{(i)}=\beta m_{t-1}^{(i)} + (1-\beta) g_t^{(i)},
\end{align}
where $\beta \in (0,1)$ is the \emph{momentum parameter} and is shared by all the honest users. Similarly to D-SGD, the server computes an aggregate of the momentums it receives as
$R_t = F{(m_t^{(1)}, \ldots, m_t^{(n)})}.$
Finally, the server updates the model to 
$\theta_{t} = \theta_{t-1} - \gamma R_t,$ where $\gamma > 0$ is the learning rate. After the $T$ iterations, the server outputs $\hat{\theta}$ by sampling uniformly from $(\theta_0,\ldots,\theta_{T-1})$. The overall procedure for robust D-SHB is summarized in Algorithm~\ref{sgd}.

\begin{algorithm}[ht!]
\caption{Robust D-SHB}
\label{sgd}
\textbf{Input:} 
Initial model $\theta_0$,
initial momentum $m_0^{(i)} = 0$ for honest workers,
robust aggregation $F$, learning rate $\gamma$, momentum coefficient $\beta$, and number of steps $T$.\\

\For{$t=1 \dots T$}{
\textcolor{blue!80}{{\textbf{Server}}} broadcasts $\theta_{t-1}$ to all workers\;
\For{\textbf{\textup{every}} \textcolor{orange}{\textbf{honest worker} $\mathbf{i}$}, \textbf{\textup{in parallel}}}{


Compute a stochastic gradient $g_t^{(i)}$ as per~\eqref{eq:stochgradient} \;
Update local momentum
$m_t^{(i)}$ as per~\eqref{eq:momentum} \; 
Send $m_t^{(i)}$ to the server \;
}

\textcolor{blue!80}{{\textbf{Server}}} aggregates the momentums:
$R_t = F{(m_t^{(1)}, \ldots, m_t^{(n)})}$\;

\textcolor{blue!80}{{\textbf{Server}}} updates the model:
$\theta_{t} = \theta_{t-1} - \gamma R_t$\;
}
Sample $\hat{\theta}$ uniformly from $(\theta_0,\ldots,\theta_{T-1})$\;

\textbf{return}  $\hat{\theta}$\;
\end{algorithm}

\subsection{Analysis of Robust D-SHB with $\cenna$}
\label{sec:dshb-analysis}
We now provide convergence guarantees for robust D-SHB with $\cenna$. To do so, we make an additional (standard) assumption on the variance of the stochastic gradients.
\begin{assumption}[Bounded variance]
\label{asp:bnd_var}
For each honest worker $w_i, i \in \H$, and all $\theta{} \in \mathbb{R}^d$, it holds that
\begin{equation*}
    \frac{1}{m}\sum_{ z \in \D_i}
    \norm{ \nabla_{\theta{}}{\ell{(\theta,z)}} - \nabla \loss_i\left(\theta{}\right)}^2
    \leq \sigma^2.
\end{equation*}
\end{assumption}
\vspace{-7pt}

We now present in Theorem~\ref{thm:conv} the extension of our results to D-SHB. Essentially, we analyze Algorithm~\ref{sgd} upon assuming a constant learning rate, that assumptions~\ref{asp:hetero} and \ref{asp:bnd_var} hold, and that $\loss_{\H}$ is $L$-smooth. 
For convenience, we introduce the following numerical constants before stating the theorem:
\begin{align*}
    &a_1 \coloneqq 36,
    a_2 \coloneqq 6 \sqrt{\loss_{\H}(\weight{0}) - \loss^*},
    a_3 \coloneqq 1728L, \\
    &a_4 \coloneqq 288L,
    a_5 \coloneqq 6L\, a_2^2,
    \text{ and } 
    \loss^* = \inf_{\weight{} \in \R^d} \loss_{\H}(\weight{}).
\end{align*}
\begin{theorem}
\label{thm:conv}
Let assumption~\ref{asp:hetero} and~\ref{asp:bnd_var} hold and recall that $\loss_\H$ is $L$-smooth.
Let $F$ be a $(f,\kappa)$-robust aggregation rule. Consider Algorithm~\ref{sgd} with momentum coefficient $\beta = \sqrt{1 - 24 \gamma L} $, and learning rate $$\gamma= \min{\left\{\frac{1}{24L}, ~ \frac{a_2}{2a_{\kappa}\sigma\sqrt{T}}\right\}}, $$
with $a_\kappa^2 \coloneqq a_3 \kappa + \frac{a_4}{n-f}$. For all $T \geq 1$, we have
\begin{align*}
   \expect{\norm{\nabla \loss_{\H} (\hat{\weight{}} )}^2} \hspace{-2pt}
    \leq  a_1 \kappa G^2
    + \frac{a_2 a_{\kappa}\sigma}{\sqrt{T}}
    + \frac{a_5}{T}
    + \frac{a_2 a_4 \sigma}{n a_\kappa T^{\nicefrac{3}{2}}}, 
\end{align*}
where the expectation is over the algorithm's randomness.
\end{theorem}


\paragraph{Tight probabilistic guarantee.}
The non-vanishing error in Theorem~\ref{thm:conv} is
in $\mathcal{O}{(\kappa G^2)}$. Hence, this error is tight when $\kappa = \mathcal{O}(\nicefrac{f}{n})$. However, the main difference with robust D-GD (Theorem~\ref{thm:conv-gd}) is that the inequality only holds in expectation, and therefore may not verify $(f,\varepsilon)$-Byzantine resilience.
Nevertheless, this result is consistent with the state-of-the-art convergence guarantees~\cite{karimireddy2022byzantinerobust}~(Theorem~II). Note that a subtle difference remains between our result and the one from \cite{karimireddy2022byzantinerobust}. The randomness of our result in Theorem~\ref{thm:conv} only depends on the random subsampling of data
and the final choice of the model. 
These are natural sources of randomness that are usually considered when studying stochastic gradient descent~\cite{bottou2018optimization}. 
On the other hand, the results from~\cite{karimireddy2022byzantinerobust} also incorporate an additional (exogenous) randomness introduced by the shuffling operation of Bucketing. This source of randomness cannot be canceled even if true gradients are computed by the workers, and it may amplify the uncertainty in the computations (see Section~\ref{sec:exp} and Appendix~\ref{app:Bucketing} for more details). As a result, we believe the probabilistic convergence guarantees from Theorem~\ref{thm:conv} to be strictly stronger than those obtained in \cite{karimireddy2022byzantinerobust}.

Once again, using Lemma~\ref{lem:cenna}, we can show that the tight (probabilistic) resilience guarantee implied by Theorem~\ref{thm:conv} holds for the larger class of $(f,\kappa)$-robust aggregation rules with $\kappa = \mathcal{O}{(1)}$.
We formalize this in Corollary~\ref{cor:nnm-dshb} below. \vspace{0.2cm}

\begin{corollary}
\label{cor:nnm-dshb}
Let Assumption~\ref{asp:hetero} hold and recall that $\loss_\H$ is $L$-smooth.
Consider Algorithm~\ref{sgd} with aggregation $F \circ \cenna$, under the same setting as Theorem~\ref{thm:conv}.
If $F$ is $(f,\kappa)$-robust with $\kappa = \mathcal{O}(1)$, then we have
\begin{align*}
   \expect{\norm{\nabla \loss_{\H} (\hat{\weight{}} )}^2}
   = \mathcal{O}{\left( \nicefrac{f}{n} G^2 + \nicefrac{1}{\sqrt{T}}\right)}, 
\end{align*}
where the expectation is over the algorithm's randomness.
\end{corollary}




\section{Experimental Evaluation}
\label{sec:exp}
In this section, we investigate the practical performances of $\cenna$. We report on a comprehensive set of experiments evaluating our solution against the state-of-the-art on three benchmark image classification tasks and under five different Byzantine attacks.

\begin{table*}[ht!]
\centering
\resizebox{\textwidth}{!}{%
\def\arraystretch{1.2}
\begin{tabular}{c|lccccc|c} 
\toprule
&\textbf{Aggregation} & \textbf{ALIE} & \textbf{FOE} & \textbf{LF} & \textbf{SF} & \textbf{Mimic} & \textbf{Worst Case}\\ \midrule
 & Krum & $10.19 \pm 00.61$ & $10.28 \pm 00.57$ & 12.76 $\pm$ 06.25 & 51.17 $\pm$ 06.14 & 86.78 $\pm$ 07.04 & $\mathcolorbox{red!50}{10.19 \pm 00.61}$\\
\hspace{-0.2cm}\circled{1}\hspace{-0.2cm}& Bucketing + Krum & $44.78 \pm 06.18$ & $36.10 \pm 15.70$ & 56.21 $\pm$ 19.49 & 47.54 $\pm$ 11.74 & 94.92 $\pm$ 03.93 & $\mathcolorbox{orange!0}{36.10 \pm 15.70}$\\
 & NNM + Krum & \textbf{78.30 $\pm$ 07.78} & \textbf{70.07 $\pm$ 04.39} & \textbf{93.07 $\pm$ 05.24} & \textbf{82.44 $\pm$ 02.86} & \textbf{97.69 $\pm$ 00.77} & \textbf{$\mathcolorbox{green!50}{70.07 \pm 04.39}$}\\ \midrule
&GM & \textbf{92.01 $\pm$ 04.35} & 65.61 $\pm$ 12.17 & 93.94 $\pm$ 03.70 &  57.86 $\pm$ 10.42 & 96.85 $\pm$ 01.57 & $\mathcolorbox{orange!0}{57.86 \pm 10.42}$\\
\hspace{-0.2cm}\circled{2}\hspace{-0.2cm} &Bucketing + GM & 39.83 $\pm$ 11.35 & 44.73 $\pm$ 16.47 & \textbf{96.22 $\pm$ 02.83} & \textbf{91.30 $\pm$ 03.91} & \textbf{97.68 $\pm$ 00.91} & $\mathcolorbox{red!50}{39.83 \pm 11.35}$\\
&NNM + GM & 81.26 $\pm$ 08.91 & \textbf{75.27 $\pm$ 02.69} & 94.23 $\pm$ 03.20 & 86.33 $\pm$ 03.73 & 97.17 $\pm$ 01.09 & $\mathcolorbox{green!50}{75.27 \pm 02.69}$\\ \midrule
&CWMed & 68.74 $\pm$ 06.99 & 19.48 $\pm$ 10.97 & 33.34 $\pm$ 17.02 & 27.96 $\pm$ 09.97 & 64.01 $\pm$ 12.77 & $\mathcolorbox{red!50}{19.48 \pm 10.97}$\\
\hspace{-0.2cm}\circled{3}\hspace{-0.2cm} &Bucketing + CWMed & 55.86 $\pm$ 10.00 & 42.80 $\pm$ 21.25 & 70.16 $\pm$ 11.65 & 50.96 $\pm$ 16.52 & 94.43 $\pm$ 03.48 & $\mathcolorbox{orange!0}{42.80 \pm 21.25}$\\
&NNM + CWMed & \textbf{80.52 $\pm$ 07.45} & \textbf{75.20 $\pm$ 08.80} & \textbf{93.42 $\pm$ 02.98} & \textbf{85.10 $\pm$ 06.05} & \textbf{97.38 $\pm$ 00.70} & $\mathcolorbox{green!50}{75.20 \pm 08.80}$\\ \midrule
&CWTM & 76.16 $\pm$ 07.68 & 69.96 $\pm$ 16.57 & 36.87 $\pm$ 21.43 & 27.45 $\pm$ 08.83 & 89.83 $\pm$ 02.83 & $\mathcolorbox{red!50}{27.45 \pm 08.83}$\\
\hspace{-0.2cm}\circled{4}\hspace{-0.2cm} &Bucketing + CWTM & 55.86 $\pm$ 10.00 & 42.80 $\pm$ 21.25 & 70.16 $\pm$ 11.65 & 50.96 $\pm$ 16.52 & 94.43 $\pm$ 03.48 & $\mathcolorbox{orange!0}{42.80 \pm 21.25}$\\
&NNM + CWTM & \textbf{79.04 $\pm$ 09.19} & \textbf{79.91 $\pm$ 03.94} & \textbf{94.75 $\pm$ 02.22} &  \textbf{84.78 $\pm$ 05.78} & \textbf{96.02 $\pm$ 03.25} & $\mathcolorbox{green!50}{79.04 \pm 09.19}$\\ \bottomrule
\end{tabular}}
\vspace{0.1cm}
\caption{Maximum test accuracies ($\%$) across $T=800$ learning steps on MNIST, under extreme heterogeneity ($\alpha = 0.1$) and five Byzantine attacks. There are $f=4$ Byzantine workers among $n=17$. The baseline accuracy (D-SHB) is $98.03 \pm 0.70 \%$. 
In each of the four horizontal blocks and under each attack, we highlight in \textbf{bold} the best accuracy. For every method, we also show the worst-case accuracy across attacks, and highlight the $\mathcolorbox{green!50}{\text{best}}$ and $\mathcolorbox{red!50}{\text{worst}}$ one in each block.}
\label{table:results_mnist}
\end{table*}

\subsection{Experimental Setup}

\paragraph{Datasets, models, and hyperparameters.} We consider three image classification datasets, namely MNIST~\cite{mnist}, Fashion-MNIST~\cite{fashion-mnist}, and CIFAR-10~\cite{cifar}; and we implement NNM on top of robust D-SHB.
Due to space limitation, we only present results on MNIST and CIFAR-10, and defer the remaining results to Appendix~\ref{app:exp_results}, including experiments on D-GD with NNM. 

On MNIST, we train a convolutional neural network (CNN) for $T = 800$ steps using a batch size $b = 25$, with a decaying learning rate starting at 0.75, and a momentum parameter $\beta = 0.9$. On CIFAR-10, we use a CNN with $b = 50$, $T = 2000$, $\gamma = 0.25$ that decays at step 1500, and $\beta = 0.9$. Furthermore, we implement our solution with four aggregation rules namely Krum, CWTM, CWMed, and GM\footnote{We implement GM using the approximation from~\cite{pillutla2019robust}.}, and compare its performance against the Bucketing~\cite{karimireddy2022byzantinerobust} method and vanilla aggregation rules.
As a benchmark, we also implement vanilla D-SHB (i.e., robust D-SHB with $F = $ average) in a setting where there are no faults ($f=0$).


\paragraph{Heterogeneity.} We simulate heterogeneity in honest workers' data by sampling from the original dataset using a Dirichlet distribution of parameter $\alpha$ (as done in~\cite{dirichlet}). We consider three heterogeneity regimes: \textit{extreme} ($\alpha = 0.1$), \textit{moderate} ($\alpha = 1$), and \textit{low} ($\alpha = 10$). A pictorial representation of the resulting heterogeneity as a function of $\alpha$ can be found in Appendix~\ref{app:exp_setup}. We run our algorithm on MNIST and Fashion-MNIST over the whole spectrum of heterogeneity defined above. However, as CIFAR-10 is considerably more challenging, we restrict the heterogeneity to $\alpha \in \{1, 10\}$.


\paragraph{Distributed system and Byzantine attacks.}
We consider a distributed system of $n=17$ workers, among which $f < \nicefrac{n}{2}$ can be Byzantine. We vary $f \in \{4, 6, 8\}$ on the MNIST dataset, and $f \in \{2, 3, 4\}$ on CIFAR-10. The Byzantine workers execute five state-of-the-art gradient attacks, namely Fall of Empires (FOE)~\cite{empire}, A Little is Enough (ALIE)~\cite{little}, Sign Flipping (SF)~\cite{allen2020byzantine}, Label Flipping~\cite{allen2020byzantine}, and Mimic~\cite{karimireddy2022byzantinerobust}. Note that for ALIE and FOE, we design optimized versions of the attacks, as done in~\cite{shejwalkar2021manipulating}. We explain this further in Appendix~\ref{app:exp_setup}.

\paragraph{Reproducibility and reusability.}
All experiments are run with five seeds from 1 to 5 for reproducibility purposes. The code will also be made available for reusability. Additional details on the setup
can be found in Appendix~\ref{app:exp_setup}.


\subsection{Empirical Results on MNIST}\label{exp_results_mnist}
In Table~\ref{table:results_mnist}, we carefully examine the performance of NNM on MNIST under extreme heterogeneity ($\alpha = 0.1$), in comparison with Bucketing and vanilla aggregation rules. For every block (i.e., every aggregation rule) and under every attack, we highlight in bold the algorithm resulting in the highest accuracy in the considered scenario.

\paragraph{NNM improves robustness.} We clearly see that our algorithm boosts the resilience of aggregation rules in Byzantine settings, and provides the most consistent behavior across attacks. In fact, for Krum, CWMed, and CWTM (i.e., blocks 1, 3, and 4 in Table~\ref{table:results_mnist}), NNM outputs the maximal accuracy under all attacks. The two other techniques (Bucketing and vanilla) showcase much weaker performances and are sometimes on par with a random classifier (e.g., 10.19\% under ALIE and 10.28\% under FOE for vanilla Krum).

The case of GM is less evident since the other techniques can outperform our method in some settings (e.g., ALIE, SF). However, a crucial observation is that there always exists at least one attack that considerably deteriorates the performance of Bucketing+GM and vanilla GM, whereas NNM+GM consistently yields desirable accuracies in all attack scenarios. Indeed, the minimum accuracy across attacks achieved by NNM+GM is 75.27\%, which is considerably better than Bucketing+GM and vanilla GM with 39.83\% and 57.86\% minimum accuracies, respectively.
Moreover, although vanilla GM scores the highest under ALIE, it showcases low accuracies of 65.61\% and 57.86\% under FOE and SF, respectively. Bucketing+GM also fails considerably under ALIE and FOE with accuracies far below 50\%. Finally, even though Bucketing+GM scores the highest under LF and Mimic, NNM+GM is also excellent under these two attacks with accuracies greater than 96\%.

\paragraph{Worst-case performance.} We show in the last column of Table~\ref{table:results_mnist} the \emph{worst-case} performance (across attacks) of every aggregation technique (i.e., every row), and rank them within each block from worst (in red) to best (in green). We argue that this is a critical metric to correctly evaluate Byzantine resilience, as the same algorithm can simultaneously greatly defend against some attacks but perform poorly against others.
Accordingly, we see from the last column that our method always displays the ``best" worst case. In fact, the lowest accuracy NNM yields across attacks and aggregation rules is 70.07\% with Krum. Furthermore, the worst case performances of the other techniques are much worse than ours, yielding values as low as 10.19\% with Krum and 39.83\% with Bucketing+GM.

Additionally, the worst case performance of Bucketing is almost equally poor for all aggregation rules (36.1\%, 39.83\%, and 42.80\%). However, all worst case accuracies of NNM are within the range 70-79\%. This highly suggests the unreliability of Bucketing, due to its subpar worst case behavior independently of the aggregation rule used.

\newpage

\subsection{Empirical Results on CIFAR-10}\label{exp_results_cifar}
In Figure~\ref{fig:plots_cifar_main}, we plot the performance of NNM and Bucketing on CIFAR-10 during $T=2000$ steps of learning, with GM, CWMed, CWTM, and Krum as base aggregation rules\footnote{In the considered setting, Bucketing with CWMed and CWTM are exactly equivalent.}. We consider two heterogeneity levels ($\alpha=1$ in row 1, and $\alpha=10$ in row 2), with two Byzantine workers executing the ALIE and LF attacks.

\begin{figure*}[ht]
    \centering
    \includegraphics[width=0.4\textwidth]{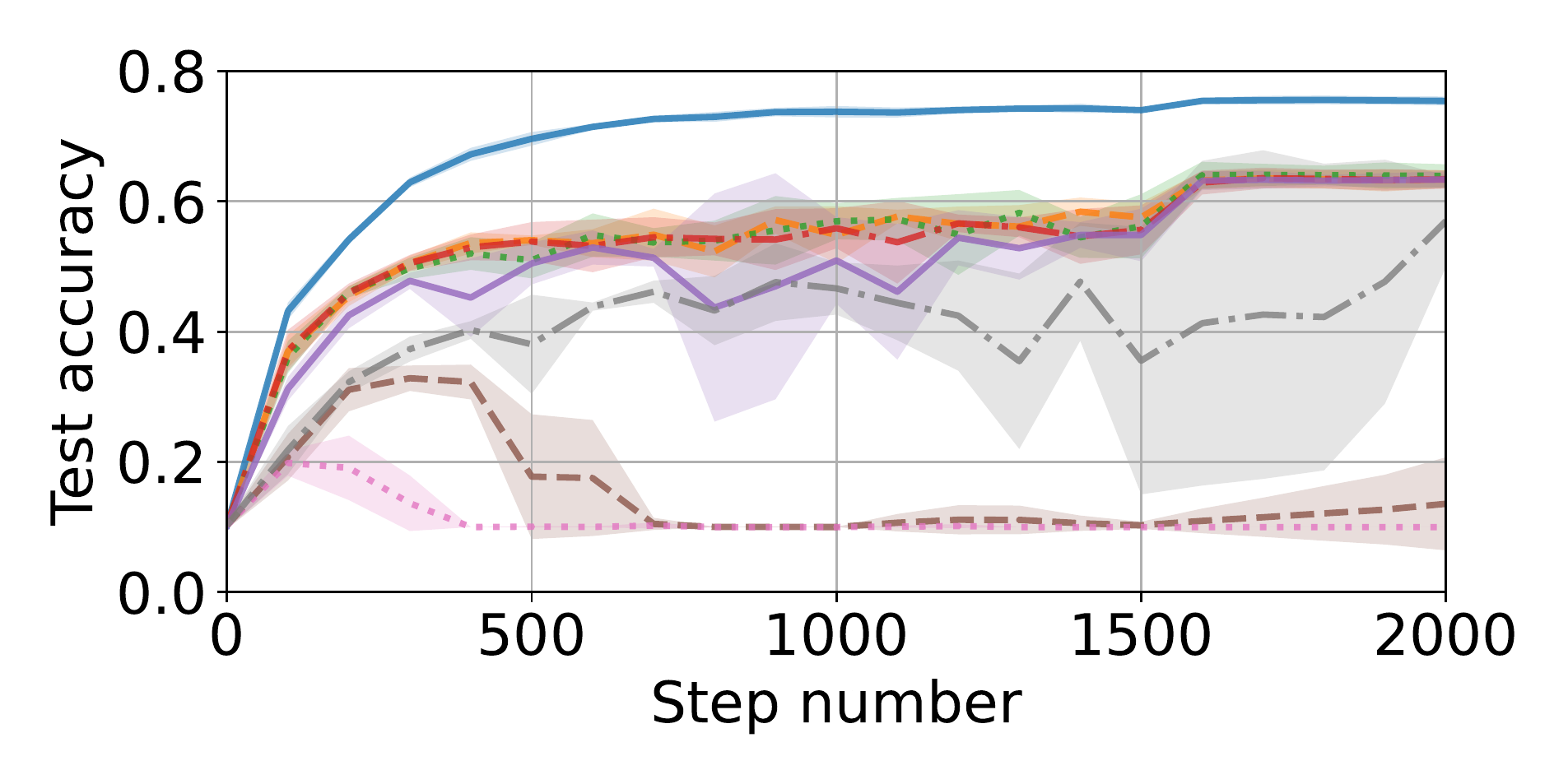}
    \includegraphics[width=0.4\textwidth]{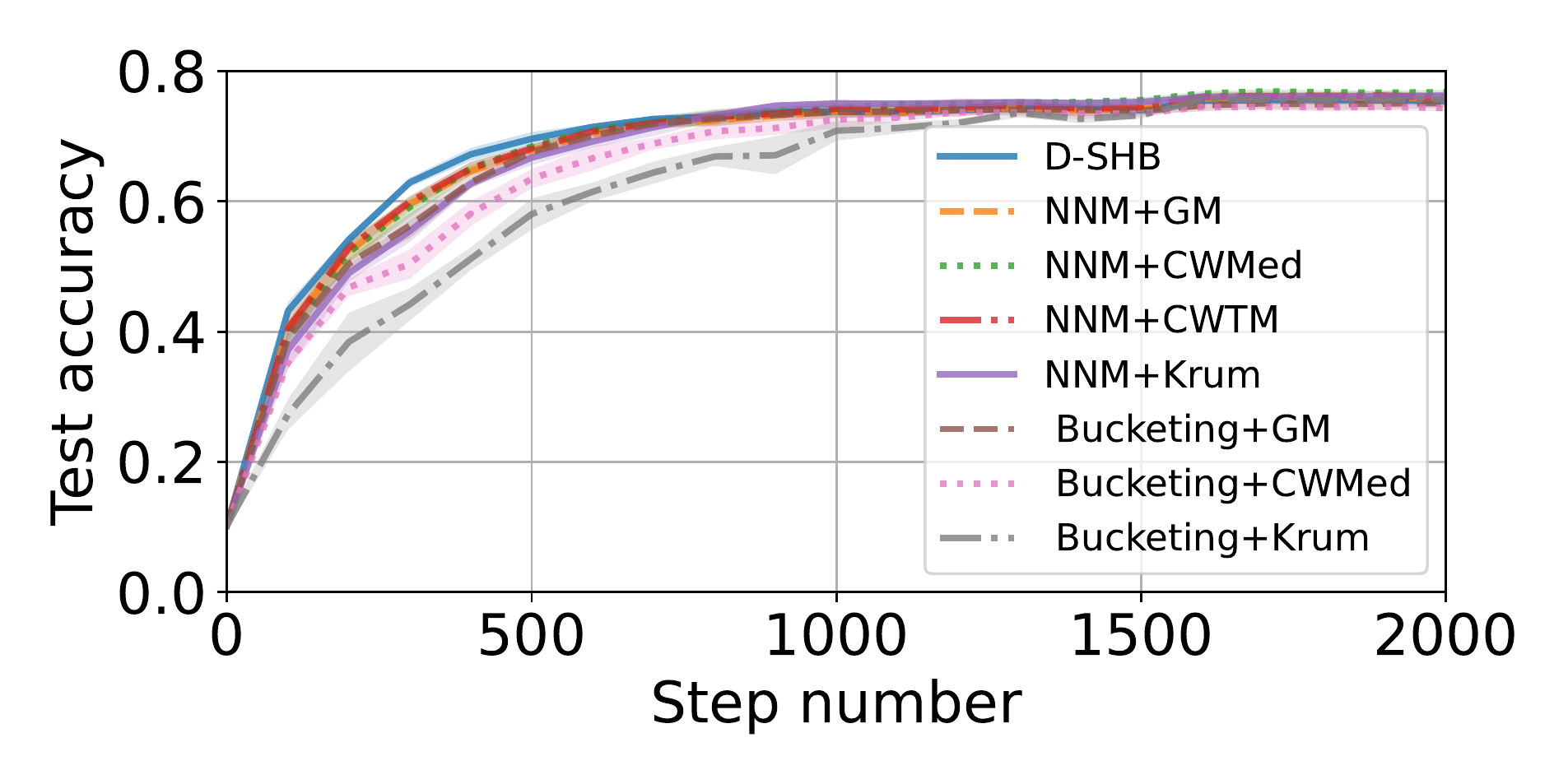}\\[-0.2cm]
    \includegraphics[width=0.4\textwidth]{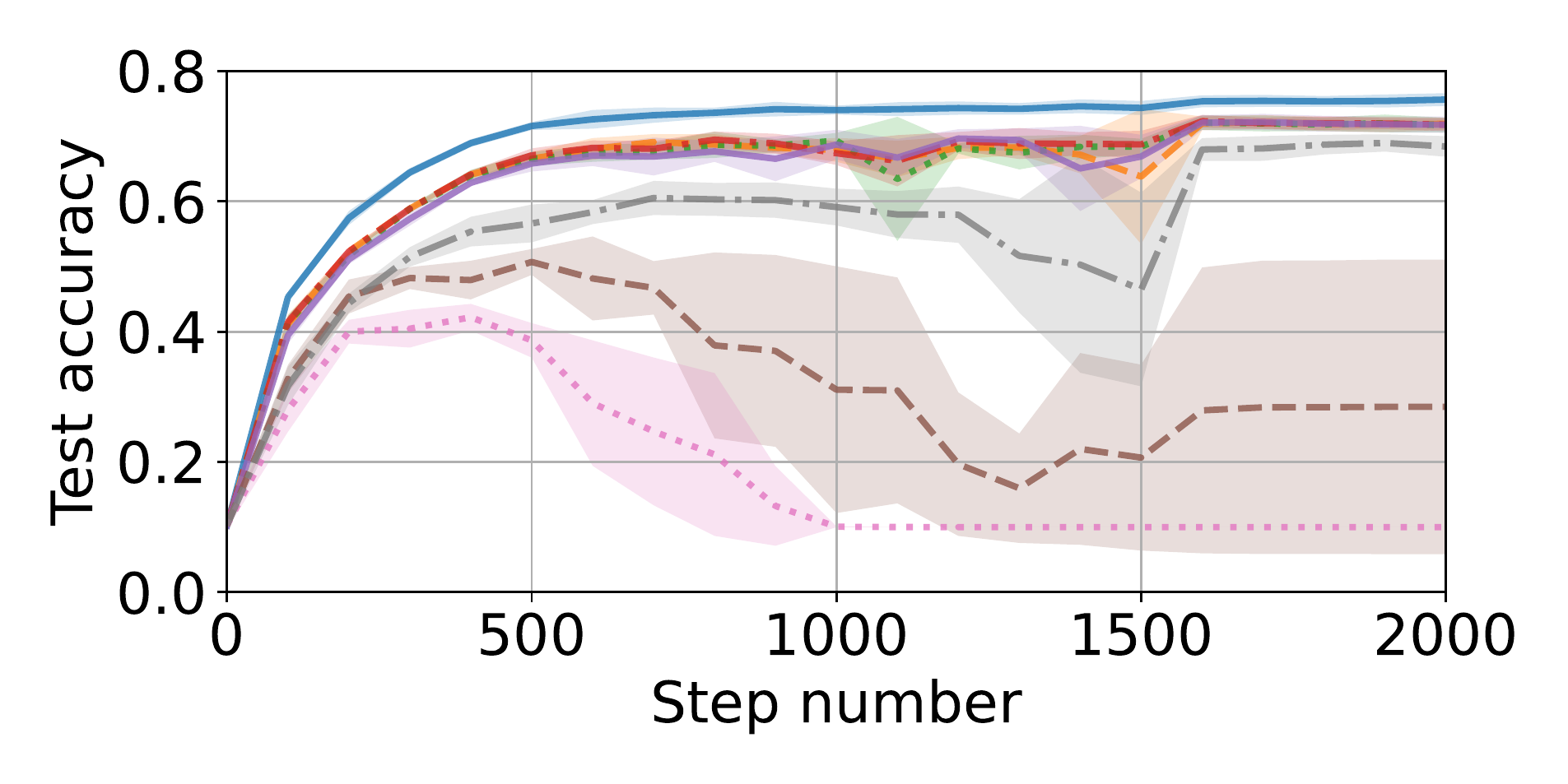}%
    \includegraphics[width=0.4\textwidth]{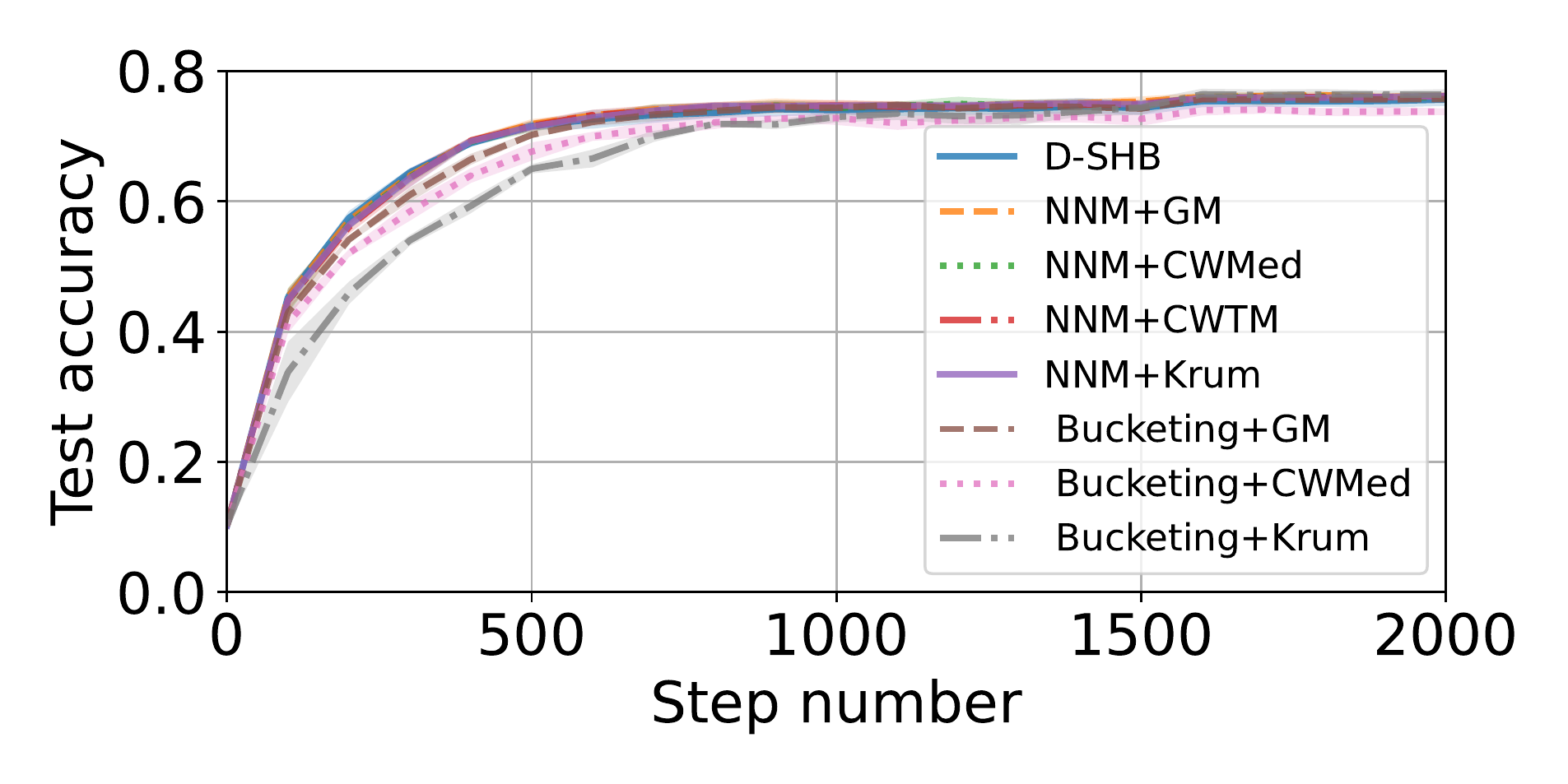}%
    \caption{Experiments on CIFAR-10 with $f=2$ Byzantine workers among $n = 17$. The Byzantine workers execute the ALIE (\textit{column 1}) and LF (\textit{column 2}) attacks. \textit{$1^{\text{st}}$ row}: moderate heterogeneity ($\alpha = 1$). \textit{$2^{\text{nd}}$ row}: low heterogeneity ($\alpha = 10$).}
\label{fig:plots_cifar_main}
\end{figure*}

\paragraph{Heterogeneity impairs learning.} Our first observation is that stronger heterogeneity regimes hinder the learning, as expected. As can be seen in Figure~\ref{fig:plots_cifar_main}, increasing $\alpha$ from 1 to 10 significantly improves all NNM aggregations (especially under ALIE), whereas Bucketing+CWMed still stagnates at 10\% and Bucketing+GM barely reaches 30\%.

\paragraph{Empirical superiority of NNM.} Under ALIE, we clearly see that NNM greatly outperforms Bucketing in both heterogeneity regimes. In particular, ALIE deteriorates the learning when using Bucketing with GM and CWMed, with both methods reaching a final accuracy close to 10\% when $\alpha = 1$. Although Bucketing+Krum has a better performance, it still displays a lower accuracy and a much larger variance than any NNM aggregation.
On the other hand, we observe that LF is a much weaker attack. Although all aggregation methods are able to converge to a desirable accuracy, our algorithm still portrays better convergence rates than Bucketing. This is particularly apparent for $\alpha = 10$ where all NNM aggregations almost exactly match D-SHB's convergence rate, whereas Bucketing converges slower, especially in the case of Krum. The remaining three attacks have a similar behavior to LF and are in Appendix~\ref{app:exp_results}.
\section{Conclusion and Future Work}
We show that robust D-GD enhanced by NNM is an optimal solution for Byzantine ML under heterogeneity. We also derive similar results for robust D-SHB (with NNM) in expectation. We believe that an interesting future direction would be to tighten these guarantees by investigating almost sure convergence rates for our stochastic solution. Such guarantees have been recently introduced in the Byzantine-free setting for SHB~\cite{sebbouh2021almost,liu22d,gadat2018stochastic}. Yet, adapting them to our setting remains an open question. In general, we believe subsequent works on Byzantine ML should strive to obtain almost sure convergence guarantees (instead of expectation), precluding the use of unnecessary randomness and certifying the robustness of these algorithms.
\vspace{-2pt}
\subsubsection*{Acknowledgements}
\vspace{-2pt}
This work has been supported in part by SNSF grants 200021\_200477 and 200021\_182542, and an EPFL-Ecocloud postdoctoral grant.
The authors are thankful to the anonymous reviewers for their constructive comments.

\bibliographystyle{plain}
\bibliography{references}




\onecolumn
\thispagestyle{empty}
\hsize\textwidth
  \linewidth\hsize \toptitlebar {\centering
  {\Large\bfseries Supplementary Materials: \\
Fixing by Mixing: A Recipe for Optimal Byzantine ML under Heterogeneity \par}}
 \bottomtitlebar
 
\section*{Organization of the Appendix}
Appendix~\ref{app:res-averaging} contains proofs related to $(f,\kappa)$-robustness.
Appendix~\ref{app:cenna} contains the analysis of $\cenna$ (proof of Lemma~\ref{lem:cenna}).
Appendix~\ref{app:Bucketing} contains a formal discussion on Bucketing~\cite{karimireddy2022byzantinerobust}.
Appendix~\ref{app:convergence-dgd} contains the analysis of robust D-GD (proof of Theorem~\ref{thm:conv-gd}).
Appendix~\ref{app:corollary} contains the analysis of robust D-GD with $\cenna$ (proof of Corollary~\ref{cor:resilience}), and the Byzantine resilience lower bound (Proposition~\ref{prop:lower-bound}).
Appendix~\ref{app:convergence-dshb} contains the analysis of robust D-SHB with $\cenna$ (proofs of Theorem~\ref{thm:conv} and Corollary~\ref{cor:nnm-dshb}).
Appendix~\ref{app:exp_setup} contains details on the experimental setup.
Appendix~\ref{app:exp_results} contains all experimental results.

\section{Robustness Analysis}
\label{app:res-averaging}

In this section, we prove all our claims related to $(f,\kappa)$-robustness.
In Section~\ref{app:gars}, we prove several existing aggregation rules to be $(f,\kappa)$-robust and give their exact robustness coefficients.
We then establish the tightness of our analysis in Section~\ref{app:kappa-lb} by proving a universal lower bound on $\kappa$, and an aggregation-specific lower bound.
Finally, we prove that $(f,\kappa)$-robustness unifies existing robustness definitions in Section~\ref{app:unification}.

We first recall the definition of $(f,\kappa)$-robustness:

\begin{repdefinition}{def:resaveraging}
Let $f < \nicefrac{n}{2}$ and $\kappa \geq 0$. An aggregation rule $\aggregation$ is said to be {\em $(f, \kappa)$-robust} if for any vectors $x_1, \ldots, \, x_n \in \R^d$, and any set $S \subseteq [n]$ of size $n-f$, 
\begin{align*}
    \norm{\aggregation(x_1, \ldots, \, x_n) - \overline{x}_S}^2 \leq \frac{\kappa}{\card{S}} \sum_{i \in S} \norm{x_i - \overline{x}_S}^2
\end{align*}
where $\overline{x}_S = \frac{1}{\card{S}} \sum_{i \in S} x_i$. We refer to $\kappa$ as the {\em robustness coefficient}.
\end{repdefinition}

\subsection{$(f,\kappa)$-robust Aggregation Rules}
\label{app:gars}
In this section, we prove that coordinate-wise Trimmed Mean (CWTM)~\cite{yin2018}, Krum~\cite{krum}, Geometric Median (GM)~\cite{small1990survey}, and coordinate-wise Median (CWMed)~\cite{yin2018} all satisfy $(f,\kappa)$-robustness.

\subsubsection{Trimmed Mean}

Let $x \in \mathbb{R}^d$, we denote by $[x]_k$, the $k$-th coordinate of $x$. Given the input vectors $x_1, \ldots, \, x_n \in \R^d$, we let $\tau_k$ denote a permutation on $[n]$ that sorts the $k$-th coordinate of the input vectors in non-decreasing order, i.e., $[x_{\tau_k(1)}]_k\leq [x_{\tau_k(2)}]_k \leq\ldots \leq [x_{\tau_k(n)}]_k$. Then, the coordinate-wise trimmed mean of $x_1, \ldots, \, x_n$, denoted by $\text{CWTM}(x_1, \ldots, x_n)$, is a vector in $\R^d$ whose $k$-th coordinate is defined as follows,
\begin{equation*}
    \left[\text{CWTM}(x_1, \ldots, x_n)\right]_k \coloneqq \frac{1}{n-2f} \sum_{j=f+1}^{n-f} [x_{\tau_k(j)}]_k. \label{eqn:def_cwtm}
\end{equation*}

We first show a general lemma simplifying the analysis of $(f,\kappa)$-robustness for coordinate-wise aggregations, by reducing the analysis to scalars without loss of generality.  
Specifically, we show that if $F$ is a coordinate-wise function, i.e., each $k$-th coordinate of $F$ denoted by $F_k$ only depends on the respective $k$-th coordinates of the inputs $\left[x_1\right]_k, \ldots , \, \left[ x_n \right]_k$, then coordinate-wise robustness implies overall robustness.

\begin{lemma}
\label{lem:cw-kappa}
Assume that $F: \R^{d \times n} \to \R^d$ is a coordinate-wise aggregation function, i.e., there exist $d$ real-valued functions $F_1,\ldots,F_d \colon \R^n \to \R$ such that for all $x_1,\ldots,x_n \in \R^d$, 
    $\left[F{(x_1,\ldots,x_n)} \right]_k = F_k{([x_1]_k,\ldots,[x_n]_k)}$.
If for each $k \in [d]$, $F_k$ is $(f,\kappa)$-robust then $F$ is $(f,\kappa)$-robust.
\end{lemma}
\begin{proof}
Consider $n$ arbitrary vectors in $\R^d$, $x_1,\ldots,x_n$, and an arbitrary set $S \subseteq [n]$ such that $\card{S}=n-f$. Assume that for each $k \in [d]$, $F_k$ is $(f,\kappa)$-robust.
Since $F$ is assume to be a coordinate-wise aggregation function, we have
\begin{align}
    \norm{F(x_1,\ldots,x_n)-\overline{x}_S}^2
    &= \sum_{k=1}^d \absv{F_k{([x_1]_k,\ldots,[x_n]_k)}-[\overline{x}_S]_k}^2. \label{eqn:cw-kappa_1}
\end{align}
Since for each $k \in [d]$, $F_k$ is $(f,\kappa)$-robust, we have
\begin{align*}
 \sum_{k=1}^d \absv{F_k{([x_1]_k,\ldots,[x_n]_k)}-[\overline{x}_S]_k}^2 &\leq \sum_{k=1}^d \frac{\kappa}{\card{S}} \sum_{i \in S} \absv{[x_i]_k-[\overline{x}_S]_k}^2 = \frac{\kappa}{\card{S}} \sum_{i \in S} \sum_{k=1}^d  \absv{[x_i]_k-[\overline{x}_S]_k}^2
    = \frac{\kappa}{\card{S}} \sum_{i \in S} \norm{x_i-\overline{x}_S}^2.
\end{align*}
Substituting from above in~\eqref{eqn:cw-kappa_1} concludes the proof.
\end{proof}

We now show an important property in Lemma~\ref{lem:tmean-main} below on the sorting of real values that proves essential in obtaining a tight $(f, \kappa)$-robustness guarantee for trimmed mean. 

\begin{lemma}
\label{lem:tmean-main}
Consider $n$ real values $x_1, \ldots, x_n$ be such that $x_1 \leq \ldots \leq x_n$.
Let $S \subseteq [n]$ of size $\card{S}=n-f$,
and $I \coloneqq \{f+1,\ldots,n-f\}$. 
We obtain that
\begin{equation*}
    \sum_{i \in I \setminus S} \absv{x_i - \overline{x}_S}^2 \leq \sum_{i \in S \setminus I} \absv{x_i - \overline{x}_S}^2.
\end{equation*}

\end{lemma}
\begin{proof}
First, note that, as $\card{I} = n-2f$ and $\card{S} = n-f$, $\card{I \setminus S} \leq f$ and $\card{S \setminus I} \geq f$. Thus, $\card{I \setminus S}  \leq  \card{S \setminus I}$. To prove the lemma, we show that there exists an injection $\phi \colon I \setminus S \to S \setminus I$ such that
\begin{equation}
\label{eq:lem-tmean}
    \forall i \in I \setminus S,~ \absv{x_i - \overline{x}_S} \leq \absv{x_{\phi(i)} - \overline{x}_S}.
\end{equation}
As $\card{I \setminus S}  \leq  \card{S \setminus I}$, we have $\sum_{i \in I \setminus S} \absv{x_{\phi(i)} - \overline{x}_S}^2 \leq \sum_{i \in S \setminus I} \absv{x_i - \overline{x}_S}^2$. Hence,~\eqref{eq:lem-tmean} proves the lemma. 

\vspace{0.2cm}

We denote by $B$ the complement of $S$ in $[n]$, i.e., $B = [n] \setminus S$. Therefore, $\card{B} = f$ and $I \setminus S = I \cap B $.
We denote $I^+ \coloneqq \{n-f+1,\ldots,n\}$, and $I^- \coloneqq \{1,\ldots,f\}$ to be the indices of values that are larger (or equal to) and smaller (or equal to) the values in $I$, respectively. Let $\card{I \cap B} = q$. As $B = [n] \setminus S$, we obtain that
\begin{align}
    \card{I^+ \cap S} = f - \card{I^+ \cap B} \geq f - \card{B \setminus (I \cap B)} = f - (f - q) = q. \label{eqn:i+b}
\end{align}
Similarly, we obtain that 
\begin{align}
    \card{I^- \cap S} \geq q. \label{eqn:i-b}
\end{align}
Recall that $q = \card{I \cap B}$ where $I \setminus S = I \cap B$.
Therefore, due to~\eqref{eqn:i+b} and~\eqref{eqn:i-b}, there exists an injection from $I \setminus S$ to $(I^- \cap S) \times (I^+ \cap S)$. Let $\psi$ be such an injection. For each $i \in I \setminus S$, $\psi(i)$ is a pair, denoted by $(\psi^-(i), \, \psi^+(i))$, in $(I^- \cap S) \times (I^+ \cap S)$. Consider an arbitrary $i \in I \setminus S$. By definition of $I^-$ and $I^+$, we have 
\begin{align*}
    x_{\psi^-(i)} \leq x_i \leq x_{\psi^+(i)}.
\end{align*}
Therefore, for any real value $y$, 
\begin{align*}
    \absv{x_i - y} \leq \max\left\{ \absv{x_{\psi^+(i)} - y}, \, \absv{x_{\psi^-(i)} - y} \right\}.
\end{align*}
The above proves~\eqref{eq:lem-tmean}, where the injection $\phi$ is defined as $\phi(i) = \argmax_{j \in \{\psi^+(i), \, \psi^-(i)\}} \absv{x_j - \overline{x}_S}$ for all $i \in I \setminus S$. 
\end{proof}

We now prove the $(f, \, \kappa)$-robustness property of trimmed mean in the proposition below.

\begin{proposition}
\label{prop:tmean}
Let $n \in \mathbb{N}^*$ and $f < \nicefrac{n}{2}$.
CWTM is $(f,\kappa)$-robust with $\kappa = \frac{6 f}{n-2f}\, \left( 1 + \frac{f}{n-2f} \right)$.

\end{proposition}

\begin{proof}
First, note that trimmed mean Trimmed Mean is a coordinate-wise aggregation, defined in Lemma~\ref{lem:cw-kappa}. 
Thus, due to Lemma~\ref{lem:cw-kappa}, it suffices to show that Trimmed Mean is $(f,\kappa)$-robust in the scalar domain, i.e., when $d=1$. \\
 
Let $x_1, \ldots, \, x_n \in \R$ and, without loss of generality, let us assume that $x_1 \leq \ldots \leq x_n$. We denote by $I \coloneqq \{f+1,\ldots,n-f\}$, and let $S$ be an arbitrary subset of $[n]$ of size $n-f$. Recall that  the set of indices selected by $\mathrm{TM}\left(x_1, \ldots, \, x_n \right) = \frac{1}{n-2f} \sum_{i \in I} x_i$. We obtain that
\begin{align*}
    \absv{\tmean{x_1,\ldots,x_n}-\overline{x}_S}^2
    &= \absv{\frac{1}{n-2f} \sum_{i \in I} x_i - \overline{x}_S}^2 
    = \absv{\frac{1}{n-2f} \sum_{i \in I} (x_i-\overline{x}_S)}^2 \\
    &= \absv{\frac{1}{n-2f} \sum_{i \in I} (x_i-\overline{x}_S) - \frac{1}{n-2f} \sum_{i \in S} (x_i-\overline{x}_S)}^2 &\Big(\text{as } \sum_{i \in S} (x_i-\overline{x}_S)=0\Big) \\
    &= \frac{1}{(n - 2f)^2} \, \absv{\sum_{i \in I \setminus S} (x_i-\overline{x}_S) - \sum_{i \in S \setminus I} (x_i-\overline{x}_S) }^2. \\
\end{align*}
Using Jensen's inequality above, we obtain that
\begin{align*}
    \absv{\tmean{x_1,\ldots,x_n}-\overline{x}_S}^2
    \leq \frac{\card{I \setminus S} + \card{S \setminus I}}{(n-2f)^2}\left(\sum_{i \in I \setminus S} \absv{x_i-\overline{x}_S}^2 + \sum_{i \in S \setminus I} \absv{x_i-\overline{x}_S}^2 \right).
\end{align*}
Note that $\card{I \setminus S} = \card{I \cup S} - \card{S} \leq n - (n-f) = f$. Similarly, $\card{S \setminus I} = \card{I \cup S} - \card{I} \leq n - (n-2f) = 2f$. Accordingly, $\card{I \setminus S} + \card{S \setminus I} \leq 3f$. Therefore, we have
\begin{align*}
    \absv{\tmean{x_1,\ldots,x_n}-\overline{x}_S}^2
    &\leq \frac{3f}{(n-2f)^2}\left(\sum_{i \in I \setminus S} \absv{x_i-\overline{x}_S}^2 + \sum_{i \in S \setminus I} \absv{x_i-\overline{x}_S}^2 \right).
\end{align*}
Recall from Lemma~\ref{lem:tmean-main} that $\sum_{i \in I \setminus S} \absv{x_i-\overline{x}_S}^2 \leq \sum_{i \in S \setminus I} \absv{x_i-\overline{x}_S}^2$.
Using this fact above we obtain that
\begin{align*}
    \absv{\tmean{x_1,\ldots,x_n}-\overline{x}_S}^2
    \leq \frac{6f}{(n-2f)^2}\sum_{i \in S \setminus I} \absv{x_i-\overline{x}_S}^2
    \leq \frac{6f}{(n-2f)^2}\sum_{i \in S} \absv{x_i-\overline{x}_S}^2.
\end{align*}
Finally, as 
$\frac{f}{(n-2f)^2} = \frac{f (n-f)}{(n-2f)^2} \, \left( \frac{1}{n-f}\right) = \Big(\frac{f}{n-2f}+\left(\frac{f}{n-2f}\right)^2\Big)\frac{1}{n-f}$, we obtain that
\begin{align}
    \absv{\tmean{x_1,\ldots,x_n}-\overline{x}_S}^2
    \leq 6\left(\frac{f}{n-2f}+\left(\frac{f}{n-2f}\right)^2\right) \frac{1}{\card{S}}\sum_{i \in S} \absv{x_i-\overline{x}_S}^2.
\end{align}
The above concludes the proof.
\end{proof}

\subsubsection{Krum}

In this section, we study a slight adaptation of the Krum algorithm first introduced in~\cite{krum}. 
Essentially, given the input vectors $x_1, \ldots, \, x_n$, Krum outputs the vector that is the nearest to its neighbors upon discarding $f$ (as opposed to $f+1$ in the original version) furthest vectors. Specifically, we denote by $\mathcal{N}_j$ the set the of indices of the $n-f$ nearest neighbors of $x_j$ in $\{x_1, \ldots, x_n\}$, with ties arbitrarily broken. 
Krum outputs the vector $x_{k^*}$ such that
\begin{equation*}
    k^* \in \argmin_{j \in [n]} \sum_{i \in \mathcal{N}_j} \norm{x_j - x_i}^2,
\end{equation*}
with ties arbitrarily broken if the set of minimizers above includes more than one element.

\begin{proposition}
Let $n \in \mathbb{N}^*$ and $f < \nicefrac{n}{2}$.
Krum is $(f,\kappa)$-robust with $\kappa = 6(1+\frac{f}{n-2f})$.
\end{proposition}
\begin{proof}
Let $n \in \mathbb{N}$, $f < \nicefrac{n}{2}$, and $x_1, \ldots, x_n \in \R^d$.
Consider any subset $S \subseteq [n]$ of size $\card{S}=n-f$.
In the following, for every $j \in [n]$, we denote by $\mathcal{N}_j$ the set the of indices of the $n-f$ nearest neighbors of $x_j$ in $\{x_1, \ldots, x_n\}$, with ties arbitrarily broken.
Observe that this implies, for every $j \in [n]$, that
\begin{equation}
\label{eq:3}
    \sum_{i \in \mathcal{N}_j} \norm{x_j-x_i}^2 \leq \sum_{i \in S} \norm{x_j-x_i}^2.
\end{equation}

Let $k^* \in [n]$ be the index selected by Krum.
By definition, it holds that
$
    k^* \in \argmin_{j \in [n]} \sum_{i \in \mathcal{N}_j} \norm{x_j - x_i}^2.
$
Therefore, leveraging~\eqref{eq:3} we have
\begin{align}
    \label{ineq2}
    \sum_{i \in \mathcal{N}_{k^*}}\norm{x_{k^*}-x_i}^2 
    &= \min_{j \in [n]} \sum_{i \in \mathcal{N}_j} \norm{x_j - x_i}^2
    \leq \min_{j \in S} \sum_{i \in \mathcal{N}_j} \norm{x_j - x_i}^2
    \leq \frac{1}{\card{S}} \sum_{j \in S} \sum_{i \in \mathcal{N}_j} \norm{x_j - x_i}^2\nonumber\\
    &\leq \frac{1}{\card{S}} \sum_{j \in S} \sum_{i \in S} \norm{x_j - x_i}^2
    = \frac{1}{\card{S}} \sum_{i, j \in S} \norm{x_j - \overline{x}_S - (x_i - \overline{x}_S)}^2 \nonumber\\
    &= \frac{1}{\card{S}} \sum_{i, j \in S} \Big (\norm{x_j - \overline{x}_S}^2 +\norm{x_i - \overline{x}_S}^2 - 2 \iprod{x_j - \overline{x}_S}{x_i - \overline{x}_S} \Big ) \nonumber\\
    &= \frac{1}{\card{S}} \Big [ \sum_{i, j \in S} \norm{x_j - \overline{x}_S}^2 + \sum_{i, j \in S}\norm{x_i - \overline{x}_S}^2 - 2 \sum_{i, j \in S} \iprod{x_j - \overline{x}_S}{x_i - \overline{x}_S} \Big ] \nonumber\\
    &= \frac{1}{\card{S}} \Big [ 2 \card{S} \sum_{i \in S} \norm{x_i - \overline{x}_S}^2 - 2 \sum_{i, j \in S} \iprod{x_j - \overline{x}_S}{x_i - \overline{x}_S} \Big ] \nonumber\\
    &= \frac{1}{\card{S}} \Big [ 2 \card{S} \sum_{i \in S} \norm{x_i - \overline{x}_S}^2 - 2 \sum_{i \in S} \iprod{ \underbrace{\sum_{j \in S} (x_j - \overline{x}_S)}_{=0}}{x_i - \overline{x}_S} \Big ]
    = 2 \sum_{i \in S} \norm{x_i - \overline{x}_S}^2.
\end{align}

Now, using Jensen's inequality, we can write for all $i \in S$,
\begin{align*}
    \norm{x_{k^*}-\overline{x}_S}^2 \leq 2\norm{x_{k^*}-x_i}^2 + 2\norm{x_i-\overline{x}_S}^2.
\end{align*}
Therefore, by rearranging the terms, we have for all $i \in S$,
\begin{align*}
    \norm{x_{k^*}-x_i}^2 \geq \frac{1}{2}\norm{x_{k^*}-\overline{x}_S}^2 - \norm{x_i-\overline{x}_S}^2.
\end{align*}
Together with the fact that $\card{S \cap \mathcal{N}_{k^*}} = \card{S} + \card{\mathcal{N}_{k^*}} - \card{S \cup \mathcal{N}_{k^*}} \geq (n-f) + (n-f) - n = n-2f$, the previous inequality implies that
\begin{align*}
    \sum_{i \in \mathcal{N}_{k^*}}\norm{x_{k^*}-x_i}^2
    &\geq \sum_{i \in S \cap \mathcal{N}_{k^*}}\norm{x_{k^*}-x_i}^2
    \geq \card{S \cap \mathcal{N}_{k^*}}\frac{1}{2}\norm{x_{k^*}-\overline{x}_S}^2 - \sum_{i \in S \cap \mathcal{N}_{k^*}}\norm{x_i-\overline{x}_S}^2\\
    &\geq \frac{n-2f}{2}\norm{x_{k^*}-\overline{x}_S}^2 - \sum_{i \in S \cap \mathcal{N}_{k^*}}\norm{x_i-\overline{x}_S}^2.
\end{align*}
By rearranging the terms, and invoking \eqref{ineq2} we can write
\begin{align*}
    \norm{x_{k^*}-\overline{x}_S}^2 
    & \leq \frac{2}{n-2f} \left [\sum_{i \in \mathcal{N}_{k^*}}\norm{x_{k^*}-x_i}^2 + \sum_{i \in S \cap \mathcal{N}_{k^*}}\norm{x_i-\overline{x}_S}^2\right] \\
    &\leq \frac{2}{n-2f} \left [\sum_{i \in \mathcal{N}_{k^*}}\norm{x_{k^*}-x_i}^2 + \sum_{i \in S}\norm{x_i-\overline{x}_S}^2\right] \\
    &\leq \frac{2}{n-2f} \left [2\sum_{i \in S}\norm{x_i-\overline{x}_S}^2+ \sum_{i \in S}\norm{x_i-\overline{x}_S}^2\right] \\
    &= \frac{6}{n-2f} \sum_{i \in S} \norm{x_i - \overline{x}_S}^2
    = \frac{6(n-f)}{n-2f} \frac{1}{\card{S}} \sum_{i \in S} \norm{x_i - \overline{x}_S}^2.
\end{align*}

We conclude by remarking that $\frac{n-f}{n-2f} = 1 + \frac{f}{n-2f}$.
\end{proof}

\subsubsection{Geometric Median}

The Geometric Median of $x_1, \ldots, \, x_n \in \R^d$ denoted by $\text{GM}(x_1, \ldots, x_n)$, is defined to be a vector that minimizes the sum of the $\ell_2$-distances to these vectors. Specifically, we have
 \begin{equation*}
     \text{GM}(x_1, \ldots, x_n) \in \argmin_{y \in \R^d} \sum_{i = 1}^n \norm{y-x_i}.
 \end{equation*}
 
\begin{proposition}
\label{prop:gm}
Let $n \in \mathbb{N}^*$ and $f < \nicefrac{n}{2}$.
GM is $(f,\kappa)$-robust with $\kappa = 4\left(1+\frac{f}{n-2f}\right)^2$.
\end{proposition}
\begin{proof}
The proof is similar to the analysis of Geometric Median in~\cite{karimireddy2022byzantinerobust}.
For completeness, we provide the full proof adapted to the definition of  $(f,\kappa)$-robustness.
Let us denote by $x^*:=\text{GM}(x_1, \ldots, x_n)$, the geometric median of the input vectors. Consider any subset $S \subseteq [n]$ of size $\card{S}=n-f$. By the reverse triangle inequality, for any $i \in S$, we have
\begin{equation}
\label{eq:GMONE}
    \norm{x^* - x_i} \geq \norm{x^* - \overline{x}_S} - \norm{x_i - \overline{x}_S}.
\end{equation}
Similarly, for any $i \in [n]\setminus S$, we obtain
\begin{equation}
\label{eq:GMTWO}
    \norm{x^* - x_i} \geq  \norm{x_i - \overline{x}_S} - \norm{x^* - \overline{x}_S}.
\end{equation}
Summing up~\eqref{eq:GMONE} and~\eqref{eq:GMTWO} over all input vectors we obtain
\begin{equation*}
    \sum_{i \in [n]} \norm{x^* - x_i} \geq (n-2f) \norm{x^* - \overline{x}_S} + \sum_{i \in [n]\setminus S} \norm{x_i - \overline{x}_S} - \sum_{i \in S} \norm{x_i - \overline{x}_S}.
\end{equation*}
Rearranging the terms, we obtain
\begin{align*}
         \norm{x^* - \overline{x}_S} &\leq \frac{1}{n-2f} \left( \sum_{i \in [n]} \norm{x^* - x_i}  - \sum_{i \in [n]\setminus S} \norm{x_i - \overline{x}_S} +  \sum_{i \in S} \norm{x_i - \overline{x}_S} \right)
\end{align*}
Note that by the definition of the geometric median, we have
\begin{equation*}
     \sum_{i \in [n]} \norm{x^* - x_i} \leq \sum_{i \in [n]} \norm{\overline{x}_S - x_i}.
\end{equation*}
Therefore,
\begin{align*}
         \norm{x^* - \overline{x}_S} &\leq \frac{1}{n-2f} \left( \sum_{i \in [n]} \norm{x_i - \overline{x}_S} - \sum_{i \in [n]\setminus S} \norm{x_i - \overline{x}_S} +  \sum_{i \in S} \norm{x_i - \overline{x}_S} \right)
         \leq  \frac{2}{n-2f} \sum_{i \in S} \norm{x_i - \overline{x}_S}.
\end{align*}
Squaring both sides and using Jensen's inequality, we obtain
\begin{align*}
    \norm{x^* - \overline{x}_S}^2 
    &\leq \frac{4(n-f)}{(n-2f)^2} \sum_{i \in S} \norm{x_i - \overline{x}_S}^2
    = \frac{4(n-f)^2}{(n-2f)^2} \cdot \frac{1}{n-f} \sum_{i \in S} \norm{x_i - \overline{x}_S}^2 \\
    &= 4 \left(1 +\frac{f}{n-2f} \right)^2 \frac{1}{|S|} \sum_{i \in S} \norm{x_i - \overline{x}_S}^2.
\end{align*}
This is the desired result.
\end{proof}


\subsubsection{Median}

For input vectors $x_1, \ldots, \, x_n$, their coordinate-wise median, denoted by $\text{CWMed}(x_1, \ldots, x_n)$, is defined to be a vector whose $k$-th coordinate, for all $k \in [d]$, is defined to be
\begin{equation}
    \left[\text{CWMed}\left(x_1, \ldots, x_n\right)\right]_k \coloneqq \text{Median}\left( [x_1]_k, \ldots [x_n]_k \right). \label{eqn:def_CWMed}
\end{equation}

\begin{proposition}
Let $n \in \mathbb{N}^*$ and $f < \nicefrac{n}{2}$.
CWMed is $(f,\kappa)$-robust with $\kappa = 4\left(1+\frac{f}{n-2f}\right)^2$.
\end{proposition}
\begin{proof}
Given that Geometric Median coincides with Median for one-dimensional inputs, it follows from Proposition~\ref{prop:gm} that, for one-dimensional inputs, Median is $(f,\kappa)$-robust with $\kappa = 4\left(1+\frac{f}{n-2f}\right)^2$.
Since Median is coordinate-wise, we deduce from Lemma~\ref{lem:cw-kappa} that the latter holds for any $d$-dimensional inputs, and conclude the proof.
\end{proof}

\subsection{Lower Bounds}
\label{app:kappa-lb}

In this section, we establish the tightness of our analysis by proving a universal lower bound on $\kappa$ in Proposition~\ref{prop:kappa-lb}, and an aggregation-specific lower bound in Proposition~\ref{prop:kappa-gar-lb}.

\begin{proposition}
\label{prop:kappa-lb}
Let $n \in \mathbb{N}^*, f<n$ and $\kappa >0$.
If $\aggregation$ is $(f,\kappa)$-robust, then $n>2f$ and $\kappa \geq \frac{f}{n-2f}$.
\end{proposition}
\begin{proof}
  Let $n \in \mathbb{N}^*, f<n$ and $\kappa >0$.
  Assume that $F$ is $(f, \kappa)$-resilient averaging aggregation rule. 
  Consider $x_1, \ldots, \, x_n$ such that $x_1 = \ldots = x_{n-f} = 0$, and $x_{n-f+1} = \ldots = x_{n} = 1$. 
  Let us first consider a set $S_0 = \left\{1,\ldots, n-f \right\}$. Since $\card{S_0} = n-f$, by definition, we have
  \begin{equation*}
    \absv{F(x_1, \ldots, \, x_n) - \overline{x}_{S_0}}^2 \leq \kappa\, \frac{1}{n-f} \sum_{i \in S_0} \absv{x_i - \overline{x}_{S_0}}^2 = 0.
  \end{equation*}
  Thus, $\aggregation(x_1, \ldots, \, x_n) = \overline{x}_{S_0} = 0$. 
  
  Now, consider another set $S_{1} = \left\{f+1,\ldots, n \right\}$.
  Observe that we necessarily have $n>2f$.
  Assume by contradiction that $2f \geq n$, i.e., $f+1 \geq n-f+1$.
  This implies that, for every $i\in S_1$, $x_i=1$.
  Therefore, $\overline{x}_{S_1}=1$.
  And, since $\aggregation$ is $(f,\kappa)$-robust, we must have
  \begin{equation*}
    \absv{F(x_1, \ldots, \, x_n) - \overline{x}_{S_1}}^2 \leq \kappa\, \frac{1}{n-f} \sum_{i \in S_1} \absv{x_i - \overline{x}_{S_1}}^2.
  \end{equation*}
  Therefore, $0=\overline{x}_{S_0} = \aggregation(x_1,\dots,x_n)=\overline{x}_{S_1}=1$, which is a contradiction.
  
  As a result, we must have $n>2f$.
  This implies that $\overline{x}_{S_1}=\frac{f}{n-f}$. 
  Thus, 
  \begin{align}
      \absv{F(x_1, \ldots, \, x_n) - \overline{x}_{S_1}}^2 = \left(\frac{f}{n-f}\right)^2. \label{eq:low_1}
  \end{align}
  Since $F$ is $(f, \kappa)$-resilient averaging rule, we have 
  \begin{align*}
      \absv{F(x_1, \ldots, \, x_n) - \overline{x}_{S_1}}^2 
      &\leq \kappa\, \frac{1}{n-f} \sum_{i \in S_1} \absv{x_i - \overline{x}_{S_1}}^2
      =\frac{\kappa}{n-f}
      \left (
      (n-2f)\absv{0-\frac{f}{n-f}}^2 +f\absv{1-\frac{f}{n-f}}^2
      \right ) \\
      &= \frac{\kappa}{n-f} 
      \left (
      (n-2f)\frac{f^2}{(n-f)^2} +f\frac{(n-2f)^2}{(n-f)^2}
      \right ) \\
      &= \frac{\kappa\, f(n-2f)}{n-f} 
      \left (
      \frac{f + n-2f}{(n-f)^2}
      \right )
      = \kappa \frac{f(n-2f)}{(n-f)^2}.
  \end{align*}
  Plugging this inequality back in \eqref{eq:low_1}, we conclude
    $\kappa \geq \frac{f}{n-2f}$.
\end{proof}

\begin{proposition}
\label{prop:kappa-gar-lb}
Let $n \in \N^*, f \geq 1$ such that $n > 2f$, and $\kappa \geq 0$.
For any $F \in \{\text{GM}, \text{CWMed}, \text{Krum}\}$, if $F$ is $(f,\kappa)$-robust then $\kappa \in \Omega{(1)}$.
\end{proposition}
\begin{proof}
Let $n \in \N^*, f \geq 1$ such that $n > 2f$ and $\kappa \geq 0$.
Let $F \in \{\text{GM}, \text{CWMed}, \text{Krum}\}$.
We will prove that if $F$ is $(f,\kappa)$-robust, then there exists an \emph{absolute} constant $c >0$ such that $\kappa \geq c$.

Consider $x_1, \ldots, x_n \in \R$ such that
$x_1, \ldots, x_{\floor*{\frac{n-f}{2}}} = -1$ and 
$x_{\floor*{\frac{n-f}{2}}+1}, \ldots, x_{n-f} = 1$.
Moreover, we set $x_{n-f+1}, \ldots, x_n = 1$.
Observe that, because $f \geq 1$, there is a strict majority of the inputs taking value $1$: the number of such inputs is $\Big(n-f - \floor*{\frac{n-f}{2}} \Big) + f = \ceil*{\frac{n-f}{2}} + f \geq \floor*{\frac{n-f}{2}} + f > \floor*{\frac{n-f}{2}}$.

Besides, recall that by setting $S = \{1,\ldots,n-f\}$, any $(f,\kappa)$-robust function $F$ should verify
  \begin{equation}
\label{eq:4}
    \absv{F(x_1, \ldots, \, x_n) - \overline{x}_{S}}^2 \leq \kappa\, \frac{1}{n-f} \sum_{i \in S} \absv{x_i - \overline{x}_{S}}^2.
  \end{equation}
  
However, the average $\overline{x}_S$ is equal to 
\begin{align}
    \label{eq:average}
    \overline{x}_S &= \frac{1}{n-f} \sum_{i \in S} x_i = \frac{1}{n-f} \left[\floor*{\frac{n-f}{2}} \times(-1) + \left(n-f - \floor*{\frac{n-f}{2}}\right)\times (+1)\right]\nonumber\\
    &= \frac{1}{n-f} \left(n-f - 2\floor*{\frac{n-f}{2}}\right)
    =
    \begin{cases}
    0 & \text{ if $n-f$ is even,}\\
    \frac{1}{n-f} & \text{ if $n-f$ is odd.}
    \end{cases}
\end{align}

Besides, the empirical average $\frac{1}{n-f} \sum_{i \in S} \absv{x_i - \overline{x}_{S}}^2 $ is equal to

\begin{align}
    \label{eq:variance}
    \frac{1}{n-f} \sum_{i \in S} \absv{x_i - \overline{x}_{S}}^2
    &= \frac{1}{n-f} \sum_{i \in S} \left(x_i^2 + \overline{x}_S^2 - 2x_i \overline{x}_S \right)
    =\frac{1}{n-f} \sum_{i \in S} x_i^2 -\overline{x}_S^2 \nonumber\\
    &= 1 - \overline{x}_S^2
    =
    \begin{cases}
    1 & \text{ if $n-f$ is even,}\\
    1 - \left(\frac{1}{n-f}\right)^2 & \text{ if $n-f$ is odd.}
    \end{cases}
\end{align}

Now, observe that the value of $F(x_1,\ldots,x_n)$ is equal to $1$ for every $F \in \{\text{GM}, \text{CWMed}, \text{Krum}\}$.
Indeed, since $f \geq 1$, a strict majority of the inputs take the value $1$, while the remaining inputs take the value $-1$.
Recall also that GM is identical to CWMed in one dimension.


By plugging this in \eqref{eq:4}, using \eqref{eq:average} and \eqref{eq:variance}, and rearranging terms, we obtain for every $F \in \{\text{GM}, \text{CWMed}, \text{Krum}\}$
\begin{align}
    \label{eq:kappa-lb}
    \kappa \geq \frac{\absv{F(x_1, \ldots, \, x_n)-\overline{x}_S}^2}{\frac{1}{n-f} \sum_{i \in S} \absv{x_i - \overline{x}_{S}}^2}
    &= \frac{\absv{1-\overline{x}_S}^2}{\frac{1}{n-f} \sum_{i \in S} \absv{x_i - \overline{x}_{S}}^2}
    = \begin{cases}
    1 & \text{ if $n-f$ is even,}\\
    \frac{\absv{1-\frac{1}{n-f}}^2}{1 - \left(\frac{1}{n-f}\right)^2} & \text{ if $n-f$ is odd.}
    \end{cases}
\end{align}
Note however that
$\frac{\absv{1-\frac{1}{n-f}}^2}{1 - \left(\frac{1}{n-f}\right)^2} = \frac{1-\frac{1}{n-f}}{1 + \frac{1}{n-f}} = \frac{n-f-1}{n-f+1} = 1 - \frac{2}{n-f+1}$.
Since $f \geq 1$, then $n \geq 2f+1 \geq 3$, and thus $\frac{\absv{1-\frac{1}{n-f}}^2}{1 - \left(\frac{1}{n-f}\right)^2} = 1 - \frac{2}{n-f+1} \geq \frac{1}{3}$.
Therefore, in both cases of \eqref{eq:kappa-lb}, we have $\kappa \geq \frac{1}{3}$ for every $F \in \{\text{GM}, \text{CWMed}, \text{Krum}\}$.
This concludes the proof.
\end{proof}

We can now see that the robustness coefficients given in Table~\ref{tab:kappa} are tight in order of magnitude.
Assume that $n \geq (2+\nu)f$ for some absolute constant $\nu>0$.
The coefficient of CWTM is of order $\mathcal{O}{\left(\frac{f}{n-2f}\right)}$, which is optimal following Proposition~\ref{prop:kappa-lb}.
The coefficients of CWMed, GM and Krum are of order $\mathcal{O}{(1)}$, which is optimal following Proposition~\ref{prop:kappa-gar-lb}.

\clearpage
\subsection{Unifying Robustness Definitions}
\label{app:unification}

In this section, we prove that $(f,\kappa)$-robustness unifies existing robustness definitions in the literature~\cite{farhadkhani2022byzantine,karimireddy2022byzantinerobust}.
Specifically, we prove that verifying our definition implies verifying the other definitions.
The reason behind this is that, in $(f,\kappa)$-robustness, we control the error on estimating the average with a smaller quantity compared to existing works.

\subsubsection{$(f,\lambda)$-resilient Averaging}

Recall the definition of $(f, \lambda)$-resilient averaging~\cite{farhadkhani2022byzantine} below.
\begin{definition}[{\bf $(f, \, \lambda)$-Resilient averaging}]
For $f < n$ and real value $\lambda \geq 0$, an aggregation rule $F$ is called {\em $(f, \, \lambda)$-resilient averaging} if for any collection of $n$ vectors $x_1, \ldots, \, x_n$, and any set $S \subseteq \{1, \ldots, \, n\}$ of size $n-f$,
\begin{align*}
    \norm{F(x_1, \ldots, \, x_n) - \overline{x}_S} \leq \lambda \max_{i, j \in S} \norm{x_i - x_j},
\end{align*}
where $\overline{x}_S \coloneqq \frac{1}{\card{S}} \sum_{i \in S} x_i$, and $\card{S}$ is the cardinality of $S$.
\end{definition}

The following proposition shows that $(f,\kappa)$-robustness implies $(f,\lambda)$-resilient averaging.

\begin{proposition}
\label{prop:kappa-lambda}
Let $n \in \mathbb{N}^*$, $f<n$ and $\kappa > 0$.
If $\aggregation$ is $(f,\kappa)$-robust, then $\aggregation$ is $(f,\lambda)$-resilient with $\lambda = \sqrt{\kappa/2}$.
\end{proposition}
\begin{proof}
Let $n \in \mathbb{N}^*$, $f<n$ and $\lambda,\kappa > 0$.
Let $x_1, \ldots, x_n \in \R^d$ and $S \subseteq \{1,\ldots,n\}$ such that $\card{S}=n-f$.

If $F$ is $(f,\kappa)$-robust, then we have
\begin{align}
    \label{eq:kappa-lambda}
    \norm{F{(x_1,\ldots,x_n)}-\overline{x}_S}^2
    \leq \kappa \frac{1}{n-f} \sum_{i \in S} \norm{x_i - \overline{x}_S}^2
    =\kappa \frac{1}{2(n-f)^2}\sum_{i,j \in S} \norm{x_i - x_j}^2
    \leq \frac{\kappa}{2} \max_{i,j \in S} \norm{x_i - x_j}^2.
\end{align}
The equality above is due to
\begin{align}   
    \label{eq:variance-pairwise}
    \frac{1}{2(n-f)^2}\sum_{i,j \in S}\norm{x_i - x_j}^2
    &= \frac{1}{2(n-f)^2}\sum_{i,j \in S}\norm{x_i - \overline{x}_S + \overline{x}_S - x_j}^2\nonumber\\
    &= \frac{1}{2(n-f)^2}\sum_{i,j \in S} \left( \norm{x_i - \overline{x}_S}^2 + \norm{\overline{x}_S - x_j}^2 -2(x_i - \overline{x}_S)(x_j - \overline{x}_S)\right)\nonumber\\
    &= \frac{1}{2(n-f)^2} \left[ (n-f)\sum_{i \in S}\norm{x_i - \overline{x}_S}^2 + (n-f)\sum_{j \in S}\norm{\overline{x}_S - x_j}^2\right] \nonumber\\
    &\quad -\frac{1}{(n-f)^2}\sum_{i,j \in S}(x_i - \overline{x}_S)(x_j - \overline{x}_S)\nonumber\\
    &=\frac{1}{n-f}\sum_{i \in S}\norm{x_i - \overline{x}_S}^2 - \frac{1}{(n-f)^2}\sum_{i \in S}(x_i - \overline{x}_S)\underbrace{\sum_{j \in S}(x_j - \overline{x}_S)}_{=0}\nonumber\\
    &=\frac{1}{n-f}\sum_{i \in S}\norm{x_i - \overline{x}_S}^2.
\end{align}
Taking the square root of both sides in \eqref{eq:kappa-lambda} concludes the proof.
\end{proof}

As a consequence of Proposition~\ref{prop:kappa-lambda}, one can measure the significant improvement over the analysis of aggregation rules in~\cite{farhadkhani2022byzantine}.
For example, the coefficient $\lambda$ proved for CWMed in~\cite{farhadkhani2022byzantine} is $\frac{n}{2(n-f)}\min{\{2\sqrt{n-f},\sqrt{d}\}}$, which is either growing with the dimension or $n$.
In contrast, Proposition~\ref{prop:kappa-lambda} shows that our analysis implies the coefficient $\lambda = \sqrt{\kappa/2} = \sqrt{2}\left(1+\frac{f}{n-2f}\right)$, which is bounded whenever $n \geq (2+\nu)f$ for some $\nu>0$.
A similar observation can be made for CWTM.

\subsubsection{$(\delta_{\text{max}},c)$-agnostic Robust Aggregation}

We recall the definition of an agnostic robust aggregator (ARAgg)~\cite{karimireddy2022byzantinerobust} below.
Note that the so-called "good" subset in the original definition of $(\delta_{\text{max}},c)$-ARAgg~\cite{karimireddy2022byzantinerobust} is only required to be of size $\card{S} \geq (1-\delta)n$, where $\delta$ is an \emph{upper bound} on the fraction of Byzantine workers.
In our formalism, $f$ is the actual number of Byzantine workers, and thus we directly require $\card{S} = n-f$ without loss of generality.
\begin{definition}[$(\delta_{\text{max}},c)${\bf -ARAgg}]
Given inputs $X_1,\ldots,X_n$ such that a subset $S, \card{S} = n-f$ with $\frac{f}{n} \leq \delta_{\text{max}} < 0.5$ satisfies $\expect{\norm{X_i - X_j}^2} \leq \rho^2$ for all $i,j \in S$. 
Then, the output $\hat{X}$ of a $(\delta_{\text{max}},c)$-ARAgg satisfies $\expect{\norm{\hat{X} -\overline{X}_S}^2} \leq c \, \frac{f}{n} \rho^2$. 
\end{definition}

The following proposition shows that $(f,\kappa)$-robustness implies $(\delta_{\text{max}},c)$-agnostic robust aggregation.
We also show that to obtain $(\delta_{\text{max}},c)$-agnostic robust aggregation with $c = \mathcal{O}{(1)}$, although there is no such theoretical requirement in \cite{karimireddy2022byzantinerobust}, it is sufficient to have $\kappa \in  \mathcal{O}{\left(\frac{f}{n-2f}\right)}$.

\begin{proposition}
Let $n \in \mathbb{N}^*, 0 < f < \nicefrac{n}{2}, \kappa > 0$, and $\frac{f}{n} \leq \delta_{\text{max}} < \frac{1}{2}$.
If $\aggregation$ is $(f,\kappa)$-robust,
then $\aggregation$ is $(\delta_{\text{max}},c)$-ARAgg with $c = \kappa \,\frac{n}{2f}$.
Furthermore, if $\kappa \in  \mathcal{O}{\left(\frac{f}{n-2f}\right)}$ then $c = \mathcal{O}{(1)}$.
\end{proposition}
\begin{proof}
Let $n \in \mathbb{N}^*, 0 < f < \nicefrac{n}{2}, \kappa > 0$, and $0 < \delta_{\text{max}} < \frac{1}{2}$.
Assume that $\aggregation$ is $(f,\kappa)$-robust.
Consider any $x_1, \ldots, x_n \in \R^d$ and $S \subseteq [n]$ such that $\card{S} = n-f$.

As $F$ is $(f,\kappa)$-robust, using \eqref{eq:variance-pairwise}, we have
\begin{align}
\label{eq:aragg-det}
    \norm{F{(x_1,\ldots,x_n)}-\overline{x}_S}^2
    \leq \frac{\kappa }{n-f} \sum_{i \in S} \norm{x_i - \overline{x}_S}^2
    = \frac{\kappa}{2(n-f)^2} \sum_{i,j \in S} \norm{x_i - x_j}^2.
\end{align}

For any random variables $X_1, \ldots, X_n$, integrating over \eqref{eq:aragg-det} with the joint probability measure of these variables then gives
\begin{align}
    \label{eq:aragg-stoch}
    \expect{\norm{F{(X_1,\ldots,X_n)} - \overline{X}_S}^2}
    &\leq \frac{\kappa}{2} \expect{\frac{1}{(n-f)^2} \sum_{i,j \in S} \norm{X_i - X_j}^2}
    = \frac{\kappa}{2} \frac{1}{\card{S}^2}\sum_{i,j \in S} \expect{\norm{X_i - X_j}^2}\nonumber\\
    &\leq \frac{\kappa}{2} \max_{i,j \in S}\expect{\norm{X_i - X_j}^2}
    = c\, \frac{f}{n} \max_{i,j \in S}\expect{\norm{X_i - X_j}^2}.
\end{align}
where $c \coloneqq \kappa \, \frac{n}{2f}$.
Thus, if $X_1, \ldots, X_n$ are such that there exists a subset $S \subseteq [n], \card{S} = n-f,$ for which $\expect{\norm{X_i - X_j}^2} \leq \rho^2$ for all $i,j \in S$.
Then, it holds that $\rho^2 \geq \max_{i,j \in S} \expect{\norm{X_i - X_j}^2}$.
This fact together with \eqref{eq:aragg-stoch} allows to conclude the desired result; that is,
\begin{align*}
    \expect{\norm{F{(X_1,\ldots,X_n)} - \overline{X}_S}^2} \leq c\, \frac{f}{n} \rho^2.
\end{align*}

Furthermore, if $\kappa = \mathcal{O}{\left(\frac{f}{n-2f}\right)}$, then we can write $\kappa \leq 2m (1-2 \delta_{\text{max}}) \frac{f}{n-2f}$ for some absolute constant $m > 0$, since $\delta_{\text{max}} \in (0,\frac{1}{2})$.
However, since $\frac{f}{n} \leq \delta_{\text{max}}$ and $\delta \mapsto \frac{2\delta}{1-2\delta}$ is non-decreasing, we have $\frac{f}{n-2f} = \frac{f}{n}(1+\frac{2f}{n-2f}) \leq \frac{f}{n}(1+\frac{2\delta_{\text{max}}}{1-2\delta_{\text{max}}})$, and thus $\kappa \leq 2m (1-2 \delta_{\text{max}}) \frac{f}{n-2f} \leq 2m\,  \frac{f}{n}$.
As a result, we have $c = \kappa\, \frac{n}{2f} \leq m = \mathcal{O}{(1)}$.
This concludes the proof.
\end{proof}

Our analysis improves over that of \cite{karimireddy2022byzantinerobust} for CWMed and Krum.
For CWMed, their rate is dimension-dependent unlike ours.
For Krum, they only prove robustness assuming $n>4f$, while our analysis holds for $n>2f$.

\clearpage
\section{Proof of Lemma~\ref{lem:cenna}: Analysis of $\cenna$}
\label{app:cenna}

\subsection{Proof Overview}

The proof of Lemma~\ref{lem:cenna} relies on the observation that $\cenna$ brings the inputs closer to the true average.
This is formalized in Lemma~\ref{lem:cenna-reduction} where we show that the empirical variance is reduced by a factor of order $\nicefrac{f}{n}$.
Note that although outputs of $\cenna$ have smaller variance than the original inputs, their average may deviate from the original average.
Then, in the proof of Lemma~\ref{lem:cenna}, we control this bias introduced by the $\cenna$ operation, and use the reduction proved in Lemma~\ref{lem:cenna-reduction} to conclude.

\paragraph{Notation.}
In the following, for every set $S \subseteq [n]$, and every vectors $x_1, \ldots, x_n \in \R^d$, we denote by $\overline{x}_S$ the average $\frac{1}{\card{S}} \sum_{i \in S} x_i$.
Let $\mu \in \R^d$.
We denote by $y_{\mu}$ the average of the $n-f$ nearest neighbors of $\mu$ in $\{x_1, \ldots, x_n\}$:
\begin{equation*}
    y_\mu \coloneqq \frac{1}{\card{\mathcal{N}_{\mu}}} \sum_{i \in \mathcal{N}_{\mu}} x_i,
\end{equation*}
where $\mathcal{N}_{\mu} \subseteq \{1, \ldots, \, n\}, \card{\mathcal{N}_{\mu}}=n-f,$ is the set of indices of the $n-f$ nearest neighbors of $\mu$.

\subsection{Proof of Supporting Lemmas}

We first prove a general lemma allowing us to control the distance between the nearest neighbor average $y_\mu$ and the true average $\overline{x}_S$ with the dispersion of $(x_i)_{i \in S}$ around the pivot $\mu$.

\begin{lemma}
\label{lem:cva-pivot}
Let $n \in \mathbb{N}^*$, $f < n$, and $\mu \in \R^d$.
For any set $S \subseteq \{1,\dots,n\}$, $\card{S}= n-f$, we have for any vectors $x_1, \dots, x_n \in \R^d$,
\begin{equation*}
    \norm{y_\mu - \overline{x}_S}^2
    \leq \frac{4f}{n-f} \frac{1}{\card{S}} \sum_{i \in S} \norm{x_i-\mu}^2.
\end{equation*}
\end{lemma}
\begin{proof}
Let $n \in \mathbb{N}^*$, $f<n$, $\mu \in \R^d$, $x_1, \ldots, x_n \in \R^d$, and $S \subseteq \{1,\dots,n\}$, $\card{S}= n-f$.
Recall that, by definition, we have
    $y_\mu = \frac{1}{n-f} \sum_{i \in \mathcal{N}_{\mu}} x_i$,
where $\mathcal{N}_{\mu} \subseteq \{1, \ldots, \, n\}, \card{\mathcal{N}_{\mu}}=n-f,$ is the set of indices of the $n-f$ nearest neighbors of $\mu$.
We then have
\begin{align*}
    \norm{y_\mu - \overline{x}_S}^2
    &= \norm{\frac{1}{n-f} \sum_{i \in \mathcal{N}_{\mu}} x_i - \frac{1}{n-f} \sum_{i \in S} x_i}^2
    = \frac{1}{(n-f)^2}
    \norm{\sum_{i \in \mathcal{N}_{\mu}} x_i - \sum_{i \in S} x_i}^2 \\
    &= \frac{1}{(n-f)^2}
    \norm{\sum_{i \in \mathcal{N}_{\mu} \setminus S} x_i - \sum_{i \in S \setminus \mathcal{N}_{\mu}} x_i}^2.
\end{align*}
Observe that, since $\card{S} = \card{\mathcal{N}_{\mu}} = n-f$, we have
$\card{S \setminus \mathcal{N}_{\mu}} = \card{\mathcal{N}_{\mu} \setminus S} = \card{S \cup \mathcal{N}_{\mu}} - \card{S} \leq n - (n-f) \leq f$.
As a result, by applying Jensen's inequality, we have
\begin{align*}
    \norm{y_\mu - \overline{x}_S}^2
    &= \frac{1}{(n-f)^2}
    \norm{\sum_{i \in \mathcal{N}_{\mu} \setminus S} x_i - \sum_{i \in S \setminus \mathcal{N}_{\mu}} x_i}^2
    = \frac{1}{(n-f)^2}
    \norm{\sum_{i \in \mathcal{N}_{\mu} \setminus S} (x_i-\mu) - \sum_{i \in S \setminus \mathcal{N}_{\mu}} (x_i-\mu)}^2 \\
    &\leq \frac{\card{S \setminus \mathcal{N}_{\mu}} + \card{\mathcal{N}_{\mu} \setminus S}}{(n-f)^2}
    \left [
    \sum_{i \in \mathcal{N}_{\mu} \setminus S} \norm{x_i-\mu}^2 +
    \sum_{i \in S \setminus \mathcal{N}_{\mu}} \norm{x_i-\mu}^2
    \right] \\
    &\leq \frac{2f}{(n-f)^2}
    \left [
    \sum_{i \in \mathcal{N}_{\mu} \setminus S} \norm{x_i-\mu}^2 +
    \sum_{i \in S \setminus \mathcal{N}_{\mu}} \norm{x_i-\mu}^2
    \right].
\end{align*}
On one hand, since $\mathcal{N}_{\mu}$ is the set of $n-f$ nearest neighbors to $\mu$, the first term can be bounded by
\begin{align*}
    \sum_{i \in \mathcal{N}_{\mu} \setminus S} \norm{x_i-\mu}^2
    \leq \sum_{i \in \mathcal{N}_{\mu}} \norm{x_i-\mu}^2
    \leq \sum_{i \in S} \norm{x_i-\mu}^2.
\end{align*}
On the other hand, the second term can be bounded by
\begin{align*}
    \sum_{i \in S \setminus \mathcal{N}_{\mu}} \norm{x_i-\mu}^2
    &\leq \sum_{i \in S} \norm{x_i-\mu}^2.
\end{align*}
We finally conclude that
\begin{equation*}
    \norm{y_\mu - \overline{x}_S}^2
    \leq \frac{4f}{(n-f)^2} \sum_{i \in S} \norm{x_i-\mu}^2.
\end{equation*}
\end{proof}
The second lemma crucially shows that the sum of the bias and the variance is reduced by a factor $\frac{8f}{n-f}$.
To do so, we specialize the first lemma by setting the pivot to be the element $x_i, i \in S$.

\begin{lemma}
\label{lem:cenna-reduction}
Let $n \in \mathbb{N}^*$, $f<n$ .
For any set $S \subseteq \{1,\dots,n\}$, $\card{S}= n-f$, for any vectors $x_1, \dots, x_n \in \R^d$, the vectors $(y_1,\ldots,y_n) = \cenna{(x_1,\ldots,x_n)}$ verify
\begin{equation*}
    \underbrace{\frac{1}{\card{S}} \sum_{i \in S} \norm{y_i - \overline{y}_S}^2}_{\text{variance}} + \underbrace{\norm{\overline{y}_S - \overline{x}_S}^{\mathrlap{2}}}_{\text{bias}}
    \leq \frac{8f}{n-f} \frac{1}{\card{S}} \sum_{i \in S} \norm{x_i-\overline{x}_S}^2.
\end{equation*}
\end{lemma}
\begin{proof}
Let $n \in \mathbb{N}^*$, $f<n$, $x_1, \ldots, x_n \in \R^d$, and $S \subseteq \{1,\dots,n\}$, $\card{S}= n-f$.
We first prove the following bias-variance decomposition
\begin{align}
    \label{eq:bias-variance}
    \frac{1}{\card{S}} \sum_{i \in S} \norm{y_i - \overline{y}_S}^2 + \norm{\overline{y}_S - \overline{x}_S}^2 = \frac{1}{\card{S}} \sum_{i \in S} \norm{y_i - \overline{x}_S}.
\end{align}
We develop the first term of the l.h.s. of the equality above and obtain
\begin{align*}
    \frac{1}{n-f}\sum_{i \in S} \norm{y_i- \overline{y}_S}^2
    &= \frac{1}{n-f}\sum_{i \in S} \norm{y_i-\overline{x}_S+\overline{x}_S- \overline{y}_S}^2\\
    &= \frac{1}{n-f}\sum_{i \in S} \norm{y_i-\overline{x}_S}^2+\norm{\overline{x}_S- \overline{y}_S}^2
    + 2\frac{1}{n-f}\sum_{i \in S} \iprod{y_i-\overline{x}_S}{\overline{x}_S- \overline{y}_S}.
\end{align*}
However, we have
\begin{align*}
    \frac{1}{n-f}\sum_{i \in S} \iprod{y_i-\overline{x}_S}{\overline{x}_S- \overline{y}_S}
    = \iprod{\underbrace{\frac{1}{n-f}\sum_{i \in S}y_i}_{= \overline{y}_S}-\overline{x}_S}{\overline{x}_S- \overline{y}_S}
    = - \norm{\overline{x}_S - \overline{y}_S}^2.
\end{align*}
As a result, we then have
\begin{align*}
    \frac{1}{n-f}\sum_{i \in S} \norm{y_i- \overline{y}_S}^2
    &= \frac{1}{n-f}\sum_{i \in S} \norm{y_i-\overline{x}_S}^2+\norm{\overline{x}_S- \overline{y}_S}^2 -2\norm{\overline{x}_S- \overline{y}_S}^2\\
    &= \frac{1}{n-f}\sum_{i \in S} \norm{y_i-\overline{x}_S}^2-\norm{\overline{x}_S- \overline{y}_S}^2.
\end{align*}
This proves \eqref{eq:bias-variance}.
Besides, for any $i \in S$, we know from Lemma~\ref{lem:cva-pivot}, with $\mu = x_i$, that
\begin{equation*}
    \norm{y_i - \overline{x}_S}^2
    \leq \frac{4f}{n-f} \frac{1}{n-f} \sum_{j \in S} \norm{x_j - x_i}^2.
\end{equation*}

As a consequence, using \eqref{eq:bias-variance}, we have
\begin{align*}
    \frac{1}{\card{S}} \sum_{i \in S} \norm{y_i - \overline{y}_S}^2 + \norm{\overline{y}_S - \overline{x}_S}^2
    &=\frac{1}{n-f} \sum_{i \in S} \norm{y_i - \overline{x}_S}^2
    \leq \frac{4f}{n-f} \frac{1}{(n-f)^2} \sum_{i, j \in S} \norm{x_j - x_i}^2\\
    &= \frac{4f}{n-f} \frac{1}{(n-f)^2} \sum_{i, j \in S} \norm{x_j - \overline{x}_S - (x_i - \overline{x}_S)}^2\\
    &= \frac{4f}{n-f} \frac{1}{(n-f)^2} \sum_{i, j \in S} \left (\norm{x_j - \overline{x}_S}^2 +\norm{x_i - \overline{x}_S}^2 - 2 \iprod{x_j - \overline{x}_S}{x_i - \overline{x}_S} \right ) \\
    &= \frac{4f}{n-f} \frac{1}{(n-f)^2} \Big [ \sum_{i, j \in S} \norm{x_j - \overline{x}_S}^2 + \sum_{i, j \in S}\norm{x_i - \overline{x}_S}^2 - 2 \sum_{i, j \in S} \iprod{x_j - \overline{x}_S}{x_i - \overline{x}_S} \Big ] \\
    &= \frac{4f}{n-f} \frac{1}{(n-f)^2} \Big [ 2(n-f) \sum_{i \in S} \norm{x_i - \overline{x}_S}^2 - 2 \sum_{i, j \in S} \iprod{x_j - \overline{x}_S}{x_i - \overline{x}_S} \Big ] \\
    &= \frac{4f}{n-f} \frac{1}{(n-f)^2} \Big [ 2(n-f) \sum_{i \in S} \norm{x_i - \overline{x}_S}^2 - 2 \sum_{i \in S} \iprod{ \underbrace{\sum_{j \in S} (x_j - \overline{x}_S)}_{=0}}{x_i - \overline{x}_S} \Big ] \\
    &= \frac{8f}{n-f} \frac{1}{n-f} \sum_{i \in S} \norm{x_i - \overline{x}_S}^2.
\end{align*}
This concludes the proof.
\end{proof}

\subsection{Proof of Lemma~\ref{lem:cenna}}
We can now prove Lemma~\ref{lem:cenna}, mainly thanks to the bias-variance trade-off guaranteed by Lemma~\ref{lem:cenna-reduction}.

\begin{replemma}{lem:cenna}
Let $f < \nicefrac{n}{2}$ and $F \colon \R^{d \times n} \to \R^d$.
If $F$ is $(f,\kappa)$-robust, then $F \circ \cenna$ is $(f,\kappa')$-robust with
$\kappa' \leq \frac{8f}{n-f}(\kappa+1).$
\end{replemma}
\begin{proof}
Let $n \in \mathbb{N}^*, f<n, x_1, \dots, x_n \in \R^d$, $S \subseteq \{1,\dots,n\}, \card{S}=n-f$.
Assume $F$ is $(f,\kappa)$-robust.
First, denote $(y_1,\ldots,y_n) \coloneqq F \circ \cenna{(x_1,\dots,x_n)}$ and $ \overline{y}_S \coloneqq \frac{1}{n-f} \sum_{i \in S} y_i$. 
Since $F$ is $(f,\kappa)$-robust, the vectors $y_1,\ldots,y_n$ satisfy
\begin{align}
    \label{eq:1}
    \norm{F \circ \cenna{(x_1,\dots,x_n)} -  \overline{y}_S}^2
    = \norm{F{(y_1,\dots,y_n)} -  \overline{y}_S}^2
    &\leq \kappa \frac{1}{\card{S}} \sum_{i \in S} \norm{y_i- \overline{y}_S}^2.
\end{align}

Now, thanks to Young's inequality we have for $c=\nicefrac{1}{\kappa}$, $(a+b)^2 \leq (1+c)a^2 + (1+\nicefrac{1}{c})b^2$.
We can then write using \eqref{eq:1}
\begin{align}
    \label{eq:2}
    \norm{F \circ \cenna{(x_1,\dots,x_n)} - \overline{x}_S}^2
    &= \norm{F \circ \cenna{(x_1,\dots,x_n)} -  \overline{y}_S +  \overline{y}_S - \overline{x}_S}^2\nonumber\\
    &\leq (1+\frac{1}{\kappa})\norm{F \circ \cenna{(x_1,\dots,x_n)} -  \overline{y}_S}^2 + (1+\kappa)\norm{ \overline{y}_S - \overline{x}_S}^2\nonumber\\
    &\leq (1+\frac{1}{\kappa})\kappa \frac{1}{\card{S}}\sum_{i \in S} \norm{y_i- \overline{y}_S}^2 + (1+\kappa)\norm{\overline{y}_S - \overline{x}_S}^2\nonumber\\
    &= \frac{1+\kappa}{n-f}\sum_{i \in S} \norm{y_i- \overline{y}_S}^2 + (1+\kappa)\norm{ \overline{y}_S - \overline{x}_S}^2.
\end{align}

Recall that Lemma~\ref{lem:cenna-reduction} shows that
$\frac{1}{\card{S}}\sum_{i \in S} \norm{y_i- \overline{y}_S}^2 + \norm{ \overline{y}_S - \overline{x}_S}^2 \leq \frac{8f}{n-f}\frac{1}{\card{S}}\sum_{i \in S} \norm{x_i-\overline{x}_S}^2.$
Plugging this in \eqref{eq:2} allows to conclude
\begin{align*}
    \norm{F \circ \cenna{(x_1,\dots,x_n)} - \overline{x}_S}^2
    &\leq \frac{8f}{n-f}\frac{1+\kappa}{n-f}\sum_{i \in S} \norm{x_i-\overline{x}_S}^2.
\end{align*}
\end{proof}
\clearpage
\section{Analysis of Bucketing}
\label{app:Bucketing}
In this section, we compare the robustness guarantees of the Bucketing algorithm~\cite{karimireddy2022byzantinerobust} with NNM. 
Bucketing is a pre-aggregation step that maps input vectors $\{x_1, \ldots, x_n \}$ to output vectors $\{y_1,\ldots,y_{\lceil{\nicefrac{n}{\bucketSize}}\rceil }\}$, where $\bucketSize$ is a parameter called the bucket size. 
The Bucketing algorithm works as follows. First, it samples a random permutation $\pi \colon [n] \to [n]$. 
Then, for each $i \in \left\{1, \ldots, \ceil*{\nicefrac{n}{\bucketSize}} \right\}$, it computes the output of bucket $i$ as\footnote{In the theoretical analysis, for simplicity, we assume that $n$ is divisible by $\bucketSize$. If this is not the case, some buckets will include less than $\bucketSize$ vectors which is less favorable for the theoretical guarantees of Bucketing but does not change the asymptotic behavior.} 
$$y_i := \frac{1}{\bucketSize} \sum_{k =\bucketSize(i-1)+1 }^{\min(\bucketSize \cdot i,n)} x_{\pi(k)}.$$

Recall that the property that enables NNM to boost the robustness of aggregations is that the heterogeneity of the output vectors $\frac{1}{\card{S}}\sum_{i \in S} \norm{y_i- \overline{x}_S}^2$ is a factor $\mathcal{O}(f/n)$ smaller than the  heterogeneity of the input vectors $\frac{1}{\card{S}}\sum_{i \in S} \norm{x_i- \overline{x}_S}^2$ (see Lemma~\ref{lem:cenna-reduction}).
But, unlike NNM, Bucketing only reduces the heterogeneity \emph{in expectation over the random permutations}
(see Lemma~1 in \cite{karimireddy2022byzantinerobust}).
Thus, there are iterations of the learning algorithm where Bucketing may not reduce the heterogeneity, and these iterations are the best opportunity for the Byzantine workers to cause the most damage to the learning.

We give in Observation~\ref{obs:reduction} below a simple example showing that deterministically reducing the heterogeneity of inputs is impossible using Bucketing in general, even in the absence of malicious inputs.
\begin{observation}
\label{obs:reduction}
Using Bucketing, it is impossible to provide a \emph{worst-case} heterogeneity reduction guarantee regardless of the value of $\bucketSize$, \emph{even in the absence of Byzantine inputs}.
\end{observation}
 \begin{proof}
Let $\pi$ be any permutation over $[n]$. We will construct an instance of $n$ inputs such that an execution of Bucketing using permutation $\pi$ does not reduce the heterogeneity.
Consider the inputs $(x_1,\ldots,x_n)$ such that for all $1 \leq i \leq \ceil*{\nicefrac{n}{\bucketSize}}$ and all $k,l \in \{s(i-1)+1, \ldots, s \cdot i$\}, it holds that $x_{\pi(k)} = x_{\pi(l)} \eqqcolon y_i$.
Thus, applying the permutation $\pi$ on these inputs yields
  $\big(\overbrace{y_1, \ldots, y_1}^{\bucketSize  \text{ times}},\ldots,\overbrace{ y_{\nicefrac{n}{\bucketSize}}, \ldots, y_{\nicefrac{n}{\bucketSize}}}^{\bucketSize  \text{ times}} \Big)$ in which each vector $y_i$ is repeated $\bucketSize$ times.
  Overall, the execution of Bucketing with $\pi$ will produce $(y_1,\ldots,y_{\nicefrac{n}{\bucketSize}})$, which has the same variance as the original inputs, and thus heterogeneity was not reduced.
 \end{proof}

Moreover, given that the buckets are randomly chosen, the (expected) reduction of heterogeneity is only achieved at the expense of an increase in the fraction of Byzantine inputs:
\begin{observation}
Bucketing increases the fraction of Byzantine workers by a factor $\bucketSize$ in the worst case.
\end{observation}
\begin{proof}
In some executions, Bucketing might assign every Byzantine input to a different bucket. In this case, since the average within each ``contaminated" bucket is arbitrarily manipulable by a single Byzantine input, the number of Byzantine output vectors is the same as the number of Byzantine input vectors. However, the total number of output vectors is $\bucketSize$ times smaller than the total number of input vectors, thereby effectively increasing the fraction of Byzantine workers.
\end{proof}
This observation implies that in order to have a worst-case robustness guarantee, the bucket size $\bucketSize$ can be at most $\left \lfloor \frac{n}{2f} \right\rfloor$, as done in \cite{karimireddy2022byzantinerobust}.
This instability in reducing the heterogeneity results in a poor estimation of the mean in some iterations. We observe this instability in practice as we show below on experiments on CIFAR-10. 

\begin{figure}[ht!]
    \centering
    \includegraphics[width=0.49\textwidth]{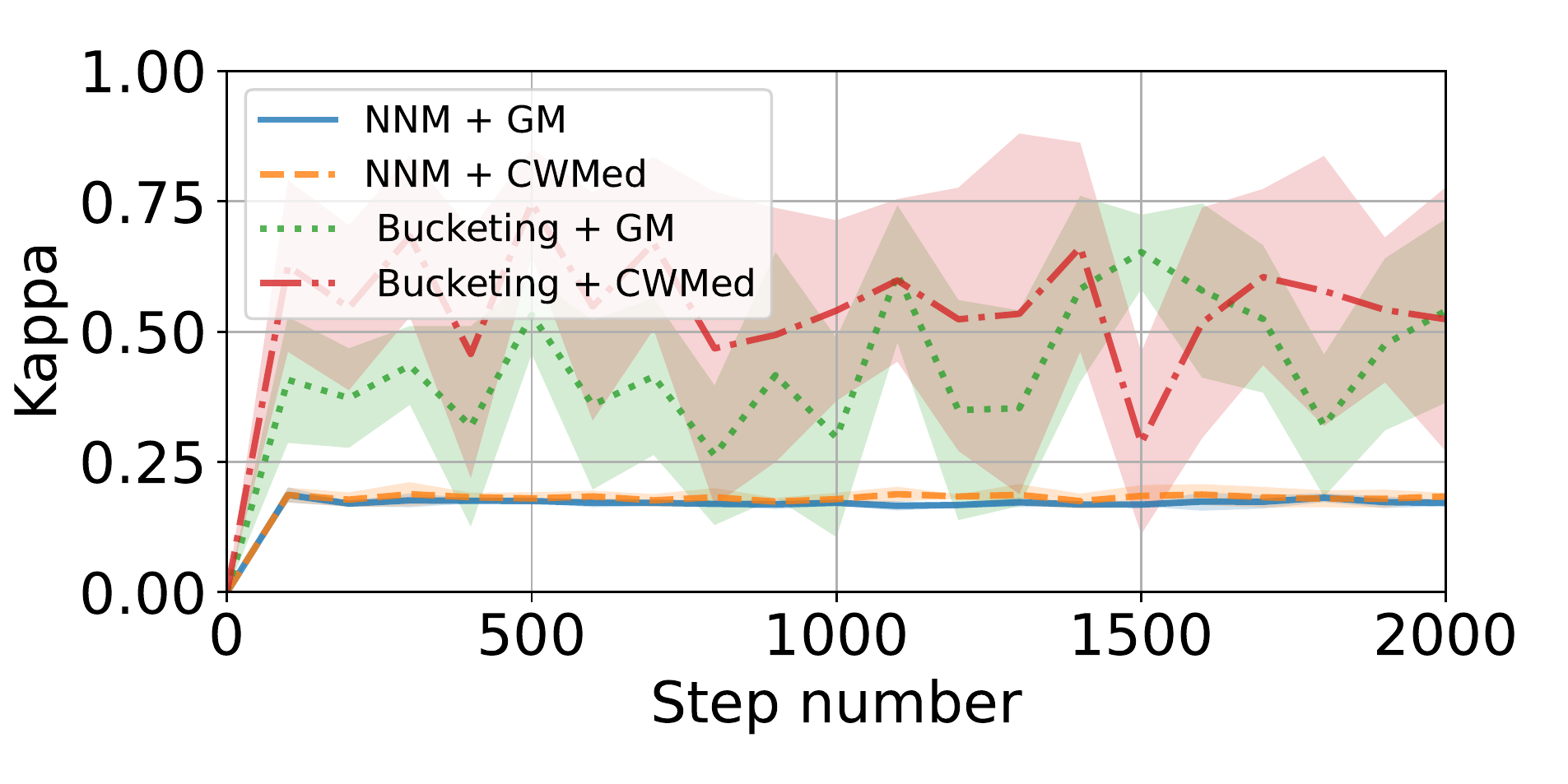}
    \includegraphics[width=0.49\textwidth]{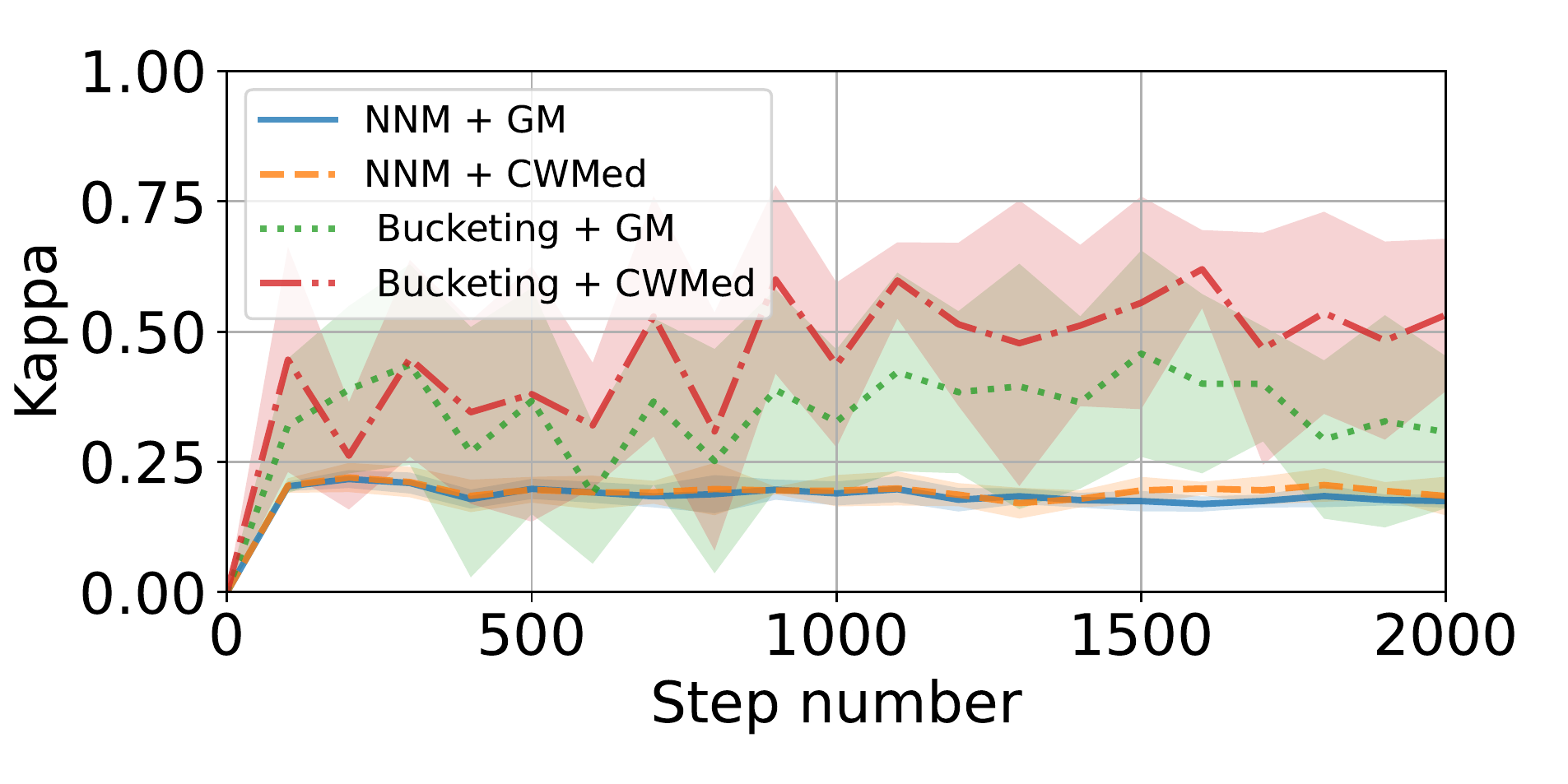}
    \vspace{-10pt}
    \caption{
    Plot of $\hat{\kappa}_t$~(defined in \eqref{eq:empirical-kappa}) on CIFAR-10 with $f=2$ Byzantine workers among $n = 17$ executing the ALIE (\textit{column 1}) and FOE (\textit{column 2}) attacks.
    The heterogeneity level is $\alpha=1$. The corresponding accuracies are in Figure~\ref{fig:plots_cifar_1}.
    }
    \label{fig:kappa}
\end{figure}

\paragraph{Experimental validation.}
Recall robust D-SHB (Algorithm~\ref{sgd}), and the corresponding notations. 
In Figure~\ref{fig:kappa}, we plot in each step $t \in [T]$ the quantity $\hat{\kappa}_t$ defined as follows
\begin{equation}
    \label{eq:empirical-kappa}
  \hat{\kappa}_t^2 \coloneqq \norm{R_t-\overline{m}_t}^2 / \frac{1}{\card{\H}}\sum_{i \in \H}\norm{m_t^{(i)}-\overline{m}_t}^2,
\end{equation}
where $\overline{m}_t \coloneqq \frac{1}{\card{\H}} \sum_{i \in \H} m_t^{(i)}$ is the average of honest momentums at step $t$.
In words, Figure~\ref{fig:kappa} shows the error in estimating the true average, scaled by the standard deviation of the honest inputs, across different steps of the learning.
The quantity in Equation~\eqref{eq:empirical-kappa} is an empirical estimation of $\kappa$, a parameter of the definition of $(f,\kappa)$-robustness in our theory.
The figure validates our insights about Bucketing and shows the superiority of $\cenna$ both in terms of stability of the error curves, as well as the quality of the average estimation. Indeed, the $\cenna$ curves are consistently below Bucketing's.

\section{Proof of Theorem~\ref{thm:conv-gd}: Analysis of Robust D-GD}
\label{app:convergence-dgd}
\begin{reptheorem}{thm:conv-gd}
Let Assumption~\ref{asp:hetero} hold and recall that $\loss_{\H}$ is $L$-smooth. 
Consider Algorithm~\ref{gd} with learning rate $\gamma = \frac{1}{L}$.
If $F$ is $(f,\kappa)$-robust, then for all $T \geq 1$
$$
   \norm{\nabla{\loss_{\H}{(\hat{\theta})}}}^2 
    \leq  4 \kappa G^2
    + \frac{4 L(\loss_\H{(\theta_0)}-\loss^*)}{T}.
$$
\end{reptheorem}
\begin{proof}
Let Assumption~\ref{asp:hetero} hold.
Assume $\loss_{\H}$ to be $L$-smooth and $F$ to be $(f,\kappa)$-robust.
Consider Algorithm~\ref{gd} with learning rate $\gamma = \frac{1}{L}$.
The update step corresponds to $\theta_t = \theta_{t-1} - \frac{1}{L} R_t$ for every $t \geq 1$.

Since $\loss_{\H}$ is $L$-smooth, we have (see~\cite{bottou2018optimization}) for all $t \geq 1$,
\begin{align}
    \label{eq:8}
    \loss_{\H}{(\theta_t)} - \loss_{\H}{(\theta_{t-1})}
    &\leq \iprod{\nabla{\loss_{\H}{(\theta_{t-1})}}}{\theta_t - \theta_{t-1}} + \frac{L}{2}\norm{\theta_t - \theta_{t-1}}^2
    = -\gamma \iprod{\nabla{\loss_{\H}{(\theta_{t-1})}}}{R_t} + \frac{1}{2}\gamma^2 L\norm{R_t}^2 \nonumber \\
    &=-\frac{1}{L} \iprod{\nabla{\loss_{\H}{(\theta_{t-1})}}}{R_t} + \frac{1}{2L}\norm{R_t}^2.
\end{align}
We expand the second term as follows
\begin{align*}
    \norm{R_t}^2
    &= \norm{R_t - \nabla{\loss_{\H}{(\theta_{t-1})}} + \nabla{\loss_{\H}{(\theta_{t-1})}}}^2\\\
    &= \norm{R_t - \nabla{\loss_{\H}{(\theta_{t-1})}}}^2 + \norm{\nabla{\loss_{\H}{(\theta_{t-1})}}}^2 + 2\iprod{\nabla{\loss_{\H}{(\theta_{t-1})}}}{R_t - \nabla{\loss_{\H}{(\theta_{t-1})}}}.
\end{align*}
Plugging this back in \eqref{eq:8}, then simplifying, yields
\begin{align*}
    \loss_{\H}{(\theta_t)} - \loss_{\H}{(\theta_{t-1})}
    &\leq -\frac{1}{L} \iprod{\nabla{\loss_{\H}{(\theta_{t-1})}}}{R_t} + \frac{1}{2L} \norm{R_t - \nabla{\loss_{\H}{(\theta_{t-1})}}}^2 + \frac{1}{2L} \norm{\nabla{\loss_{\H}{(\theta_{t-1})}}}^2 \\ 
    &\quad + \frac{1}{L}\iprod{\nabla{\loss_{\H}{(\theta_{t-1})}}}{R_t - \nabla{\loss_{\H}{(\theta_{t-1})}}}\\
    &= \frac{1}{2L} \norm{R_t - \nabla{\loss_{\H}{(\theta_{t-1})}}}^2
    -\frac{1}{2L} \norm{\nabla{\loss_{\H}{(\theta_{t-1})}}}^2.
\end{align*}
Upon rearranging terms and multiplying both sides by $2L$ we get
\begin{align*}
    \norm{\nabla{\loss_{\H}{(\theta_{t-1})}}}^2
    &\leq \norm{R_t - \nabla{\loss_{\H}{(\theta_{t-1})}}}^2 + 2L(\loss_{\H}{(\theta_{t-1})} - \loss_{\H}{(\theta_{t})}).
\end{align*}
By the $(f,\kappa)$-robustness property of $F$ and Assumption~\ref{asp:hetero}, we can bound the first term as follows
\begin{align}
    \label{eq:6}
    \norm{R_t - \nabla{\loss_{\H}{(\theta_{t-1})}}}^2
    &= \norm{F{\left(g_{t}^{(1)}, \ldots, g_{t}^{(n)}\right)} - \frac{1}{\card{\H}} \sum_{i \in \H} g_{t}^{(i)}}^2
    \leq \frac{\kappa}{\card{\H}} \sum_{i \in \H} \norm{g_t^{(i)}-\frac{1}{\card{\H}} \sum_{i \in \H} g_{t}^{(i)}}^2 \nonumber\\
    &= \frac{\kappa}{\card{\H}} \sum_{i \in \H} \norm{\nabla{\loss_i{(\theta_{t-1})}}-\nabla{\loss_{\H}{(\theta_{t-1})}}}^2
    \leq \kappa G^2.
\end{align}
As a result, we have
\begin{align*}
    \norm{\nabla{\loss_{\H}{(\theta_{t-1})}}}^2
    &\leq \kappa G^2 + 2L\left(\loss_{\H}{(\theta_{t-1})} - \loss_{\H}{(\theta_{t})}\right).
\end{align*}
By taking the average over $t$ from $1$ to $T$, and since $\loss_{\H}{(\theta_T)} \geq \inf_{\theta \in \R^d} \loss_\H{(\theta)} = \loss^*$, we have
\begin{align}
    \label{eq:7}
    \frac{1}{T}\sum_{t=0}^{T-1} \norm{\nabla{\loss_{\H}{(\theta_{t})}}}^2
    \leq \kappa G^2 + \frac{2L\left(\loss_{\H}{(\theta_0)} - \loss_{\H}{(\theta_{T})}\right)}{T}
    \leq \kappa G^2 + \frac{2L\left(\loss_{\H}{(\theta_0)} - \loss^*\right)}{T}.
\end{align}

\textbf{Final step.}
Recall from Algorithm~\ref{gd} that $\hat{\theta} = \theta_{\tau-1}$ with $\tau \in \argmin_{1 \leq t \leq T} \norm{R_{t}}$.
Thus, using Jensen's inequality, \eqref{eq:6} and then \eqref{eq:7}, we have
\begin{align*}
    \norm{\nabla{\loss_{\H}{(\hat{\theta})}}}^2
    &= \min_{1 \leq t \leq T} \norm{R_{t}}^2
    = \min_{1\leq t \leq T} \norm{\nabla{\loss_{\H}{(\theta_{t-1})}} + (R_{t}-\nabla{\loss_{\H}{(\theta_{t-1})}})}^2 \\
    &\leq \min_{1\leq t \leq T} \left \{ 2\norm{\nabla{\loss_{\H}{(\theta_{t-1})}}}^2 + 2\norm{R_{t}-\nabla{\loss_{\H}{(\theta_{t-1})}}}^2 \right \}
    \leq 2\min_{1\leq t \leq T} \norm{\nabla{\loss_{\H}{(\theta_{t-1})}}}^2 + 2\kappa G^2 \\
    &\leq 2\, \frac{1}{T} \sum_{t=0}^{T-1} \norm{\nabla \loss_{\H} \left(\weight{t} \right)}^2
    + 2\kappa G^2
    \leq 2\kappa G^2 + \frac{4L(\loss_{\H}(\theta_0)-\loss^*)}{T} + 2\kappa G^2 \\
    &= 4\kappa G^2 + \frac{4L(\loss_{\H}(\theta_0)-\loss^*)}{T}.
\end{align*}
This concludes the proof.
\end{proof}

\section{Proof of Corollary~\ref{cor:resilience}: Analysis of Robust D-GD with NNM}\label{app:corollary}


\begin{repcorollary}{cor:resilience}
Let Assumption~\ref{asp:hetero} hold and recall that $\loss_\H$ is $L$-smooth.
Consider Algorithm~\ref{gd} with $\gamma = \nicefrac{1}{L}$ and aggregation $F \circ \cenna$.
If $F$ is $(f,\kappa)$-robust with $\kappa = \mathcal{O}(1)$, then Algorithm~\ref{gd} is $(f,\varepsilon)$-Byzantine resilient with
$$\varepsilon = \mathcal{O}{\left(\nicefrac{f}{n}G^2 + \nicefrac{1}{T}\right)}.$$
\end{repcorollary}
\begin{proof}
Let Assumption~\ref{asp:hetero} hold.
Assume $\loss_\H$ to be $L$-smooth and $F$ to be $(f,\kappa)$-robust.
First recall that, by Lemma~\ref{lem:cenna}, the composition $F \circ \cenna$ is $(f,\kappa')$-robust with $\kappa' = \frac{8f}{n-f}(\kappa+1)$.

Consider Algorithm~\ref{gd} with learning rate $\gamma = \nicefrac{1}{L}$ and aggregation rule $F \circ \cenna$.
Following Theorem~\ref{thm:conv-gd}, we have for every $T \geq 1$,
\begin{equation*}
    \norm{\nabla \loss_{\H} \left(\hat{\weight{}} \right)}^2
    \leq  4\kappa' G^2
    + \frac{4L(\loss_{\H}(\theta_0)-\loss^*)}{T}
    = \frac{32f}{n-f}(\kappa+1)G^2 + \frac{4L(\loss_{\H}(\theta_0)-\loss^*)}{T}.
\end{equation*}

Recall that, as $f < \nicefrac{n}{2}$, we have $\frac{f}{n-f} \leq \frac{2f}{n}$.
Thus, if $\kappa = \mathcal{O}{(1)}$, we can write
\begin{equation*}
    \norm{\nabla \loss_{\H} \left(\hat{\weight{}} \right)}^2
    = \mathcal{O}{\left(\nicefrac{f}{n} G^2 + \nicefrac{1}{T}\right)}.
\end{equation*}
This concludes the proof.
\end{proof}

We provide the proof of the lower bound on Byzantine resilience  under heterogeneity~(Proposition~\ref{prop:lower-bound}) below.
\begin{repproposition}{prop:lower-bound}
If a learning algorithm $\mathcal{A}$ is $(f,\varepsilon)$-Byzantine resilient for every collection of smooth loss functions $\loss_1, \ldots, \loss_n$ satisfying Assumption~\ref{asp:hetero},
then $\varepsilon = \Omega \left(\nicefrac{f}{n}\, G^2 \right)$.
\end{repproposition}
\begin{proof}
The proof is similar to that of Theorem~III~\cite{karimireddy2022byzantinerobust}.
Assume learning algorithm $\mathcal{A}$ is $(f,\varepsilon)$-Byzantine resilient for every collection of smooth loss functions $\loss_1, \ldots, \loss_n$ satisfying Assumption~\ref{asp:hetero}.
Consider the following quadratic loss functions $\loss_1 = \ldots = \loss_{n-f} = \theta \mapsto \frac{1}{2}\norm{\theta}^2$ and $\loss_{n-f+1} = \ldots = \loss_n = \theta \mapsto \frac{1}{2}\norm{\theta - z}^2$, where $z \in \R^d$ is such that $\norm{z}^2 = \frac{(n-f)^2}{f(n-2f)}G^2$.
Consider the two situations $\H_1 = \{1,\ldots,n-f\}$ and $\H_2 = \{f+1,\ldots,n\}$.

We first show that the loss functions satisfy Assumption~\ref{asp:hetero} in both situations.
This is straightforward in situation $\H_1$ since honest losses are identical.
In situation $\H_2$, we have for all $\theta \in \R^d$, 
\begin{align*}
  \nabla{\loss_{\H_2}{(\theta)}} = \frac{1}{n-f}\sum_{i \in \H_2} \nabla{\loss_{i}{(\theta)}} = \frac{n-2f}{n-f} \theta + \frac{f}{n-f}(\theta - z) = \theta - \frac{f}{n-f}z.
\end{align*}
Therefore, thanks to the choice of $z$, we now show that Assumption~\ref{asp:hetero} holds, as for all $\theta \in \R^d$ we have
\begin{align*}
    \frac{1}{\card{\H_2}} \sum_{i \in \H_2} \norm{\nabla{\loss_{i}{(\theta)}} - \nabla{\loss_{\H_2}{(\theta)}}}^2 = \frac{n-2f}{n-f}\norm{\frac{f}{n-f}z}^2 + \frac{f}{n-f}\norm{\frac{n-2f}{n-f}z}^2= \frac{f(n-2f)}{(n-f)^2}\norm{z}^2 = G^2.
\end{align*}
Now, since learning algorithm $\mathcal{A}$ is $(f,\varepsilon)$-Byzantine resilient, it outputs $\hat{\theta}$ such that $\norm{\nabla{\loss_{\H_1}{(\hat{\theta})}}}^2 \leq \varepsilon$ and $\norm{\nabla{\loss_{\H_2}{(\hat{\theta})}}}^2 \leq \varepsilon$.
Note that situations $\H_1$ and $\H_2$ are indistinguishable to algorithm $\mathcal{A}$ because it ignores the Byzantine identities, and thus $\hat{\theta}$ is the same in both situations.
Therefore, invoking Jensen's inequality, we have
\begin{align*}
\varepsilon 
&\geq \max{\left\{\norm{\nabla{\loss_{\H_1}{(\hat{\theta})}}}^2,\norm{\nabla{\loss_{\H_2}{(\hat{\theta})}}}^2\right\}}
\geq \frac{1}{2} \left(\norm{\nabla{\loss_{\H_1}{(\hat{\theta})}}}^2 + \norm{\nabla{\loss_{\H_2}{(\hat{\theta})}}}^2 \right)\\
&= \frac{1}{2} \left(\norm{\hat{\theta}}^2 + \norm{\hat{\theta} - \frac{f}{n-f}z}^2 \right)
\geq \frac{1}{4} \norm{\frac{f}{n-f}z}^2 
= \frac{1}{4} \left(\frac{f}{n-f}\right)^2 \frac{(n-f)^2}{f(n-2f)} G^2
= \frac{1}{4}\frac{f}{n-2f} G^2.
\end{align*}
Since $n -2f \leq n$, we obtain $\varepsilon \geq \frac{1}{4} \nicefrac{f}{n}\, G^2$, which concludes the proof.
\end{proof}

\section{Proof of Theorem~\ref{thm:conv}: Analysis of Robust D-SHB}
\label{app:convergence-dshb}


\subsection{Proof Outline}
\label{sec:resultsskeleton}
Our analysis of robust D-SHB (Algorithm~\ref{sgd}) follows \cite{farhadkhani2022byzantine} and consists of four elements:
(i) {\em Momentum drift} (Lemma~\ref{lem:mmt_drift}),
(ii) {\em Aggregation error} (Lemma~\ref{lem:drift}),
(iii) {\em Momentum deviation} (Lemma~\ref{lem:dev}), and
(iv) {\em Descent bound} (Lemma~\ref{lem:growth_Q}).
We combine these elements to obtain the final convergence result stated in Theorem~\ref{thm:conv}.
The originality of our analysis, compared to \cite{farhadkhani2022byzantine}, is (i) the tighter analysis of the aggregation error thanks to $(f,\kappa)$-robustness, and (ii) the extension of the momentum drift analysis to the heterogeneous setting.
There are other subtle differences, such as the choice of the learning rate.

\noindent \paragraph{Notation.} Recall that for each step $t$, for each honest worker $w_i$,
\begin{equation}
    m_t^{(i)} = \beta m_{t-1}^{(i)} + (1-\beta) g_t^{(i)} \label{eqn:mmt_i}
\end{equation}
where $m_0^{(i)} = 0$ by convention. As we analyze Algorithm~\ref{sgd} with aggregation $F$, we denote 
\begin{align}
    R_t \coloneqq F{\left(\mmt{1}{t}, \ldots, \mmt{n}{t} \right)}, \label{eqn:R}
\end{align}
and 
\begin{align}
    \theta_t = \theta_{t  - 1} - \gamma R_t. \label{eqn:SGD}
\end{align}
We denote by $\mathcal{P}_t$ the history from steps $1$ to $t$. Specifically, 
\[\P_t \coloneqq \left\{\weight{0}, \ldots, \, \weight{t-1}; ~ \mmt{i}{1}, \ldots, \, \mmt{i}{t-1}; i = 1, \ldots, \, n \right\}.\] 
By convention, $\P_1 = \{ \weight{0}\}$. We denote by $\condexpect{t}{\cdot}$ and $\expect{\cdot}$ the conditional expectation $\expect{\cdot ~ \vline ~ \P_t}$ and the total expectation, respectively. 
Thus, $\expect{\cdot} = \condexpect{1}{ \cdots \condexpect{T}{\cdot}}$.


\subsubsection{Momentum Drift}
\label{sec:drift}

Along the trajectory $\theta_0,\ldots,\theta_{t-1}$, the honest workers' local momentums may drift away from each other.
This is in part due to the heterogeneity between the training sets.
This induces a dissimilarity between honest workers' local gradients that we can control thanks to Assumption~\ref{asp:hetero}.
The drift is also due to the stochasticity of the local gradients.

We show in Lemma~\ref{lem:mmt_drift} below that the bound on the {\em drift} between the honest workers' momentums can be controlled by tuning the momentum coefficient $\beta$.
In fact, the smaller $(1-\beta)$ the smaller the bound on the drift.
Recall that we denote by $\H \subseteq [n]$ the set of $n-f$ honest workers. We define
\begin{align}
    \AvgMmt{t} \coloneqq \nicefrac{1}{(n-f)} \sum_{i \in \H} \mmt{i}{t}, \label{eqn:def_avg_mmt}
\end{align}
the average of honest workers' local momentums.
We obtain the following bound on the momentum drift, i.e., Lemma~\ref{lem:mmt_drift}, proof of which can be found in Appendix~\ref{app:mmt_drift}.

\begin{lemma}
\label{lem:mmt_drift}
Suppose that assumptions~\ref{asp:hetero} and ~\ref{asp:bnd_var} hold true. Consider Algorithm~\ref{sgd}. For each $t \in [T]$, we obtain that
\begin{align*}
    \expect{\frac{1}{\card{\H}}\sum_{i \in \H}\norm{\mmt{i}{t} - \AvgMmt{t}}^2 } 
    \leq 3\left(\frac{1-\beta}{1+\beta} \right) \,  \left ( 1+\frac{1}{n-f} \right)\sigma^2
    + 3 G^2.
\end{align*}
\end{lemma}

\subsubsection{Aggregation Error}
\label{sec:error}
By building upon this first lemma and the $(f,\kappa)$-robustness property, we can obtain a bound on the error between the aggregate $R_t$  and $\overline{m}_t$ the average momentum of honest workers for the case. Specifically, when defining the error
\begin{align}
    \drift{t} \coloneqq R_t - \AvgMmt{t}, \label{eqn:drift}
\end{align}
we get the following bound on the error in Lemma~\ref{lem:drift}, proof of which can be found in Appendix~\ref{app:lem_drift}.
\begin{lemma}
\label{lem:drift}
Suppose that assumptions~\ref{asp:hetero} and ~\ref{asp:bnd_var} hold true.
Assume $F$ is $(f,\kappa)$-robust.
Consider Algorithm~\ref{sgd} with aggregation $F$. For each step $t \in [T]$, we obtain that
\begin{align*}
    \expect{\norm{\drift{t}}^2}
    \leq 6 \kappa \frac{1-\beta}{1+\beta}\sigma^2 + 3\kappa G^2.
\end{align*}
\end{lemma}
\subsubsection{Momentum Deviation}
\label{sec:deviation}

Next, we study the momentum deviation; i.e., the distance between the average honest momentum $\AvgMmt{t}$ and the true gradient $\nabla \loss_{\H}(\weight{t-1})$ in an arbitrary step $t$. Specifically, we define {\em deviation} to be
\begin{align}
    \dev{t} \coloneqq \AvgMmt{t} - \nabla \loss_{\H}\left( \weight{t-1} \right), \label{eqn:dev}
\end{align}
and obtain in Lemma~\ref{lem:dev} below an upper bound on the growth of the deviation over the learning steps $t \in [T]$. (Proof of Lemma~\ref{lem:dev} can be found in Appendix~\ref{app:lem_dev}.) 

\begin{lemma}
\label{lem:dev}
Suppose that Assumption \ref{asp:bnd_var} holds.
Recall that $\loss_\H$ is $L$-smooth.
Consider Algorithm~\ref{sgd} with $T > 1$. For all $t \geq 2$ we obtain that
\begin{align*}
    \expect{\norm{\dev{t}}^2} \leq & \beta^2 c \expect{\norm{\dev{t-1}}^2} +  4 \gamma L ( 1 + \gamma L) \beta^2  \expect{\norm{\nabla \loss_{\H}(\weight{t-2})}^2} +(1 - \beta)^2 \frac{\sigma^2}{(n-f)} \\
    & + 2 \gamma L ( 1 + \gamma L)\beta^2  \expect{\norm{\drift{t-1}}^2},
\end{align*}
where $c \coloneqq (1 + \gamma L ) \left(1 + 4 \gamma   L \right)$.
\end{lemma}

\subsubsection{Descent Bound} 
\label{sec:growth}

Finally, we analyze the fourth element, i.e., the growth of cost function $\loss_{\H}(\weight{})$ along the trajectory of Algorithm~\ref{sgd}. From~\eqref{eqn:SGD} and~\eqref{eqn:R}, we obtain that, for each step $t$,
\begin{align*}
    \weight{t} = \weight{t-1} - \gamma  R_t= \weight{t-1} - \gamma    \, \AvgMmt{t} - \gamma  \left(R_t -   \, \AvgMmt{t} \right),
\end{align*}
Furthermore, by~\eqref{eqn:drift}, $R_t -   \, \AvgMmt{t} = \drift{t}$. Thus, for all $t$,
\begin{align}
    \weight{t} = \weight{t-1} - \gamma    \, \AvgMmt{t} - \gamma  \drift{t}. \label{eqn:sgd_new}
\end{align}
This means that Algorithm~\ref{sgd} can actually be treated as distributed SGD with a momentum term that is subject to perturbation proportional to $\drift{t}$ at each step $t$. This perspective leads us to Lemma~\ref{lem:growth_Q}, proof of which can be found in Appendix~\ref{app:growth_Q}. 

\begin{lemma} 
\label{lem:growth_Q}
Recall that $\loss_\H$ is $L$-smooth.
Consider Algorithm~\ref{sgd}. For all $t \in [T]$, we obtain that
\begin{align*}
    \expect{2 \loss_{\H}(\weight{t}) - 2 \loss_{\H}(\weight{t-1})} \leq & - \gamma    \left( 1 - 4 \gamma  L  \right) \expect{\norm{\nabla \loss_{\H}(\weight{t-1})}^2}  + 2 \gamma    \left( 1 + 2 \gamma  L   \right) \expect{ \norm{\dev{t}}^2} \\
    & + 2 \gamma  \left(  1 + \gamma  L \right) \expect{\norm{\drift{t}}^2}.
\end{align*}
\end{lemma}

Putting all of the previous lemmas together, we prove Theorem~\ref{thm:conv} in Section~\ref{app:conv} below.
As a corollary of Theorem~\ref{thm:conv} and Lemma~\ref{lem:cenna}, we prove Corollary~\ref{cor:nnm-dshb} in Section~\ref{app:dshb-corollary}.

\subsection{Proof of Theorem~\ref{thm:conv}}
\label{app:conv}
We recall the theorem statement below for convenience. Recall that
\begin{align}
    \loss^* = \inf_{\weight{} \in \R^d} \loss_{\H}(\weight{}), ~ a_1= 36, a_2 = 6 \sqrt{\loss_{\H}(\weight{0}) - \loss^*},
    a_3 = 1728L,
    a_4 = 288L,
    \text{ and } a_5 = 6L\, a_2^2. \label{eqn:a1a2}
\end{align}
\begin{reptheorem}{thm:conv}
Let assumption~\ref{asp:hetero} and~\ref{asp:bnd_var} hold and recall that $\loss_\H$ is $L$-smooth.
Let $F$ be a $(f,\kappa)$-robust aggregation rule. Consider Algorithm~\ref{sgd} with momentum coefficient $\beta = \sqrt{1 - 24 \gamma L} $, and learning rate $$\gamma= \min{\left\{\frac{1}{24L}, ~ \frac{a_2}{2a_{\kappa}\sigma\sqrt{T}}\right\}}, $$
with $a_\kappa^2 \coloneqq a_3 \kappa + \frac{a_4}{n-f}$. For all $T \geq 1$
\begin{align*}
   \expect{\norm{\nabla \loss_{\H} (\hat{\weight{}} )}^2} \hspace{-2pt}
    \leq  a_1 \kappa G^2
    + \frac{a_2 a_{\kappa}\sigma}{\sqrt{T}}
    + \frac{a_5}{T}
    + \frac{a_2 a_4 \sigma}{n a_\kappa T^{\nicefrac{3}{2}}}, 
\end{align*}
where the expectation is over the randomness of the algorithm.
\end{reptheorem}
\begin{proof}
Define 
\begin{equation}
    \gamma_o \coloneqq \frac{1}{18L}.
\end{equation}
Note that as specified in the theorem statement, by definition of $\gamma$, we have
\begin{equation}
\gamma \leq \frac{1}{24L} \leq \gamma_o.
    \label{eqn:gamma_gamma_o}
\end{equation}
Therefore, $\beta = \sqrt{1 - 24  \gamma L}$ (as defined) is a well-defined real value in $(0, 1)$. 

To obtain the convergence result we 
define the Lyapunov function to be
\begin{align}
    V_t \coloneqq \expect{2 \loss_{\H}(\weight{t-1}) + z \norm{\dev{t}}^2} ~ \text{ and } z = \frac{1}{8L}. \label{eqn:lyap_func}
\end{align}
We consider an arbitrary $t \in [T]$. \\

{\bf Invoking Lemma~\ref{lem:dev}.} Upon substituting from Lemma~\ref{lem:dev}, we obtain that 
\begin{align}
    \expect{ z \norm{\dev{t+1}}^2 - z \norm{\dev{t}}^2} \leq & z \beta^2 c \expect{\norm{\dev{t}}^2} +  4 z \gamma L ( 1 + \gamma L) \beta^2   \expect{\norm{\nabla \loss_{\H}(\weight{t-1})}^2} + z (1 - \beta)^2 \frac{\sigma^2}{n-f} \nonumber \\
    & + 2 z \gamma L ( 1 + \gamma L)\beta^2 \expect{\norm{\drift{t}}^2} - z \expect{\norm{\dev{t}}^2}. \label{eqn:dev_gamma}
\end{align}
Recall that 
\begin{align}
    c = (1 + \gamma L) \left(1 + 4 \gamma  L \right) = 1 + 5 \gamma L + 4 \gamma^2  L^2. \label{eqn:zeta_expand}
\end{align}

{\bf Invoking Lemma~\ref{lem:growth_Q}.} Substituting from Lemma~\ref{lem:growth_Q} we obtain that
\begin{align}
    \expect{2 \loss_{\H}(\weight{t}) - 2 \loss_{\H}(\weight{t-1})}
    &\leq - \gamma   \left( 1 - 4 \gamma  L \right) \expect{\norm{\nabla \loss_{\H}(\weight{t-1})}^2} 
    + 2 \gamma   \left( 1 + 2 \gamma  L  \right) \expect{\norm{\dev{t}}^2} \nonumber \\
    &\quad + 2\gamma  \left(1 + \gamma  L \right) \expect{\norm{\drift{t}}^2}. \label{eqn:growth_gamma}
\end{align}
Substituting from~\eqref{eqn:dev_gamma} and~\eqref{eqn:growth_gamma} in~\eqref{eqn:lyap_func} we obtain that
\begin{align}
    V_{t+1} - V_t = & \expect{2 \loss_{\H}(\weight{t}) - 2 \loss_{\H}(\weight{t-1})} + \expect{ z \norm{\dev{t+1}}^2 - z \norm{\dev{t}}^2} \nonumber \\
    \leq & - \gamma   \left( 1 - 4 \gamma  L \right) \expect{\norm{\nabla \loss_{\H}(\weight{t-1})}^2} 
    + 2 \gamma   \left( 1 + 2 \gamma  L  \right) \expect{\norm{\dev{t}}^2}
    + 2\gamma  \left(  1 + \gamma  L \right) \expect{\norm{\drift{t}}^2} \nonumber \\
    & +  z \beta^2 c \expect{\norm{\dev{t}}^2} +  4 z \gamma L ( 1 + \gamma L) \beta^2   \expect{\norm{\nabla \loss_{\H}(\weight{t-1})}^2} + z (1 - \beta)^2 \frac{\sigma^2}{n-f} \nonumber \\
    & + 2 z \gamma L ( 1 + \gamma L)\beta^2 \expect{\norm{\drift{t}}^2} - z \expect{\norm{\dev{t}}^2}. \label{eqn:Vt-t}
\end{align}
Upon re-arranging the R.H.S.~in~\eqref{eqn:Vt-t} we obtain that
\begin{align*}
    V_{t+1} - V_t \leq & - \gamma \left(  \left( 1 - 4 \gamma L \right) - 4 z L( 1 + \gamma L) \beta^2   \right) \expect{\norm{\nabla \loss_{\H}(\weight{t-1})}^2} +  z (1 - \beta)^2 \frac{\sigma^2}{n-f} \nonumber \\
    & + \left( 2 \gamma  \left( 1 + 2 \gamma L  \right) +  z \beta^2 c - z \right)  \expect{\norm{\dev{t}}^2}  + 2\gamma \left( 1 + \gamma L + z L (1 + \gamma L) \beta^2 \right) \expect{\norm{\drift{t}}^2}.
\end{align*}
For simplicity, we define
\begin{align}
    A \coloneqq \left( 1 - 4 \gamma L \right) - 4 z L( 1 + \gamma L) \beta^2 , \label{eqn:def_At}
\end{align}
\begin{align}
    B \coloneqq 2 \gamma  \left( 1 + 2 \gamma L  \right) +  z \beta^2 c - z, \label{eqn:def_Bt}
\end{align}
and
\begin{align}
    C \coloneqq 2\gamma \left( 1 + \gamma L + zL (1 + \gamma L) \beta^2 \right), \label{eqn:def_Ct}
\end{align}
Thus,
\begin{align}
    V_{t+1} - V_t \leq - A \gamma \expect{\norm{\nabla \loss_{\H}(\weight{t-1})}^2} + B  \expect{\norm{\dev{t}}^2}  + C  \expect{\norm{\drift{t}}^2} +  z (1 - \beta)^2 \frac{\sigma^2}{n-f}. \label{eqn:after_At}
\end{align}
We now analyse below the terms $A$, $B$ and $C$.\\

{\bf Term $A$.} Recall from~\eqref{eqn:gamma_gamma_o} that $ \gamma \leq \gamma_{o} = \frac{1}{18L} $. Upon using this in~\eqref{eqn:def_At}, and the facts that $z = \frac{1}{8L}$ and $\beta^2 < 1$,  we obtain that
\begin{align}
    A \geq 1 - 4 \gamma_o L  - \frac{4L}{8L} ( 1 + \gamma_o L ) \geq \frac{1}{2} - \frac{9\gamma_oL}{2}\geq \frac{1}{4}. \label{eqn:At_3}
\end{align}

{\bf Term $B$.} 
Substituting $c$ from~\eqref{eqn:zeta_expand} in~\eqref{eqn:def_Bt} we obtain that
\begin{align*}
    B &= 2\gamma \left( 1 + 2\gamma L \right) + z \beta^2 \left( 1 + 5 \gamma L + 4 \gamma^2  L^2 \right) - z \\
    & =  - \left(1 - \beta^2 \right) z +\gamma \left( 2 + 4\gamma L +5 z \beta^2 L + 4 z \beta^2 L \gamma L \right).
\end{align*}
Using the facts that $\beta^2 \leq 1$ and $\gamma \leq \gamma_o \leq \frac{1}{18L}$, and then substituting  $z = \frac{1}{8L}$ we obtain that
\begin{align}
    B &\leq \frac{-(1-\beta^2)}{8L} + \gamma \left( 2 + \frac{4}{18} + \frac{5}{8} + \frac{4}{18 \times 8} \right)  \leq \frac{-(1-\beta^2)}{8L} + 3 \gamma
    \leq  \frac{-(1-\beta^2) + 24 \gamma L}{8L} = 0, \label{eqn:Bt_2} 
\end{align}
where the last equality follows from the fact that $1 - \beta^2 = 24 \gamma L$.\\

{\bf Term $C$.} Substituting $z = \frac{1}{8L}$ in~\eqref{eqn:def_Ct}, and then using the fact that $\beta^2 < 1$, we obtain that
\begin{align*}
    C = 2 \gamma \left( 1 + \gamma L + \frac{1}{8} (1 + \gamma L) \right) \leq \frac{9 \gamma}{4} \left( 1 +\gamma L \right).
\end{align*}
As $\gamma \leq \gamma_o \leq \frac{1}{18 L}$, from above we obtain that
\begin{align}
    C \leq \frac{9 \gamma}{4} \left( 1 +\frac{1}{18} \right) \leq 3 \gamma. \label{eqn:Ct_2}
\end{align}


{\bf Combining terms $A$, $B$, and $C$.} Finally, substituting from~\eqref{eqn:At_3},~\eqref{eqn:Bt_2}, and ~\eqref{eqn:Ct_2} in~\eqref{eqn:after_At} (and recalling that $z = \frac{1}{8L}$) we obtain that
\begin{align*}
    V_{t+1} - V_t \leq - \frac{ \gamma}{4} \expect{\norm{\nabla \loss_{\H}(\weight{t-1})}^2} + 3\gamma \expect{\norm{\drift{t}}^2} +   (1 - \beta)^2 \frac{\sigma^2}{8L(n-f)}.
\end{align*}
As the above is true for an arbitrary $t \in [T]$, by taking summation on both sides from $t = 1$ to $t = T$ we obtain that
\begin{align*}
    V_{T+1} - V_1 \leq - \frac{\gamma}{4} \sum_{t = 1}^{T} \expect{\norm{\nabla \loss_{\H}(\weight{t-1})}^2}  + 3 \gamma \sum_{t = 1}^{T} \expect{\norm{\drift{t}}^2} + (1 - \beta)^2 \frac{\sigma^2}{8L(n-f)} T. 
\end{align*}
Thus,
\begin{align} 
    \frac{\gamma}{4} \sum_{t = 1}^{T} \expect{\norm{\nabla \loss_{\H}(\weight{t-1})}^2} \leq V_1 - V_{T+1} + 3 \gamma  \sum_{t = 1}^{T} \expect{\norm{\drift{t}}^2} +   (1 - \beta)^2 \frac{\sigma^2}{8L(n-f)} T. \label{eqn:lyap_before_rational} 
\end{align}
Note that, as $\beta > 0$, and $1 - \beta^2 = 24 \gamma L$, we have
\begin{align*}
    (1 - \beta)^2 = \frac{\left(1 - \beta^2\right)^2}{\left( 1 + \beta \right)^2} \leq \left( 1 - \beta^2 \right)^2 = 576 \gamma^2 L^2.
\end{align*}
Substituting from above in~\eqref{eqn:lyap_before_rational} we obtain that
\begin{align*}
    \frac{\gamma}{4} \sum_{t = 1}^{T} \expect{\norm{\nabla \loss_{\H}(\weight{t-1})}^2} \leq V_1 - V_{T+1} + 3 \gamma  \sum_{t = 1}^{T} \expect{\norm{\drift{t}}^2} +  \frac{ 576 \gamma^2 L^2  \sigma^2}{8L(n-f)} \, T.
\end{align*}
Multiplying both sides by $4/\gamma$ we obtain that
\begin{align}
    \sum_{t = 1}^{T} \expect{\norm{\nabla \loss_{\H}(\weight{t-1})}^2} &\leq \frac{4 \left(V_1 - V_{T+1} \right)}{\gamma} + 12 \sum_{t = 1}^{T} \expect{\norm{\drift{t}}^2} +  \frac{576  \gamma L \sigma^2}{2(n-f)} \, T. \label{eqn:lyap_before_rational_2} 
\end{align}

{\bf Invoking Lemma~\ref{lem:drift}.} 
Next, we use Lemma~\ref{lem:drift} to derive an upper bound on $\sum_{t = 1}^{T} \expect{\norm{\drift{t}}^2} $.
Since $F$ is $(f,\kappa)$-robust, we have from Lemma~\ref{lem:drift} that for all $t \in [T]$,
\begin{align*}
    \expect{\norm{\drift{t}}^2}
    \leq 6\kappa\frac{1-\beta}{1+\beta}\sigma^2 + 3\kappa G^2.
\end{align*}
By summing over $t$ from $1$ to $T$, we obtain that
\begin{align}
    \sum_{t = 1}^{T} \expect{\norm{\drift{t}}^2} 
    & \leq 6\kappa\frac{1-\beta}{1+\beta}\sigma^2 T + 3\kappa G^2 T. 
    \label{eqn:sum_drift_t}
\end{align}
As $\beta > 0$, and the fact that $ 1 - \beta^2 = 24 \gamma L$, we have
\begin{align*}
    \frac{1 - \beta}{1 + \beta} = \frac{1 - \beta^2}{( 1 + \beta)^2}  \leq 1 - \beta^2 = 24 \gamma L.
\end{align*}
Substituting the above in~\eqref{eqn:sum_drift_t},
we obtain that
\begin{align*}
    \sum_{t = 1}^{T} \expect{\norm{\drift{t}}^2} 
    &\leq (24 \times 6) \sigma^2 \kappa \gamma L T + 3 \kappa G^2 T = 144 \sigma^2 \kappa \gamma L T + 3 \kappa G^2 T.
\end{align*}
Substituting from above in~\eqref{eqn:lyap_before_rational_2} we obtain that
\begin{align*}
    \sum_{t = 1}^{T} \expect{\norm{\nabla \loss_{\H}(\weight{t-1})}^2} 
    \leq & \frac{4 \left(V_1 - V_{T+1} \right)}{\gamma} + (12 \times 144) \sigma^2\kappa \gamma LT + (12 \times 3) \kappa G^2 T
    + \frac{288  \gamma L \sigma^2}{(n-f)}.
\end{align*}
Recall that 
\begin{align*}
    a_1 = (12 \times 3) = 36, a_3 = (12 \times 144) L = 1728 L, \text{ and } a_4 = 288L
\end{align*}
Thus, from above we obtain that
\begin{align*}
    \sum_{t = 1}^{T} \expect{\norm{\nabla \loss_{\H}(\weight{t-1})}^2} 
    &\leq \frac{4 \left(V_1 - V_{T+1} \right)}{\gamma} + a_3 \kappa \sigma^2 \gamma T +  \frac{ a_4 \sigma^2 }{(n-f)} \, \gamma T + a_1 \kappa G^2 T.
\end{align*}
Diving both sides by $T$ we obtain that
\begin{align}
    \frac{1}{T}\sum_{t = 1}^{T} \expect{\norm{\nabla \loss_{\H}(\weight{t-1})}^2} 
    &\leq \frac{4 \left(V_1 - V_{T+1} \right)}{\gamma T} + a_3 \kappa \sigma^2 \, \gamma +  \frac{a_4 \sigma^2}{(n-f)} \, \gamma 
    + a_1 \kappa G^2.  \label{eqn:before_Qstar}
\end{align}


{\bf Analyzing $V_t$.} Recall that $\loss^* = \inf_{\weight{} \in \R^d} \loss_{\H}(\weight{})$. Note that for an arbitrary $t$, by definition of $V_t$ in~\eqref{eqn:lyap_func}, 
\[V_t - 2 \loss^* = 2 \expect{\loss_{\H}(\weight{t-1}) - \loss^*} + z \expect{\norm{\dev{t}}^2} \geq 0 + z \expect{\norm{\dev{t}}^2} \geq 0.\]
Thus, 
\begin{align}
    V_1 - V_{T+1} = V_1 - 2 \loss^* - \left(V_{T+1} -  2 \loss^* \right) \leq V_1 - 2 \loss^*. \label{eqn:v1-vt}
\end{align}
Moreover,
\begin{align}
    V_1 = 2 \loss_{\H}(\weight{0}) + z \expect{\norm{\dev{1}}^2}. \label{eqn:bnd_v2}
\end{align}
By definition of $\dev{t}$ in~\eqref{eqn:dev}, the definition of $\AvgMmt{t}$ in~\eqref{eqn:def_avg_mmt}, and the fact that $m_0^{(i)} = 0$ for all $i \in \H$, we obtain that
\begin{align*}
    \expect{\norm{\dev{1}}^2} = \expect{\norm{\AvgMmt{1} - \nabla \loss_{\H}(\weight{0})}^2} = \expect{\norm{(1 - \beta) \overline{g}_1 - \nabla \loss_{\H}(\weight{0})}^2} 
\end{align*}
where $\overline{g}_1$, defined in~\eqref{eqn:over_g}, is the average of $n-f$ honest workers' stochastic gradients in step $1$. Expanding the R.H.S.~above we obtain that
\begin{align*}
    \expect{\norm{\dev{1}}^2} = (1 - \beta)^2 \expect{ \norm{\overline{g}_1 - \nabla \loss_{\H}(\weight{0})}^2 } + \beta^2 \norm{\nabla \loss_{\H}(\weight{0})}^2 - 2 \beta (1 - \beta) \iprod{\expect{\overline{g}_1} - \nabla \loss_{\H}(\weight{0})}{ \nabla \loss_{\H}(\weight{0})}.
\end{align*}
Recall that $\expect{\overline{g}_1} = \nabla \loss_{\H}(\weight{0})$, and that (due to Assumption~\ref{asp:bnd_var}) $\expect{ \norm{\overline{g}_1 - \nabla \loss_{\H}(\weight{0})}^2} \leq \sigma^2/ (n-f)$. Therefore, 
\begin{align*}
    \expect{\norm{\dev{1}}^2} \leq \frac{(1 - \beta)^2 \sigma^2}{(n-f)} + \beta^2 \norm{\nabla \loss_{\H}(\weight{0})}^2.
\end{align*}
Recall that $\loss_\H$ is $L$-smooth. Thus, $\norm{\nabla \loss_{\H}(\weight{0})}^2 \leq 2L(\loss_{\H}{(\weight{0})}-\loss^*)$ (see~\cite{nesterov2018lectures}, Theorem 2.1.5).
Therefore,
\begin{align*}
    \expect{\norm{\dev{1}}^2} \leq \frac{(1 - \beta)^2 \sigma^2}{(n-f)} + 2\beta^2 L (\loss_{\H}{(\weight{0})}-\loss^*).
\end{align*}
Substituting from above in~\eqref{eqn:bnd_v2} we obtain that
\begin{align*}
    V_1 \leq 2 \loss_{\H}(\weight{0}) + z \left( \frac{(1 - \beta)^2 \sigma^2}{(n-f)} + 2\beta^2 L (\loss_{\H}{(\weight{0})}-\loss^*) \right).
\end{align*}
Recall that $(1 - \beta)^2 \leq \left(1 - \beta^2\right)^2 = 576 \gamma^2 L^2$. Using this, and the facts that $\beta^2 < 1$ and $z = \frac{1}{8L}$, we obtain that
\begin{align*}
    V_1 &\leq 2 \loss_{\H}(\weight{0}) +  \frac{1}{4}(\loss_{\H}{(\weight{0})}-\loss^*)+ \frac{72 \gamma^2 L^2   \sigma^2}{L(n-f)}.
\end{align*}
Recall that $a_4 = 288L $. Therefore, 
\begin{align*}
    V_1 \leq 2 \loss_{\H}(\weight{0}) + \frac{1}{4}(\loss_{\H}{(\weight{0})}-\loss^*) + \frac{a_4 \sigma^2}{4(n-f)} \, \gamma^2.
\end{align*}
Substituting the above in~\eqref{eqn:v1-vt} we obtain that
\begin{align*}
    V_1 - V_{T+1} \leq 2 \loss_{\H}(\weight{0}) - 2 \loss^* +  \frac{1}{4}(\loss_{\H}{(\weight{0})}-\loss^*) + \frac{a_4 \sigma^2}{4(n-f)} \, \gamma^2
    = \frac{9}{4}(\loss_{\H}{(\weight{0})}-\loss^*) + \frac{a_4 \sigma^2}{4(n-f)} \, \gamma^2. 
\end{align*}
Substituting from above in~\eqref{eqn:before_Qstar} we obtain that
\begin{align*}
    \frac{1}{T}\sum_{t = 1}^{T} \expect{\norm{\nabla \loss_{\H}(\weight{t-1})}^2} 
    \leq & \frac{(4 \times \frac{9}{4})(\loss_{\H}{(\weight{0})}-\loss^*) }{\gamma T} + \left(\frac{a_4 \sigma^2}{n-f}\right) \frac{\gamma}{T} + a_3 \kappa \sigma^2 \, \gamma +  \frac{a_4 \sigma^2}{(n-f)} \, \gamma
    + a_1 \kappa G^2.
\end{align*}
Upon re-arranging the terms on R.H.S.~above we obtain that
\begin{align*}
    \frac{1}{T}\sum_{t = 1}^{T} \expect{\norm{\nabla \loss_{\H}(\weight{t-1})}^2} 
    \leq & a_1 \kappa G^2 +\frac{9(\loss_{\H}(\weight{0}) - \loss^*)}{\gamma T}
    + \left( a_3 \kappa  + \frac{a_4}{n-f} \right) \sigma^2 \gamma + \left( \frac{a_4 \sigma^2}{n-f}\right) \frac{\gamma}{T}. 
\end{align*}
Recall that $a_2^2 =  36(\loss_{\H}(\weight{0}) - \loss^*)$ and $a_\kappa^2 \coloneqq a_3 \kappa + \frac{a_4}{n-f}$, we obtain that 
\begin{align}
    \frac{1}{T}\sum_{t = 1}^{T} \expect{\norm{\nabla \loss_{\H}(\weight{t-1})}^2} 
    \leq & a_1\kappa G^2 + \frac{a_2^2}{4\gamma T} + a_\kappa^2 \sigma^2 \gamma + \left( \frac{a_4 \sigma^2}{n-f}\right) \frac{\gamma}{T}
    \label{eqn:after_Qstar}
\end{align}

{\bf Final step.} Recall that by definition
\begin{align*}
    \gamma = \min{\left\{\frac{1}{24L}, ~ \frac{a_2}{2 a_\kappa \sigma \sqrt{T}}\right\}},
\end{align*}
and thus $\frac{1}{\gamma} = \max{\left\{24L, \frac{2a_\kappa \sigma\sqrt{T}}{a_2}\right\}} \leq 24L + \frac{2a_\kappa \sigma\sqrt{T}}{a_2}$.

Upon substituting this value of $\gamma$ in~\eqref{eqn:after_Qstar}, and recalling that $a_5 = 6L a_2^2$, we obtain that
\begin{align*}
    \frac{1}{T}\sum_{t = 1}^{T} \expect{\norm{\nabla \loss_{\H}(\weight{t-1})}^2} 
    \leq& a_1\kappa G^2 + \frac{a_2 a_\kappa \sigma}{\sqrt{T}}
    + \frac{24L a_2^2}{4T}
    + \frac{a_2 a_4 \sigma^2}{2(n-f)a_\kappa T^{\nicefrac{3}{2}}}\\
    \leq& a_1\kappa G^2 + \frac{a_2 a_\kappa \sigma}{\sqrt{T}}
    + \frac{a_5}{T}
    + \frac{a_2 a_4 \sigma}{n a_\kappa T^{\nicefrac{3}{2}}}. &(\text{since} ~~  n \geq 2f)
\end{align*}
Finally, recall from Algorithm~\ref{sgd} that $\hat{\weight{}}$ is chosen randomly from the set of computed parameter vectors $\left(\weight{0}, \ldots, \, \weight{T-1} \right)$. Thus, $\expect{\norm{\nabla \loss_{\H}\left(\hat{\weight{}} \right)}^2} = \frac{1}{T}\sum_{t = 1}^{T} \expect{\norm{\nabla \loss_{\H}(\weight{t-1})}^2}$. Substituting this above proves the theorem.
\end{proof}

\clearpage
\subsection{Proof of Corollary~\ref{cor:nnm-dshb}: Analysis of Robust D-SHB with NNM}\label{app:dshb-corollary}


\begin{repcorollary}{cor:nnm-dshb}
Let Assumption~\ref{asp:hetero} hold and recall that $\loss_\H$ is $L$-smooth.
Consider Algorithm~\ref{sgd} with aggregation $F \circ \cenna$, under the same setting as Theorem~\ref{thm:conv}.
If $F$ is $(f,\kappa)$-robust with $\kappa = \mathcal{O}(1)$, then we have
\begin{align*}
   \expect{\norm{\nabla \loss_{\H} (\hat{\weight{}} )}^2}
   = \mathcal{O}{\left( \nicefrac{f}{n} G^2 + \nicefrac{1}{\sqrt{T}}\right)}, 
\end{align*}
where the expectation is over the randomness of the algorithm.
\end{repcorollary}
\begin{proof}
Let Assumption~\ref{asp:hetero} hold.
Assume $\loss_\H$ to be $L$-smooth and $F$ to be $(f,\kappa)$-robust.
First recall that, by Lemma~\ref{lem:cenna}, the composition $F \circ \cenna$ is $(f,\kappa')$-robust with $\kappa' = \frac{8f}{n-f}(\kappa+1)$.

Consider Algorithm~\ref{gd} with aggregation rule $F \circ \cenna$, and learning rate and momentum coefficient set as in Theorem~\ref{thm:conv}.
Following Theorem~\ref{thm:conv}, and recalling the constants defined in \eqref{eqn:a1a2}, we have for every $T \geq 1$,
\begin{align*}
    \expect{\norm{\nabla \loss_{\H} (\hat{\weight{}} )}^2}
    &\leq  a_1 \kappa' G^2
    + \frac{a_2 a_{\kappa'}\sigma}{\sqrt{T}}
    + \frac{a_5}{T}
    + \frac{a_2 a_4 \sigma}{n a_{\kappa'} T^{\nicefrac{3}{2}}}\\
    &= \frac{8 a_1 f}{n-f}(\kappa+1) G^2
    + \frac{a_2 a_{\kappa'}\sigma}{\sqrt{T}}
    + \frac{a_5}{T}
    + \frac{a_2 a_4 \sigma}{n a_{\kappa'} T^{\nicefrac{3}{2}}}.
\end{align*}
Recall that $a_{\kappa'}^2 = a_3 \kappa' + \frac{a_4}{n-f} = \frac{8 a_3 f}{n-f}(\kappa+1) + \frac{a_4}{n-f}$.
As $\kappa \geq 0$ and $\kappa = \mathcal{O}{(1)}$ by assumption, we have $a_{\kappa'} = \Theta{(1)}$.
Now, ignoring constants and the last two terms on the RHS above (since they are dominated by $\frac{1}{\sqrt{T}}$), we obtain
\begin{equation*}
    \expect{\norm{\nabla \loss_{\H} (\hat{\weight{}} )}^2}
    = \mathcal{O}{\left(\frac{f}{n-f}(\kappa+1) G^2 + \frac{1}{\sqrt{T}}\right)}.
\end{equation*}

Recall that, as $f < \nicefrac{n}{2}$, we have $\frac{f}{n-f} \leq \frac{2f}{n}$.
Thus, if $\kappa = \mathcal{O}{(1)}$, we can write
\begin{equation*}
    \expect{\norm{\nabla \loss_{\H} (\hat{\weight{}} )}^2}
    = \mathcal{O}{\left(\nicefrac{f}{n} G^2 + \nicefrac{1}{\sqrt{T}}\right)}.
\end{equation*}
This concludes the proof.
\end{proof}

\clearpage
\subsection{Proof of Supporting Lemmas}
\subsubsection{Proof of Lemma~\ref{lem:mmt_drift}}
\label{app:mmt_drift}

We now recall Lemma~\ref{lem:mmt_drift} below, and present its proof.
\begin{replemma}{lem:mmt_drift}
Suppose that assumptions~\ref{asp:bnd_var} and~\ref{asp:hetero} hold true. Consider Algorithm~\ref{sgd}. For each $t \in [T]$, we have
\begin{align*}
    \expect{\frac{1}{\card{\H}}\sum_{i \in \H}\norm{\mmt{i}{t} - \AvgMmt{t}}^2 } 
    \leq 3\left(\frac{1-\beta}{1+\beta} \right) \,  \left ( 1+\frac{1}{n-f} \right)\sigma^2
    + 3 G^2.
\end{align*}
\end{replemma}
\begin{proof}
Recall from~\eqref{eqn:def_avg_mmt} that 
\begin{align*}
    \AvgMmt{t} \coloneqq \nicefrac{1}{(n-f)} \sum_{i \in \H} \mmt{i}{t}. 
\end{align*}
We consider an arbitrary $i \in \H$. For simplicity we define
\begin{align}
    \overline{g}_t \coloneqq \nicefrac{1}{(n-f)} \sum_{j \in \H}g^{(j)}_t. \label{eqn:over_g}
\end{align}
Now, we consider an arbitrary step $t \in [T]$. Expanding the sum in~\eqref{eqn:mmt_i} we obtain that
\begin{align*}
    \mmt{i}{t} = (1-\beta)\sum_{k=1}^t \beta^{t-k}g^{(i)}_k.
\end{align*}
Therefore, applying Jensen's inequality, we write
\begin{align}
    \expect{\frac{1}{\card{\H}}\sum_{i \in \H}\norm{\mmt{i}{t}-\AvgMmt{t}}^2}
    &= (1-\beta)^2\expect{ \frac{1}{\card{\H}}\sum_{i \in \H} \norm{\sum_{k=1}^t \beta^{t-k}(g^{(i)}_k - \overline{g}_t)}^2} \nonumber\\
    &\leq 3(1-\beta)^2 \expect{\frac{1}{\card{\H}}\sum_{i \in \H}\norm{\sum_{k=1}^t \beta^{t-k}(g^{(i)}_k - \nabla{\loss_i{(\theta_{k-1})}})}^2} \nonumber\\
    &\quad+ 3(1-\beta)^2 \expect{\frac{1}{\card{\H}}\sum_{i \in \H}\norm{\sum_{k=1}^t \beta^{t-k}(\overline{g}_k - \nabla{\loss_{\H}{(\theta_{k-1})}})}^2} \nonumber\\
    &\quad+ 3(1-\beta)^2 \expect{\frac{1}{\card{\H}}\sum_{i \in \H}\norm{\sum_{k=1}^t \beta^{t-k}(\nabla{\loss_i{(\theta_{k-1})}} - \nabla{\loss_{\H}{(\theta_{k-1})}})}^2}.
    \label{eq:lem1}
\end{align}
We obtain below upper bounds for each of the three terms on the right-hand side of~\eqref{eq:lem1}.

\vspace{0.2cm}
For the first term on the R.H.S in~\eqref{eq:lem1} we denote
\begin{align*}
    A_t \coloneqq \expect{\norm{\sum_{k=1}^t \beta^{t-k}(g^{(i)}_k - \nabla{\loss_i{(\theta_{k-1})}})}^2}.
\end{align*}
We obtain that
\begin{align*}
    A_t &= \expect{\norm{\sum_{k=1}^{t-1} \beta^{t-k}(g^{(i)}_k - \nabla{\loss_i{(\theta_{k-1})}}) + (g^{(i)}_t - \nabla{\loss_i{(\theta_{t-1})}})}^2} \\
    &= \expect{\norm{\sum_{k=1}^{t-1} \beta^{t-k}(g^{(i)}_k - \nabla{\loss_i{(\theta_{k-1})}})}^2}
    + \expect{\norm{g^{(i)}_t - \nabla{\loss_i{(\theta_{t-1})}}}^2} \\
    &\quad + 2\expect{\iprod{\sum_{k=1}^{t-1} \beta^{t-k}(g^{(i)}_k - \nabla{\loss_i{(\theta_{k-1})}})}{g^{(i)}_t - \nabla{\loss_i{(\theta_{t-1})}}}}
\end{align*}
The last term on the right-hand side is zero by the total law of expectation, and the unbiasedness of stochastic gradients (Assumption~\ref{asp:bnd_var}):
\begin{align*}
    \expect{\iprod{\sum_{k=1}^{t-1} \beta^{t-k}(g^{(i)}_k - \nabla{\loss_i{(\theta_{k-1})}})}{g^{(i)}_t - \nabla{\loss_i{(\theta_{t-1})}}}}
    &= \expect{\condexpect{t}{\iprod{\sum_{k=1}^{t-1} \beta^{t-k}(g^{(i)}_k - \nabla{\loss_i{(\theta_{k-1})}})}{g^{(i)}_t - \nabla{\loss_i{(\theta_{t-1})}}}}} \\
    &= \expect{\iprod{\sum_{k=1}^{t-1} \beta^{t-k}(g^{(i)}_k - \nabla{\loss_i{(\theta_{k-1})}})}{\underbrace{\condexpect{t}{g^{(i)}_t - \nabla{\loss_i{(\theta_{t-1})}}}}_{=0}}} \\
    &=0.
\end{align*}
By Assumption~\ref{asp:bnd_var}, we have
$\expect{\norm{g^{(i)}_t-\nabla{\loss_i{(\theta_{t-1})}}}^2} \leq \sigma^2$.
Thus, we have
\begin{align*}
    A_t 
    &\leq \expect{\norm{\sum_{k=1}^{t-1} \beta^{t-k}(g^{(i)}_k - \nabla{\loss_i{(\theta_{k-1})}})}^2}
    + \sigma^2
    = \beta^2 A_{t-1} + \sigma^2.
\end{align*}
By recursion, we obtain
\begin{align*}
    A_t \leq \beta^{2(t-1)}A_1 + \sigma^2 \sum_{l=0}^{t-2}\beta^{2l}.
\end{align*}
As $A_1 = \expect{\norm{g^{(i)}_1 - \nabla{\loss_i{(\theta_0)}}}^2} \leq \sigma^2$, from above we obtain that
\begin{align}
    A_t \leq \sigma^2\sum_{l=0}^{t-1}\beta^{2l}
    \leq \frac{\sigma^2}{1-\beta^2}. \label{eqn:mmt_dft_at}
\end{align}
\vspace{0.2cm}

For the second term on the R.H.S in~\eqref{eq:lem1}, we denote
\begin{align*}
    B_t \coloneqq \expect{\norm{\sum_{k=1}^t \beta^{t-k}(\overline{g}_k - \nabla{\loss_{\H}{(\theta_{k-1})}})}^2}.
\end{align*}
By Assumption~\ref{asp:bnd_var} and the mutual independence of stochastic gradients we have 
\begin{align*}
    \expect{\norm{\overline{g}_t - \nabla{\loss_{\H}{(\theta_{t-1})}}}^2} 
    = \frac{1}{(n-f)^2}\sum_{i \in \H}\expect{\norm{g^{(i)}_t - \nabla{\loss_{\H}{(\theta_{t-1})}}}^2}
    \leq \frac{\sigma^2}{n-f}.
\end{align*}
Thus, similar to the bound on $A_t$, we obtain that
\begin{align}
    B_t \leq \frac{\sigma^2}{(1-\beta^2)(n-f)}. \label{eqn:mmt_dft_bt}
\end{align}
\vspace{0.2cm}

For the third term on the R.H.S in~\eqref{eq:lem1}, we denote
\begin{align*}
    C_t \coloneqq \expect{\frac{1}{\card{\H}}\sum_{i \in \H}\norm{\sum_{k=1}^t \beta^{t-k}(\nabla{\loss_i{(\theta_{k-1})}} - \nabla{\loss_{\H}{(\theta_{k-1})}})}^2}.
\end{align*}
By applying Jensen's inequality, followed by Assumption~\ref{asp:hetero}, we have
\begin{align}
    C_t 
    &\leq \left(\sum_{k=1}^t \beta^{t-k}\right)\expect{\frac{1}{\card{\H}}\sum_{i \in \H}\sum_{k=1}^t \beta^{t-k} \norm{\nabla{\loss_i{(\theta_{k-1})}} - \nabla{\loss_{\H}{(\theta_{k-1})}}}^2} \nonumber \\
    &= \left(\sum_{k=1}^t \beta^{t-k}\right)\expect{\sum_{k=1}^t \beta^{t-k} \frac{1}{\card{\H}}\sum_{i \in \H}\norm{\nabla{\loss_i{(\theta_{k-1})}} - \nabla{\loss_{\H}{(\theta_{k-1})}}}^2} \nonumber \\
    &\leq \left(\sum_{k=1}^t \beta^{t-k}\right)^2 \cdot G^2
    \leq \frac{G^2}{(1-\beta)^2}. \label{eqn:mmt_dft_ct} 
\end{align}
Finally, substituting from~\eqref{eqn:mmt_dft_at},~\eqref{eqn:mmt_dft_bt} and~\eqref{eqn:mmt_dft_ct} in~\eqref{eq:lem1} we obtain that
\begin{align*}
    \expect{\frac{1}{\card{\H}}\sum_{i \in \H}\norm{\mmt{i}{t}-\AvgMmt{t}}^2}
    &\leq 3(1-\beta)^2 \frac{1}{\card{\H}}\sum_{i \in \H} \left( A_t+B_t+C_t \right) \nonumber \\
    &\leq 3(1-\beta)^2 \frac{1}{\card{\H}}\sum_{i \in \H} \left(\frac{\sigma^2}{1-\beta^2}
    +\frac{\sigma^2}{(1-\beta^2)(n-f)}
    +\frac{G^2}{(1-\beta)^2}\right) \nonumber \\
    &= 3\,\frac{1-\beta}{1+\beta} \left ( 1+\frac{1}{n-f} \right)\sigma^2
    + 3G^2. 
\end{align*}
The above proves the lemma.
\end{proof}

\subsubsection{Proof of Lemma~\ref{lem:drift}}
\label{app:lem_drift}

\begin{replemma}{lem:drift}
Suppose that assumptions~\ref{asp:hetero} and ~\ref{asp:bnd_var} hold true.
Assume $F$ is $(f,\kappa)$-robust.
Consider Algorithm~\ref{sgd} with aggregation $F$. For each step $t \in [T]$, we obtain that
\begin{align*}
    \expect{\norm{\drift{t}}^2}
    \leq 6 \kappa \frac{1-\beta}{1+\beta}\sigma^2 + 48\kappa G^2.
\end{align*}
\end{replemma}

\begin{proof}
Recall from~\eqref{eqn:R} and~\eqref{eqn:drift}, respectively, that
\begin{align*}
    m_t \coloneqq F  \left(\mmt{1}{t}, \ldots, \, \mmt{n}{t} \right) ~ \text{ and } ~ \drift{t} \coloneqq m_t - \AvgMmt{t}.
\end{align*}
We consider an arbitrary step $t$. 
Since $F$ is $(f,\kappa)$-robust, we obtain that
\begin{equation}
\label{eqn:lem_drift_t}
    \norm{\drift{t}}^2 
    = \norm{m_t - \AvgMmt{t}}^2
    \leq \frac{\kappa}{n-f}\sum_{i \in \H} \norm{\mmt{i}{t}-\AvgMmt{t}}^2.
\end{equation}
Upon taking total expectations on both sides we obtain that
\begin{align}
    \expect{\norm{\drift{t}}^2} \leq \frac{\kappa}{n-f} \sum_{i \in \H} \expect{\norm{\mmt{i}{t} - \AvgMmt{t}}^2}. \label{eqn:before_lemma_drift}
\end{align}
From Lemma~\ref{lem:mmt_drift}, under Assumption~\ref{asp:bnd_var}, we have
\begin{align*}
    \expect{\frac{1}{\card{\H}}\sum_{i 
    \in \H}\norm{\mmt{i}{t} - \AvgMmt{t}}^2}
    &\leq 3\,\frac{1-\beta}{1+\beta} \left ( 1+\frac{1}{n-f} \right)\sigma^2
    + 3G^2 \\
    &\leq 6\,\frac{1-\beta}{1+\beta}\sigma^2 + 3 G^2.
    &(n-f\geq1)
\end{align*}
Substituting from above in~\eqref{eqn:before_lemma_drift} proves the lemma, i.e., we conclude that
\begin{align*}
     \expect{\norm{\drift{t}}^2}
     &\leq 6\kappa\frac{1-\beta}{1+\beta}\sigma^2 + 3\kappa G^2.
\end{align*}
This concludes the proof.
\end{proof}

\subsubsection{Proof of Lemma~\ref{lem:dev}}
\label{app:lem_dev}

The proof of Lemma~\ref{lem:dev} is similar to that of Lemma~3 in \cite{farhadkhani2022byzantine}.
We recall the lemma and the proof below for completeness.

\begin{replemma}{lem:dev}
Suppose that Assumption \ref{asp:bnd_var} holds.
Recall that $\loss_\H$ is $L$-smooth.
Consider Algorithm~\ref{sgd} with $T > 1$. For all $t \geq 2$ we obtain that
\begin{align*}
    \expect{\norm{\dev{t}}^2} \leq & \beta^2 c \expect{\norm{\dev{t-1}}^2} +  4 \gamma L ( 1 + \gamma L) \beta^2  \expect{\norm{\nabla \loss_{\H}(\weight{t-2})}^2} +(1 - \beta)^2 \frac{\sigma^2}{(n-f)} \\
    & + 2 \gamma L ( 1 + \gamma L)\beta^2  \expect{\norm{\drift{t-1}}^2},
\end{align*}
where $c \coloneqq (1 + \gamma L ) \left(1 + 4 \gamma   L \right)$.
\end{replemma}

\begin{proof}
Recall from~\eqref{eqn:dev} that 
\begin{align*}
    \dev{t} \coloneqq \AvgMmt{t} - \nabla \loss_{\H}\left( \weight{t - 1} \right).
\end{align*}
Consider an arbitrary step $t > 1$. Substituting from~\eqref{eqn:mmt_i} and~\eqref{eqn:def_avg_mmt} we obtain that
\begin{align*}
    \dev{t} = \beta \, \AvgMmt{t-1} + (1 - \beta) \, \overline{g}_{t} - \nabla \loss_{\H}\left( \weight{t - 1} \right).
\end{align*}
Upon adding and subtracting $\beta \nabla \loss_{\H}(\weight{t-2})$ and $\beta \nabla \loss_{\H}(\weight{t-1})$ on the R.H.S.~above we obtain that
\begin{align*}
    \dev{t} & = \beta \, \AvgMmt{t-1} - \beta \nabla \loss_{\H}(\weight{t-2}) + (1 - \beta) \, \overline{g}_{t} - \nabla \loss_{\H}\left( \weight{t} \right) + \beta \nabla \loss_{\H}(\weight{t-1}) + \beta \nabla \loss_{\H}(\weight{t-2}) - \beta \nabla \loss_{\H}(\weight{t-1}) \\
    & = \beta \left( \AvgMmt{t-1} - \nabla \loss_{\H}(\weight{t-2}) \right) + (1 - \beta) \, \overline{g}_{t} - (1 - \beta) \nabla \loss_{\H}\left( \weight{t-1} \right) + \beta \left( \nabla \loss_{\H}(\weight{t-2}) - \nabla \loss_{\H}(\weight{t-1})  \right).
\end{align*}
As $\AvgMmt{t-1} - \nabla \loss_{\H}(\weight{t-2}) = \dev{t-1}$ (by~\eqref{eqn:dev}), from above we obtain that
\begin{align*}
    \dev{t} = \beta \dev{t-1} + (1 - \beta) \, \left( \overline{g}_{t} - \nabla \loss_{\H}\left( \weight{t-1} \right) \right) + \beta \left( \nabla \loss_{\H}(\weight{t-2}) - \nabla \loss_{\H}(\weight{t-1}) \right).
\end{align*}
Therefore, 
\begin{align*}
    \norm{\dev{t}}^2 = & \beta^2 \norm{\dev{t-1}}^2 + (1 - \beta)^2 \norm{ \overline{g}_{t} - \nabla \loss_{\H}\left( \weight{t-1} \right)}^2 + \beta^2 \norm{\nabla \loss_{\H}(\weight{t-2}) - \nabla \loss_{\H}(\weight{t-1}) }^2 + 2 \beta (1 - \beta) \iprod{\dev{t-1}}{\overline{g}_{t} - \nabla \loss_{\H}\left( \weight{t-1} \right)} \\
    & + 2 \beta^2 \iprod{\dev{t-1}}{\nabla \loss_{\H}(\weight{t-2}) - \nabla \loss_{\H}(\weight{t-1})} + 2 \beta ( 1- \beta) \iprod{\overline{g}_{t} - \nabla \loss_{\H}\left( \weight{t-1} \right)}{\nabla \loss_{\H}(\weight{t-2}) - \nabla \loss_{\H}(\weight{t-1})}.
\end{align*}
By taking conditional expectation $\condexpect{t}{\cdot}$ on both sides, and recalling that $\dev{t-1}$, $\weight{t-1}$ and $\weight{t-2}$ are deterministic values when the history $\P_t$ is given, we obtain that 
\begin{align*}
    \condexpect{t}{\norm{\dev{t}}^2} = & \beta^2 \norm{\dev{t-1}}^2 + (1 - \beta)^2 \condexpect{t}{\norm{ \overline{g}_{t} - \nabla \loss_{\H}\left( \weight{t-1} \right)}^2} + \beta^2 \norm{\nabla \loss_{\H}(\weight{t-2}) - \nabla \loss_{\H}(\weight{t-1}) }^2 + \\
    & 2 \beta (1 - \beta) \iprod{\dev{t-1}}{\condexpect{t}{\overline{g}_{t}} - \nabla \loss_{\H}\left( \weight{t-1} \right)} + 2 \beta^2 \iprod{\dev{t-1}}{\nabla \loss_{\H}(\weight{t-2}) - \nabla \loss_{\H}(\weight{t-1})}\\
    & + 2 \beta ( 1- \beta) \iprod{\condexpect{t}{\overline{g}_{t}} - \nabla \loss_{\H}\left( \weight{t-1} \right)}{\nabla \loss_{\H}(\weight{t-2}) - \nabla \loss_{\H}(\weight{t-1})}.
\end{align*}
Recall that $\overline{g}_t \coloneqq \nicefrac{1}{(n-f)} \sum_{j \in \H}g^{(i)}_t$. Thus, we have $\condexpect{t}{\overline{g}_{t}} = \nabla \loss_{\H}(\weight{t-1})$. Using this above we obtain that
\begin{align*}
    \condexpect{t}{\norm{\dev{t}}^2} = & \beta^2 \norm{\dev{t-1}}^2 + (1 - \beta)^2 \condexpect{t}{\norm{ \overline{g}_{t} - \nabla \loss_{\H}\left( \weight{t - 1} \right)}^2} + \beta^2 \norm{\nabla \loss_{\H}(\weight{t-2}) - \nabla \loss_{\H}(\weight{t - 1}) }^2 \\
    & + 2 \beta^2 \iprod{\dev{t-1}}{\nabla \loss_{\H}(\weight{t-2}) - \nabla \loss_{\H}(\weight{t - 1})}.
\end{align*}
Also, by Assumption~\ref{asp:bnd_var} and the fact that $\gradient{j}{t}$'s for $j \in \H$ are independent of each other, we obtain that $\condexpect{t}{\norm{ \overline{g}_{t} - \nabla \loss_{\H}\left( \weight{t - 1} \right)}^2} \leq \frac{\sigma^2}{(n-f)}$. Thus,
\begin{align*}
    \condexpect{t}{\norm{\dev{t}}^2} \leq \beta^2 \norm{\dev{t-1}}^2 + (1 - \beta)^2 \frac{\sigma^2}{(n-f)} + \beta^2 \norm{\nabla \loss_{\H}(\weight{t-2}) - \nabla \loss_{\H}(\weight{t - 1}) }^2 + 2 \beta^2 \iprod{\dev{t-1}}{\nabla \loss_{\H}(\weight{t-2}) - \nabla \loss_{\H}(\weight{t - 1})}.
\end{align*}
By the Cauchy-Schwartz inequality, $\iprod{\dev{t-1}}{\nabla \loss_{\H}(\weight{t-2}) - \nabla \loss_{\H}(\weight{t - 1})} \leq \norm{\dev{t-1}} \norm{\nabla \loss_{\H}(\weight{t-2}) - \nabla \loss_{\H}(\weight{t - 1})}$. 
Since $\loss_\H$ is $L$-smooth, we have $\norm{ \nabla \loss_{\H}(\weight{t-2}) - \nabla \loss_{\H}(\weight{t-1})} \leq L \norm{\weight{t-1} - \weight{t-2}}$. 
Recall from~\eqref{eqn:SGD} that $\weight{t} = \weight{t-1} - \gamma  m_t$. Thus,$\norm{\nabla \loss_{\H}(\weight{t-2}) - \nabla \loss_{\H}(\weight{t-1})} \leq \gamma  L \norm{m_{t-1}}$. Using this above we obtain that
\begin{align*}
    \condexpect{t}{\norm{\dev{t}}^2} \leq \beta^2 \norm{\dev{t-1}}^2 + (1 - \beta)^2 \frac{\sigma^2}{(n-f)} + \gamma^2 \beta^2 L^2 \norm{m_{t-1}}^2 + 2 \gamma  \beta^2 L \norm{\dev{t-1}} \norm{m_{t-1}}.
\end{align*}
As $2 ab \leq a^2 + b^2$, from above we obtain that
\begin{align}
    \condexpect{t}{\norm{\dev{t}}^2} & \leq \beta^2 \norm{\dev{t-1}}^2 + (1 - \beta)^2 \frac{\sigma^2}{(n-f)} + \gamma^2 \beta^2 L^2 \norm{m_{t-1}}^2 + \gamma  L \beta^2 \left( \norm{\dev{t-1}}^2 +  \norm{m_{t-1}}^2\right) \nonumber \\
    & = (1 + \gamma L ) \beta^2 \norm{\dev{t-1}}^2 + (1 - \beta)^2 \frac{\sigma^2}{(n-f)} + \gamma L (1 + \gamma L) \beta^2  \norm{m_{t-1}}^2. \label{eqn:dev_before_ab}
\end{align}
By definition of $\drift{t}$ in~\eqref{eqn:drift}, we have $m_{t-1} = \drift{t-1} +  \AvgMmt{t-1}$. Thus, owing to the triangle inequality and the fact that $2 ab \leq a^2 + b^2$, we have $\norm{m_{t-1}}^2 \leq 2 \norm{\drift{t-1}}^2 + 2  \norm{\AvgMmt{t-1}}^2$. Similarly, by definition of $\dev{t}$ in~\eqref{eqn:dev}, we have $\norm{\AvgMmt{t-1}}^2 \leq 2 \norm{\dev{t-1}}^2 + 2 \norm{\nabla \loss_{\H}(\weight{t-2})}^2$. Thus, $\norm{m_{t-1}}^2 \leq 2 \norm{\drift{t-1}}^2 + 4  \norm{\dev{t-1}}^2 + 4  \norm{\nabla \loss_{\H}(\weight{t-2})}^2$. 
Using this in~\eqref{eqn:dev_before_ab} we obtain that
\begin{align*}
    \condexpect{t}{\norm{\dev{t}}^2} \leq & (1 + \gamma L ) \beta^2 \norm{\dev{t-1}}^2 + (1 - \beta)^2 \frac{\sigma^2}{(n-f)} \\
    & + 2 \gamma L( 1 + \gamma L)\beta^2  \left( \norm{\drift{t-1}}^2 + 2  \norm{\dev{t-1}}^2 + 2  \norm{\nabla \loss_{\H}(\weight{t-2})}^2 \right).
\end{align*}
By re-arranging the terms on the R.H.S.~we get
\begin{align*}
    \condexpect{t}{\norm{\dev{t}}^2} \leq & \beta^2 (1 + \gamma L ) \left(1 + 4 \gamma   L \right) \norm{\dev{t-1}}^2 +  4 \gamma  L( 1 + \gamma L) \beta^2   \norm{\nabla \loss_{\H}(\weight{t-2})}^2 +(1 - \beta)^2 \frac{\sigma^2}{(n-f)} \\
    & + 2 \gamma  L( 1 + \gamma L)\beta^2 \norm{\drift{t-1}}^2.
\end{align*}
Substituting $c = (1 + \gamma  L) \left(1 + 4 \gamma   L \right)$ above we obtain that
\begin{align*}
    \condexpect{t}{\norm{\dev{t}}^2} \leq & \beta^2 c \norm{\dev{t-1}}^2 +  4 \gamma L ( 1 + \gamma L) \beta^2 \norm{\nabla \loss_{\H}(\weight{t-2})}^2 +(1 - \beta)^2 \frac{\sigma^2}{(n-f)} + 2 \gamma L ( 1 + \gamma L)\beta^2 \norm{\drift{t-1}}^2.
\end{align*}
Recall that $t$ in the above is an arbitrary value in $[T]$ greater than $1$. Hence, upon taking total expectation on both sides above proves the lemma.

\end{proof}

\subsubsection{Proof of Lemma~\ref{lem:growth_Q}}
\label{app:growth_Q}

The proof of Lemma~\ref{lem:growth_Q} is similar to that of Lemma~4 in \cite{farhadkhani2022byzantine}.
We recall the lemma and the proof below for completeness.

\begin{replemma}{lem:growth_Q} 
Recall that $\loss_\H$ is $L$-smooth.
Consider Algorithm~\ref{sgd}. For all $t \in [T]$, we obtain that
\begin{align*}
    \expect{2 \loss_{\H}(\weight{t}) - 2 \loss_{\H}(\weight{t-1})} \leq & - \gamma    \left( 1 - 4 \gamma  L  \right) \expect{\norm{\nabla \loss_{\H}(\weight{t-1})}^2}  + 2 \gamma    \left( 1 + 2 \gamma  L   \right) \expect{ \norm{\dev{t}}^2} \\
    & + 2 \gamma  \left(  1 + \gamma  L \right) \expect{\norm{\drift{t}}^2}.
\end{align*}
\end{replemma}

\begin{proof}
Consider an arbitrary step $t$. Since $\loss_\H$ is $L$-smooth, we have (see~\cite{bottou2018optimization})
\begin{align*}
    \loss_{\H}(\weight{t}) - \loss_{\H}(\weight{t-1}) \leq \iprod{\weight{t} - \weight{t-1}}{\nabla \loss_{\H}(\weight{t-1})} + \frac{L}{2} \norm{\weight{t} - \weight{t-1}}^2.
\end{align*}
Substituting from~\eqref{eqn:sgd_new}, i.e., $\weight{t} = \weight{t-1} - \gamma    \, \AvgMmt{t} - \gamma  \drift{t}$, we obtain that
\begin{align*}
    \loss_{\H}(\weight{t}) - \loss_{\H}(\weight{t-1}) &\leq - \gamma   \iprod{\AvgMmt{t}}{\nabla \loss_{\H}(\weight{t-1})} - \gamma  \iprod{\drift{t}}{\nabla \loss_{\H}(\weight{t-1})} + \gamma^2 \frac{L}{2} \norm{ \, \AvgMmt{t} + \drift{t}}^2 \\
    & = - \gamma   \iprod{\AvgMmt{t} - \nabla \loss_{\H}(\weight{t-1}) + \nabla \loss_{\H}(\weight{t-1})}{\nabla \loss_{\H}(\weight{t-1})} - \gamma  \iprod{\drift{t}}{\nabla \loss_{\H}(\weight{t-1})} + \gamma^2 \frac{L}{2} \norm{ \, \AvgMmt{t} + \drift{t}}^2.
\end{align*}
By Definition~\eqref{eqn:dev}, $\AvgMmt{t} - \nabla \loss_{\H}(\weight{t-1}) = \dev{t}$. Thus, from above we obtain that (scaling by factor of $2$)
\begin{align}
    2 \loss_{\H}(\weight{t}) - 2 \loss_{\H}(\weight{t-1}) \leq - 2 \gamma   \norm{\nabla \loss_{\H}(\weight{t-1})}^2 - 2 \gamma   \iprod{\dev{t}}{\nabla \loss_{\H}(\weight{t-1})} - 2 \gamma  \iprod{\drift{t}}{\nabla \loss_{\H}(\weight{t-1})} + \gamma^2 L \norm{ \, \AvgMmt{t} + \drift{t}}^2. \label{eqn:norm_1}
\end{align}
Now, we consider the last three terms on the R.H.S.~separately. Using Cauchy-Schwartz inequality, and the fact that $2 ab \leq \frac{1}{c} a^2 + c b^2$ for any $c > 0$, we obtain that (by substituting $c = 2$)
\begin{align}
    2 \mnorm{\iprod{\dev{t}}{\nabla \loss_{\H}(\weight{t-1})}} \leq 2 \norm{\dev{t}} \norm{\nabla \loss_{\H}(\weight{t-1})} \leq \frac{2}{1} \norm{\dev{t}}^2 + \frac{1}{2} \norm{\nabla \loss_{\H}(\weight{t-1})}^2 . \label{eqn:rho_1}
\end{align}
Similarly, 
\begin{align}
    2 \mnorm{\iprod{\drift{t}}{\nabla \loss_{\H}(\weight{t-1})}} \leq 2 \norm{\drift{t}} \norm{\nabla \loss_{\H}(\weight{t-1})} \leq \frac{ 2}{ 1} \norm{\drift{t}}^2 +  \frac{1}{2} \norm{\nabla \loss_{\H}(\weight{t-1})}^2. \label{eqn:rho_2}
\end{align}
Finally, using triangle inequality and the fact that $2ab \leq a^2 + b^2$ we have
\begin{align}
    \norm{ \, \AvgMmt{t} + \drift{t}}^2 & \leq 2  \, \norm{\AvgMmt{t}}^2 + 2 \norm{\drift{t}}^2 = 2  \, \norm{\AvgMmt{t} - \nabla \loss_{\H}(\weight{t}) + \nabla \loss_{\H}(\weight{t-1})}^2 + 2 \norm{\drift{t}}^2 \nonumber \\
    & \leq 4  \, \norm{\dev{t}}^2 + 4  \, \norm{\nabla \loss_{\H}(\weight{t-1})}^2 + 2 \norm{\drift{t}}^2. \quad \quad [\text{since} ~ ~ \AvgMmt{t} - \nabla \loss_{\H}(\weight{t-1}) = \dev{t}] \label{eqn:last_term}
\end{align}
Substituting from~\eqref{eqn:rho_1},~\eqref{eqn:rho_2} and~\eqref{eqn:last_term} in~\eqref{eqn:norm_1} we obtain that
\begin{align*}
    2 \loss_{\H}(\weight{t}) - 2 \loss_{\H}(\weight{t-1}) \leq & - 2 \gamma   \norm{\nabla \loss_{\H}(\weight{t-1})}^2 + \gamma   \left( 2 \norm{\dev{t}}^2 + \frac{1}{2}\norm{\nabla \loss_{\H}(\weight{t-1})}^2 \right) + \gamma  \left( 2 \norm{\drift{t}}^2 + \frac{1}{2} \norm{\nabla \loss_{\H}(\weight{t-1})}^2 \right) \nonumber \\
    & + \gamma^2 L \left( 4  \, \norm{\dev{t}}^2 + 4  \, \norm{\nabla \loss_{\H}(\weight{t-1})}^2 + 2 \norm{\drift{t}}^2 \right).
\end{align*}
Upon re-arranging the terms in the R.H.S.~we obtain that
\begin{align*}
    2 \loss_{\H}(\weight{t}) - 2 \loss_{\H}(\weight{t-1}) \leq - \gamma   \left( 1 - 4 \gamma  L \right) \norm{\nabla \loss_{\H}(\weight{t-1})}^2  + 2 \gamma   \left( 1 + 2 \gamma  L  \right) \norm{\dev{t}}^2 + 2 \gamma  \left(  1 + \gamma  L \right) \norm{\drift{t}}^2.
\end{align*}
As $t$ is arbitrarily chosen from $[T]$, taking expectation on both sides above proves the lemma.
\end{proof}
\newpage
\section{Experimental Setup}\label{app:exp_setup}

\subsection{Detailed Experimental Setup}

The architecture of the models, as well as additional details on the experimental setup, are presented in Table~\ref{table_exp}.
Note that CNN stands for convolutional neural network, and NLL refers to the negative log likelihood loss.
In order to present the architecture of the models used, we introduce the following compact notation.\\

\noindent \fcolorbox{black}{gainsboro!40}{
\parbox{0.98\textwidth}{
L(\#outputs) represents a \textbf{fully-connected linear layer}, R stands for \textbf{ReLU activation}, S stands for \textbf{log-softmax}, C(\#channels) represents a \textit{fully-connected 2D-convolutional layer} (kernel size 5, padding 0, stride 1), M stands for \textbf{2D-maxpool} (kernel size 2), B stands for \textbf{batch-normalization}, and D represents \textbf{dropout} (with fixed probability 0.25).
}}.

\newcolumntype{P}[1]{>{\centering\arraybackslash}p{#1}}
\begin{table}[h!]
\centering
\def\arraystretch{1.5}
\begin{tabular}{|P{15em}||P{15em}|P{15em}|} 
 \hline
 \textit{Dataset} & MNIST \& Fashion-MNIST & CIFAR-10 \\ [0.5ex] 
 \hline
 \textit{Model type} & CNN & CNN\\ 
 \hline
 \textit{Model architecture} & C(20)-R-M-C(20)-R-M-L(500)-R-L(10)-S & (3,32×32)-C(64)-R-B-C(64)-R-B-M-D-C(128)-R-B-C(128)-R-B-M-D-L(128)-R-D-L(10)-S\\
 \hline
 \textit{Loss} & NLL & NLL\\
  \hline
  \textit{Gradient clipping} & 2 & 5\\
  \hline
  \textit{$\ell_2$-regularization} & $10^{-4}$ & $10^{-2}$\\
  \hline
  \textit{Number of steps} & $T= 800$ & $T = 2000$\\
  \hline
  \textit{Learning rate} & \begin{equation*}
    \gamma_t= \frac{0.75}{1 + \floor{\frac{t}{50}}}
  \end{equation*} & \begin{equation*}
    \gamma_t=
    \begin{cases}
      0.25 & t \leq 1500 \\
      0.025 & 1500 < t \leq 2000
    \end{cases}
  \end{equation*}\\
  \hline
  \textit{Momentum parameter} & $\beta = 0.9$ & $\beta = 0.9, 0.99$\\
  \hline
  \textit{Batch size} & $b = 25$ & $b = 50$\\
  \hline
 \textit{Total number of workers} & $n=17$ & $n=17$\\
 \hline
 \textit{Number of Byzantine workers} & $f= 4, 6, 8$ & $f = 2, 3, 4$\\
 \hline
\end{tabular}
\caption{Model Architectures, Hyperparameters, and Distributed Settings}
\label{table_exp}
\end{table}

On all datasets, we implement $\cenna{}$ and Bucketing with four aggregation rules namely geometric median (GM), coordinate-wise median (CWMed), coordinate-wise trimmed mean (CWTM), and Krum. In every setting, we execute the Bucketing algorithm with bucket size $s = \floor{\frac{n}{2f}}$~\cite{karimireddy2022byzantinerobust}. We also implement the vanilla aggregation rules.\\
Note that when $f=4$ (out of $n=17$), Bucketing+CWMed and Bucketing+CWTM are exactly equivalent. In fact, when $f=4$, buckets are of size $s=2$. Therefore, applying Bucketing results in a total of $n'=9$ buckets, out of which $f'=4$ buckets might be potentially Byzantine (i.e., contaminated by at least one Byzantine worker). Therefore, we can clearly see that executing CWMed and CWTM with $n'=9$ and $f'=4$ (i.e., post Bucketing) outputs the same vector. Accordingly, we only show the performance of the learning under Bucketing+CWMed in the plots of Appendix~\ref{app:exp_results} when $f=4$.
Finally, note that we use folklore techniques from deep learning to improve accuracy, e.g. gradient clipping. The latter may in fact have an additional positive impact on the robust aggregation in our setting, which we also are investigating.

\subsection{Dataset Preprocessing}
MNIST receives an input image normalization of mean $0.1307$ and standard deviation $0.3081$, while Fashion-MNIST is expanded with horizontally flipped images.
Furthermore, the images of CIFAR-10 are horizontally flipped, and per channel normalization is also applied with means 0.4914, 0.4822, 0.4465 and standard deviations 0.2023, 0.1994, 0.2010.

\subsection{Byzantine Attacks}
In our experiments, the Byzantine workers execute five state-of-the-art gradient attacks, namely A Little is Enough (ALIE)~\cite{little}, Fall of Empires (FOE)~\cite{empire}, Sign-flipping (SF)~\cite{allen2020byzantine}, Label-flipping (LF)~\cite{allen2020byzantine}, and Mimic~\cite{karimireddy2022byzantinerobust}.
The first three attacks rely on the same primitive explained below.\\
Let $a_t$ be the attack vector in step $t$ and $\eta \geq 0$ a fixed real number. In every step $t$, the Byzantine workers send to the server the Byzantine vector $B_t = \overline{s}_t + \eta a_t$, where $\overline{s}_t$ is an estimation of the true gradient (or momentum) at step $t$.\\
Experimentally, when running the gradient descent (GD) algorithm, we set $\overline{s}_t = \frac{1}{|\mathcal{H}|} \sum\limits_{i \in \mathcal{H}} g_t^{(i)}$, where $g_t^{(i)}$ is the gradient computed by honest worker $w_i$ in step $t$. However, when executing the stochastic heavy ball (SHB) method, we set $\overline{s}_t = \frac{1}{|\mathcal{H}|} \sum\limits_{i \in \mathcal{H}} m_t^{(i)}$, where $m_t^{(i)}$ is the momentum vector sent by honest worker $w_i$ in step $t$.

\begin{itemize}
    \item \textbf{ALIE:} In this attack, $a_t = \sigma_t$, where $\sigma_t$ is coordinate-wise standard deviation of $\overline{s}_t$.
    \item \textbf{FOE:} In this attack, $a_t = - \overline{s}_t$. All Byzantine workers thus send $(1 - \eta) \overline{s}_t$ in step $t$.
    \item \textbf{SF:} In this attack, $a_t = - \overline{s}_t$, and $\eta = 2$. All Byzantine workers thus send $B_t = a_t = - \overline{s}_t$ in step $t$.
\end{itemize}

In our experiments, as done in~\cite{shejwalkar2021manipulating}, we implement optimized versions of ALIE and FOE, where the optimal $\eta$ is determined greedily by linearly searching over a defined range of values. In every step $t$, we pick the value of $\eta$ that maximises the L2-distance between the output of the aggregation rule $R_t$ at the server and the average of honest gradients or momentums, i.e., $\overline{s}_t$. In other words, we pick the value of $\eta$ that causes the maximum damage (in terms of L2-distance from the average of honest inputs) by the Byzantine workers.

LF and Mimic on the other hand are executed as follows:
\begin{itemize}
    \item \textbf{LF:} Every Byzantine worker computes its gradient on flipped labels. Since the labels $l$ for MNIST, Fashion-MNIST, and CIFAR-10 are in $\{0, 1, ..., 9\}$, the Byzantine workers flip the labels by computing $l' = 9 - l$ on the batch, where $l'$ is the flipped/modified label.
    \item \textbf{Mimic:} All Byzantine workers ``mimic" a certain honest worker by simply sending its gradient or momentum to the server. In order to chose which honest worker to mimic during the learning, we adopt the heuristic introduced by~\cite{karimireddy2022byzantinerobust}.
\end{itemize}

\subsection{Data Heterogeneity}
In order to simulate heterogeneity in our experiments, we sample from the original datasets (i.e., MNIST, Fashion-MNIST, and CIFAR-10) using a \textbf{Dirichlet} distribution of parameter $\alpha$, which indicates the level of heterogeneity induced in the workers' datasets. The smaller is $\alpha$, the more heterogeneous is the setting (i.e., the more likely it is that workers possess samples from only one class). We consider three heterogeneity regimes in our experiments, namely \textit{extreme} where $\alpha = 0.1$, \textit{moderate} where $\alpha = 1$, and \textit{low} where $\alpha = 10$. The corresponding distributions of the number of samples across workers and class labels (depending on $\alpha$) are shown in Figure~\ref{fig:distribution}.

\begin{figure*}[ht!]
    \centering
    \includegraphics[width=0.33\textwidth]{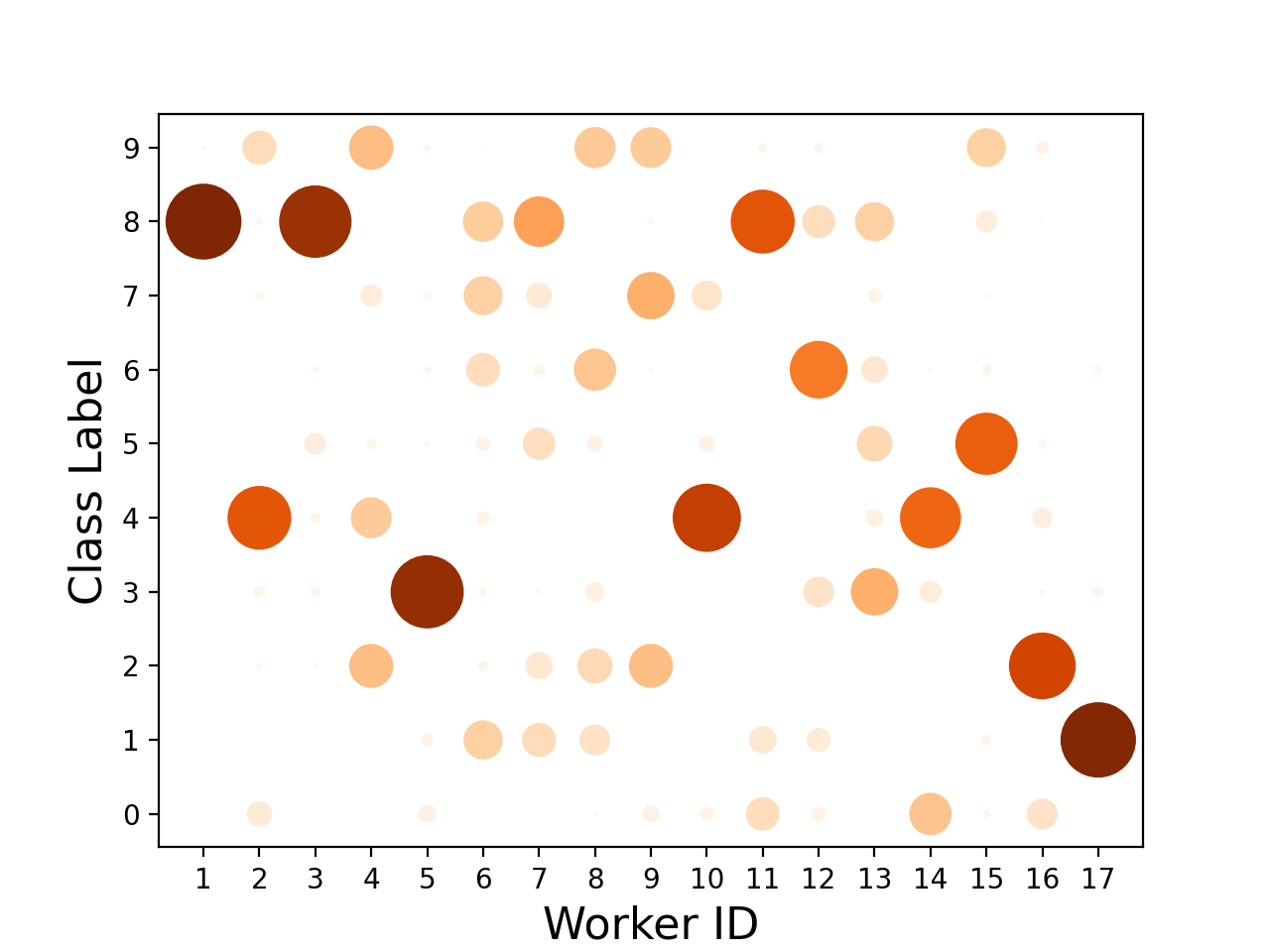}%
    \includegraphics[width=0.33\textwidth]{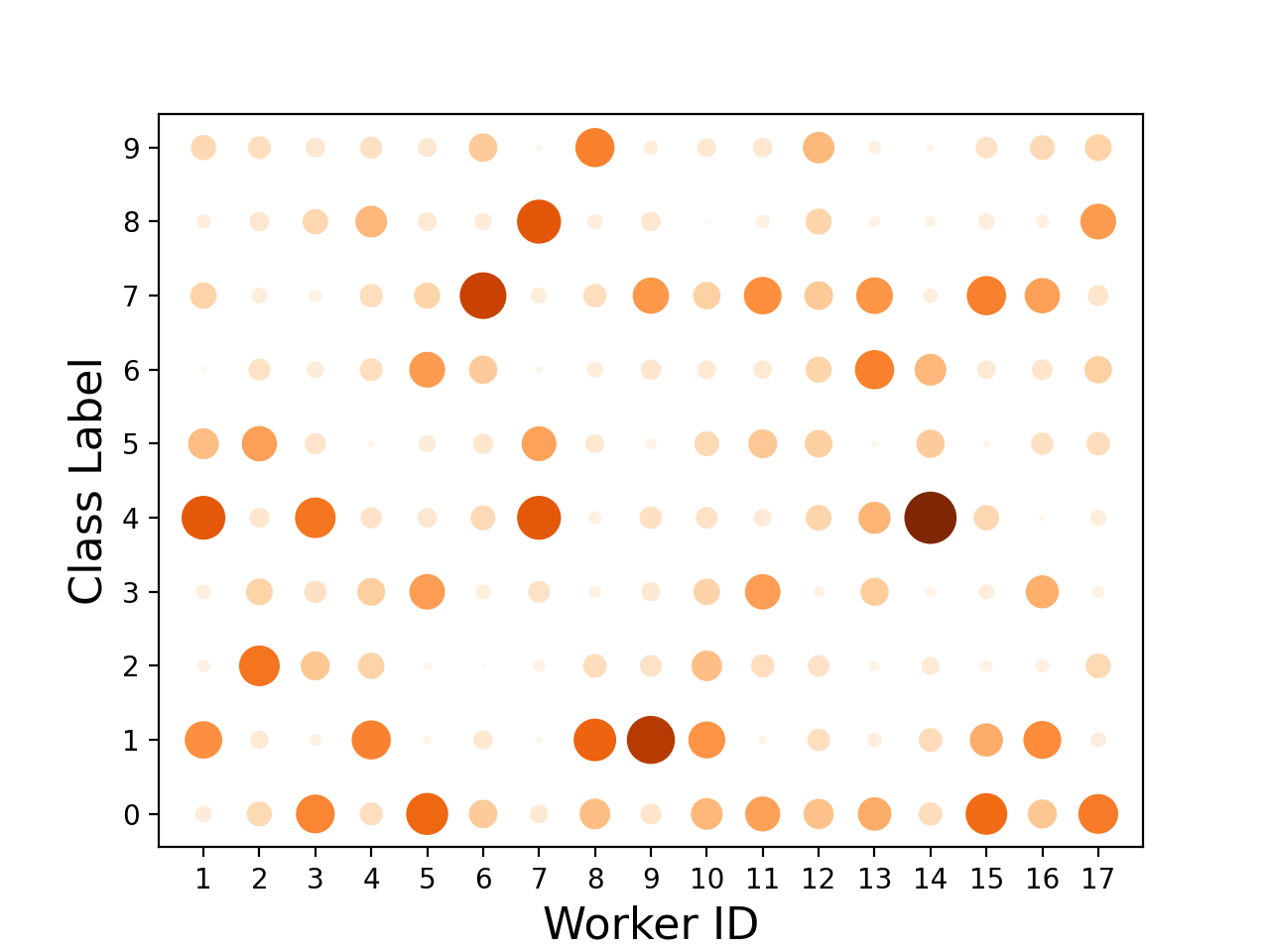}%
    \includegraphics[width=0.33\textwidth]{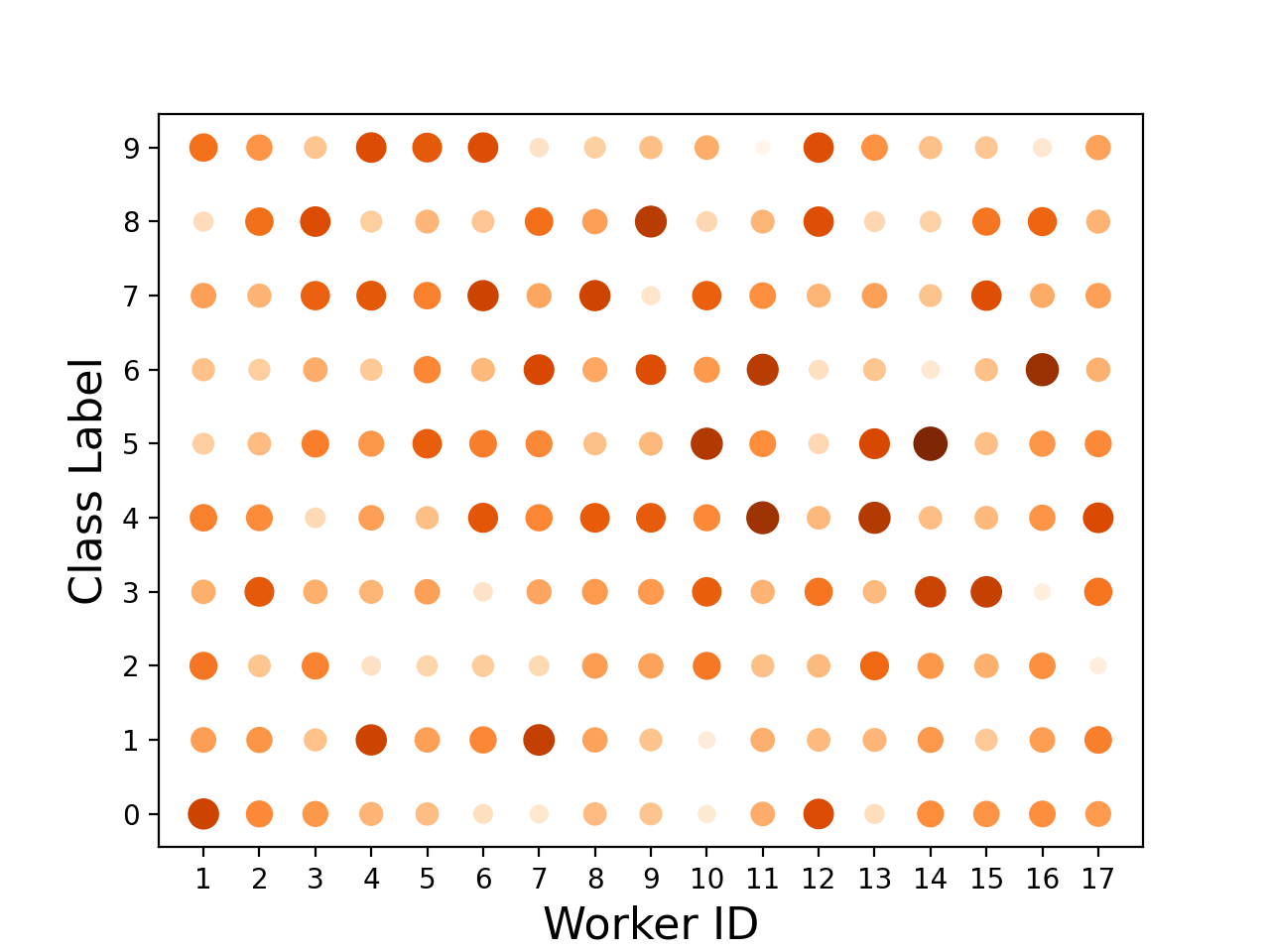}%
    \caption{Distribution of data samples across workers and class labels on MNIST/CIFAR-10 when sampling from a Dirichlet distribution of parameter $\alpha$. \textit{Left}: $\alpha = 0.1$, \textit{Middle}: $\alpha = 1$, \textit{Right}: $\alpha = 10$.}
\label{fig:distribution}
\end{figure*}
\newpage
\section{Full Experimental Results}\label{app:exp_results}

Section~\ref{app:exp_results_mnist} presents our results on the MNIST dataset. Section~\ref{app:exp_results_fashion} contains our results on Fashion-MNIST. Section~\ref{app:exp_results_cifar} shows the performance of our algorithm on CIFAR-10.

\subsection{Comprehensive Results on MNIST}\label{app:exp_results_mnist}
In this section, we present the entirety of our results on the MNIST dataset.
\subsubsection{Results on D-SHB}
We consider three Byzantine regimes: $f = 4$, $f = 6$, and $f=8$ (largest possible) out of $n=17$ workers in total. We also consider three heterogeneity regimes: $\alpha=0.1$ (extreme), $\alpha=1$ (moderate), and $\alpha=10$ (low).
We compare the performance of $\cenna{}$ and Bucketing when executed with four aggregation rules namely Krum~\cite{krum}, GM~\cite{small1990survey}, CWMed~\cite{yin2018}, and CWTM~\cite{yin2018}.
The plots are presented below and complement our (partial) results in Table~\ref{table:results_mnist} in Section~\ref{exp_results_mnist} of the main paper.
We note that in the two strongest Byzantine settings, i.e., when $f > 4$, Bucketing can only be applied with bucket size $s=1$~\cite{karimireddy2022byzantinerobust}. This means that when $f=6, 8$, executing Bucketing boils down to running the vanilla aggregation rules.

\begin{figure*}[ht!]
    \centering
    \includegraphics[width=0.5\textwidth]{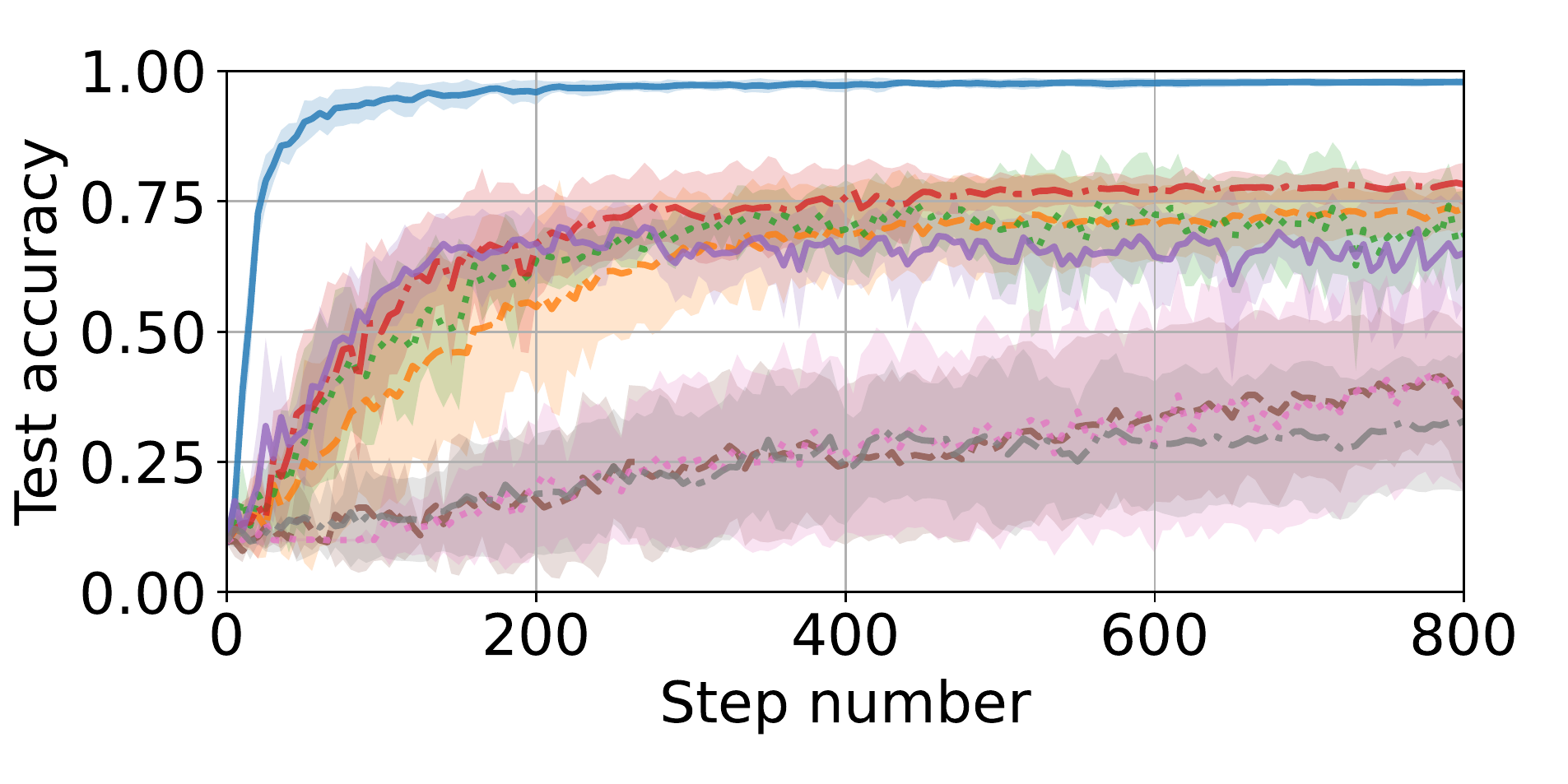}%
    \includegraphics[width=0.5\textwidth]{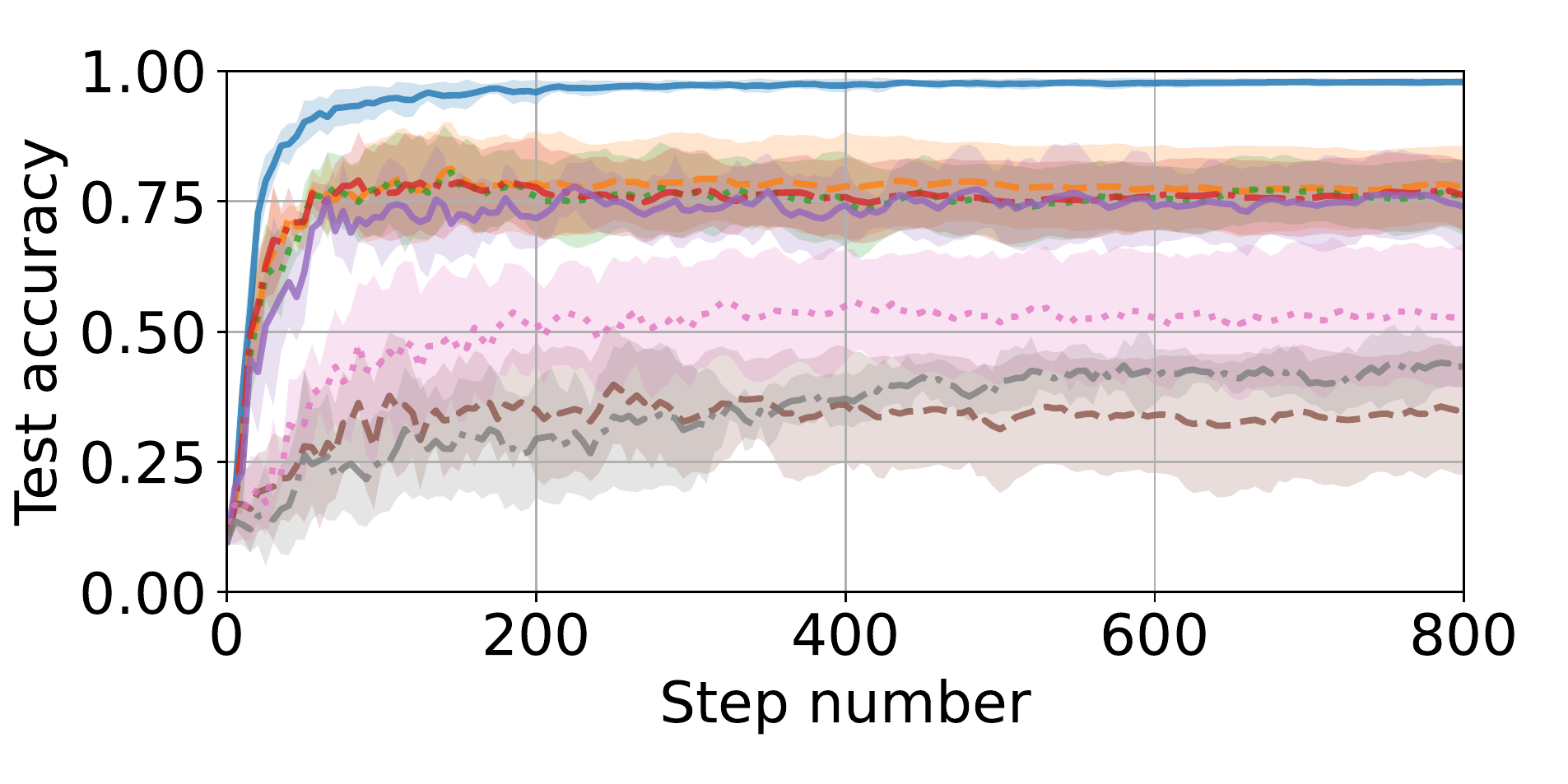}\\%
    \includegraphics[width=0.5\textwidth]{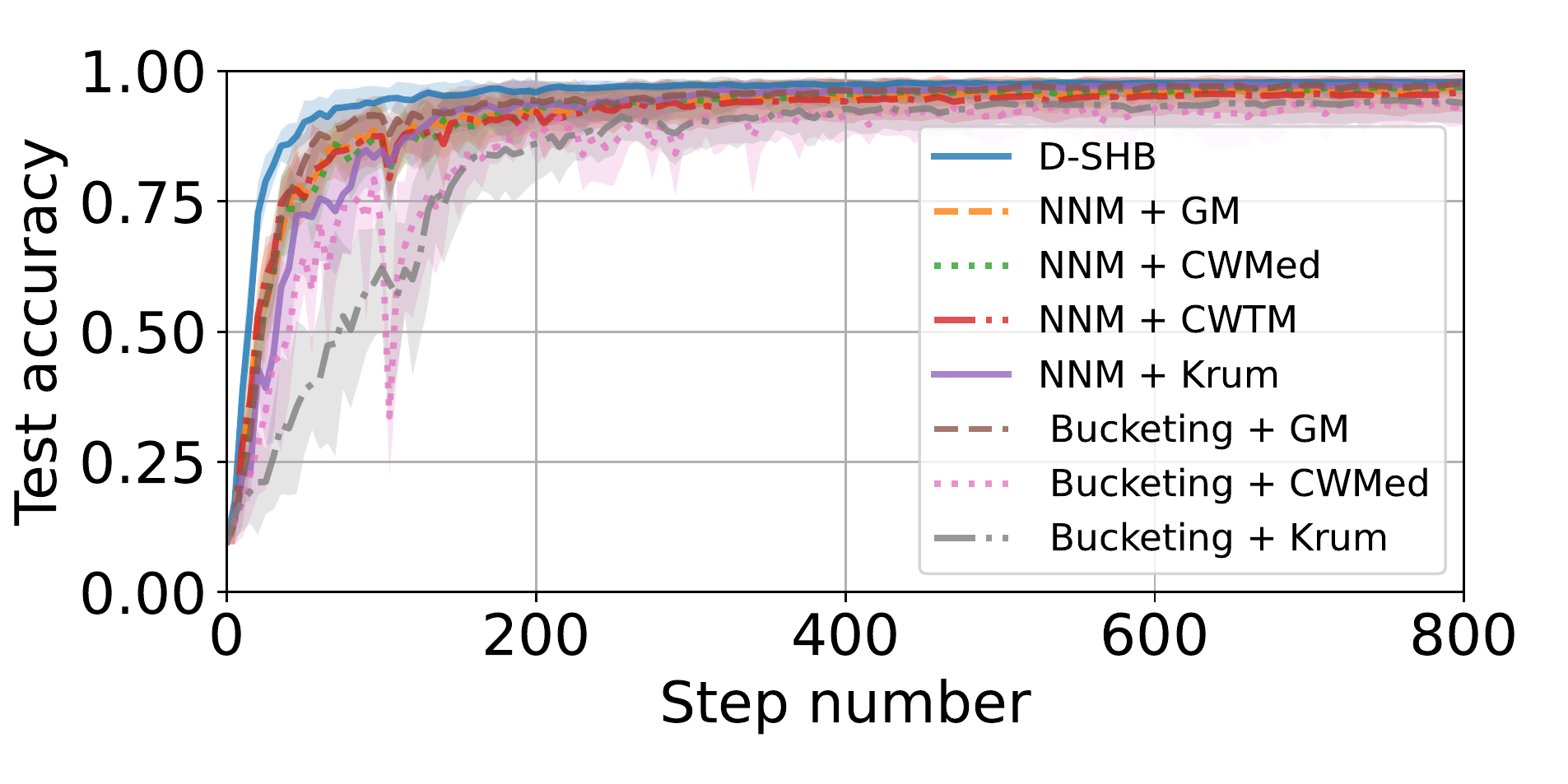}%
    \includegraphics[width=0.5\textwidth]{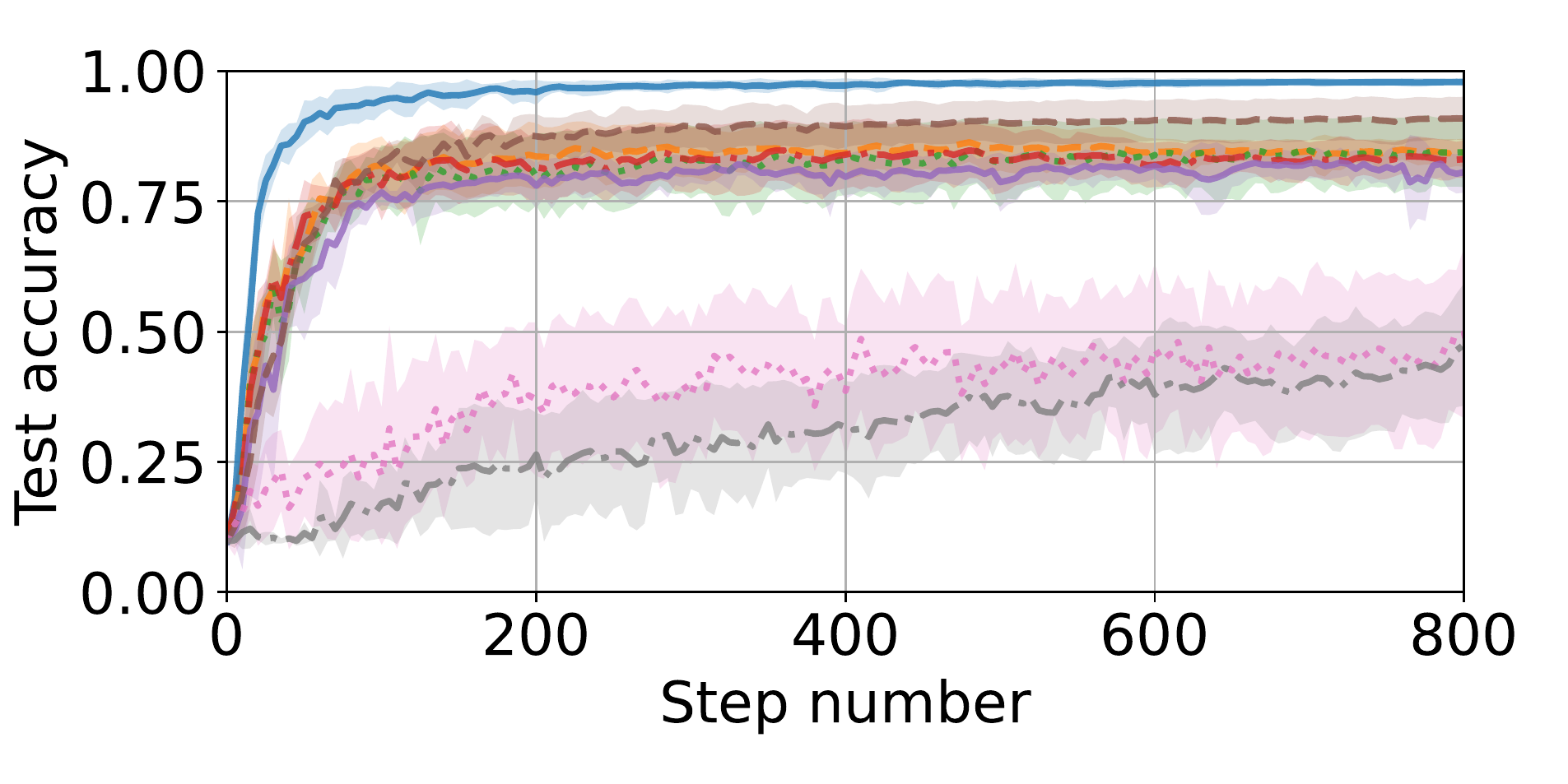}\\%
     \includegraphics[width=0.5\textwidth]{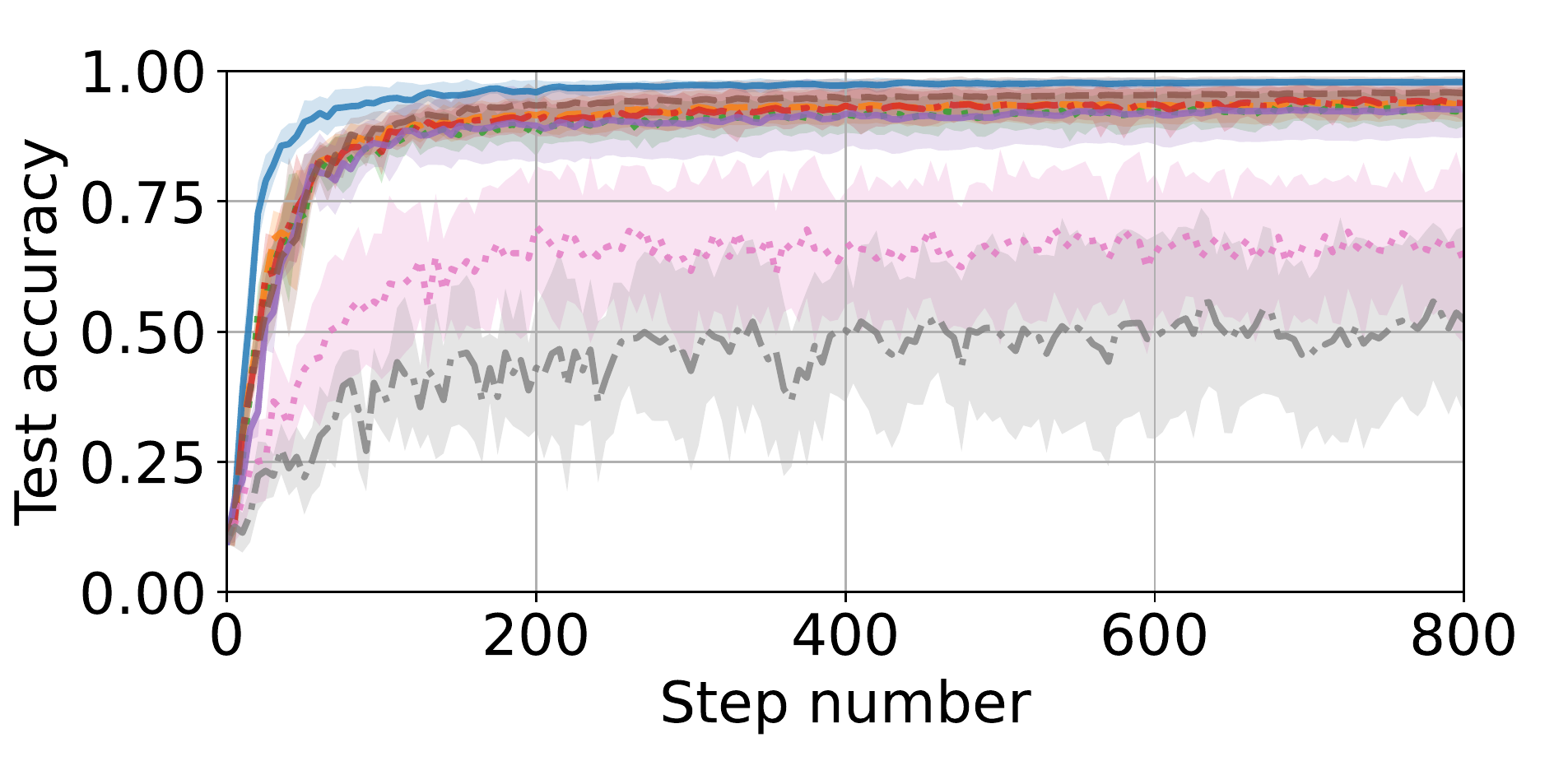}%
    \caption{Experiments on MNIST using robust D-SHB with $f = 4$ among $n = 17$ workers, with $\beta = 0.9$ and $\alpha = 0.1$. The Byzantine workers execute FOE (\textit{row 1, left}), ALIE (\textit{row 1, right}), Mimic (\textit{row 2, left}), SF (\textit{row 2, right}), and LF (\textit{row 3}).}
\label{fig:plots_mnist_1}
\end{figure*}

\begin{figure*}[ht!]
    \centering
    \includegraphics[width=0.5\textwidth]{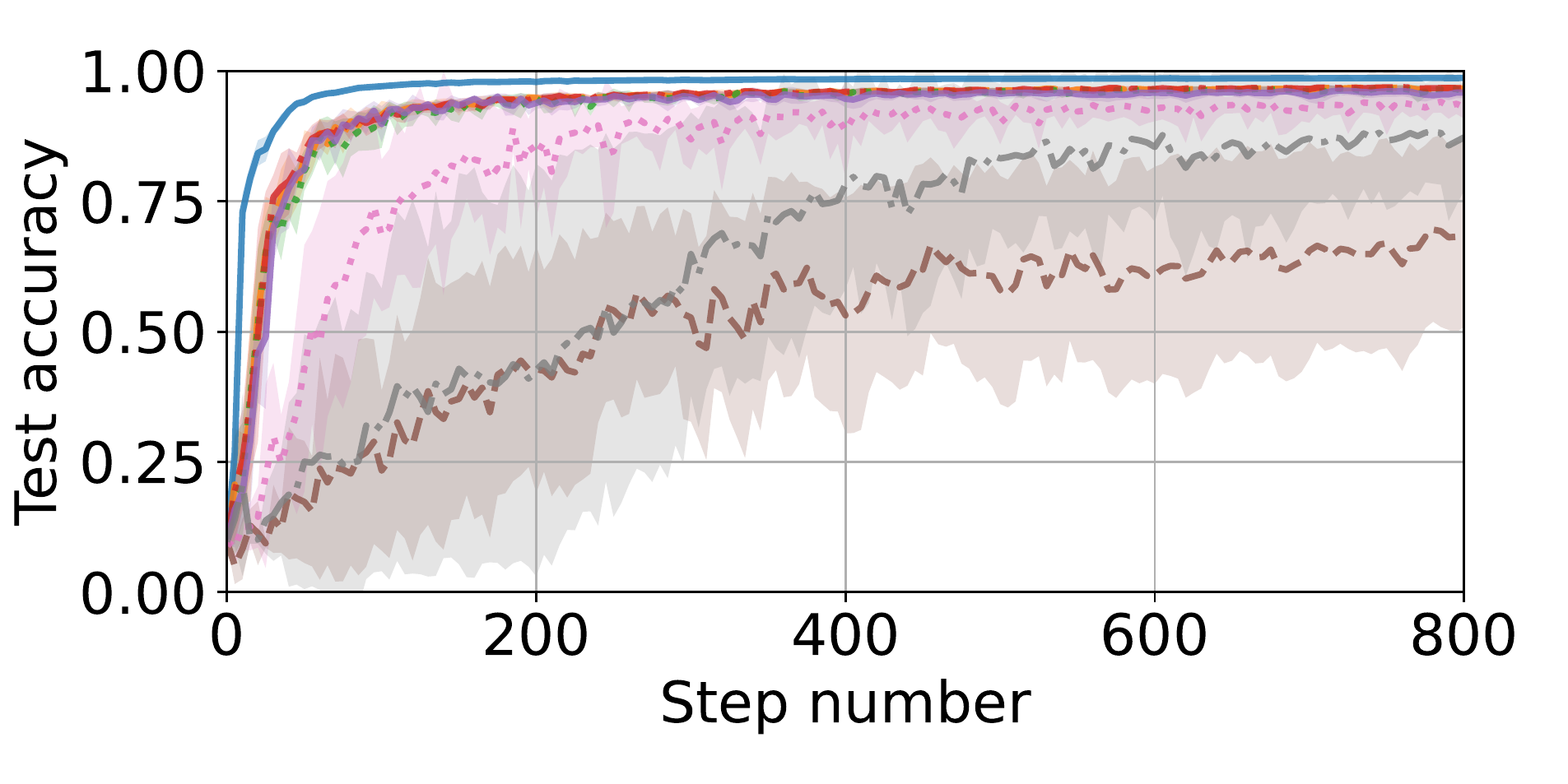}%
    \includegraphics[width=0.5\textwidth]{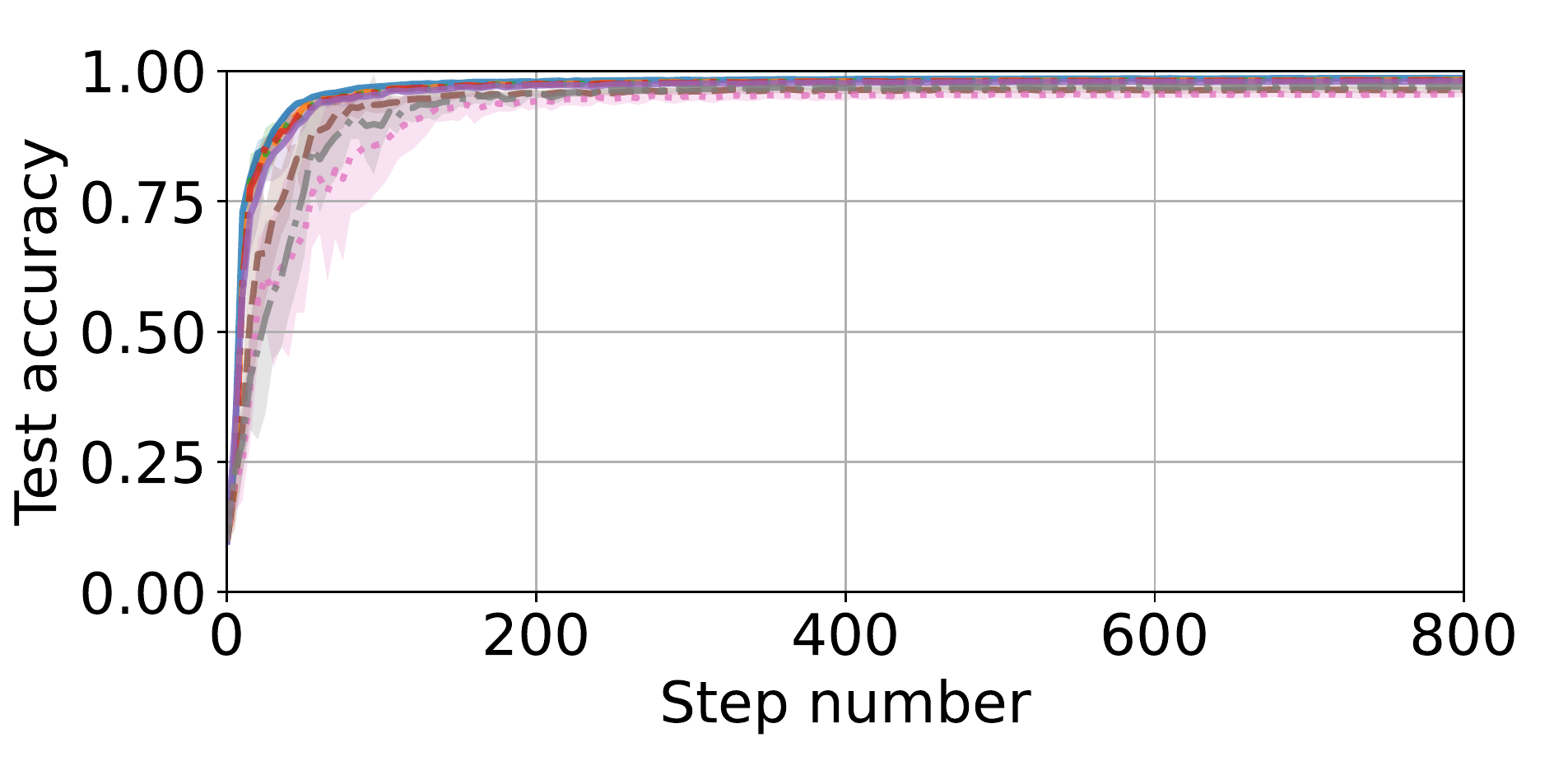}\\%
    \includegraphics[width=0.5\textwidth]{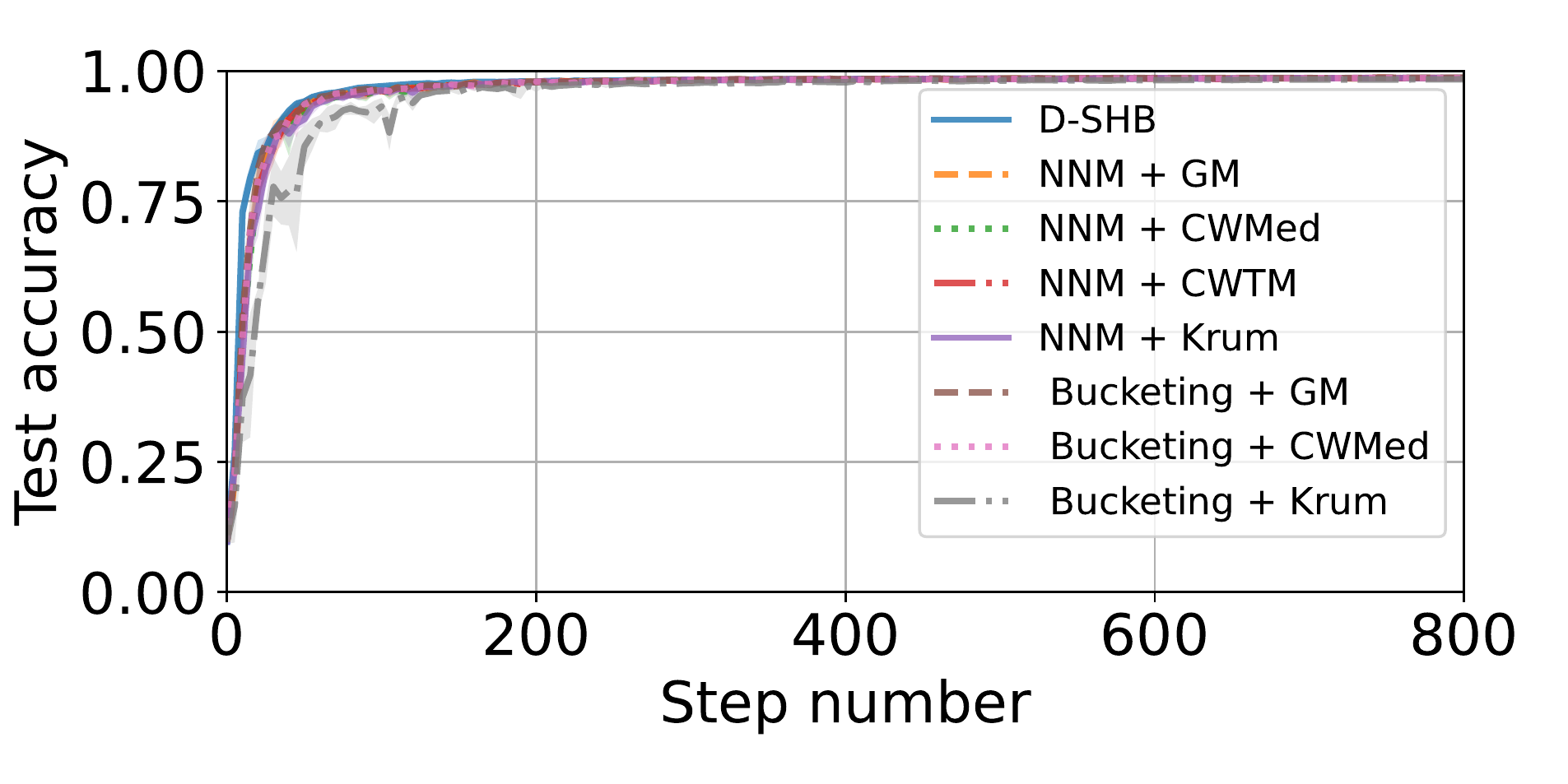}%
    \includegraphics[width=0.5\textwidth]{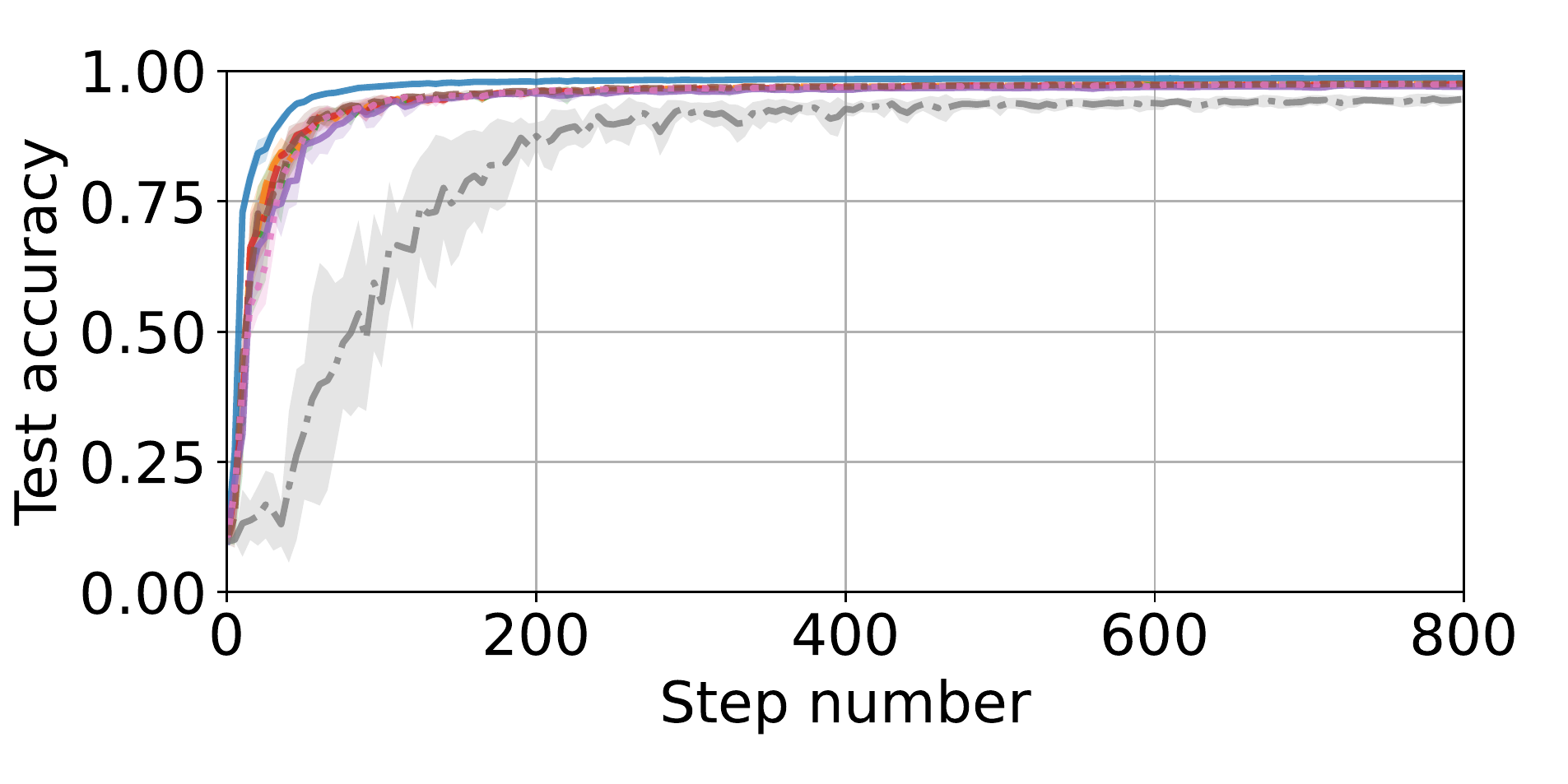}\\%
     \includegraphics[width=0.5\textwidth]{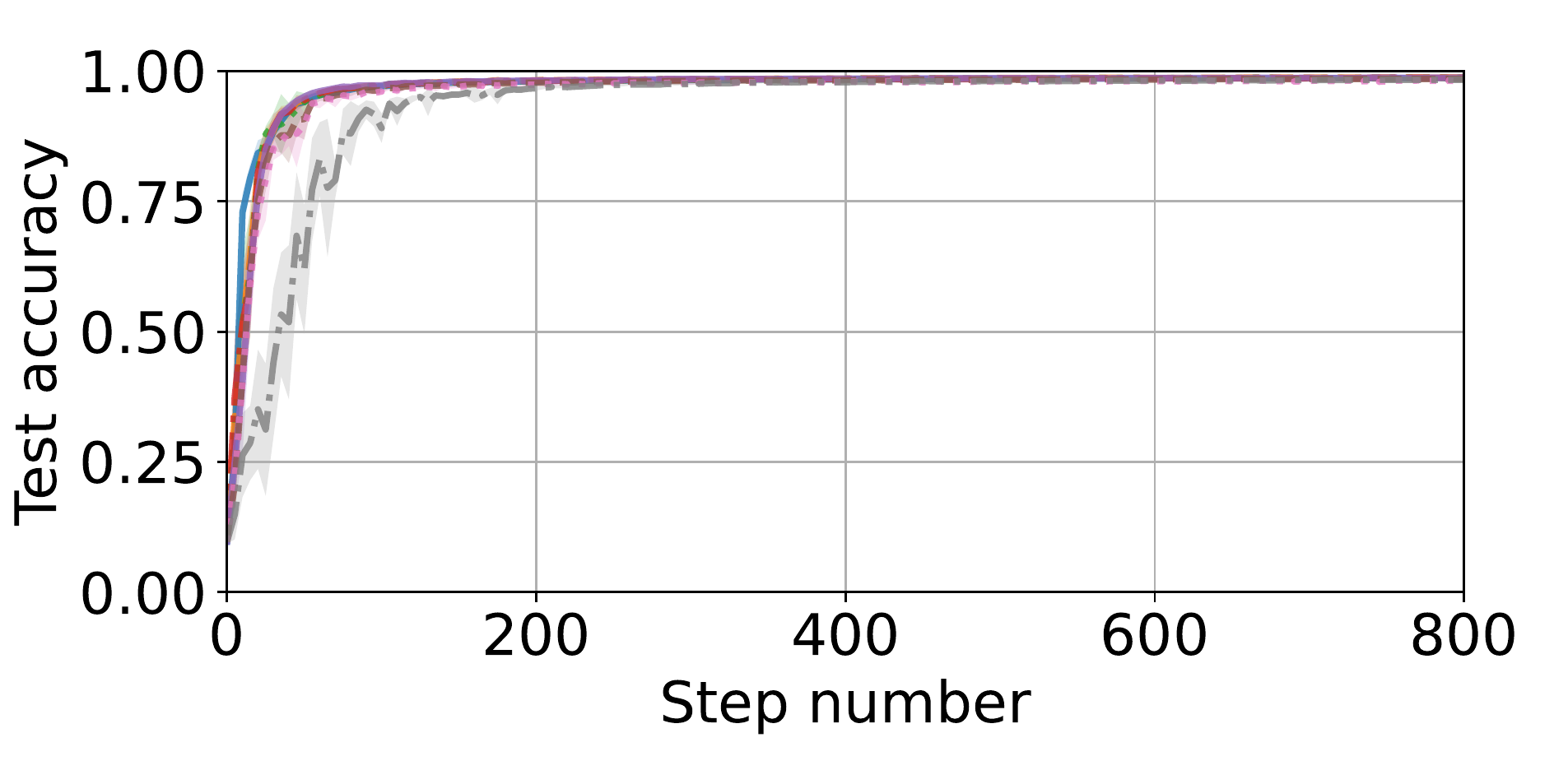}%
    \caption{Experiments on MNIST using robust D-SHB with $f = 4$ among $n = 17$ workers, with $\beta = 0.9$ and $\alpha = 1$. The Byzantine workers execute FOE (\textit{row 1, left}), ALIE (\textit{row 1, right}), Mimic (\textit{row 2, left}), SF (\textit{row 2, right}), and LF (\textit{row 3}).}
\label{fig:plots_mnist_2}
\end{figure*}

\begin{figure*}[ht!]
    \centering
    \includegraphics[width=0.5\textwidth]{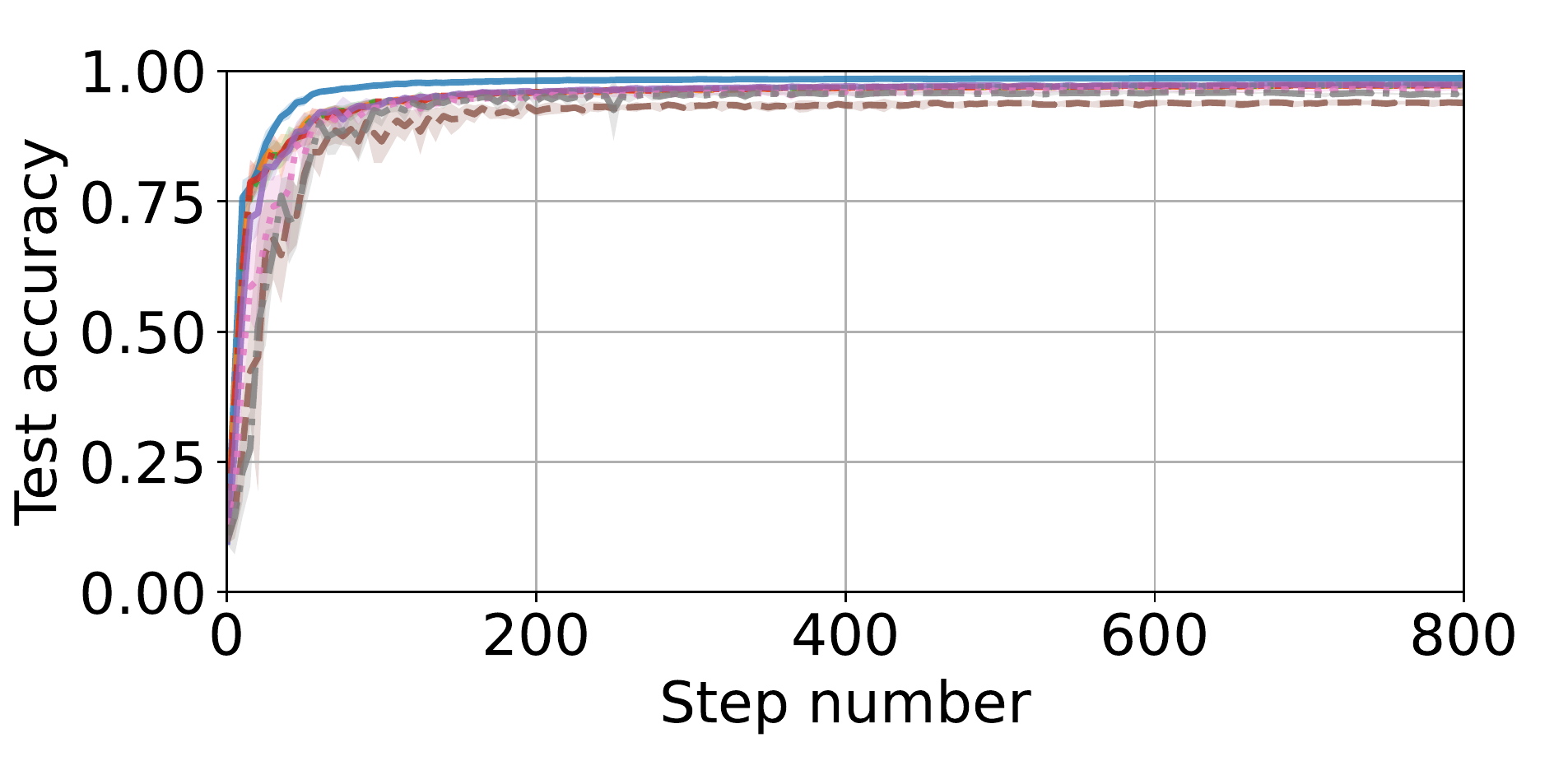}%
    \includegraphics[width=0.5\textwidth]{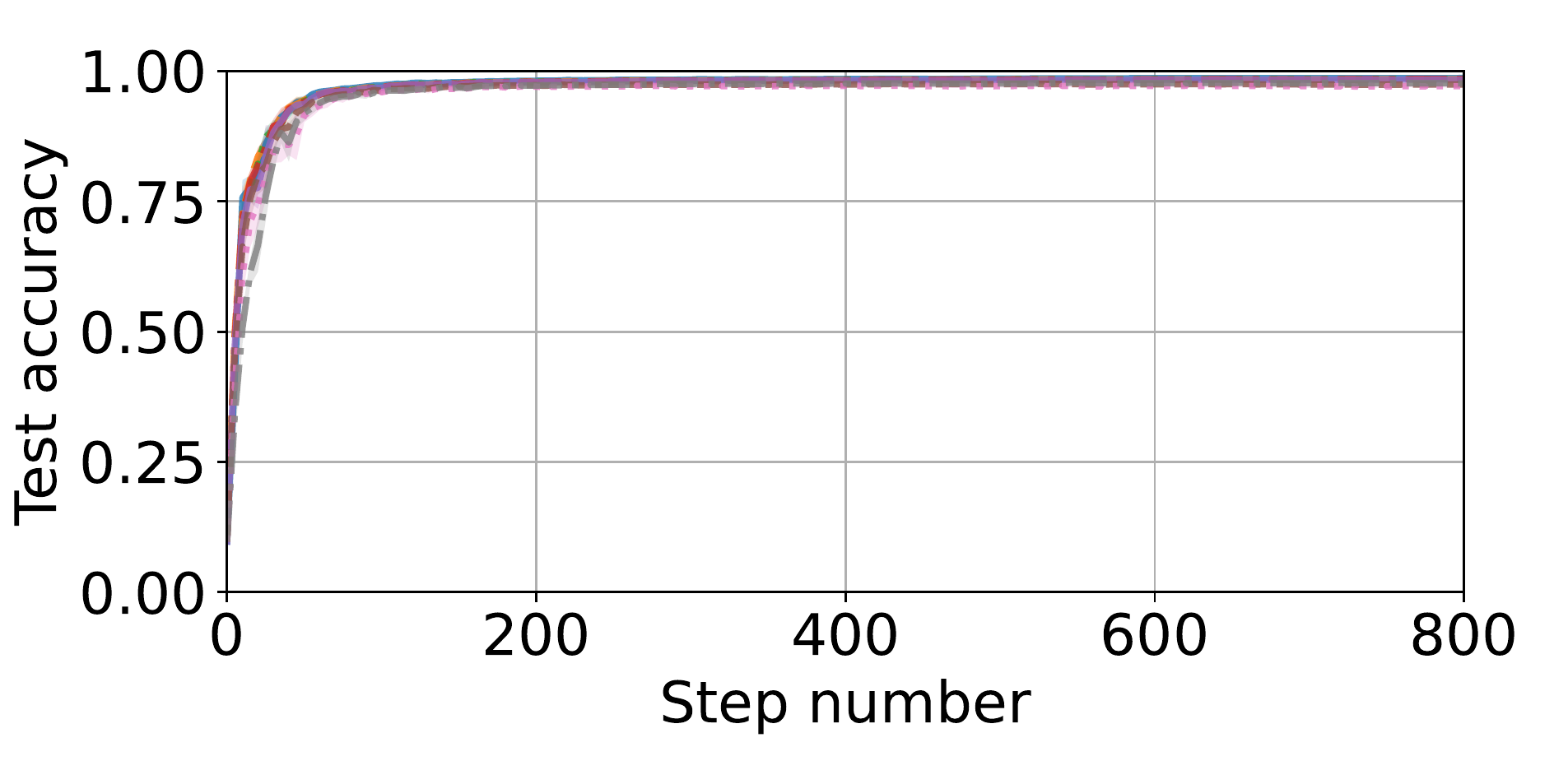}\\%
    \includegraphics[width=0.5\textwidth]{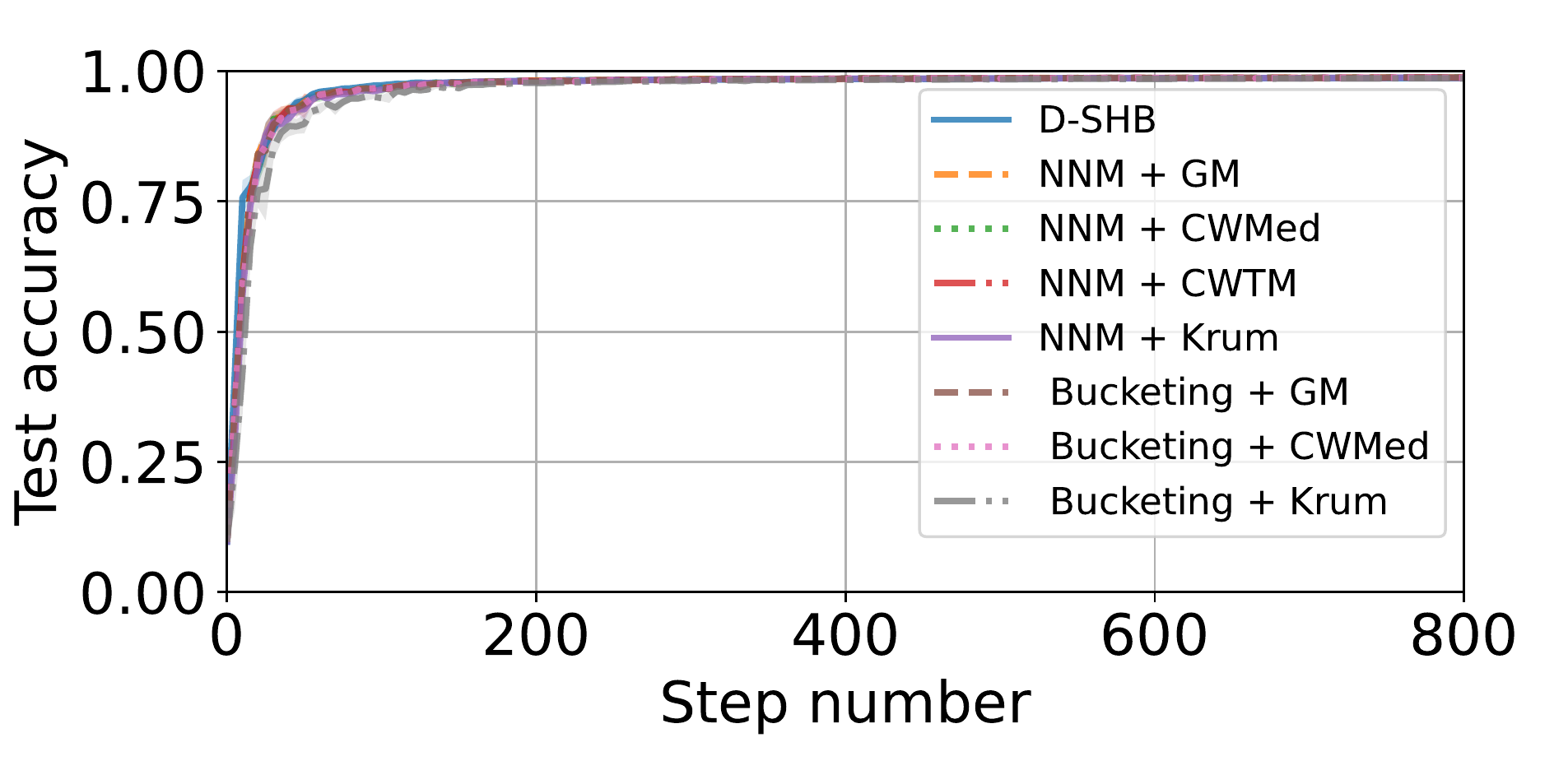}%
    \includegraphics[width=0.5\textwidth]{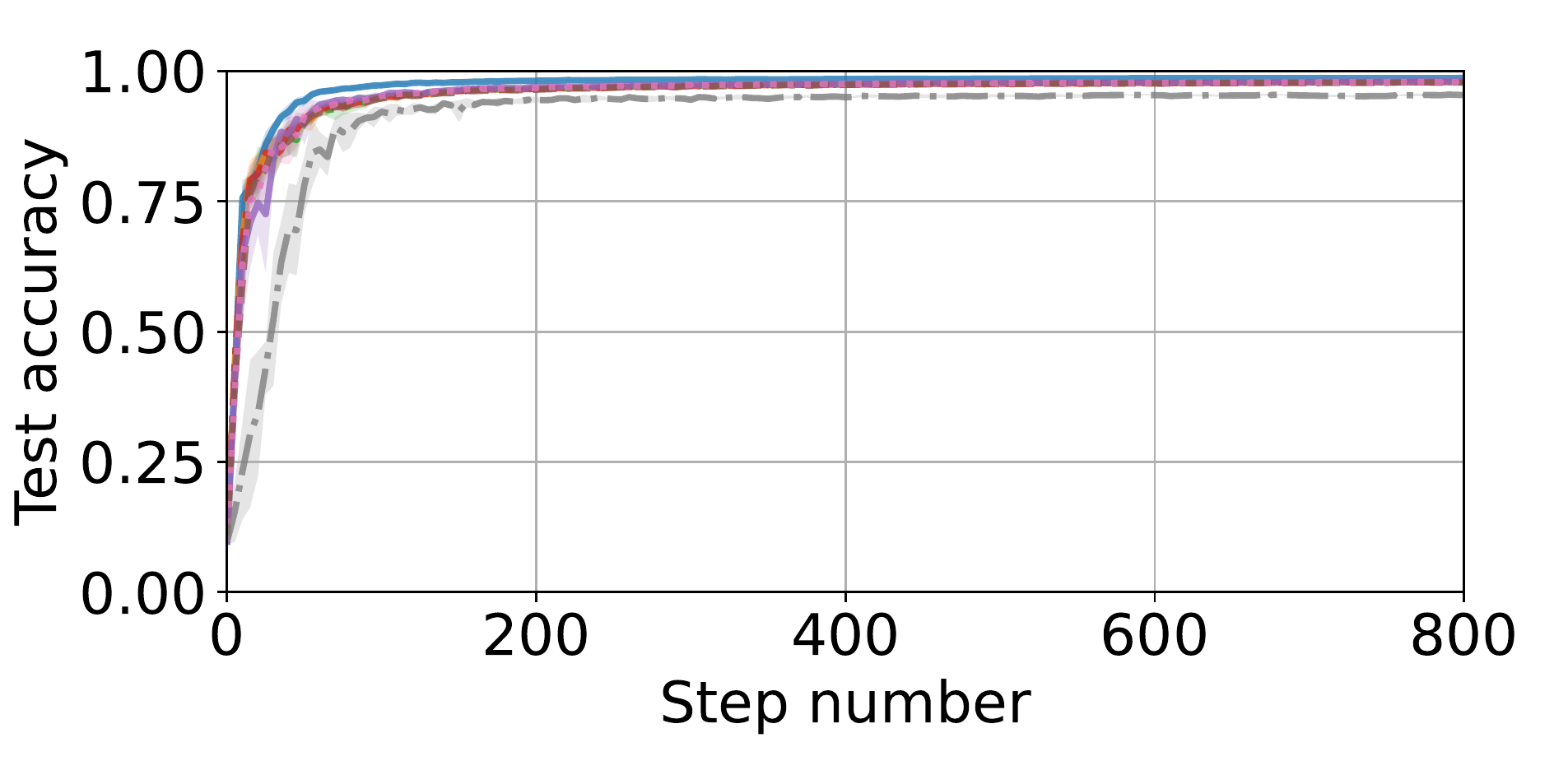}\\%
     \includegraphics[width=0.5\textwidth]{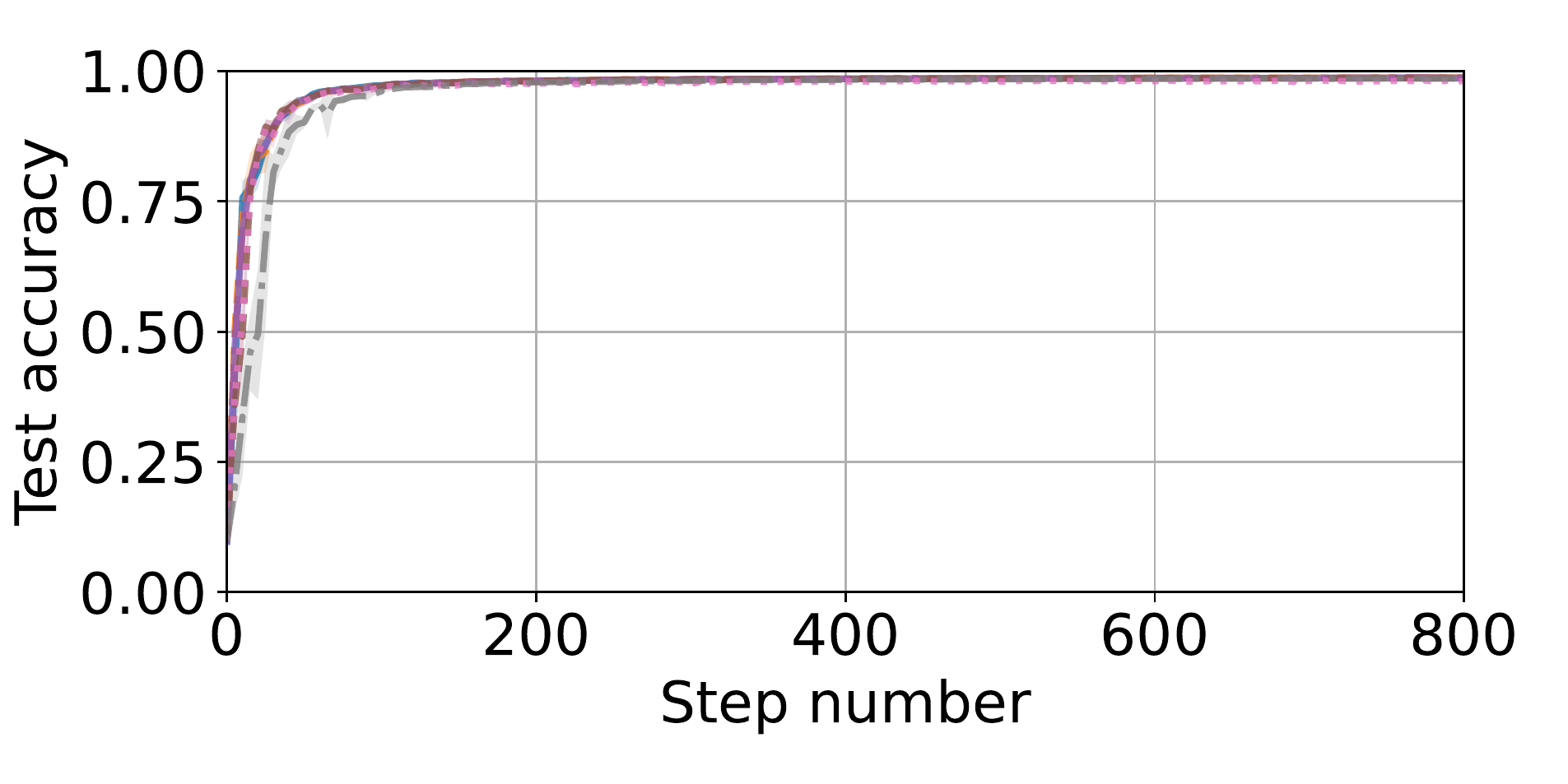}%
    \caption{Experiments on MNIST using robust D-SHB with $f = 4$ among $n = 17$ workers, with $\beta = 0.9$ and $\alpha = 10$. The Byzantine workers execute FOE (\textit{row 1, left}), ALIE (\textit{row 1, right}), Mimic (\textit{row 2, left}), SF (\textit{row 2, right}), and LF (\textit{row 3}).}
\label{fig:plots_mnist_3}
\end{figure*}

\begin{figure*}[ht!]
    \centering
    \includegraphics[width=0.5\textwidth]{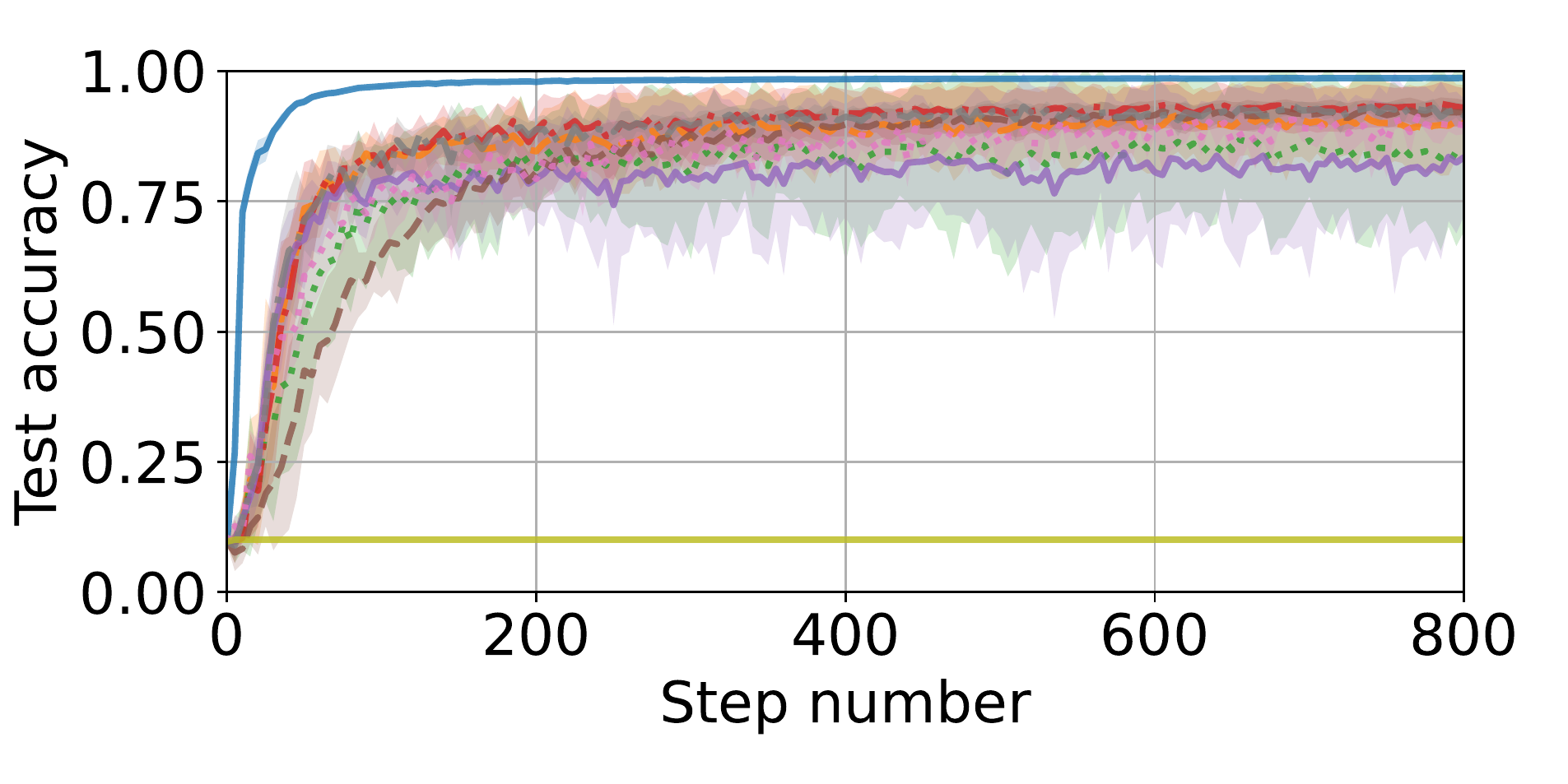}%
    \includegraphics[width=0.5\textwidth]{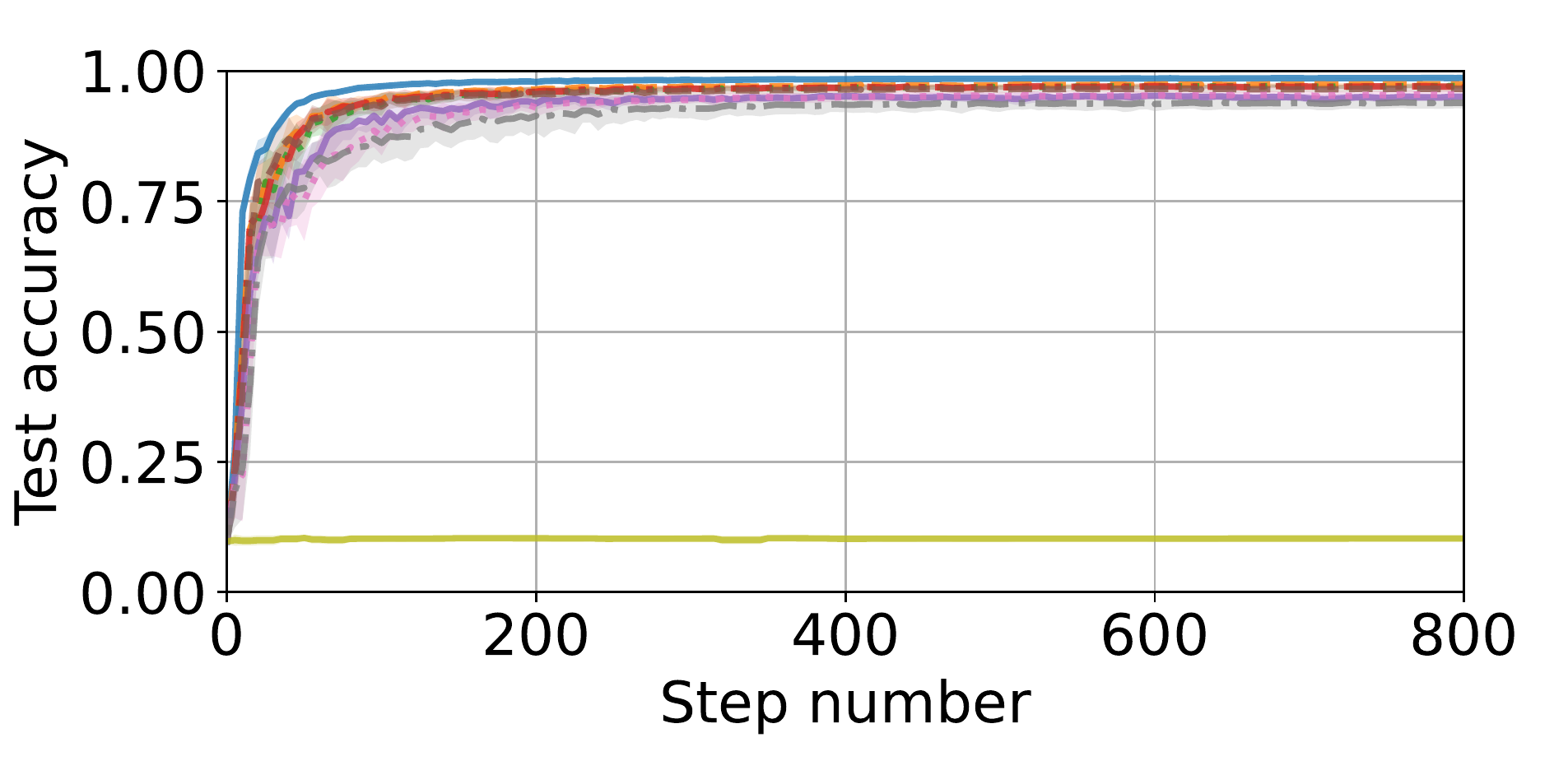}\\%
    \includegraphics[width=0.5\textwidth]{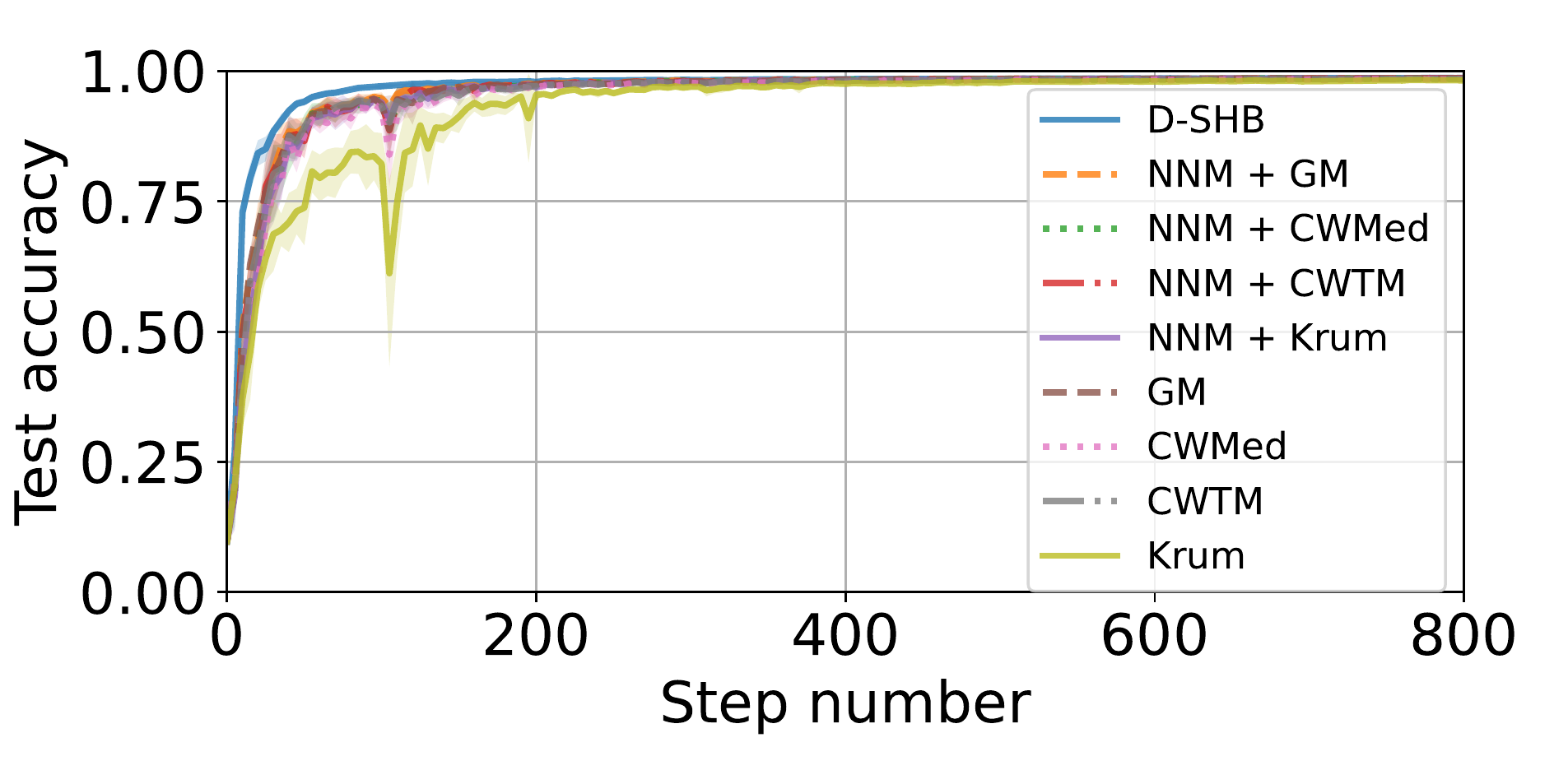}%
    \includegraphics[width=0.5\textwidth]{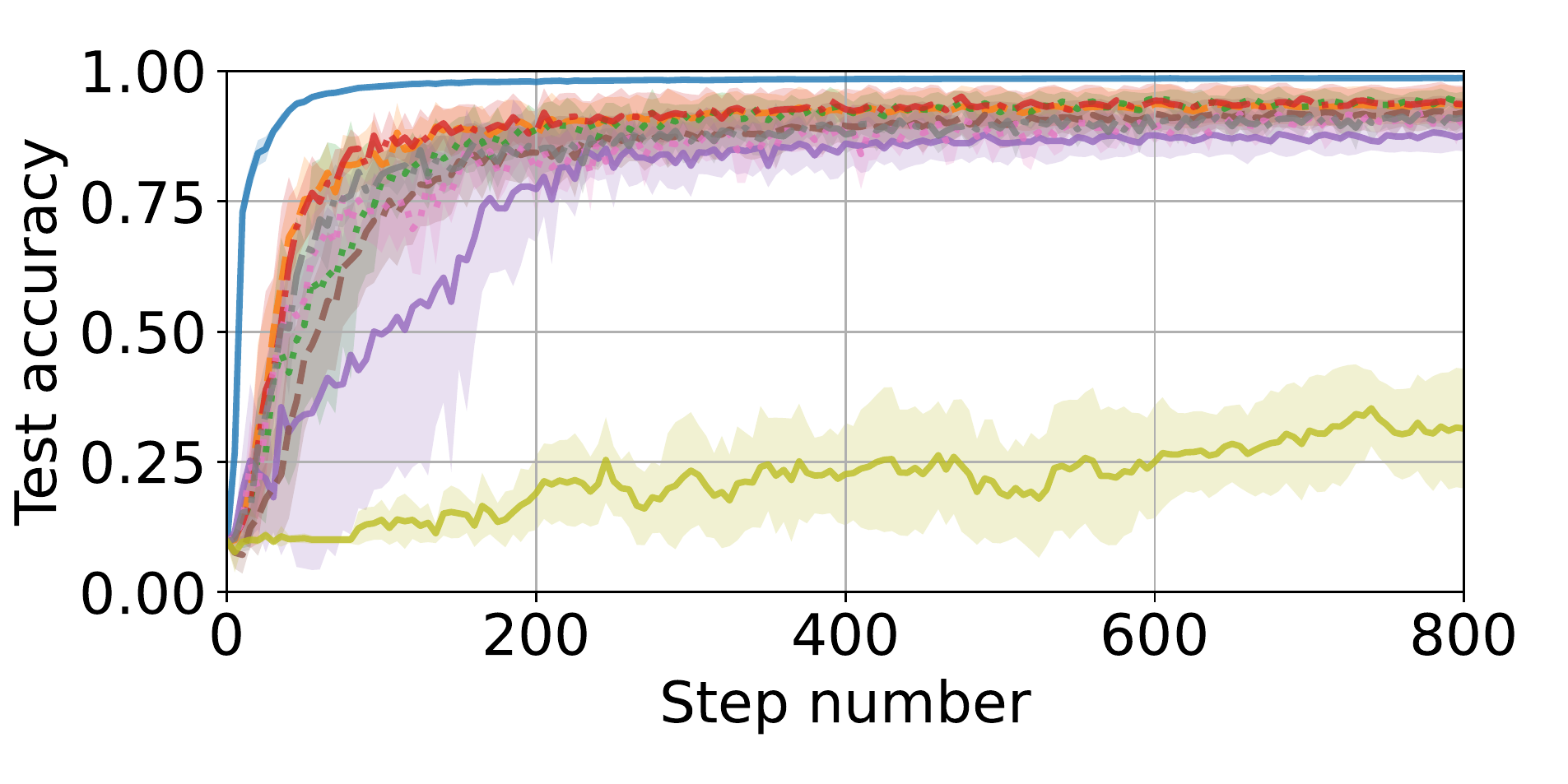}\\%
    \includegraphics[width=0.5\textwidth]{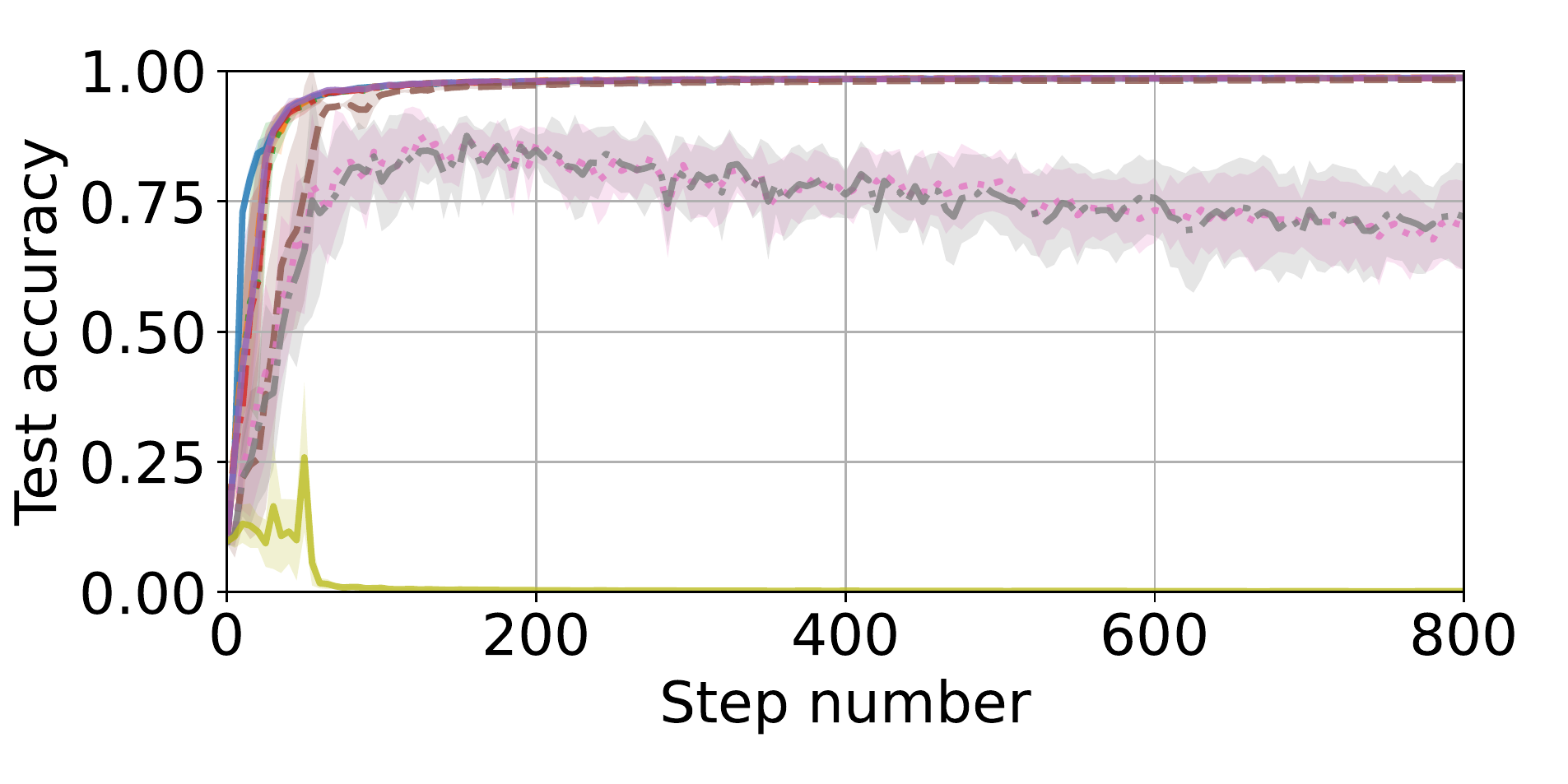}%
    \caption{Experiments on MNIST using robust D-SHB with $f = 6$ among $n = 17$ workers, with $\beta = 0.9$ and $\alpha = 1$. The Byzantine workers execute FOE (\textit{row 1, left}), ALIE (\textit{row 1, right}), Mimic (\textit{row 2, left}), SF (\textit{row 2, right}), and LF (\textit{row 3}).}
\label{fig:plots_mnist_4}
\end{figure*}

\begin{figure*}[ht!]
    \centering
    \includegraphics[width=0.5\textwidth]{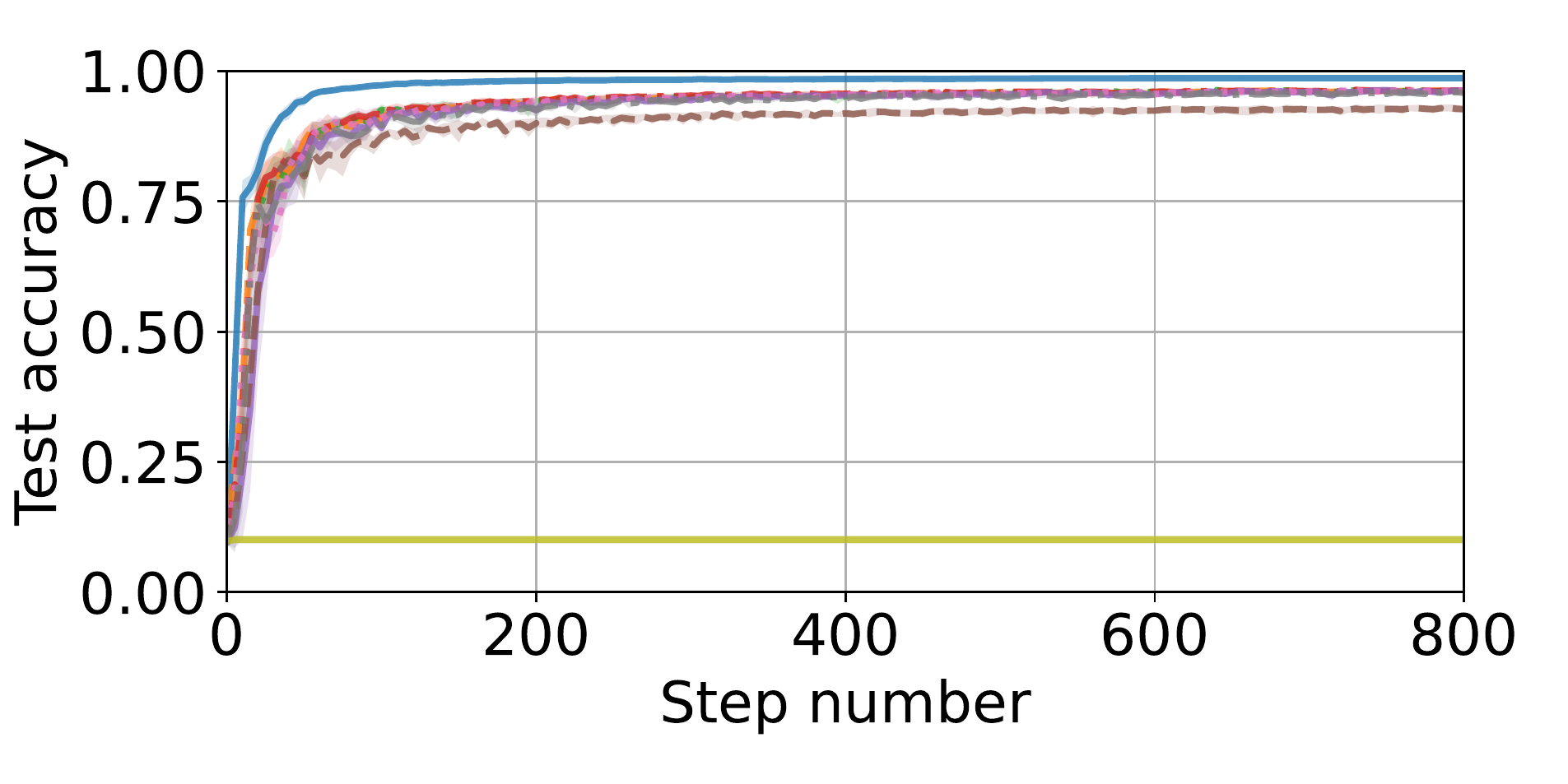}%
    \includegraphics[width=0.5\textwidth]{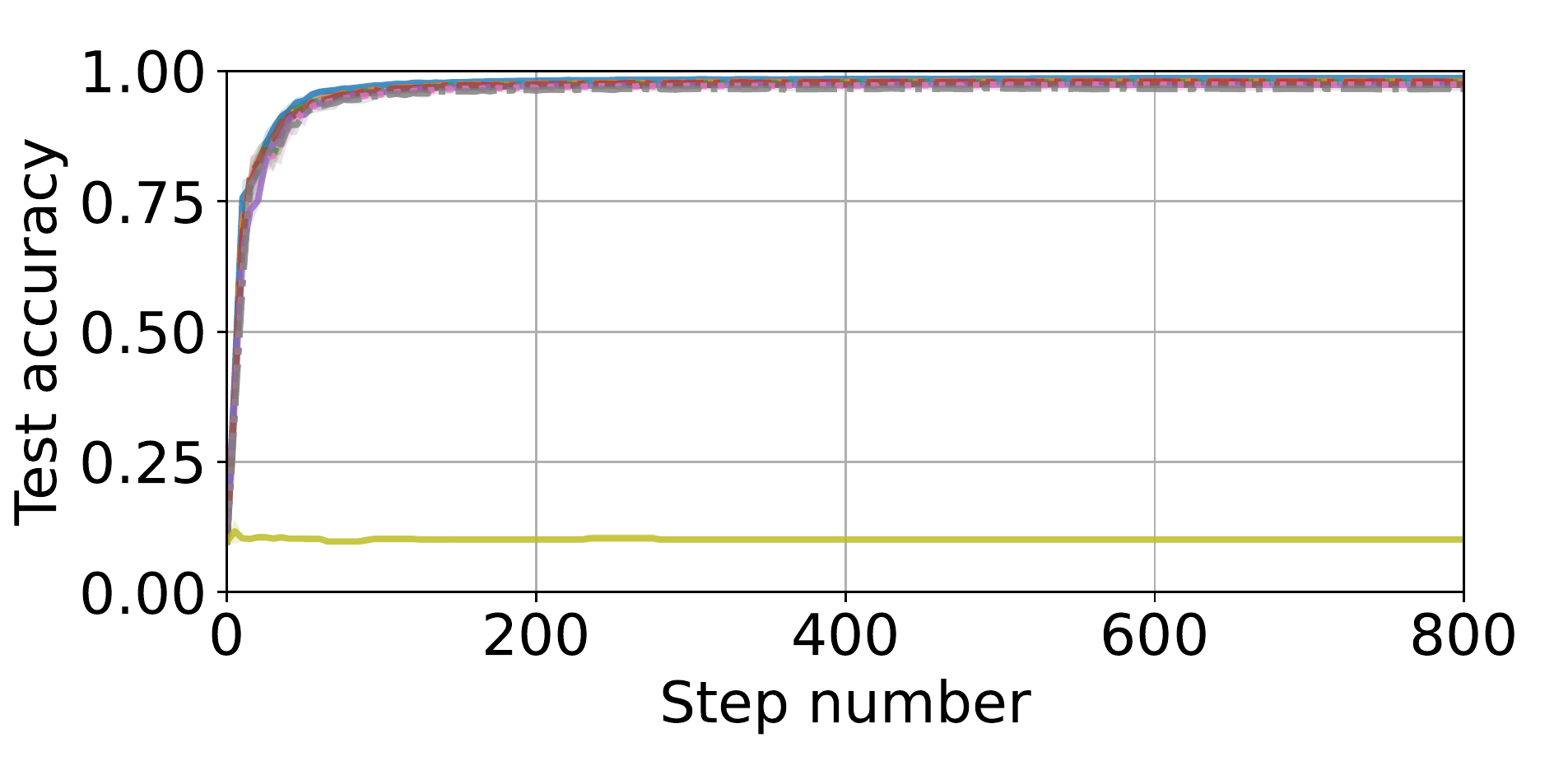}\\%
    \includegraphics[width=0.5\textwidth]{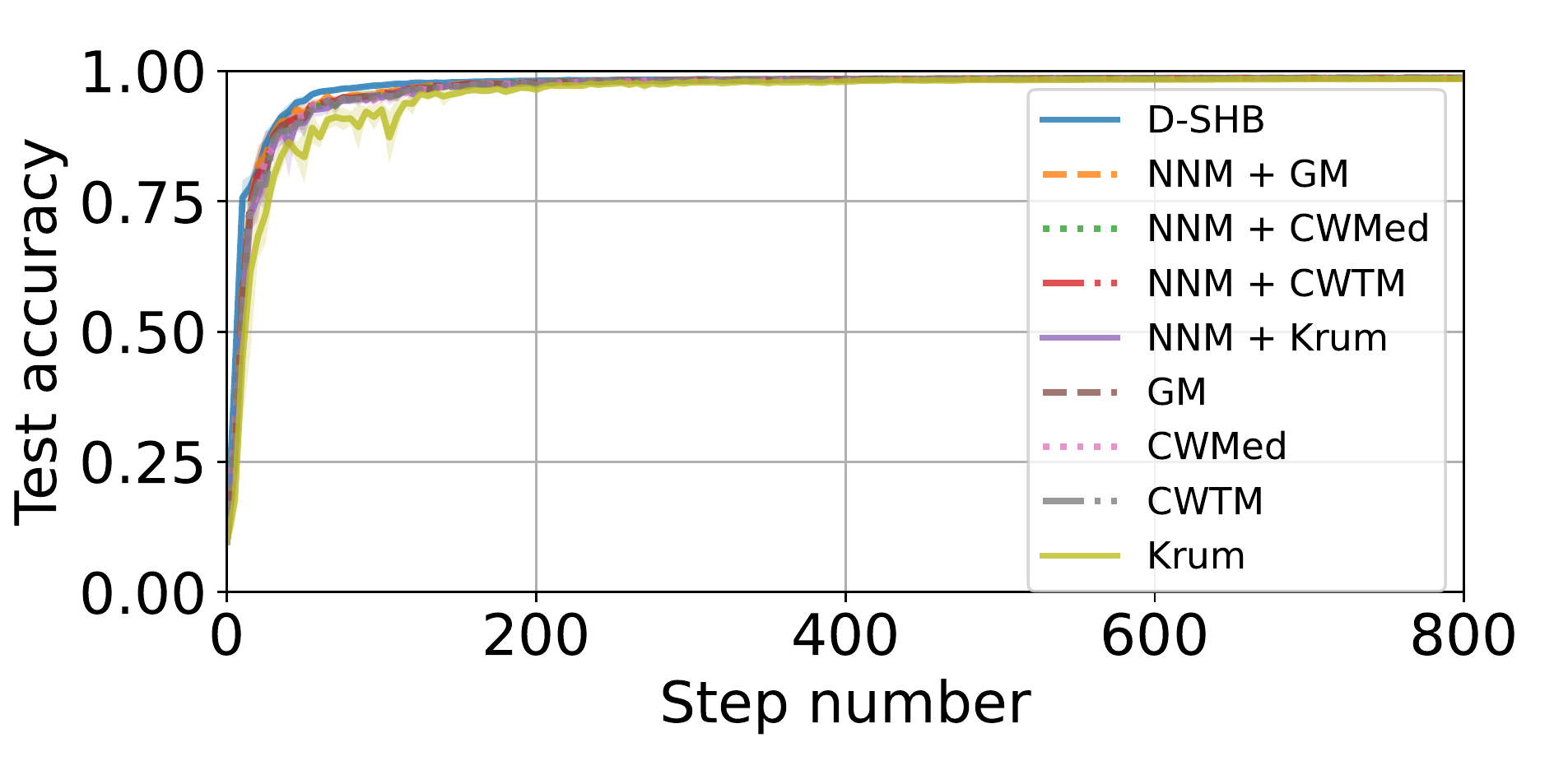}%
    \includegraphics[width=0.5\textwidth]{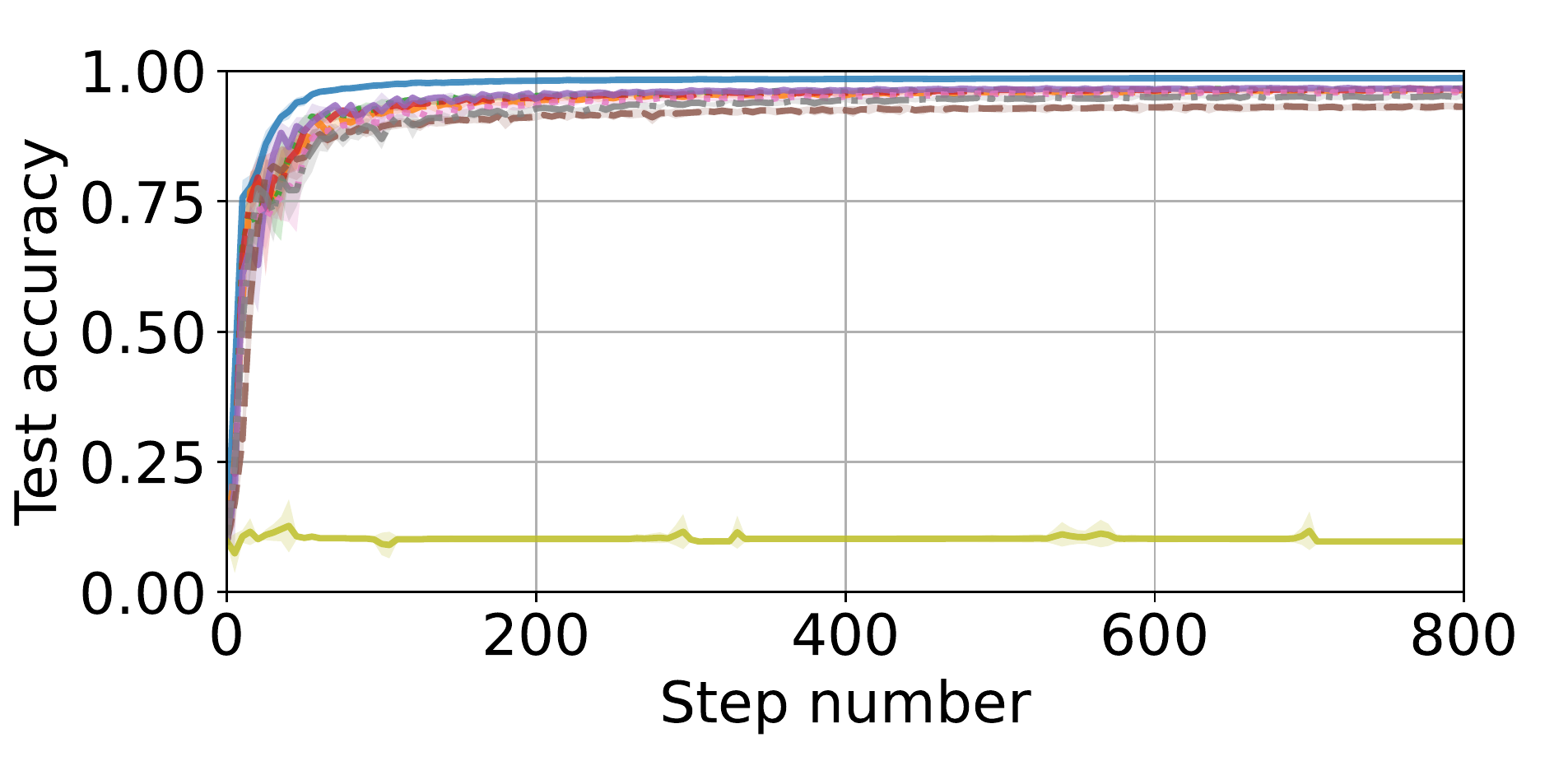}\\%
     \includegraphics[width=0.5\textwidth]{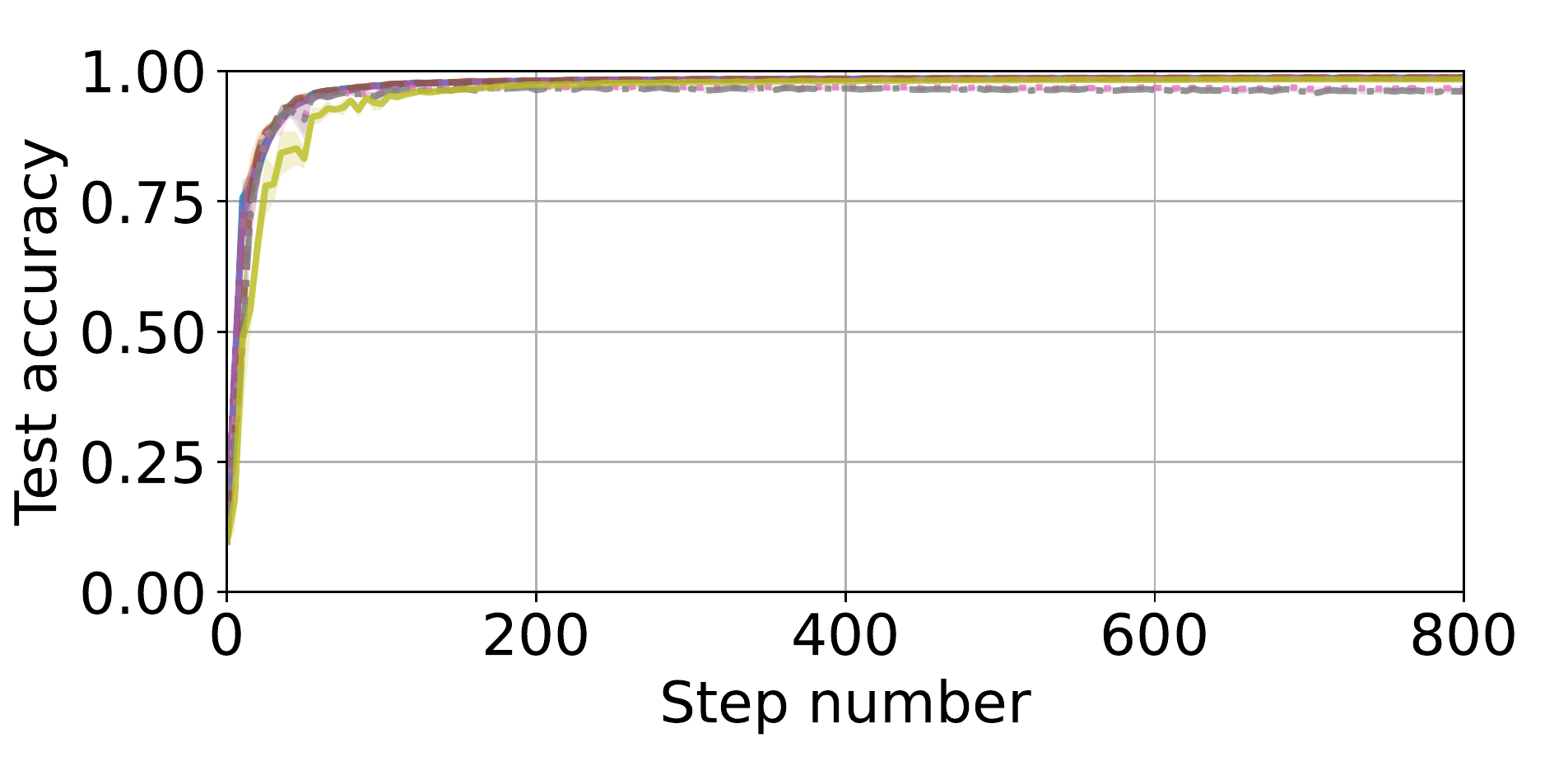}%
    \caption{Experiments on MNIST using robust D-SHB with $f = 6$ among $n = 17$ workers, with $\beta = 0.9$ and $\alpha = 10$. The Byzantine workers execute FOE (\textit{row 1, left}), ALIE (\textit{row 1, right}), Mimic (\textit{row 2, left}), SF (\textit{row 2, right}), and LF (\textit{row 3}).}
\label{fig:plots_mnist_5}
\end{figure*}

\begin{figure*}[ht!]
    \centering
    \includegraphics[width=0.5\textwidth]{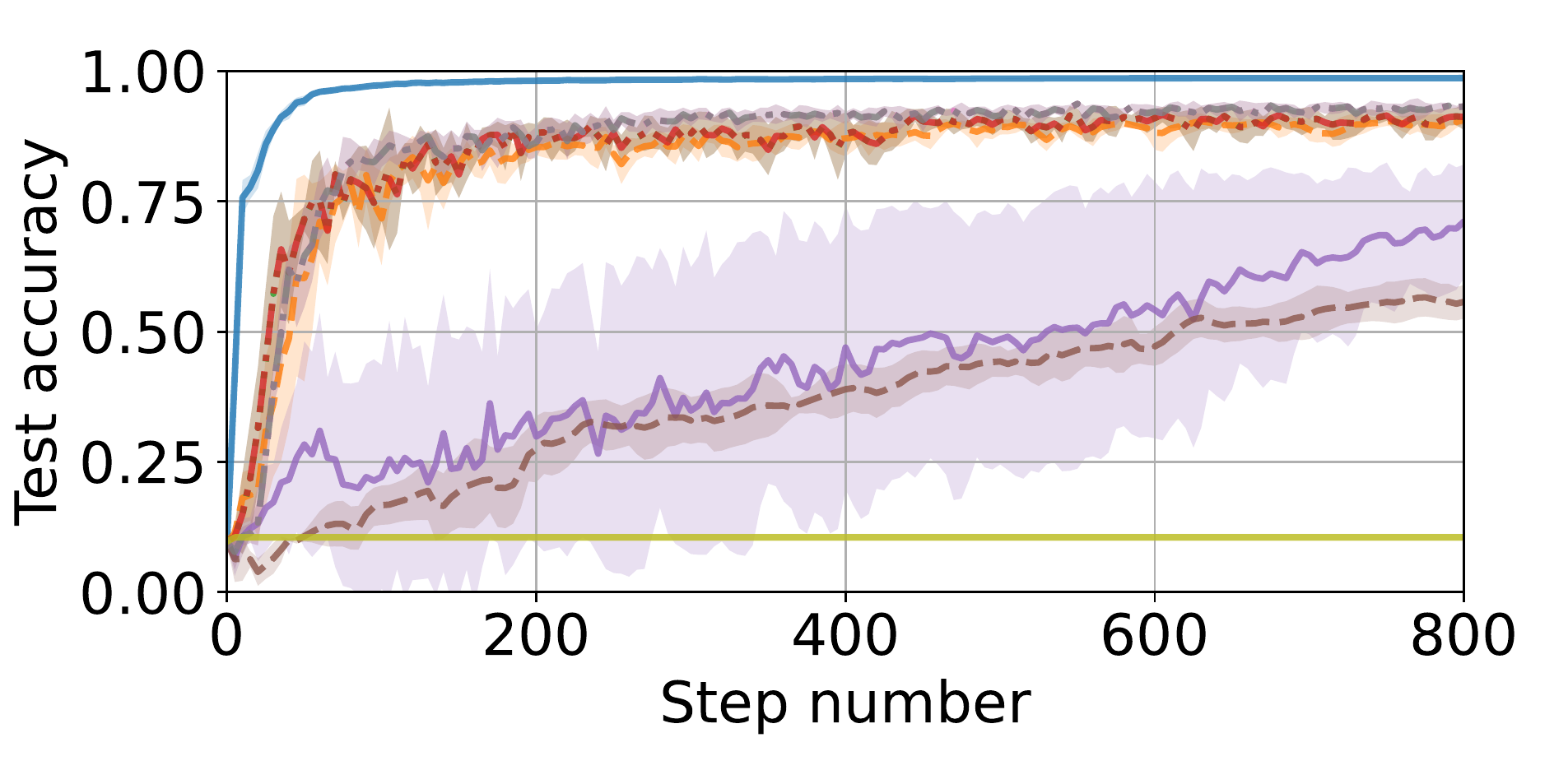}%
    \includegraphics[width=0.5\textwidth]{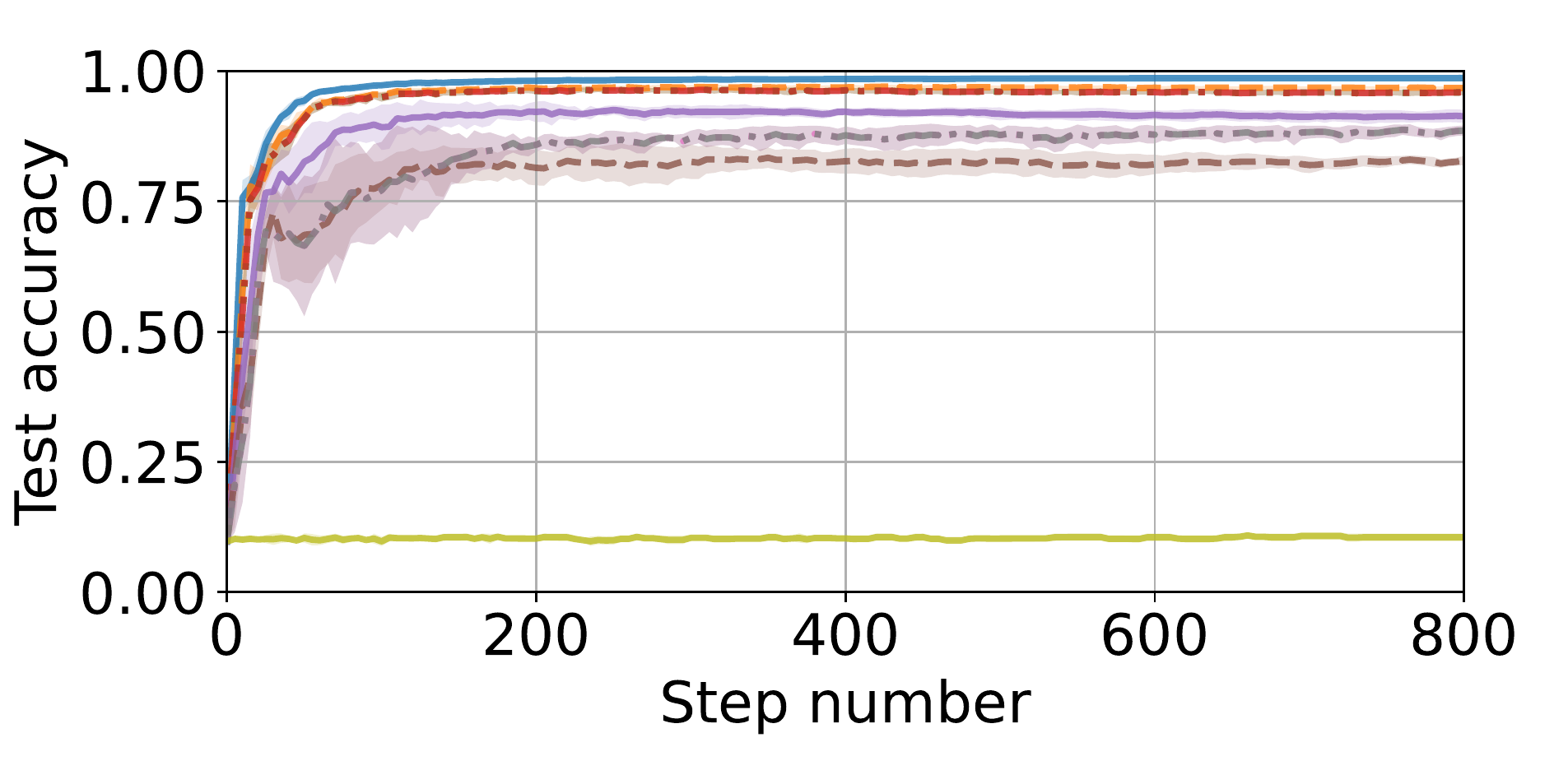}\\%
    \includegraphics[width=0.5\textwidth]{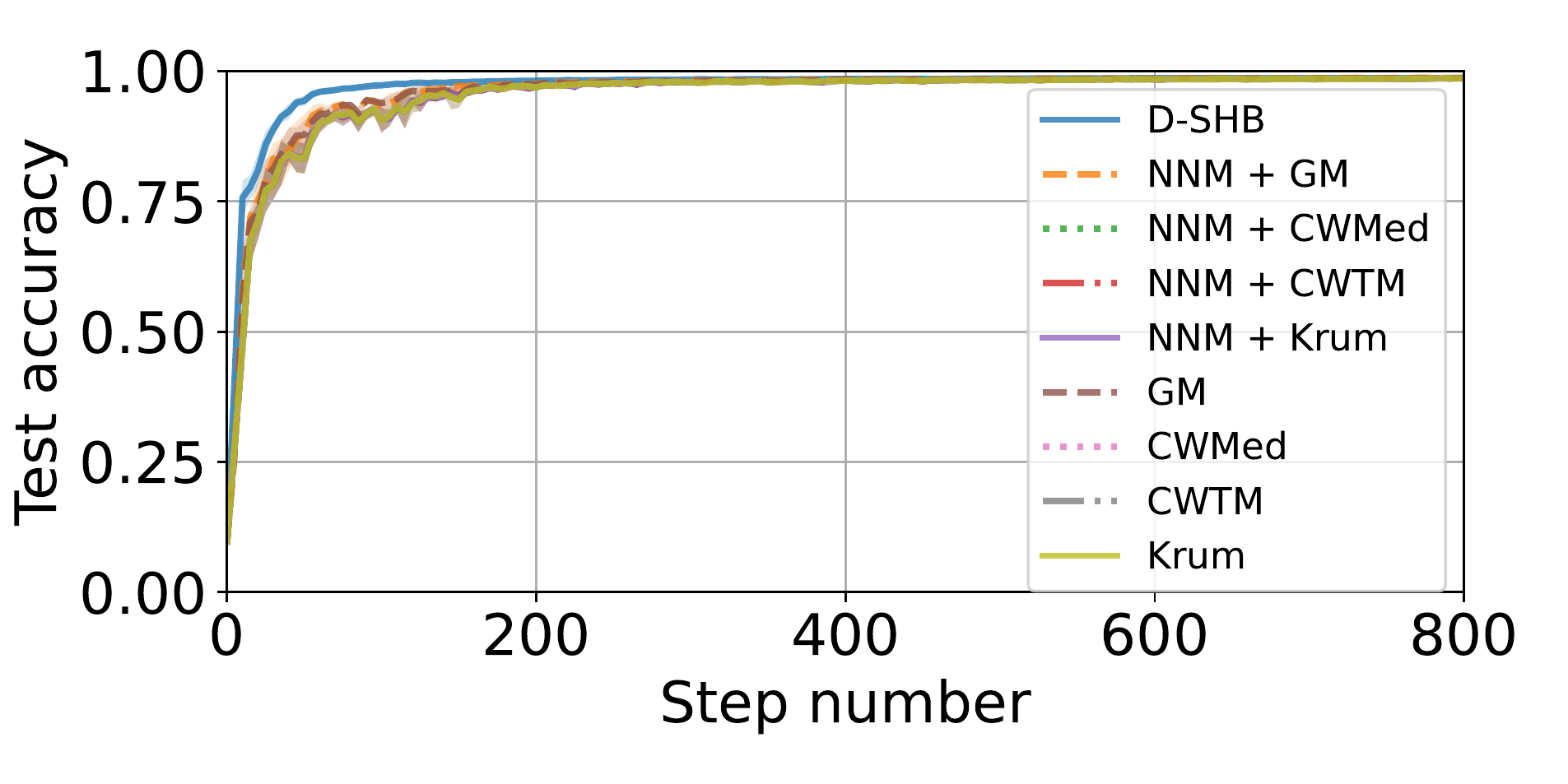}%
    \includegraphics[width=0.5\textwidth]{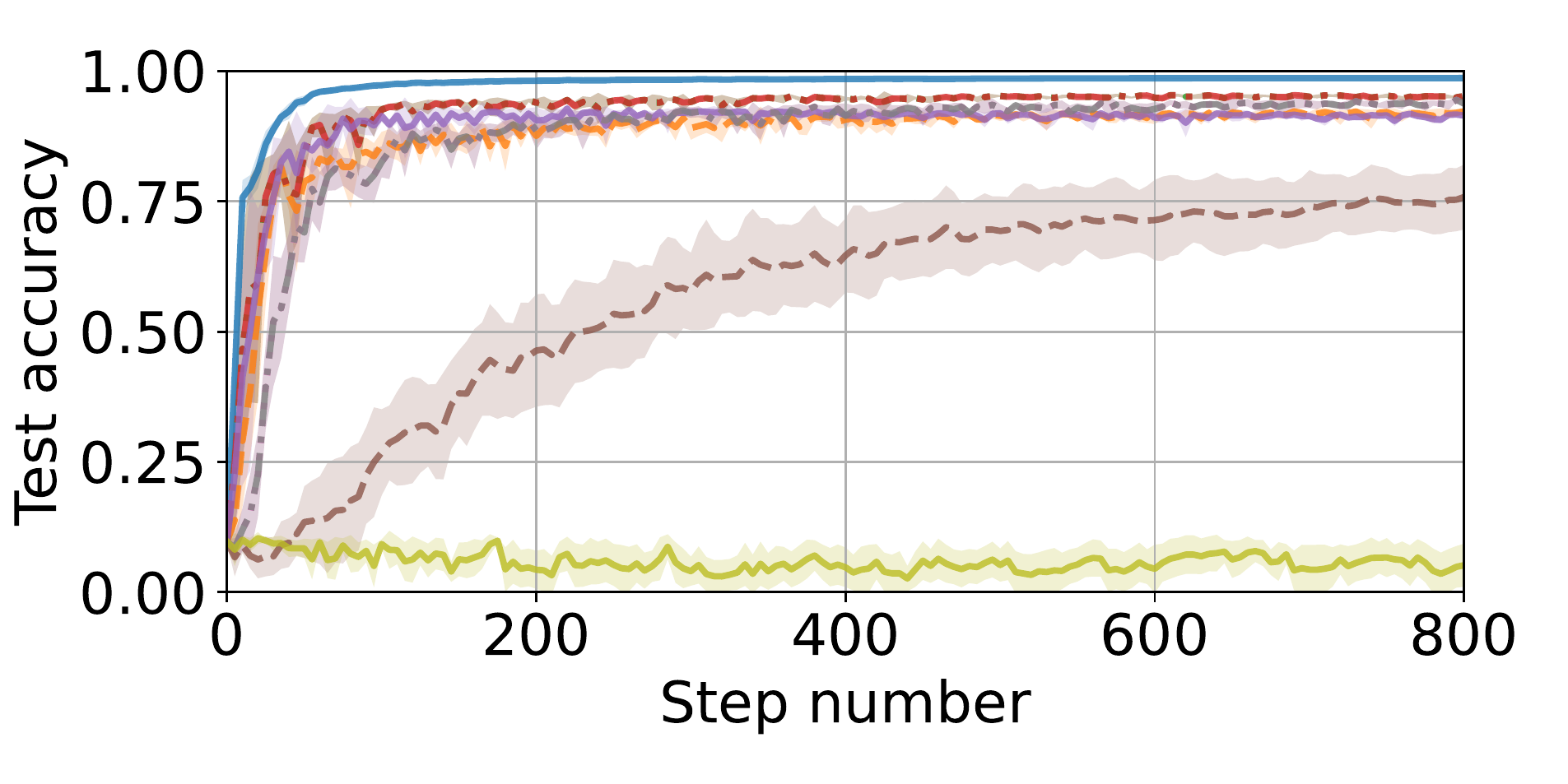}\\%
     \includegraphics[width=0.5\textwidth]{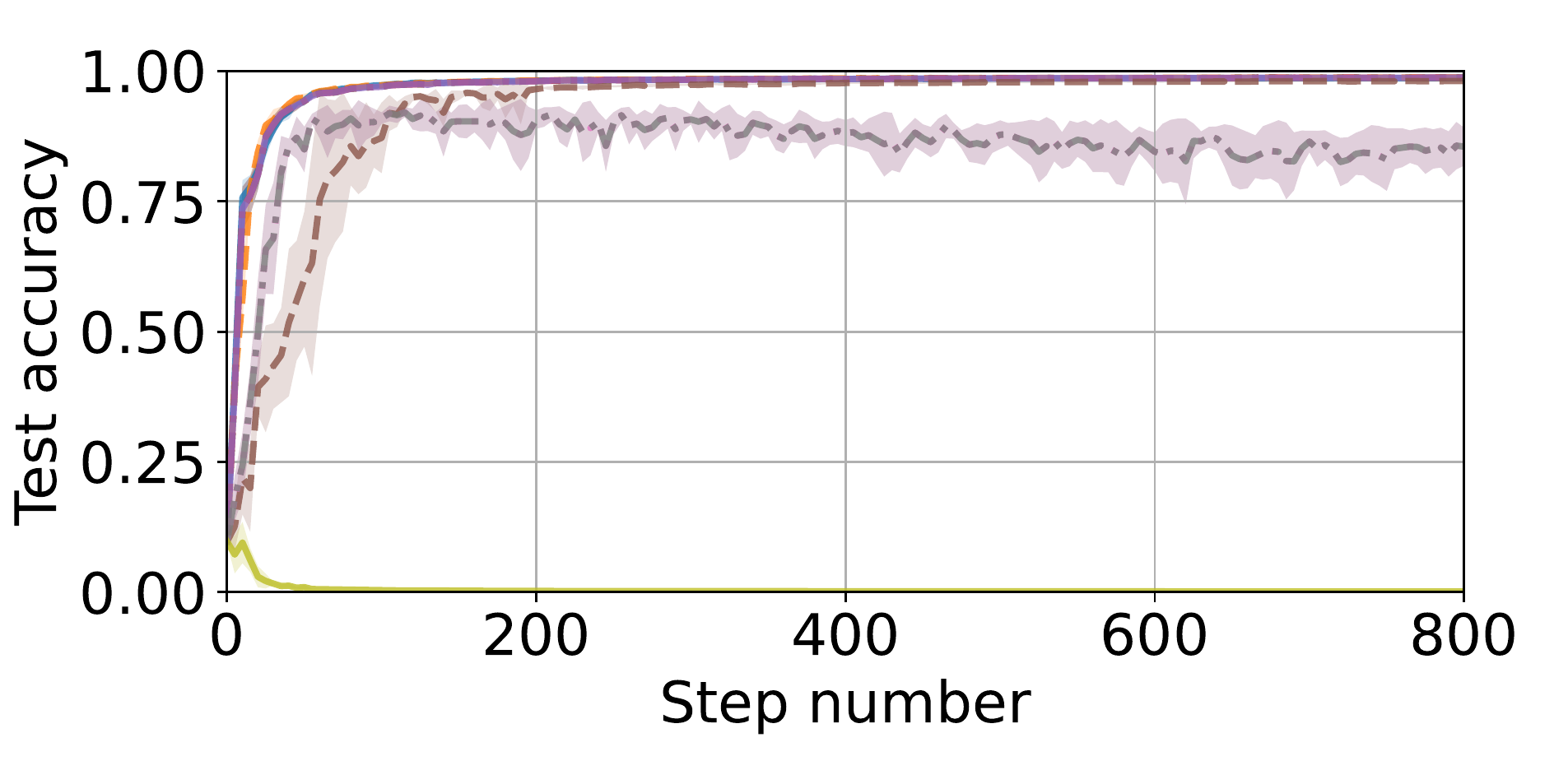}%
    \caption{Experiments on MNIST using robust D-SHB with $f = 8$ among $n = 17$ workers, with $\beta = 0.9$ and $\alpha = 10$. The Byzantine workers execute FOE (\textit{row 1, left}), ALIE (\textit{row 1, right}), Mimic (\textit{row 2, left}), SF (\textit{row 2, right}), and LF (\textit{row 3}).}
\label{fig:plots_mnist_6}
\end{figure*}

\clearpage
\subsubsection{Results on D-GD}
We also test our algorithm with robust D-GD on MNIST in one Byzantine regime: $f=4$ Byzantine workers among a total of $n=17$ workers. We also consider three heterogeneity settings: $\alpha=0.1$, $\alpha=1$, and $\alpha=10$. The results are shown below, and convey the same observations made in Appendix~\ref{app:exp_results_mnist}.

\begin{figure*}[ht!]
    \centering
    \includegraphics[width=0.5\textwidth]{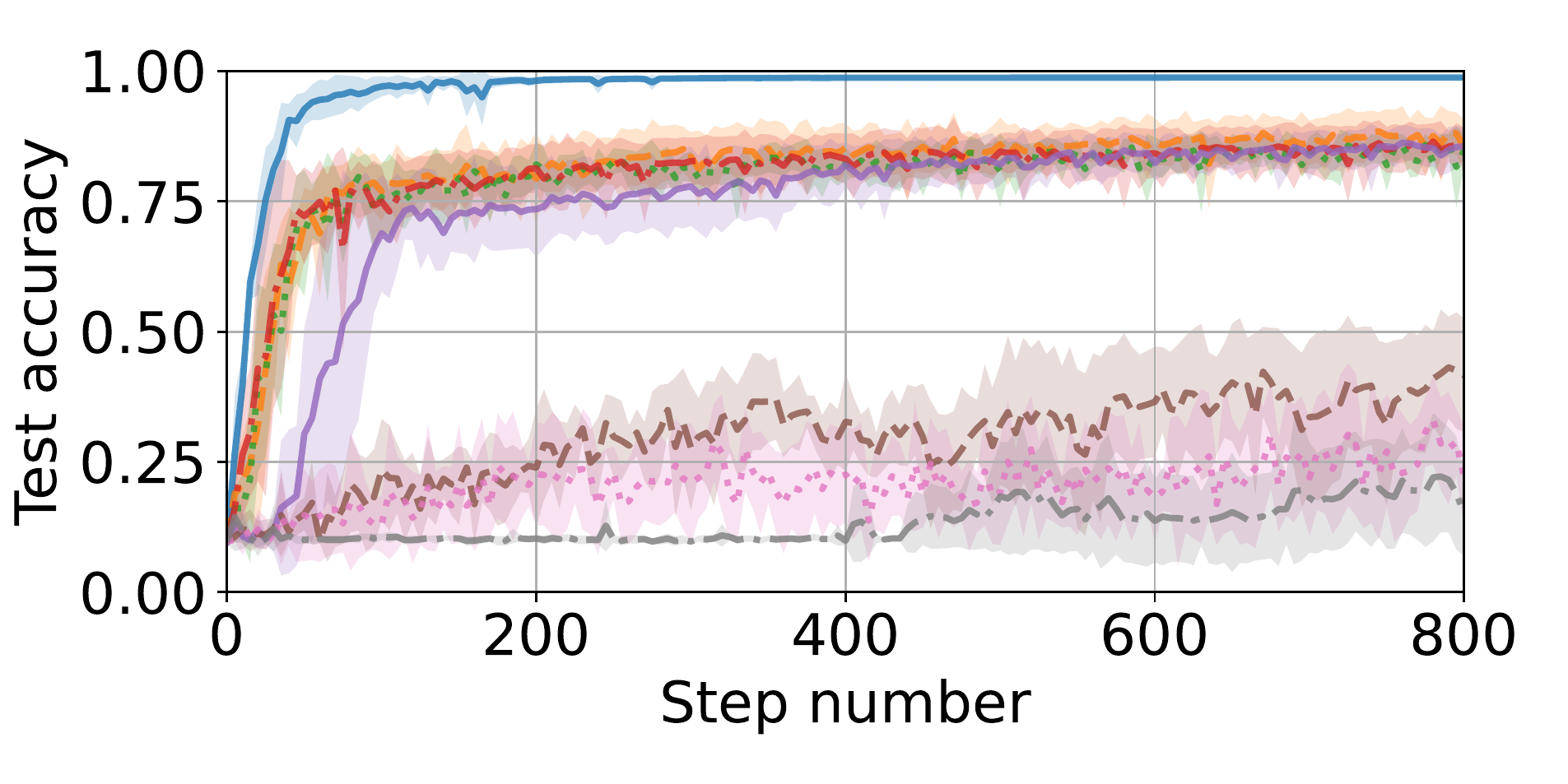}%
    \includegraphics[width=0.5\textwidth]{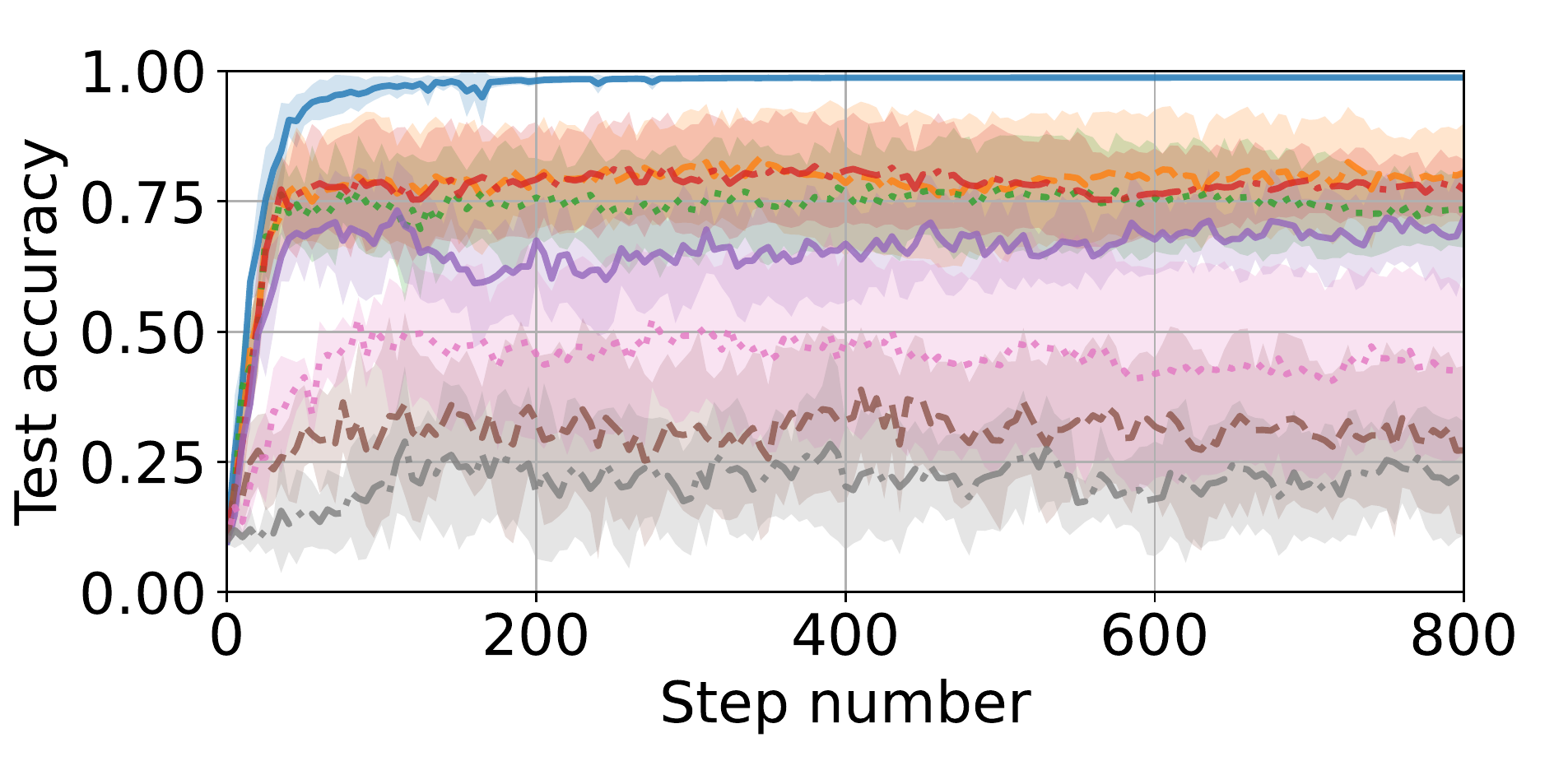}\\%
    \includegraphics[width=0.5\textwidth]{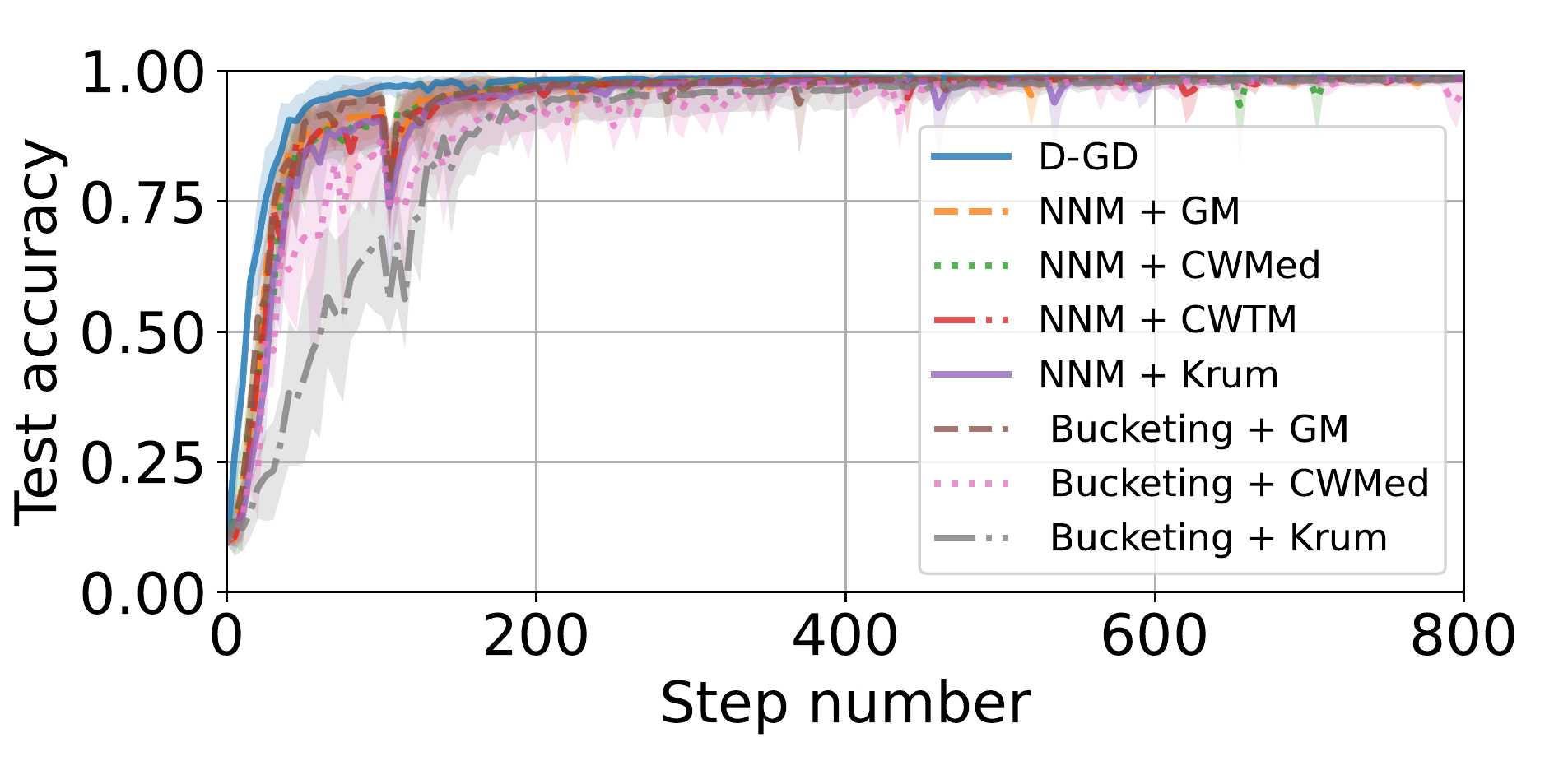}%
    \includegraphics[width=0.5\textwidth]{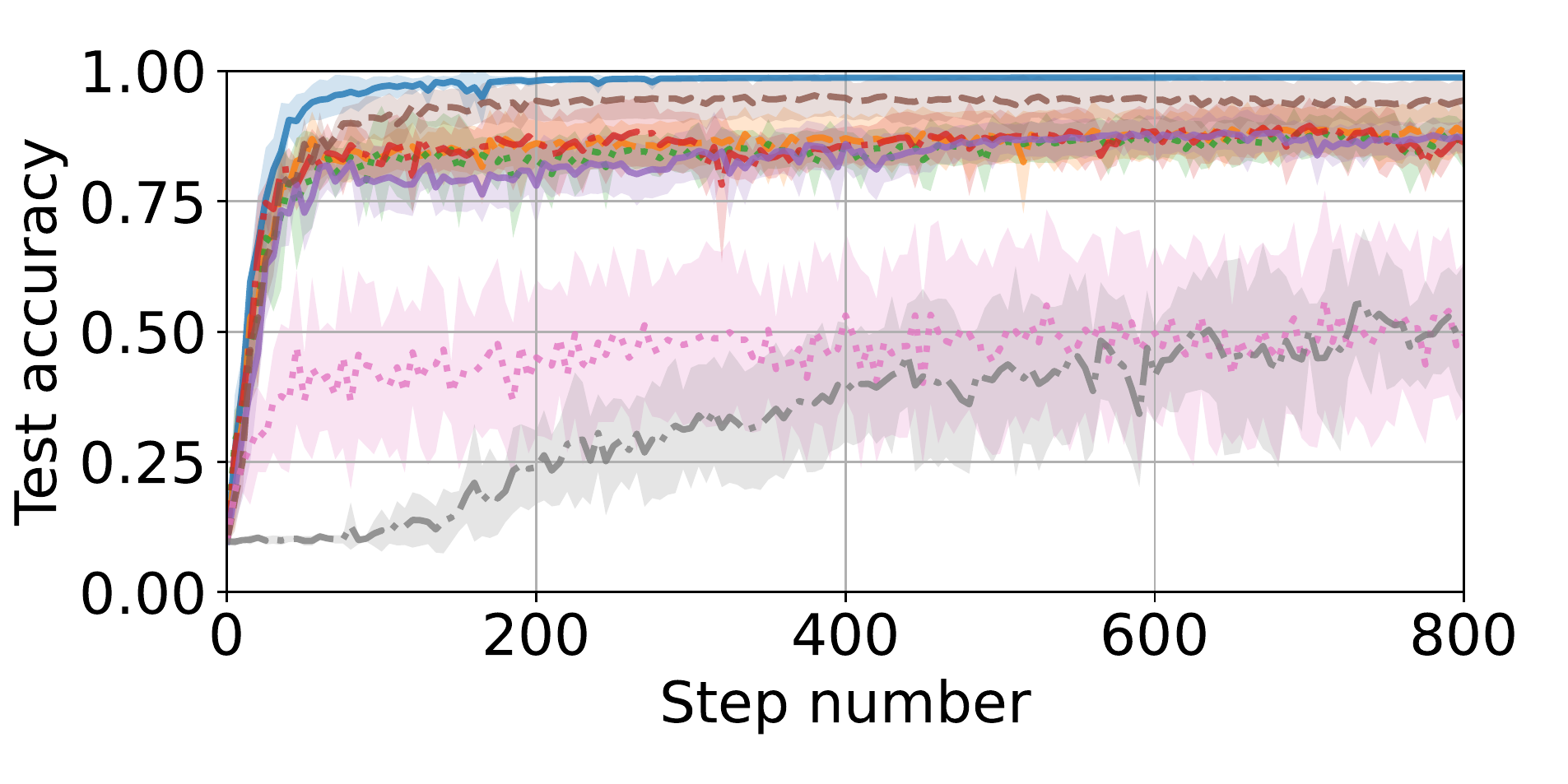}\\%
     \includegraphics[width=0.5\textwidth]{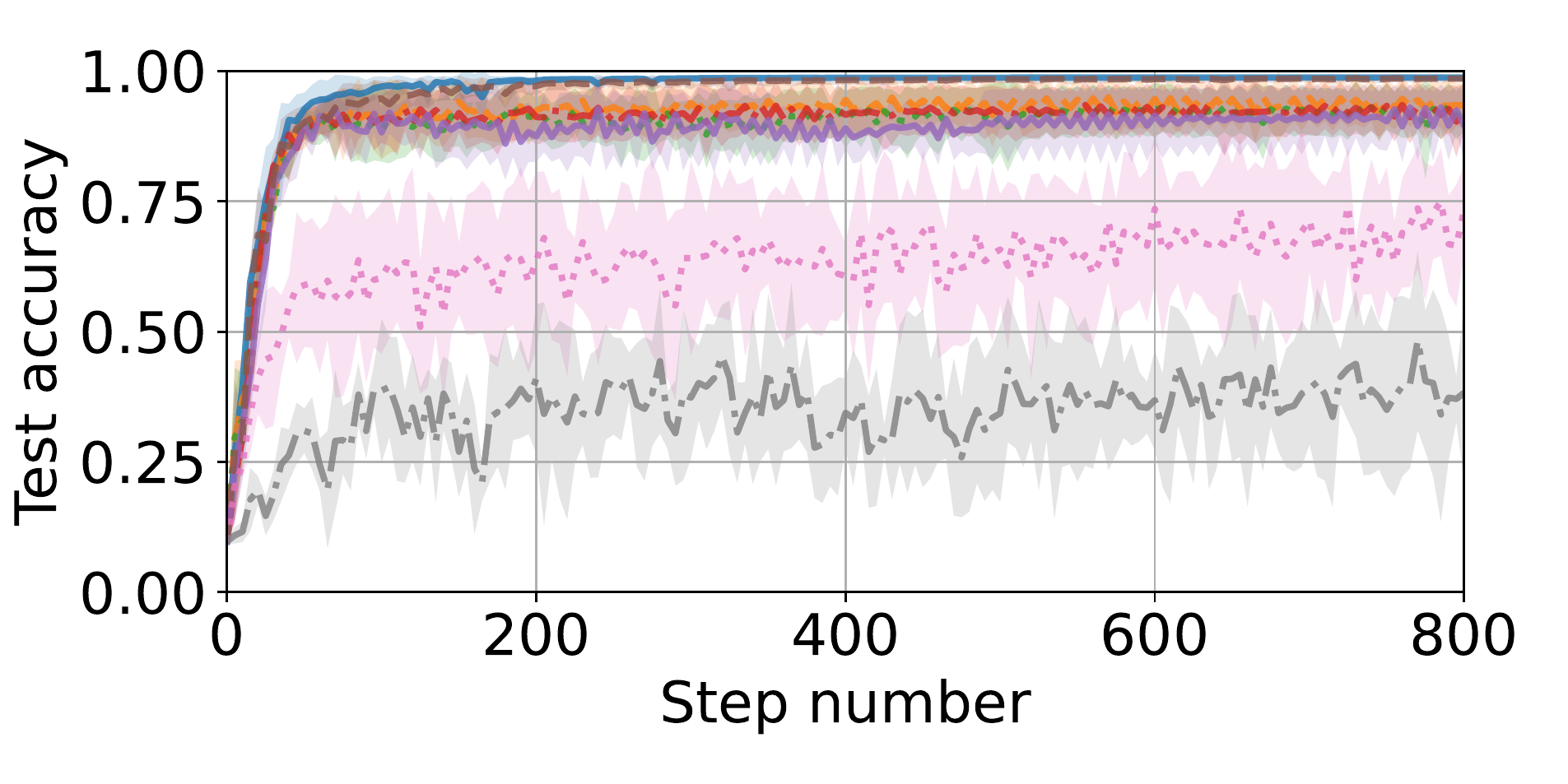}%
    \caption{Experiments on MNIST using robust D-GD with $f = 4$ among $n = 17$ workers, with $\alpha = 0.1$. The Byzantine workers execute the FOE (\textit{row 1, left}), ALIE (\textit{row 1, right}), Mimic (\textit{row 2, left}), SF (\textit{row 2, right}), and LF (\textit{row 3}) attacks.}
\label{fig:plots_mnist_gd}
\end{figure*}

\begin{figure*}[ht!]
    \centering
    \includegraphics[width=0.5\textwidth]{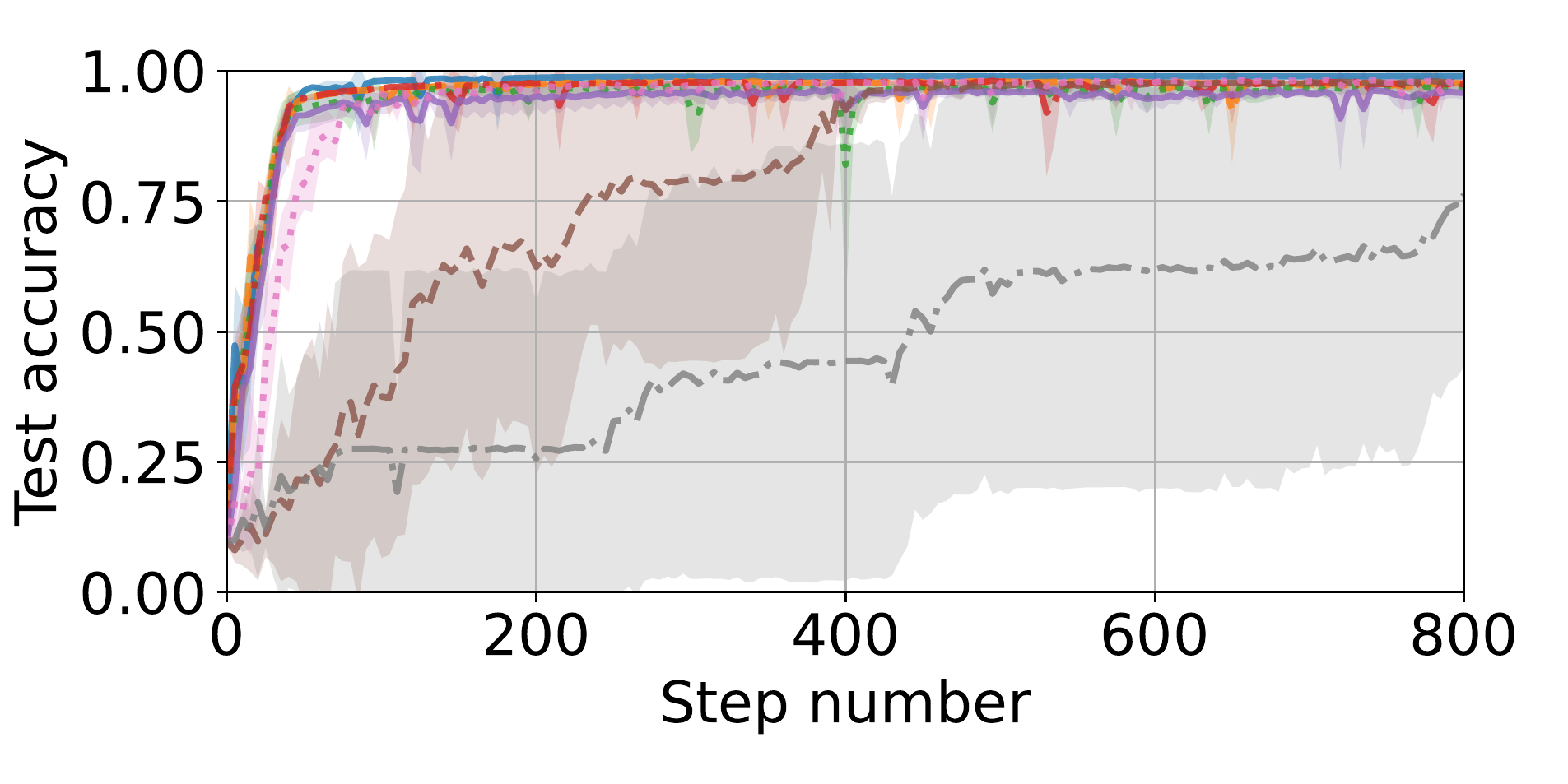}%
    \includegraphics[width=0.5\textwidth]{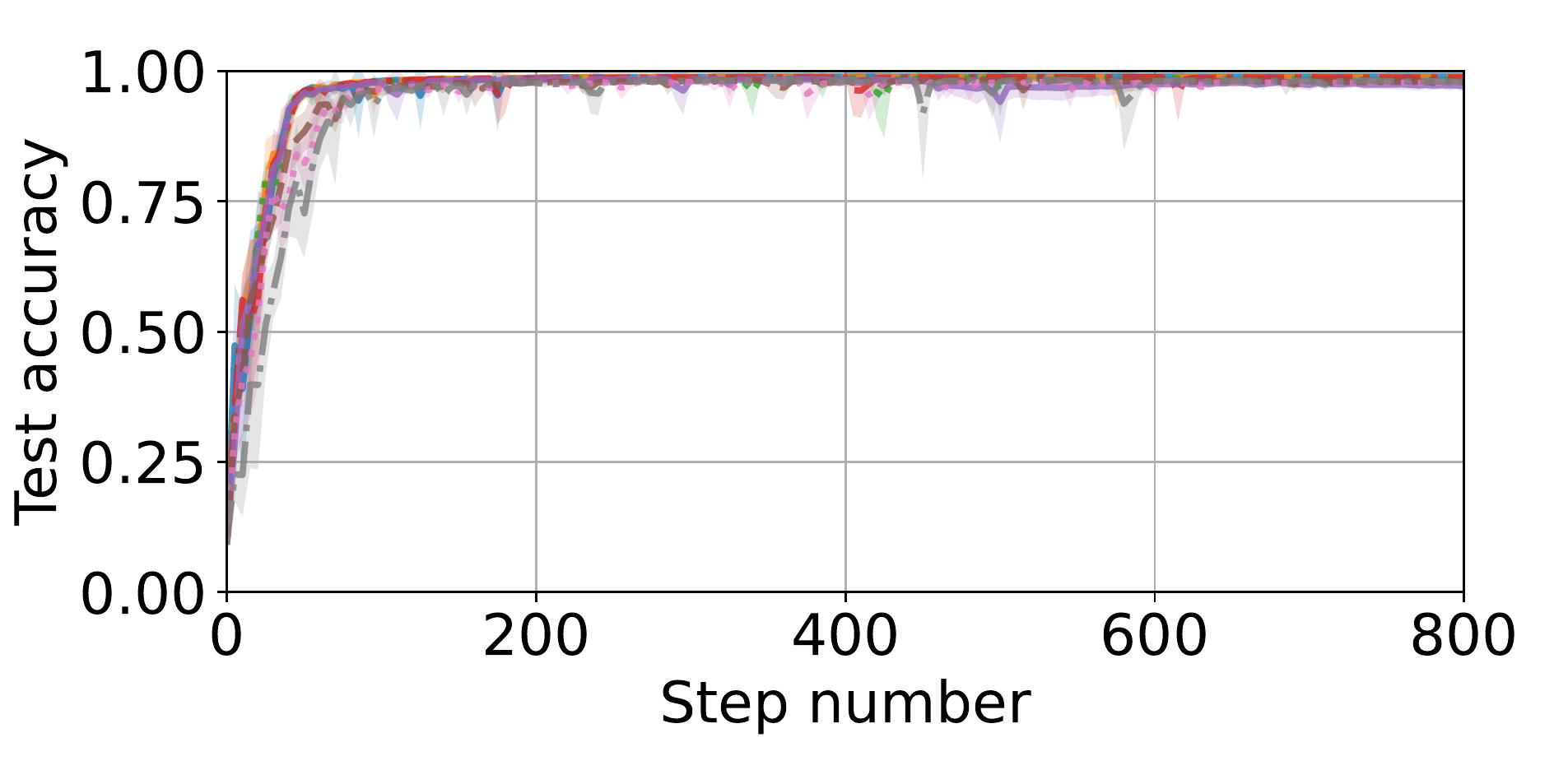}\\%
    \includegraphics[width=0.5\textwidth]{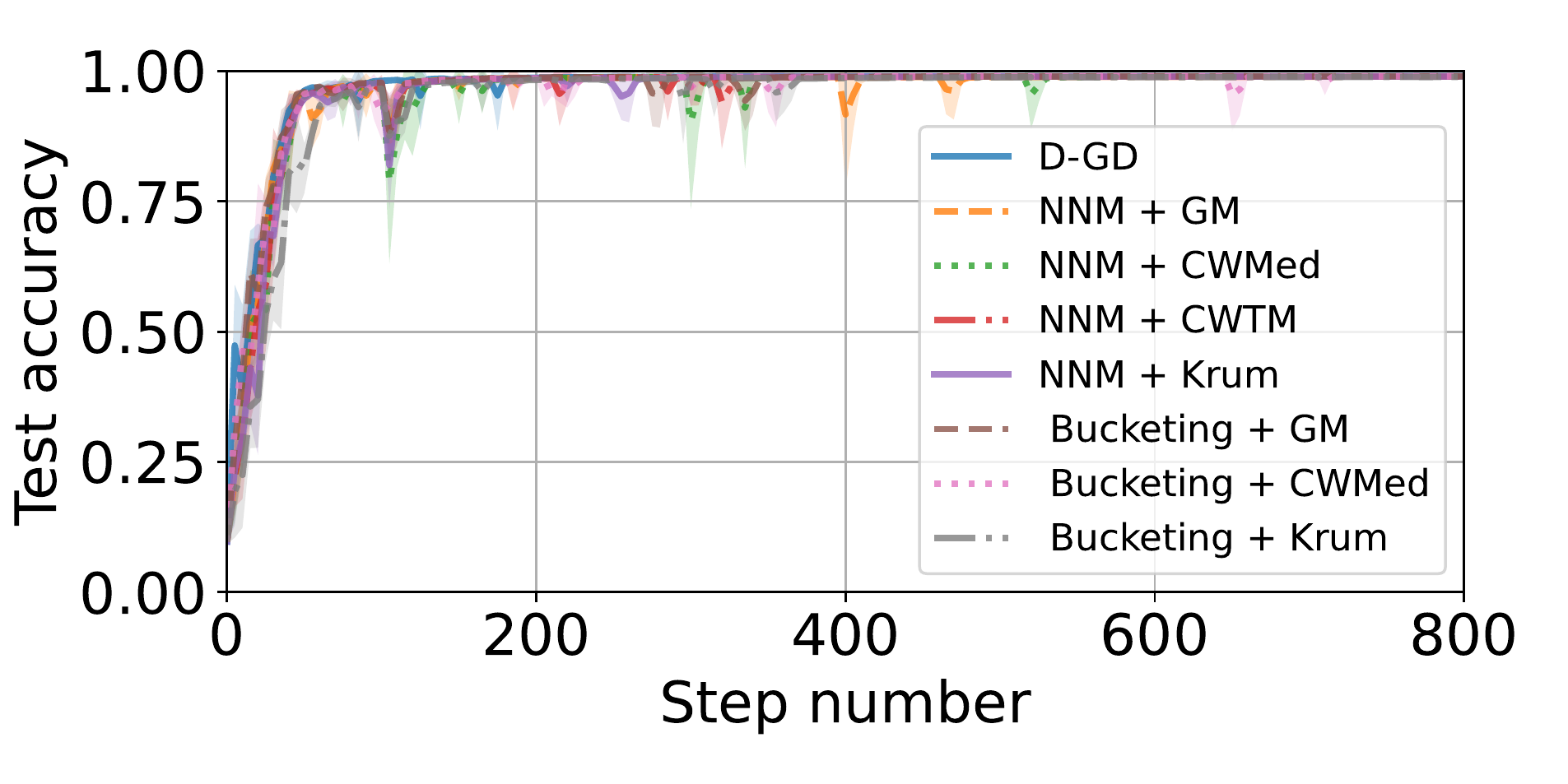}%
    \includegraphics[width=0.5\textwidth]{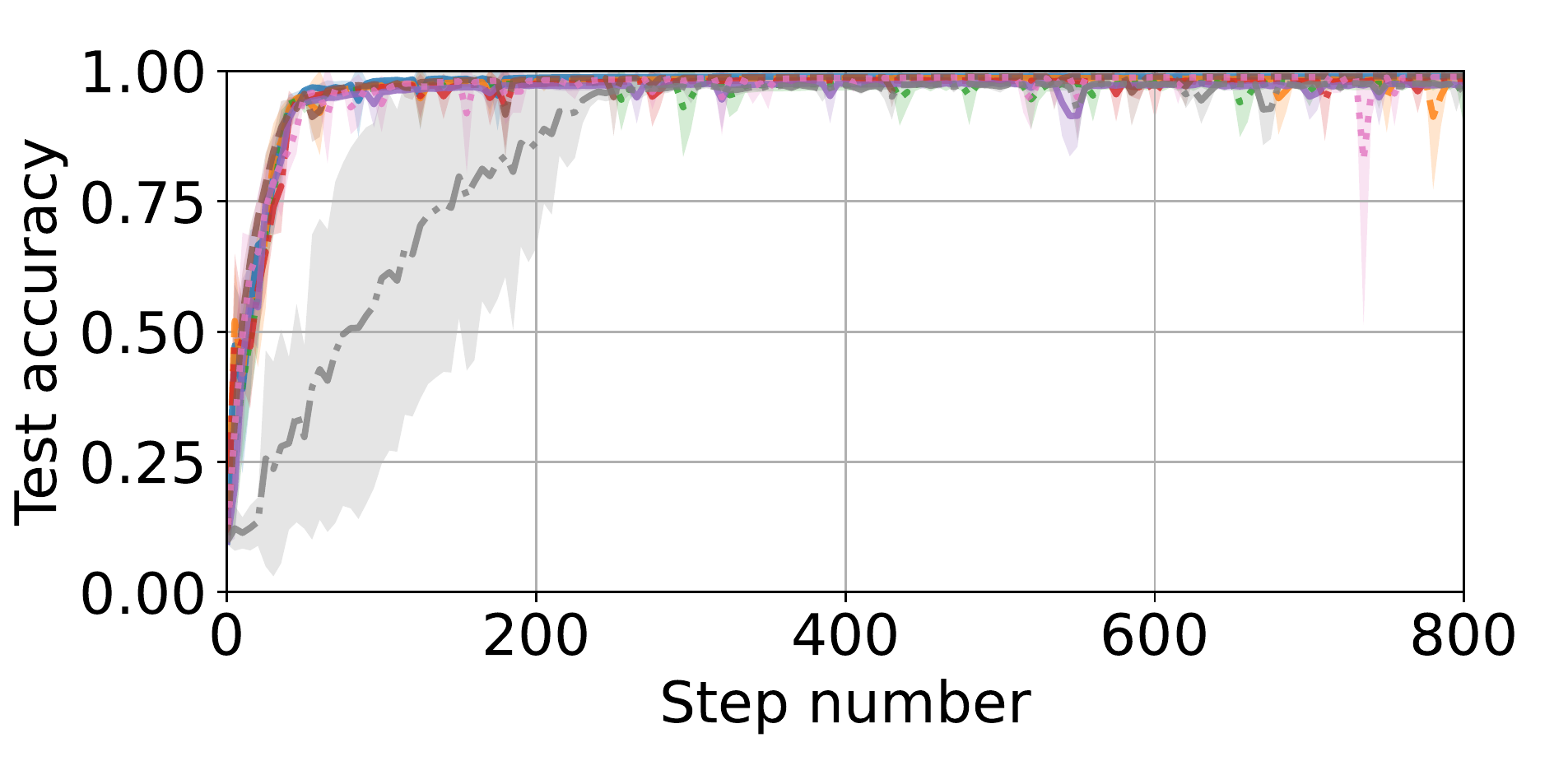}\\%
     \includegraphics[width=0.5\textwidth]{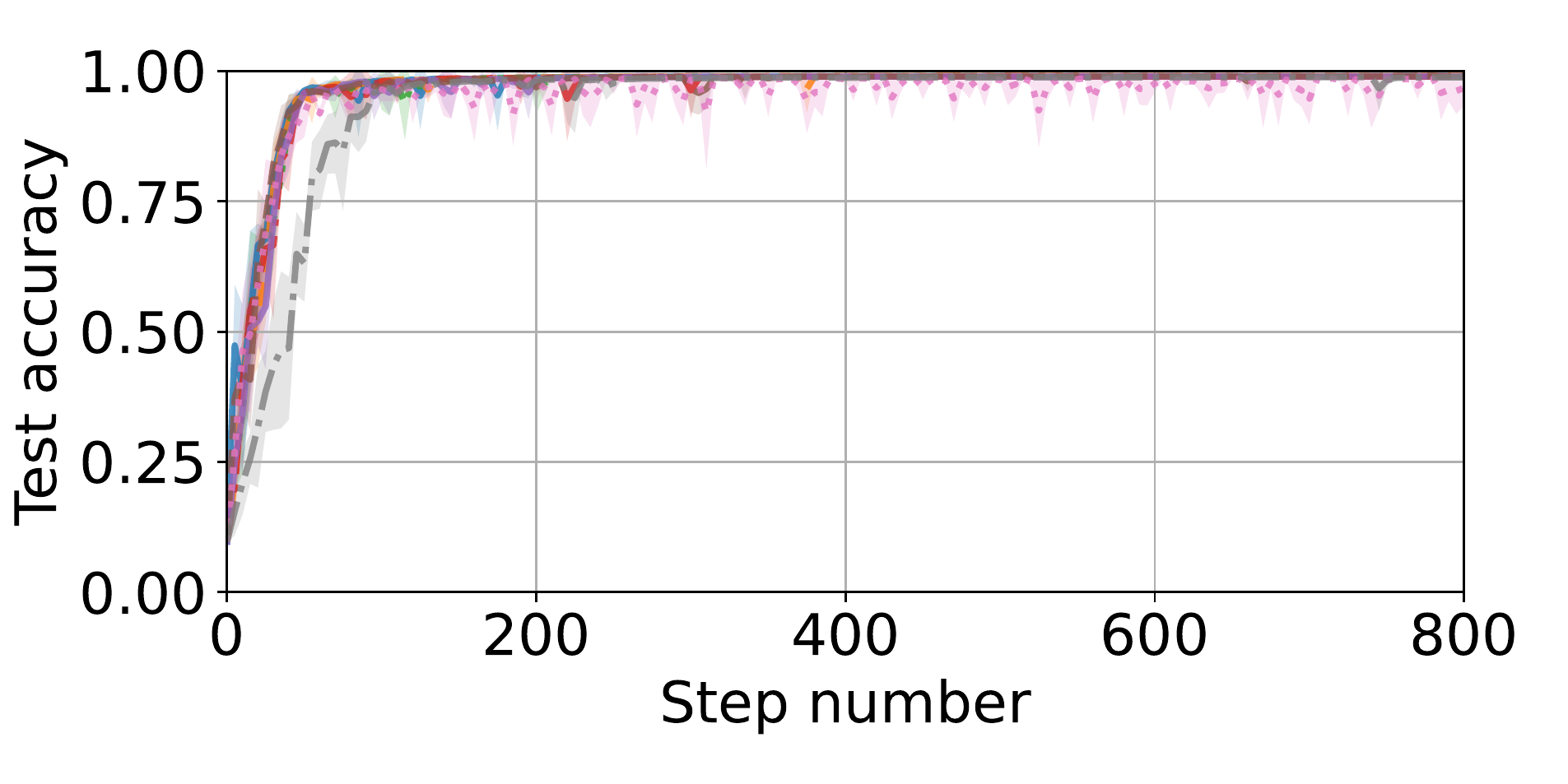}%
    \caption{Experiments on MNIST using robust D-GD with $f = 4$ among $n = 17$ workers, with $\alpha = 1$. The Byzantine workers execute the FOE (\textit{row 1, left}), ALIE (\textit{row 1, right}), Mimic (\textit{row 2, left}), SF (\textit{row 2, right}), and LF (\textit{row 3}) attacks.}
\label{fig:plots_mnist_gd_2}
\end{figure*}

\begin{figure*}[ht!]
    \centering
    \includegraphics[width=0.5\textwidth]{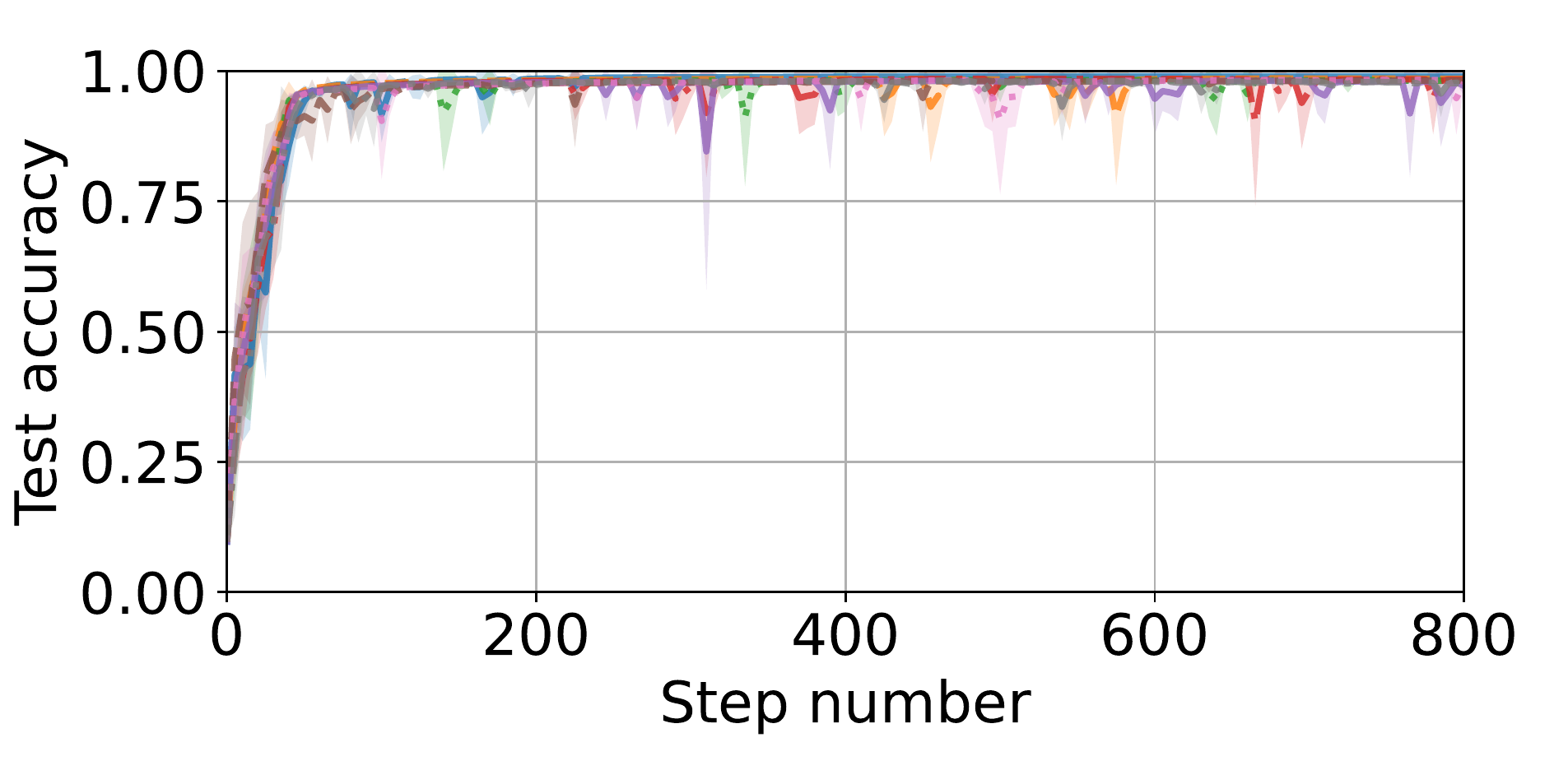}%
    \includegraphics[width=0.5\textwidth]{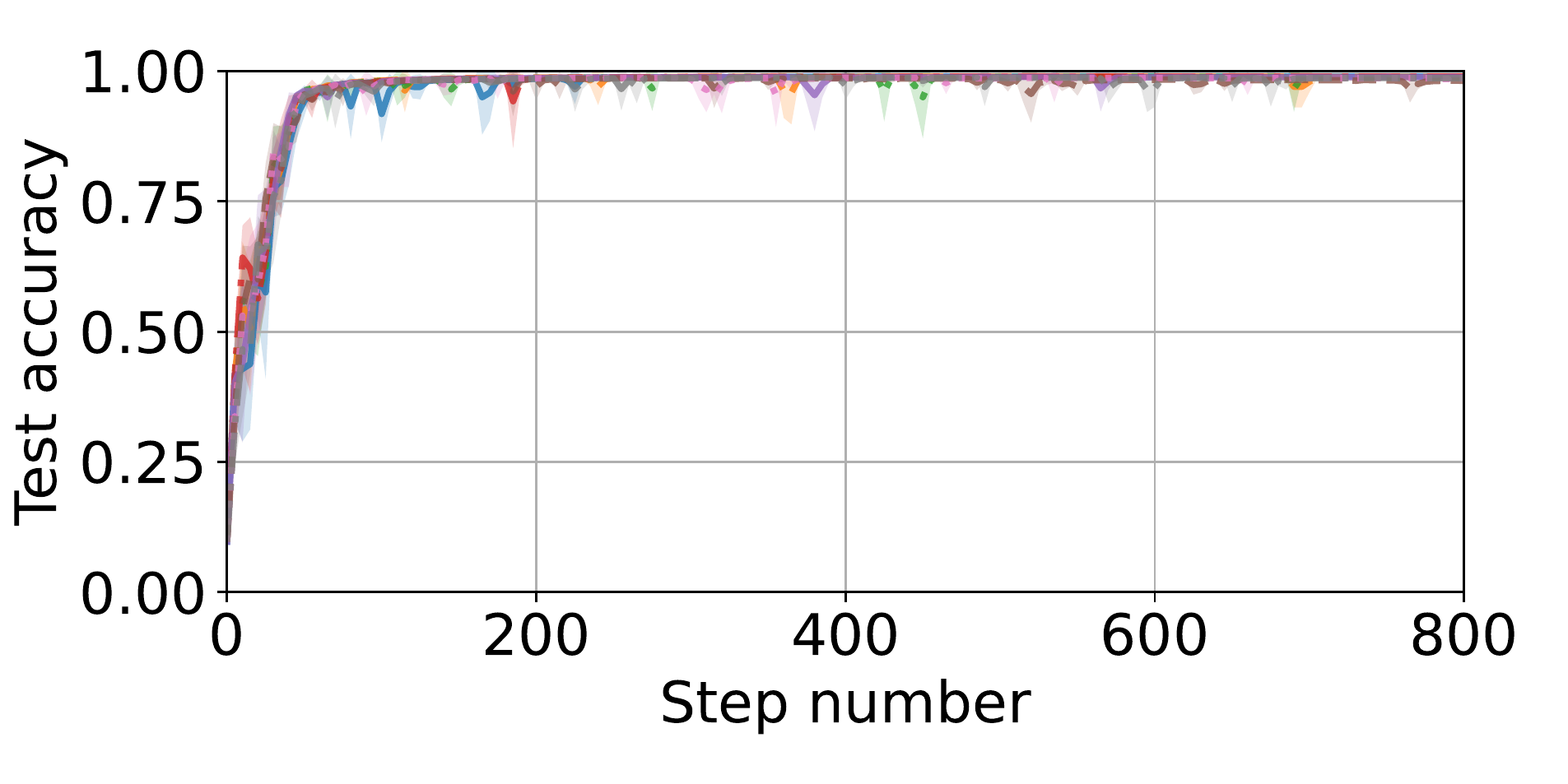}\\%
    \includegraphics[width=0.5\textwidth]{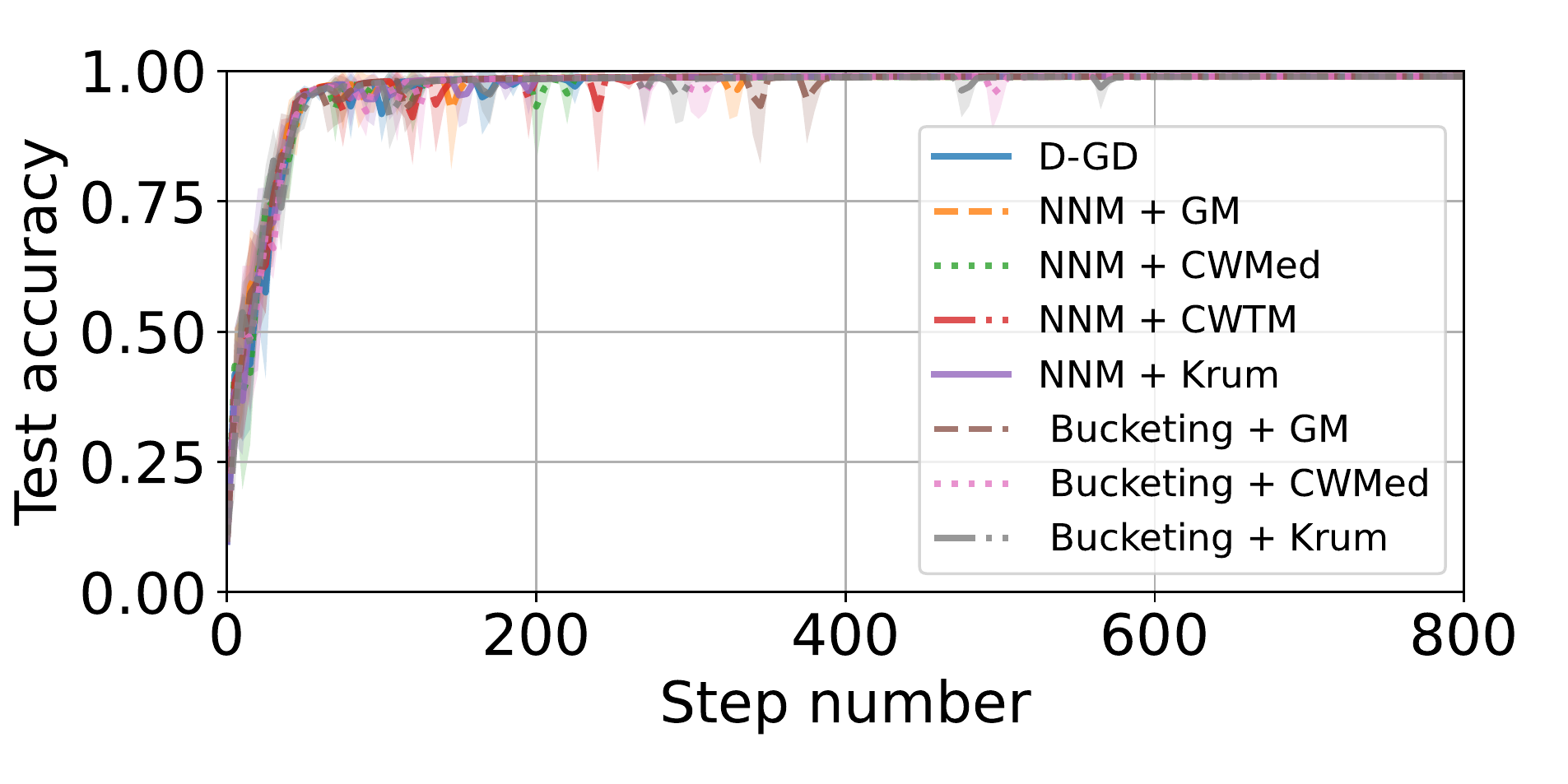}%
    \includegraphics[width=0.5\textwidth]{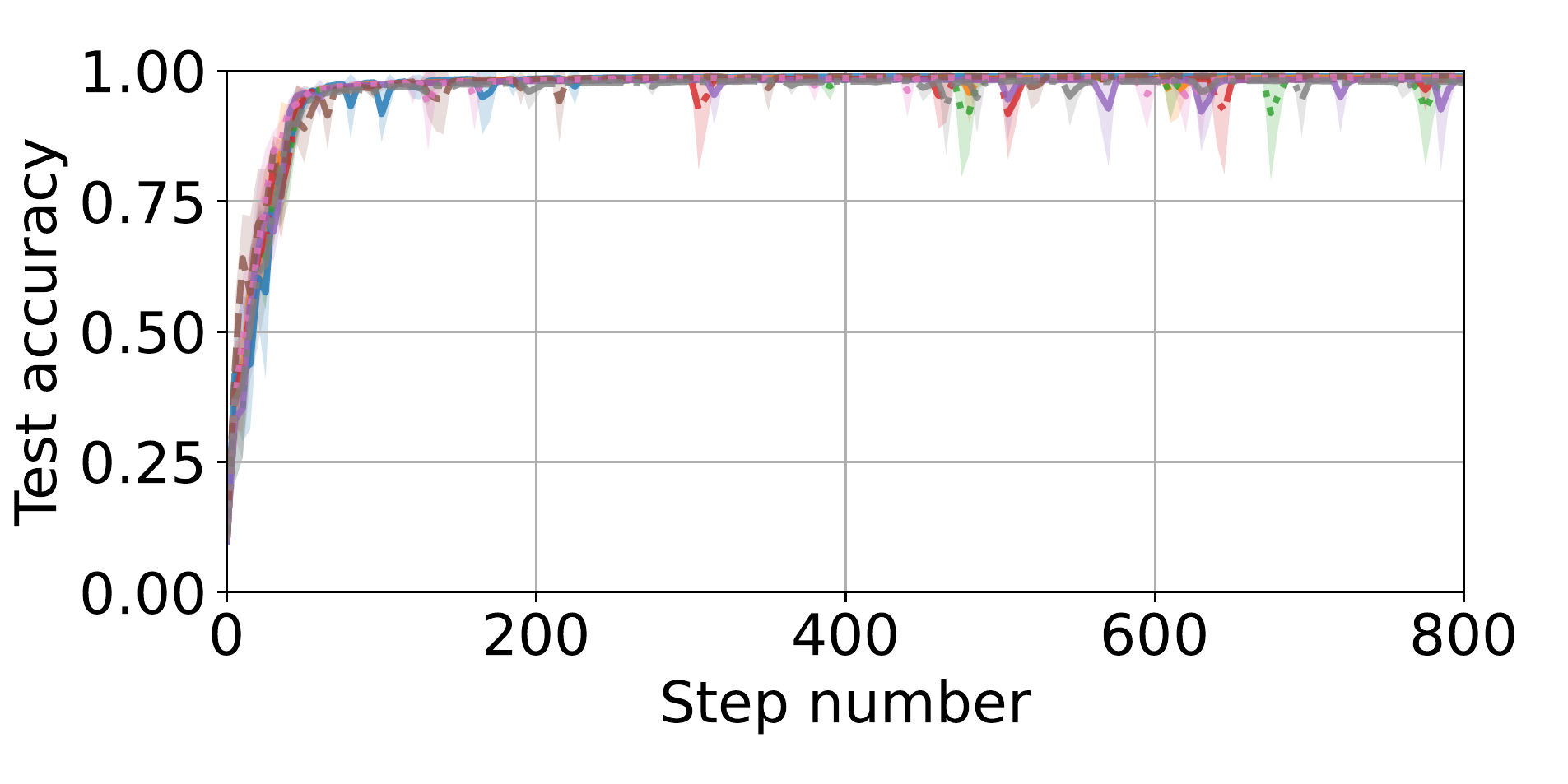}\\%
     \includegraphics[width=0.5\textwidth]{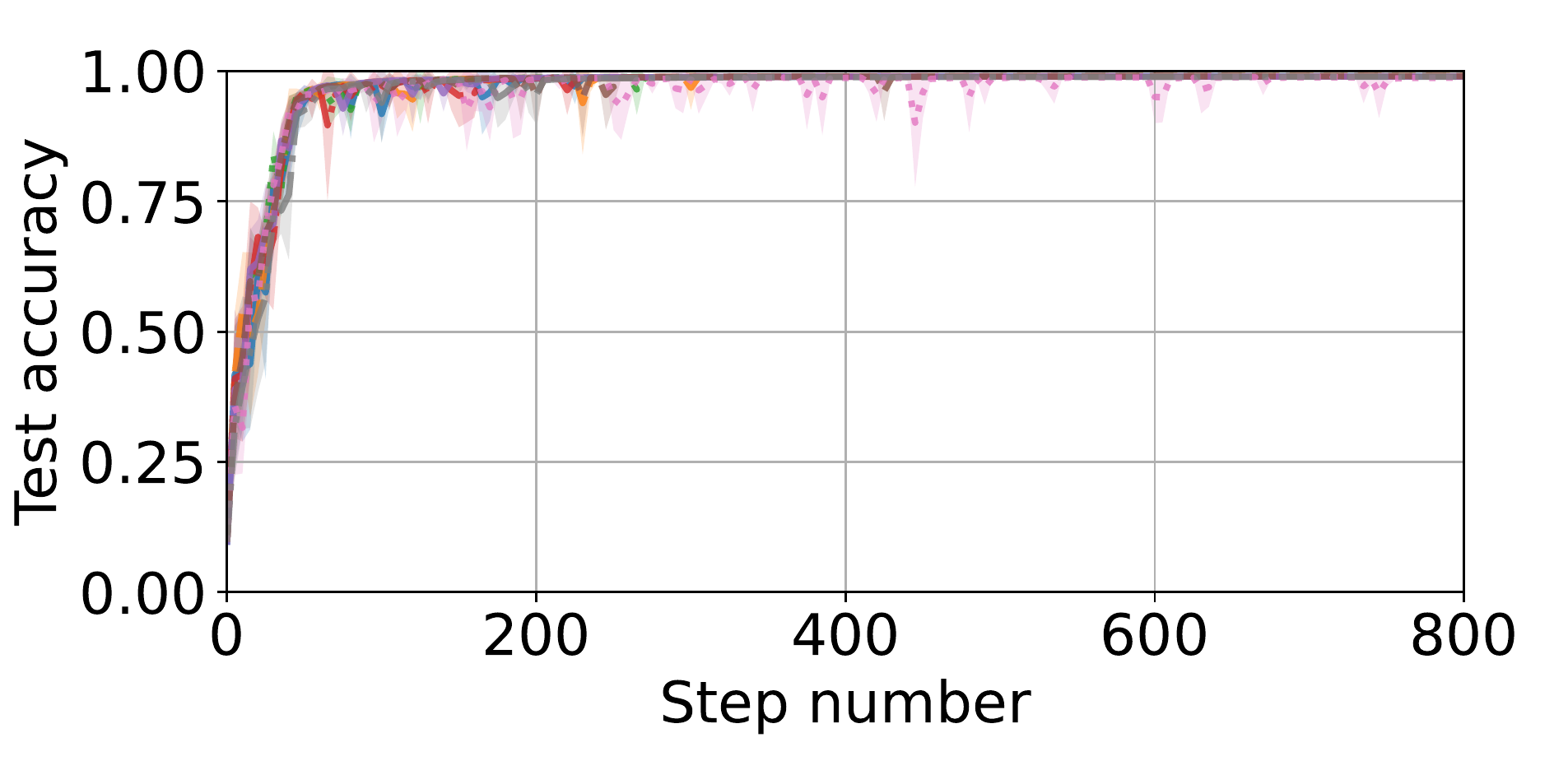}%
    \caption{Experiments on MNIST using robust D-GD with $f = 4$ among $n = 17$ workers, with $\alpha = 10$. The Byzantine workers execute the FOE (\textit{row 1, left}), ALIE (\textit{row 1, right}), Mimic (\textit{row 2, left}), SF (\textit{row 2, right}), and LF (\textit{row 3}) attacks.}
\label{fig:plots_mnist_gd_3}
\end{figure*}

\clearpage
\subsubsection{Results on Other Aggregation Rules: MDA, MultiKrum, and MeaMed}
We also execute $\cenna{}$ with three additional aggregation rules namely MDA~\cite{rousseeuw1985multivariate,collaborativeElMhamdi21}, MeaMed~\cite{meamed}, and Multi-Krum~\cite{krum}. 
The experimental setup is the following. We consider three Byzantine regimes: $f = 4$, $f = 6$, and $f=8$ (largest possible) out of $n=17$ workers in total. We also consider three heterogeneity regimes: $\alpha=0.1$ (extreme), $\alpha=1$ (moderate), and $\alpha=10$ (low).
We compare the performance of $\cenna{}$ and Bucketing when executed with MDA, MultiKrum, and MeaMed.
The results (shown below) are very similar to the ones presented in Appendix~\ref{app:exp_results_mnist}, and convey exactly the same observations made in Section~\ref{sec:exp}. Indeed, $\cenna{}$ combined with these methods provides consistently good performance on MNIST in three regimes of data heterogeneity and under five state-of-the-art Byzantine attacks. Given that no guarantees have yet been derived for these aggregation techniques, an interesting future direction would be to prove that MDA, MultiKrum, and MeaMed are $(f, \kappa)$-robust with $\kappa \in \mathcal{O}(1)$, confirming the modularity of our solution as a reliable framework for heterogeneous Byzantine learning.

\begin{figure*}[ht!]
    \centering
    \includegraphics[width=0.5\textwidth]{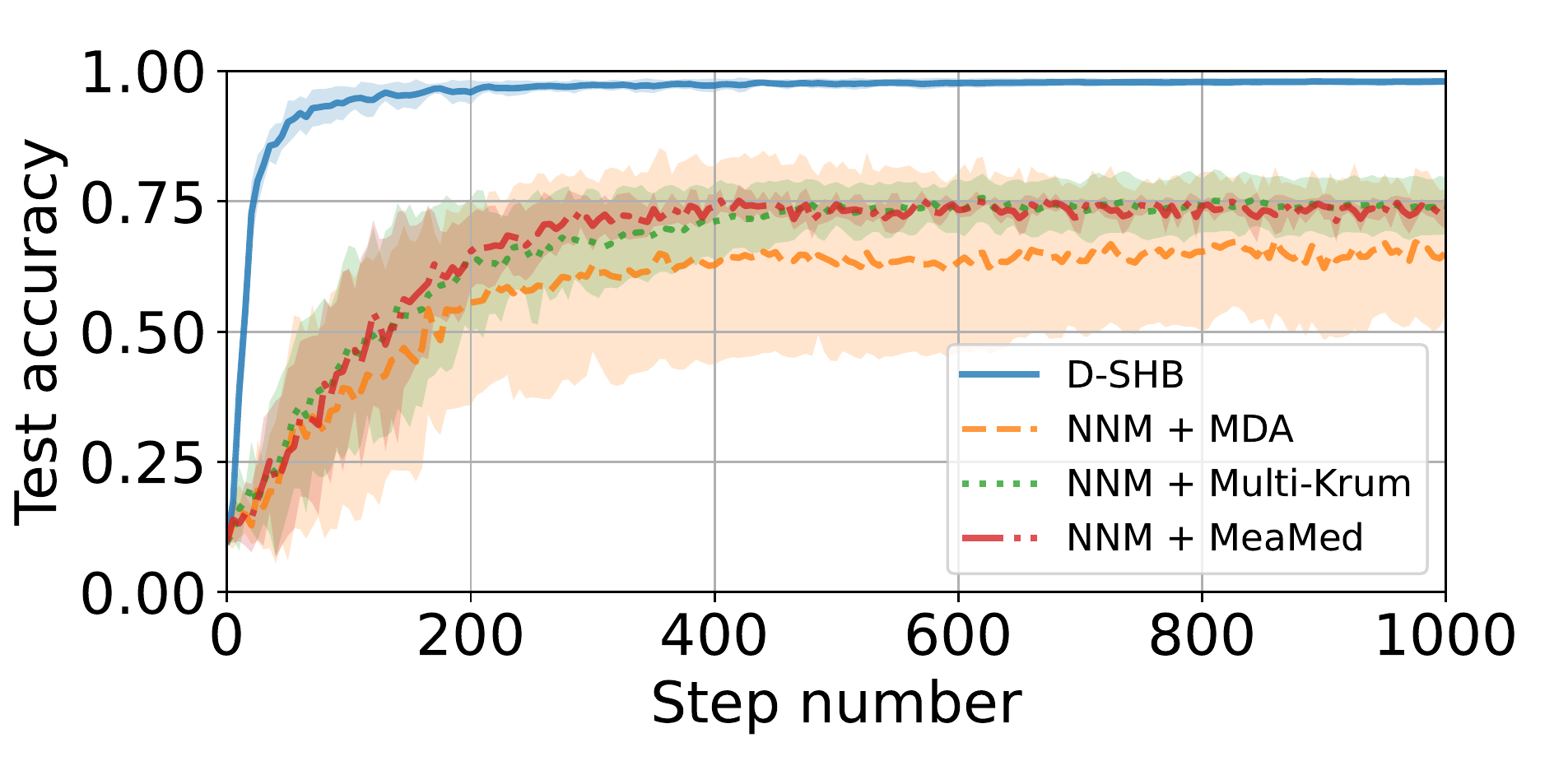}%
    \includegraphics[width=0.5\textwidth]{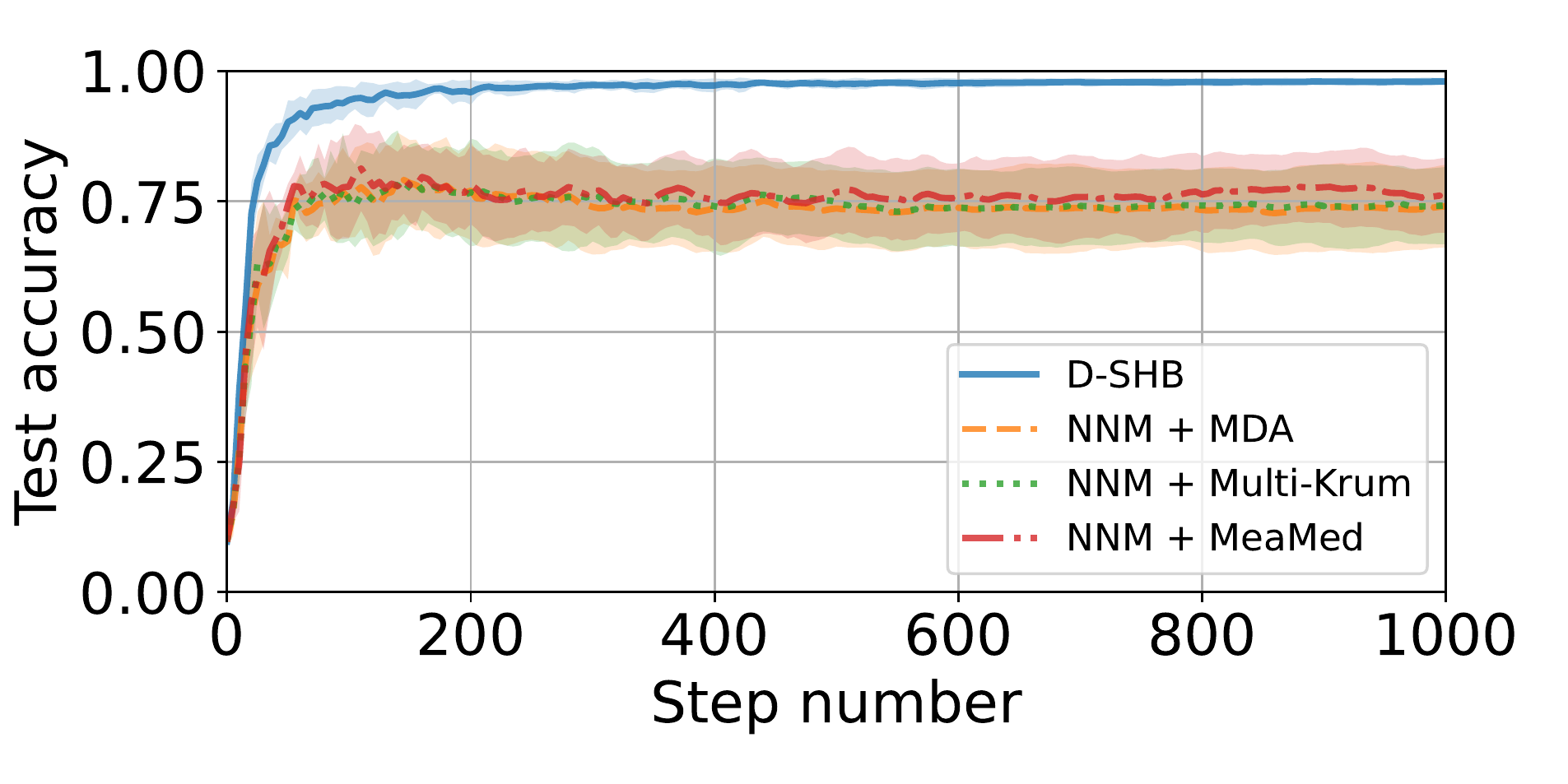}\\%
    \includegraphics[width=0.5\textwidth]{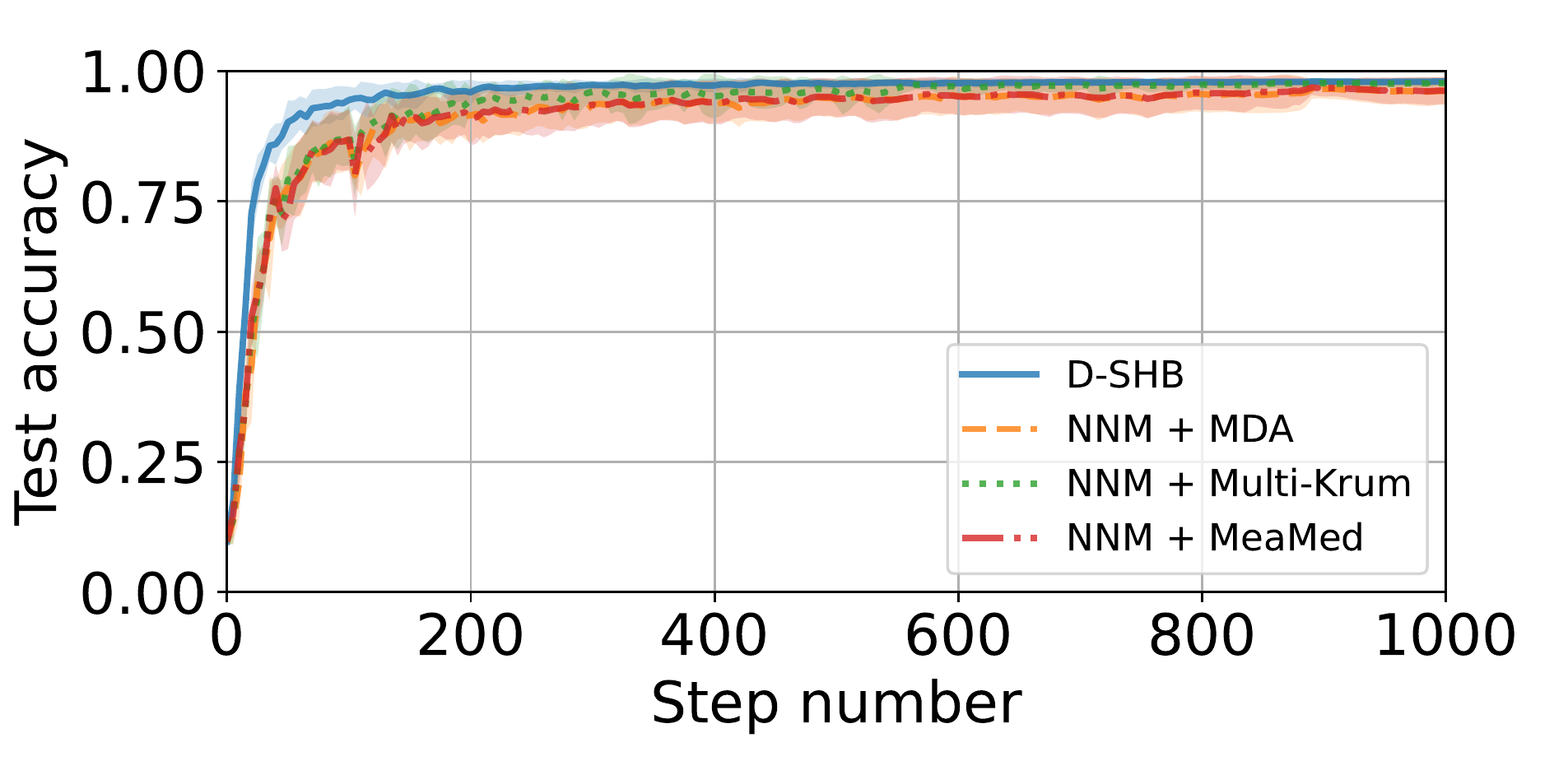}%
    \includegraphics[width=0.5\textwidth]{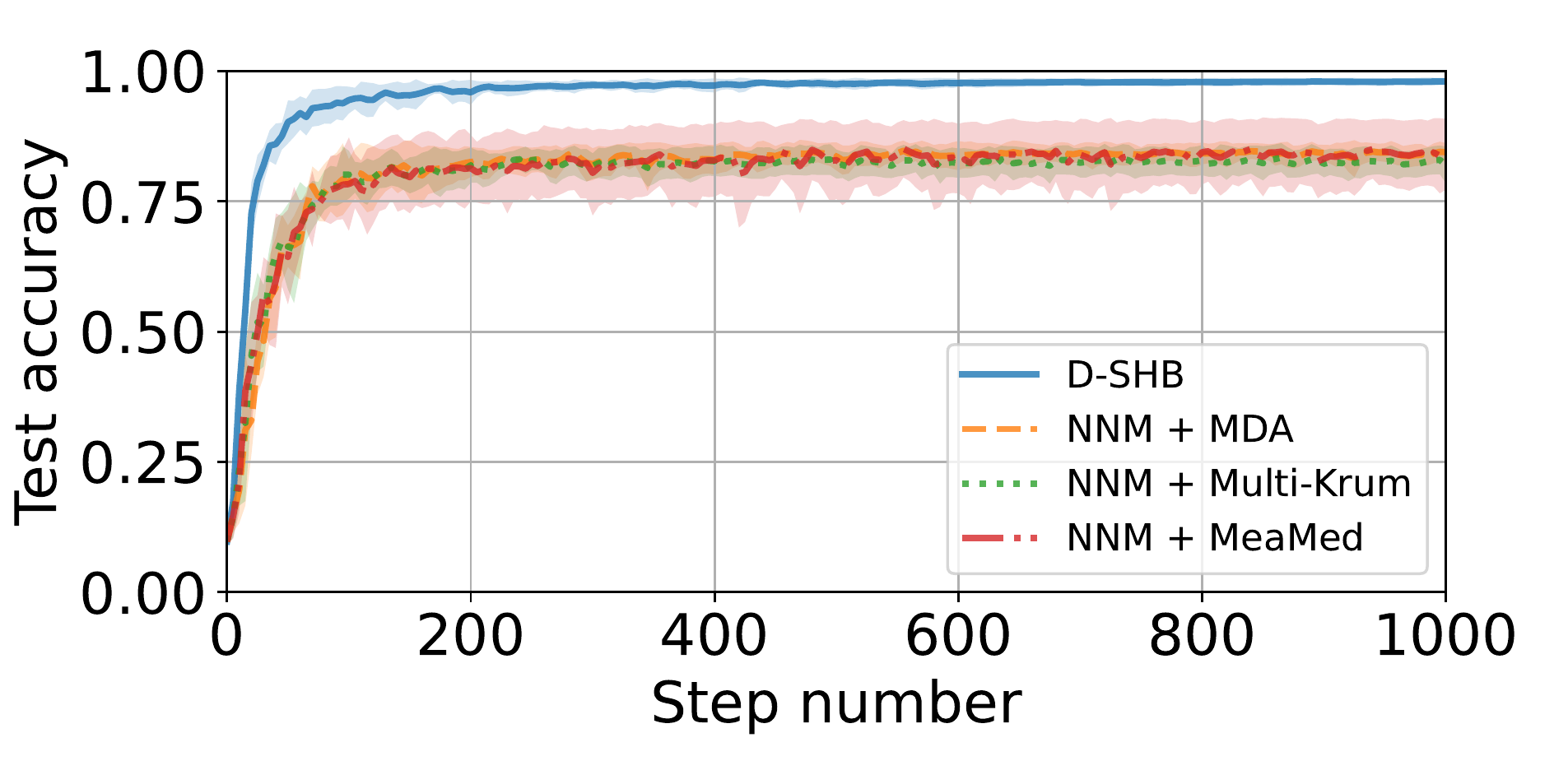}\\%
     \includegraphics[width=0.5\textwidth]{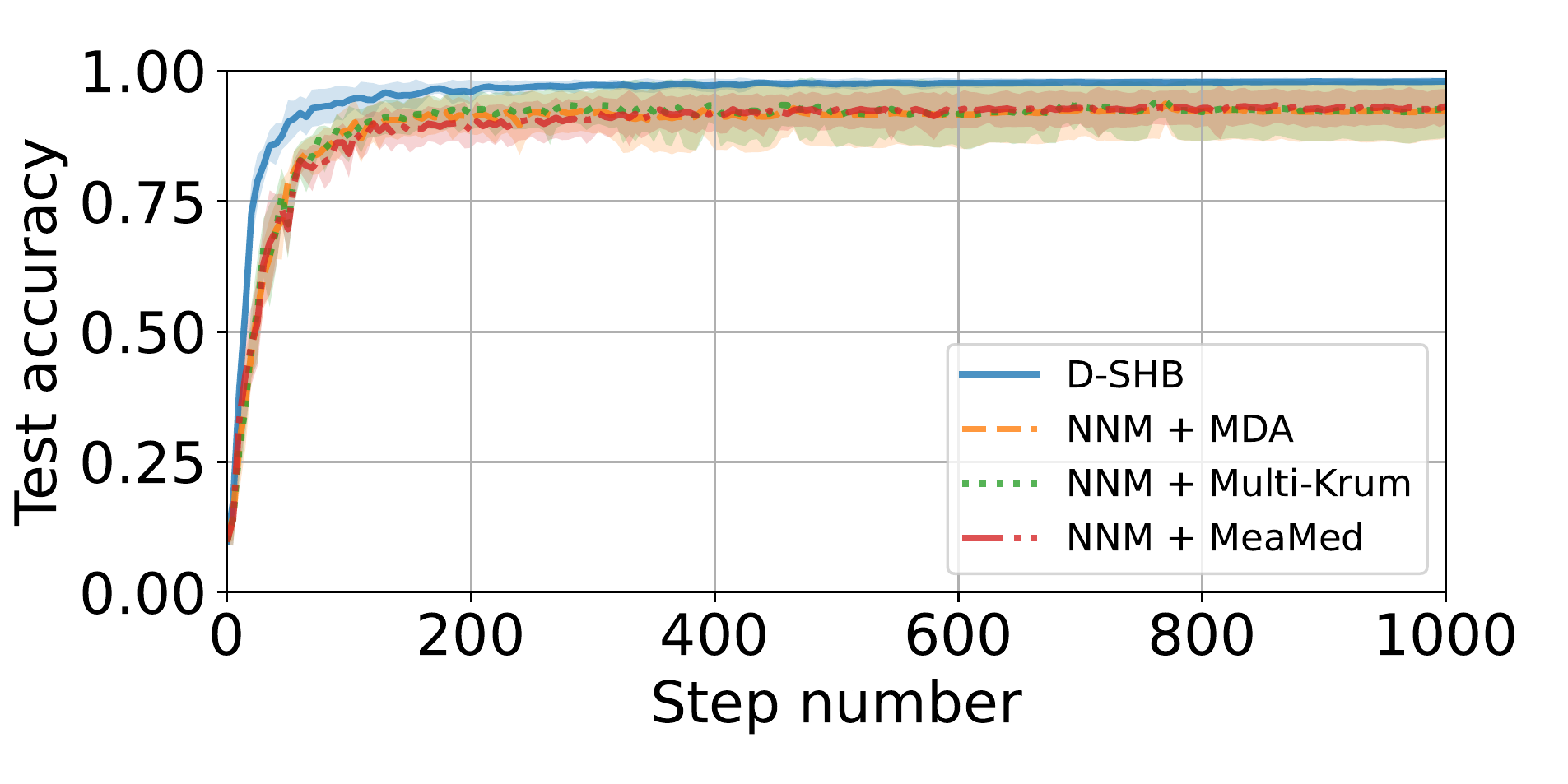}%
    \caption{Experiments on MNIST using robust D-SHB with $f = 4$ Byzantine among $n = 17$ workers, with $\beta = 0.9$ and $\alpha = 0.1$. The Byzantine workers execute the FOE (\textit{row 1, left}), ALIE (\textit{row 1, right}), Mimic (\textit{row 2, left}), SF (\textit{row 2, right}), and LF (\textit{row 3}) attacks.}
\label{fig:plots_other_1}
\end{figure*}

\begin{figure*}[ht!]
    \centering
    \includegraphics[width=0.5\textwidth]{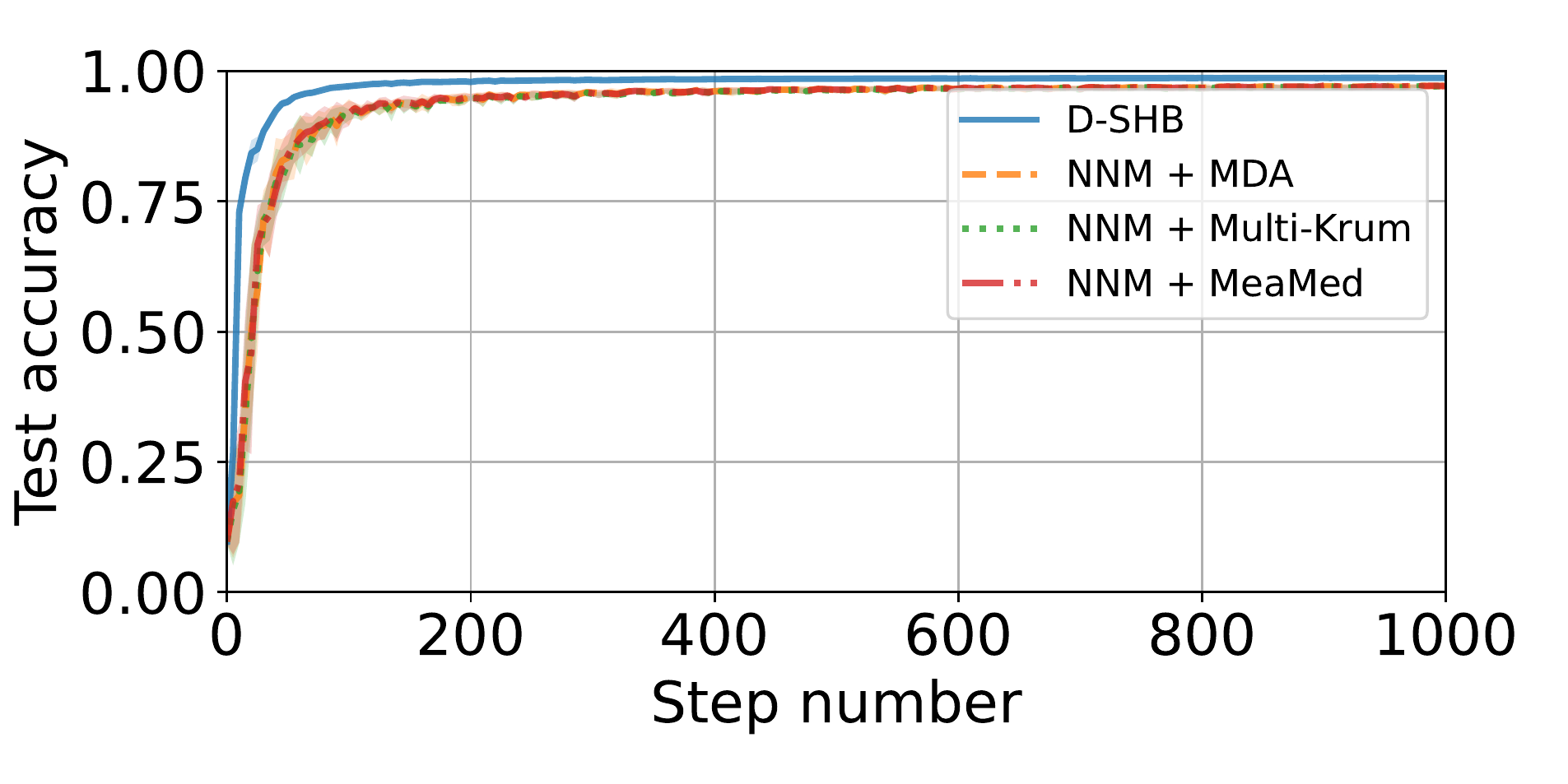}%
    \includegraphics[width=0.5\textwidth]{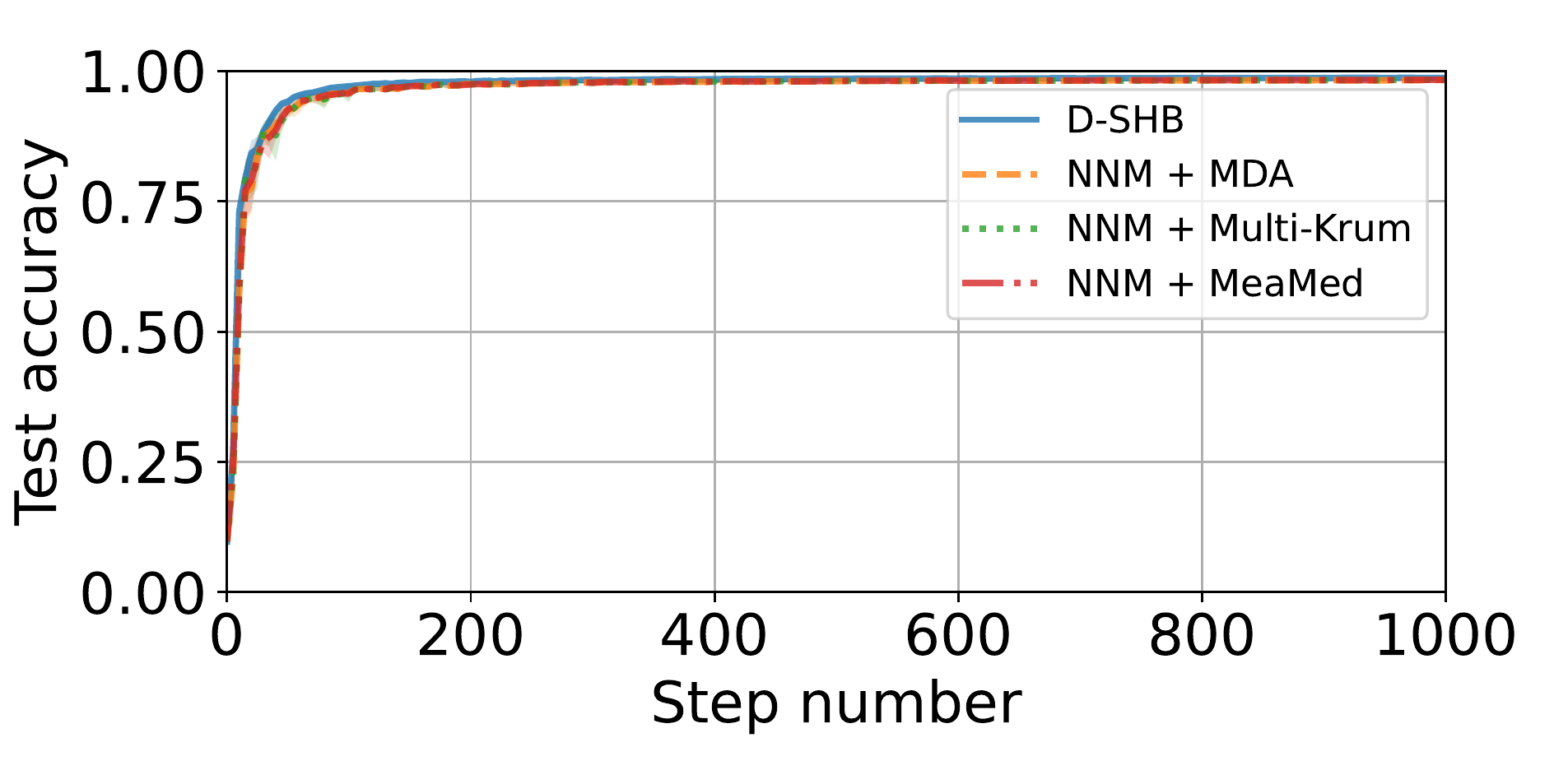}\\%
    \includegraphics[width=0.5\textwidth]{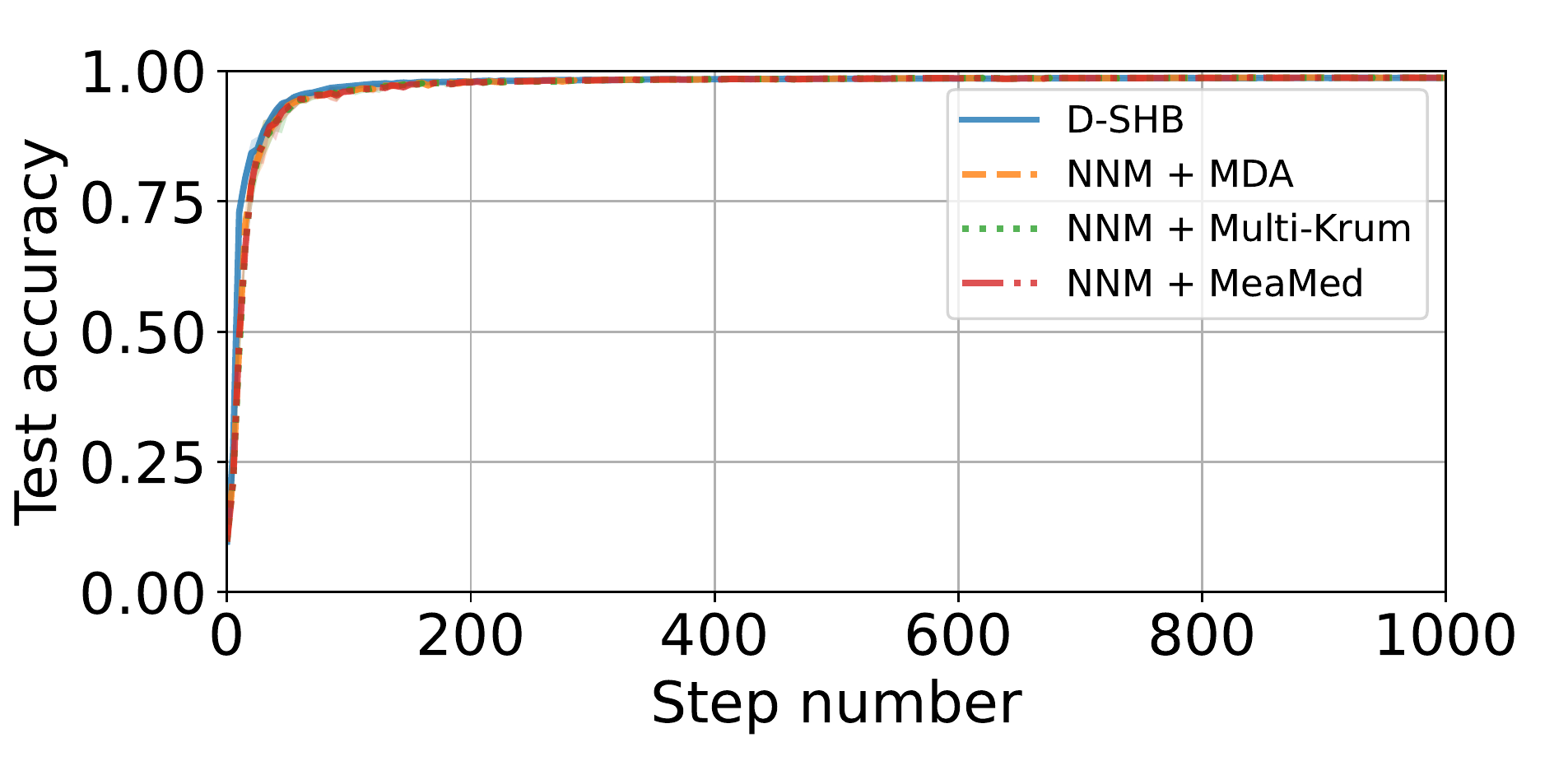}%
    \includegraphics[width=0.5\textwidth]{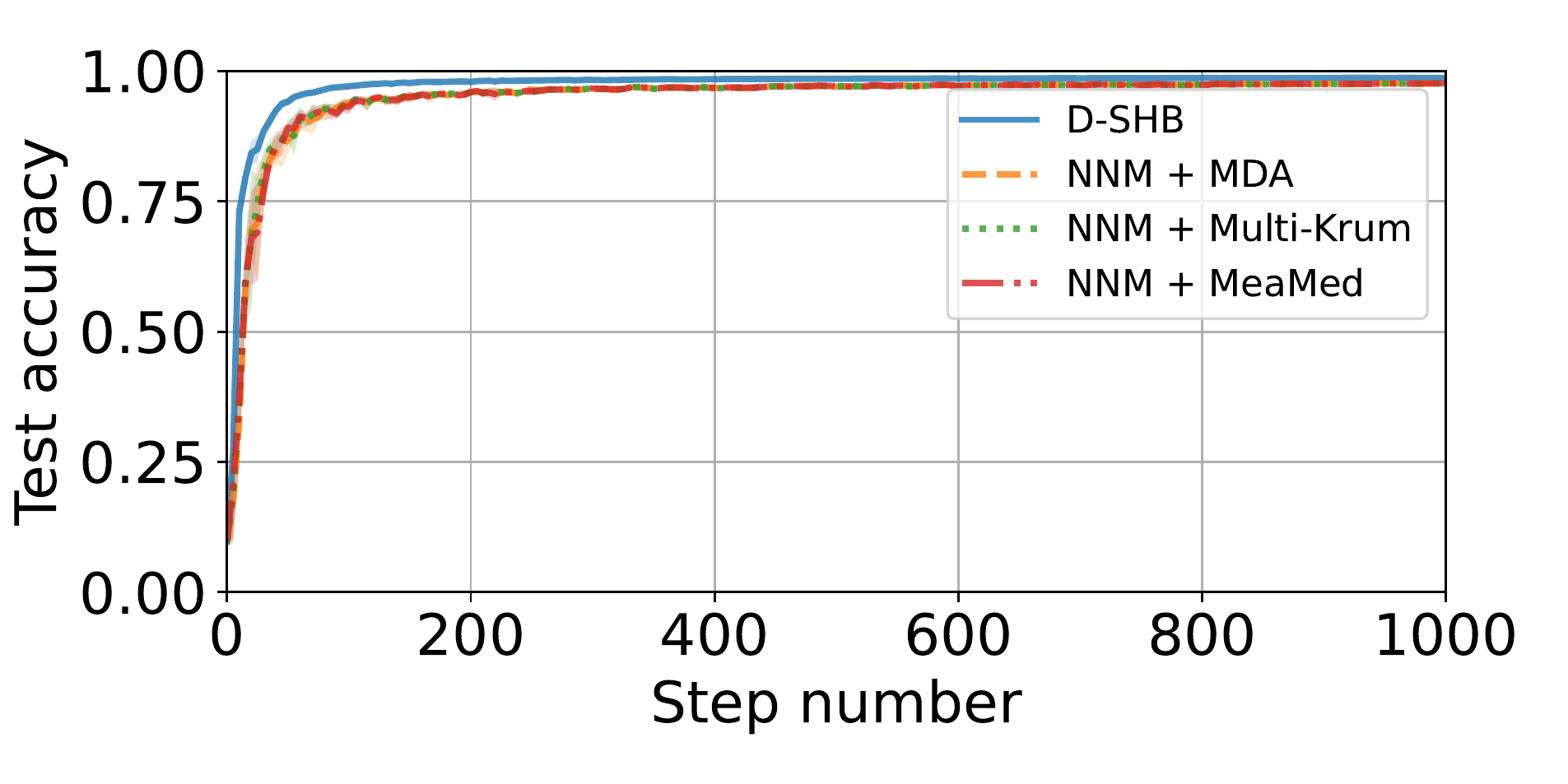}\\%
     \includegraphics[width=0.5\textwidth]{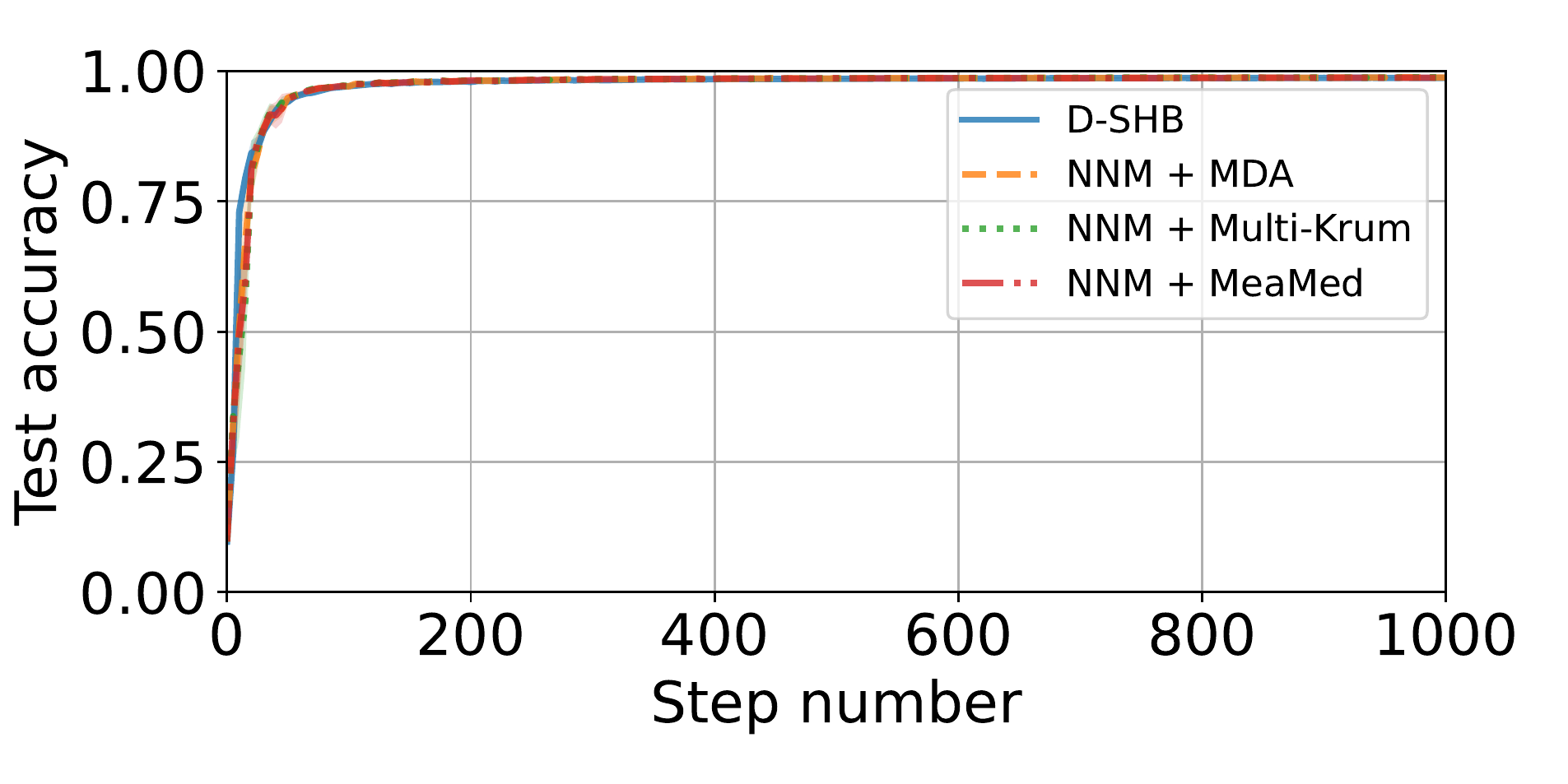}%
    \caption{Experiments on MNIST using robust D-SHB with $f = 4$ Byzantine among $n = 17$ workers, with $\beta = 0.9$ and $\alpha = 1$. The Byzantine workers execute the FOE (\textit{row 1, left}), ALIE (\textit{row 1, right}), Mimic (\textit{row 2, left}), SF (\textit{row 2, right}), and LF (\textit{row 3}) attacks.}
\label{fig:plots_other_2}
\end{figure*}

\begin{figure*}[ht!]
    \centering
    \includegraphics[width=0.5\textwidth]{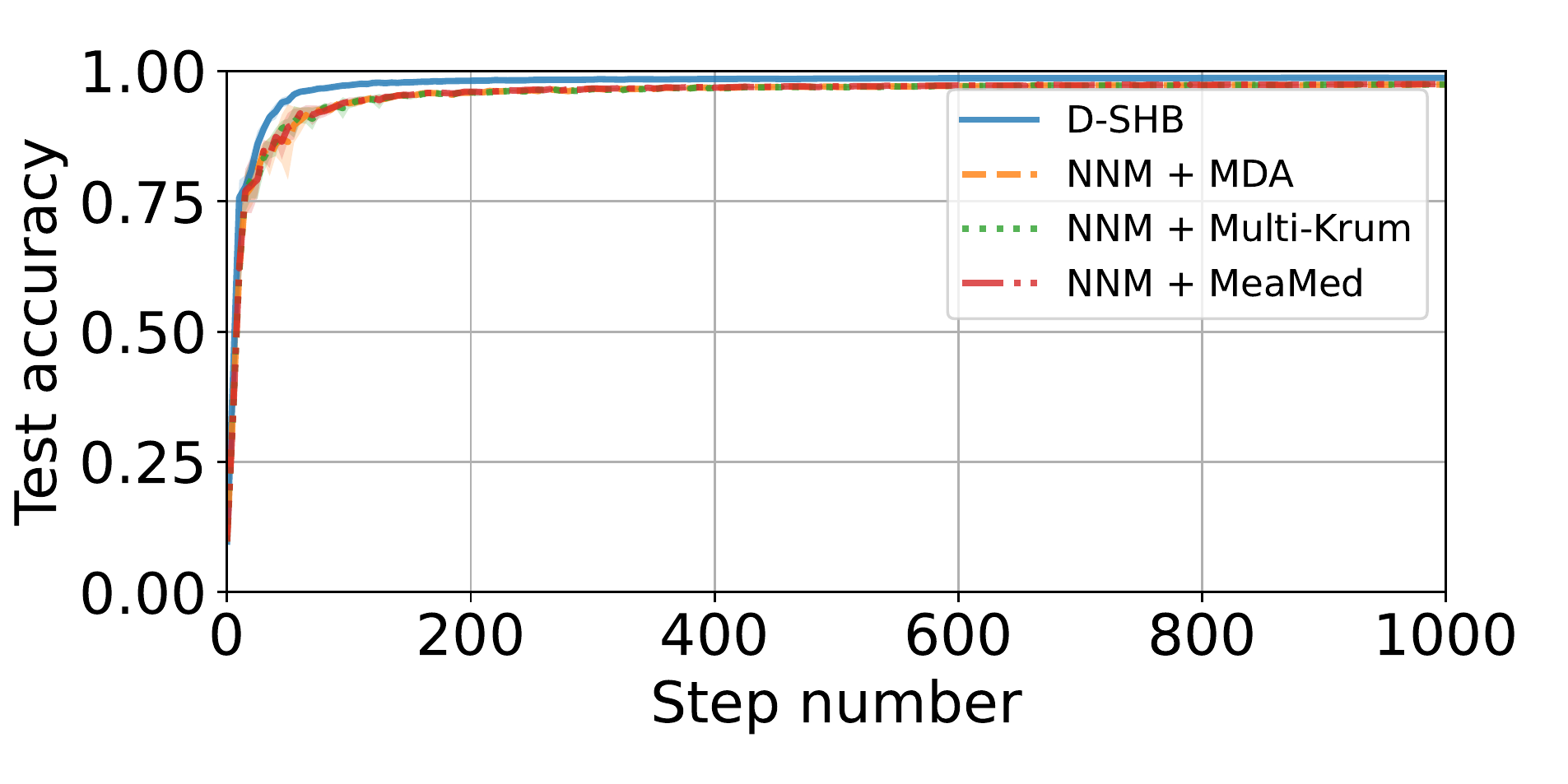}%
    \includegraphics[width=0.5\textwidth]{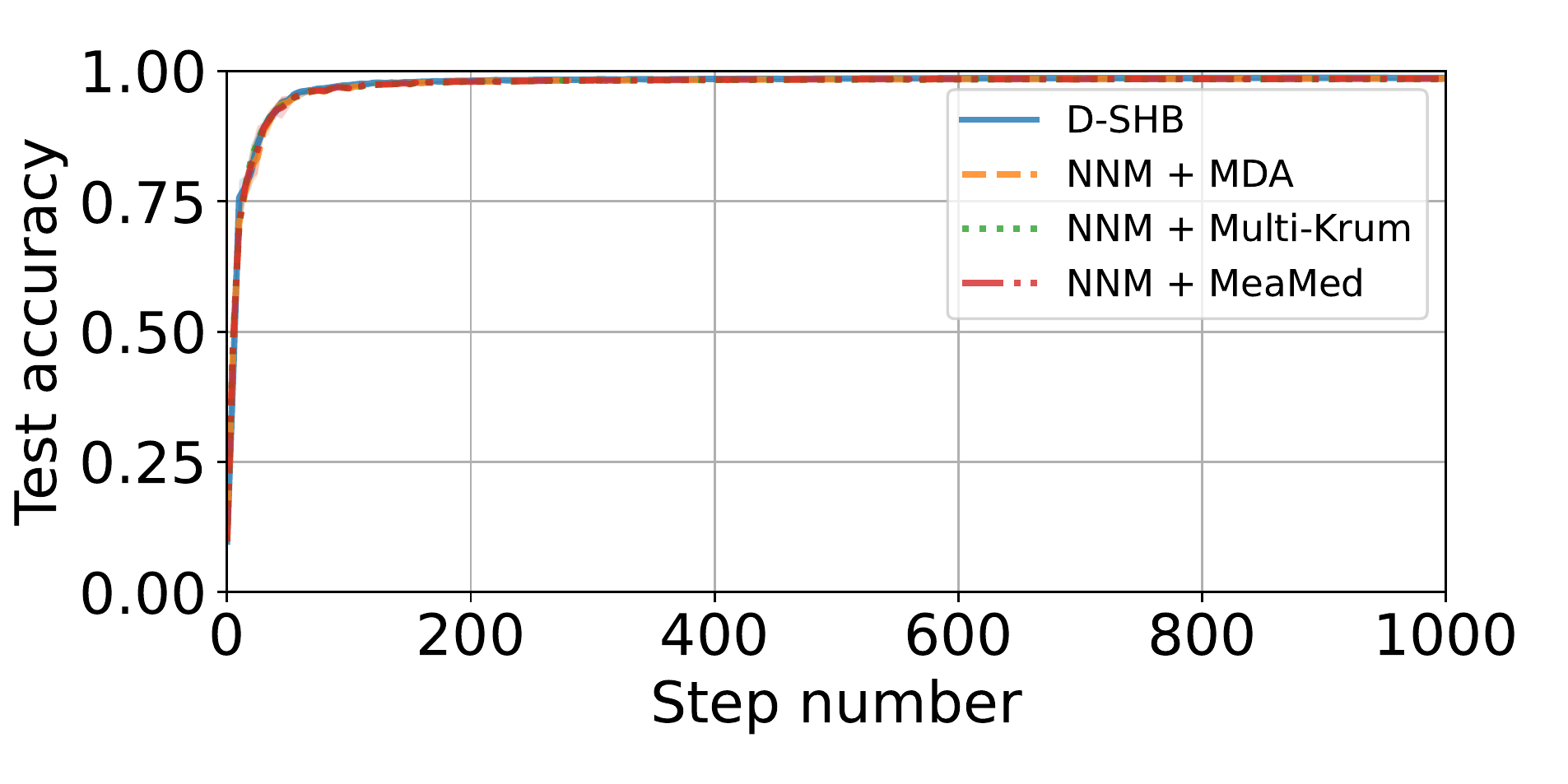}\\%
    \includegraphics[width=0.5\textwidth]{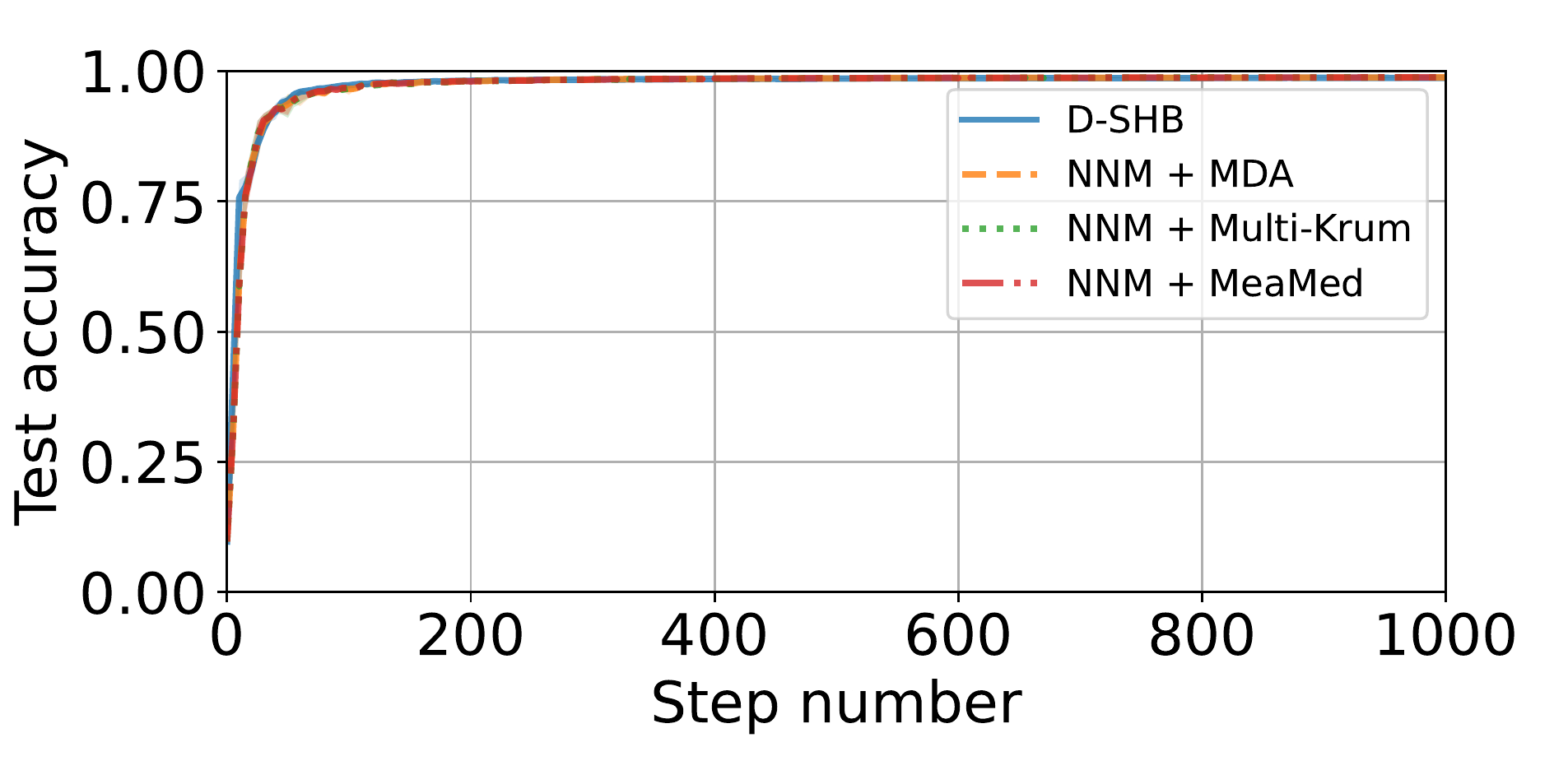}%
    \includegraphics[width=0.5\textwidth]{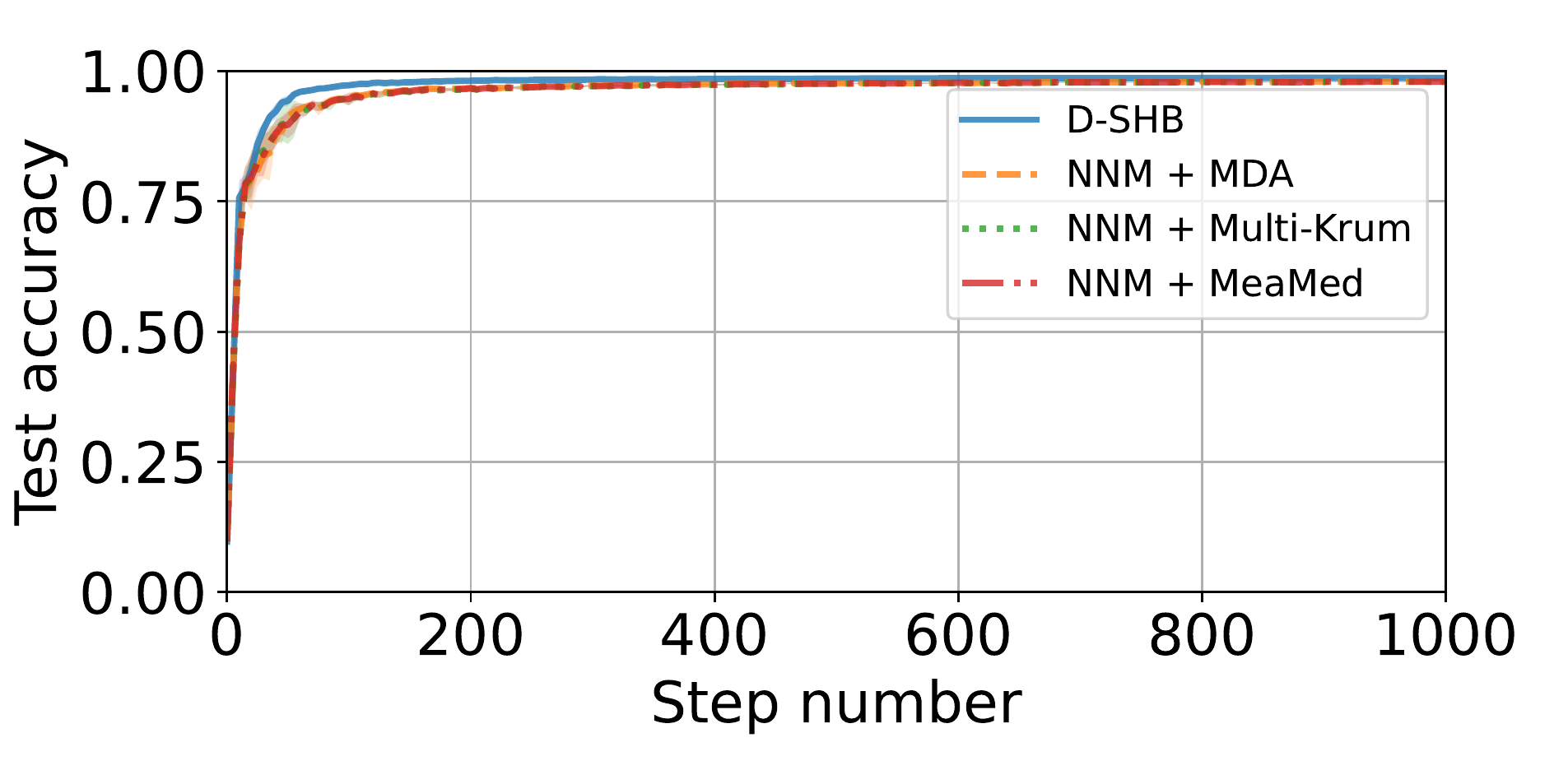}\\%
     \includegraphics[width=0.5\textwidth]{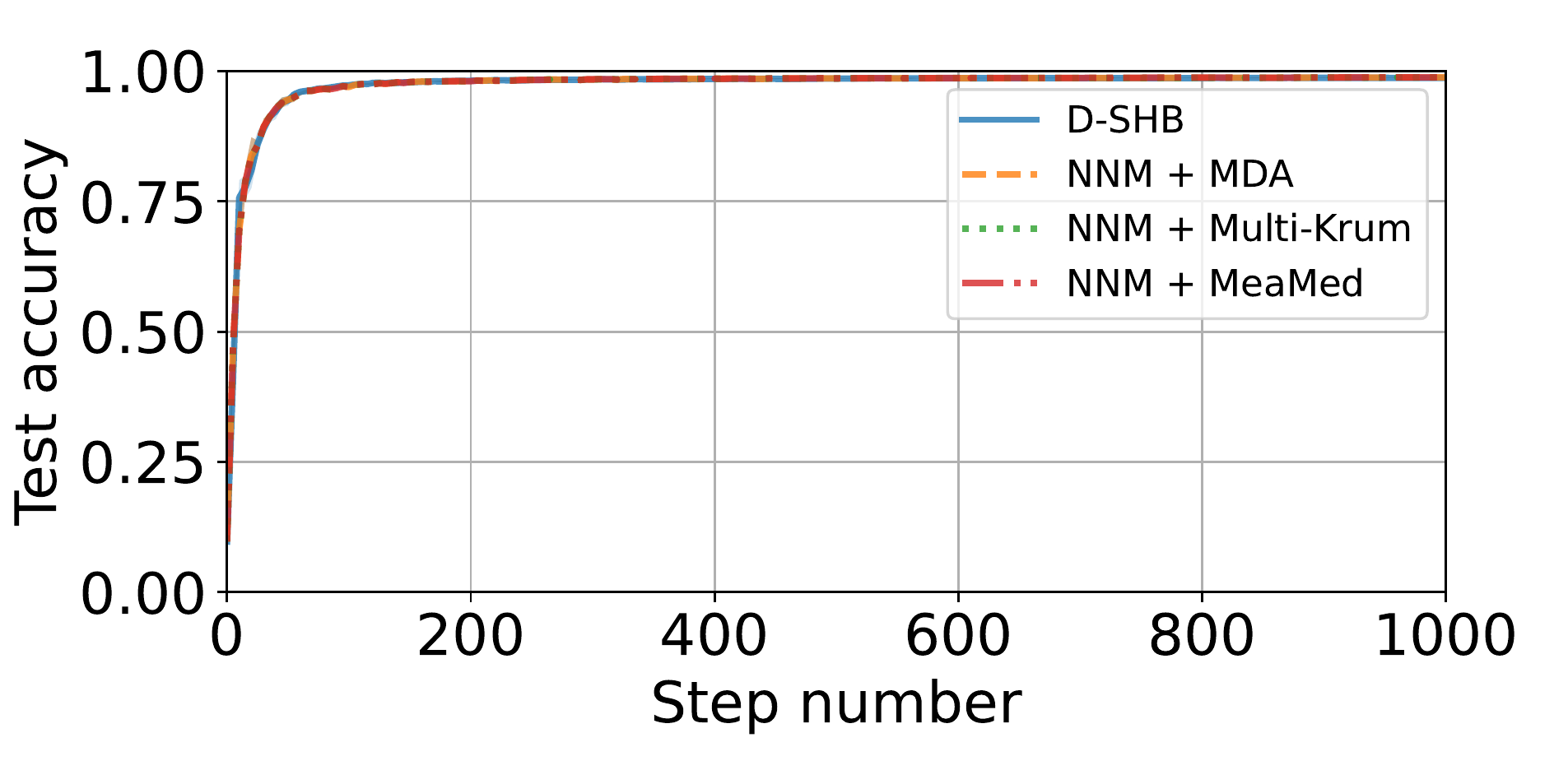}%
    \caption{Experiments on MNIST using robust D-SHB with $f = 4$ Byzantine among $n = 17$ workers, with $\beta = 0.9$ and $\alpha = 10$. The Byzantine workers execute the FOE (\textit{row 1, left}), ALIE (\textit{row 1, right}), Mimic (\textit{row 2, left}), SF (\textit{row 2, right}), and LF (\textit{row 3}) attacks.}
\label{fig:plots_other_3}
\end{figure*}

\begin{figure*}[ht!]
    \centering
    \includegraphics[width=0.5\textwidth]{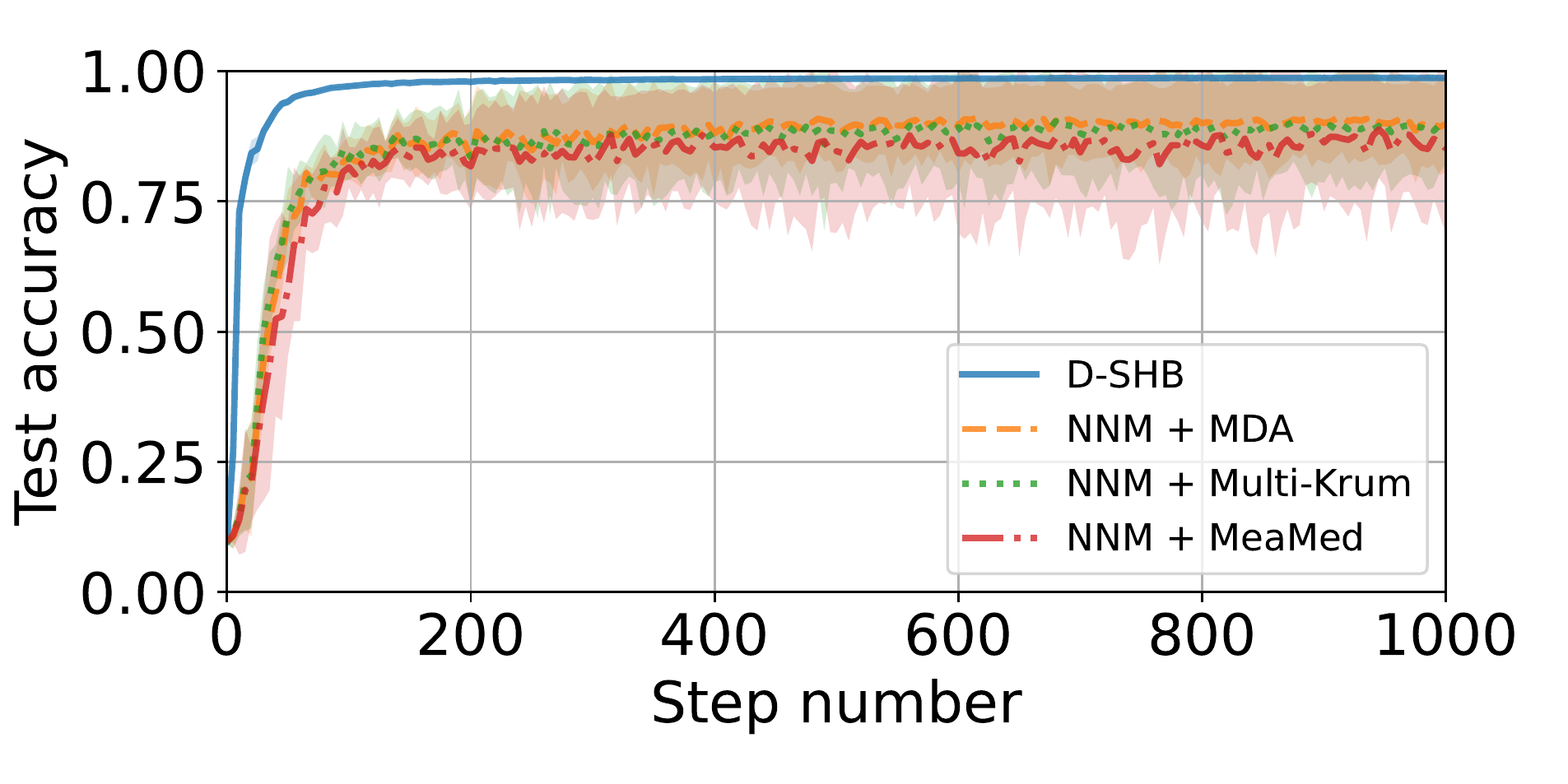}%
    \includegraphics[width=0.5\textwidth]{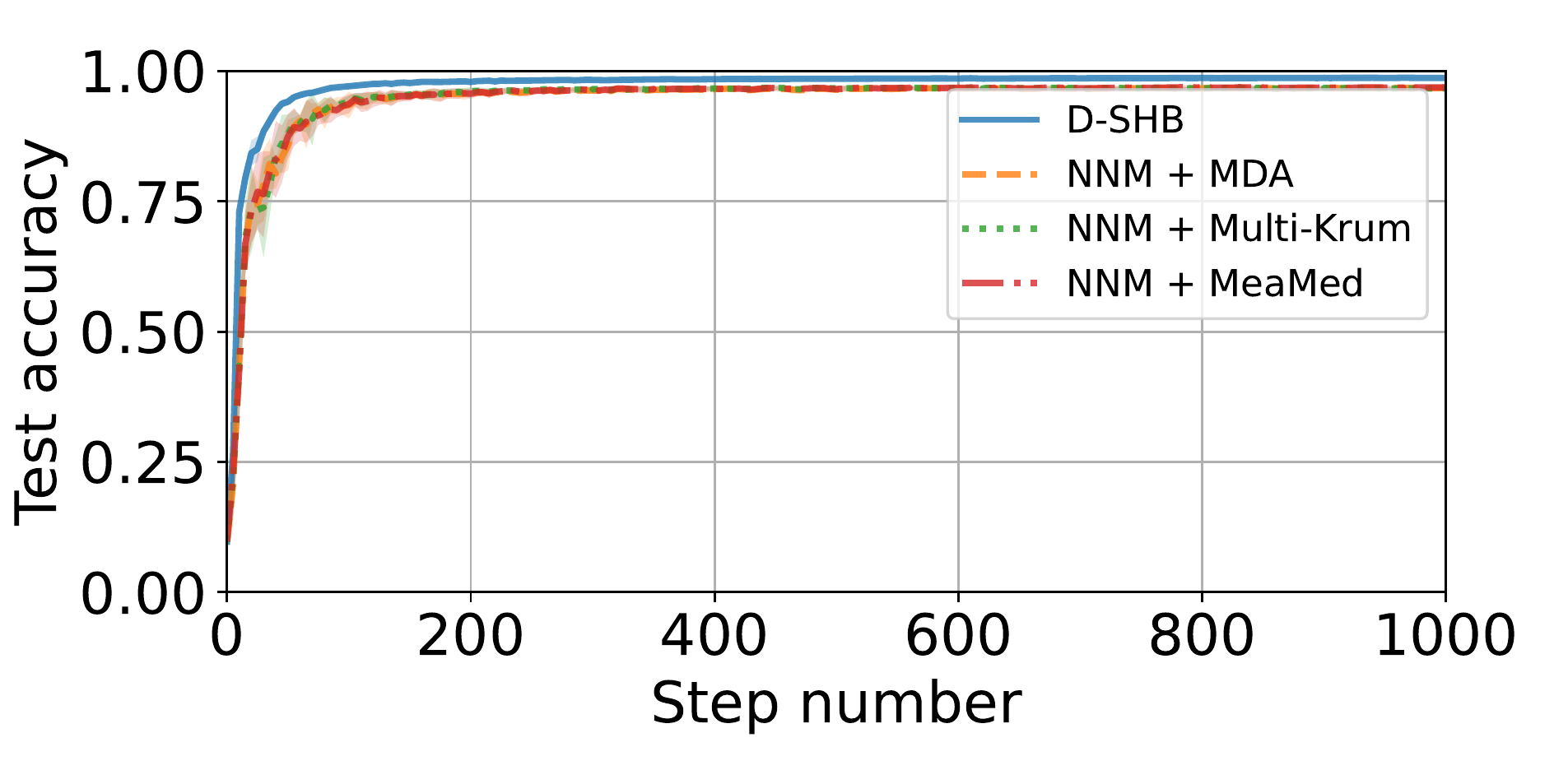}\\%
    \includegraphics[width=0.5\textwidth]{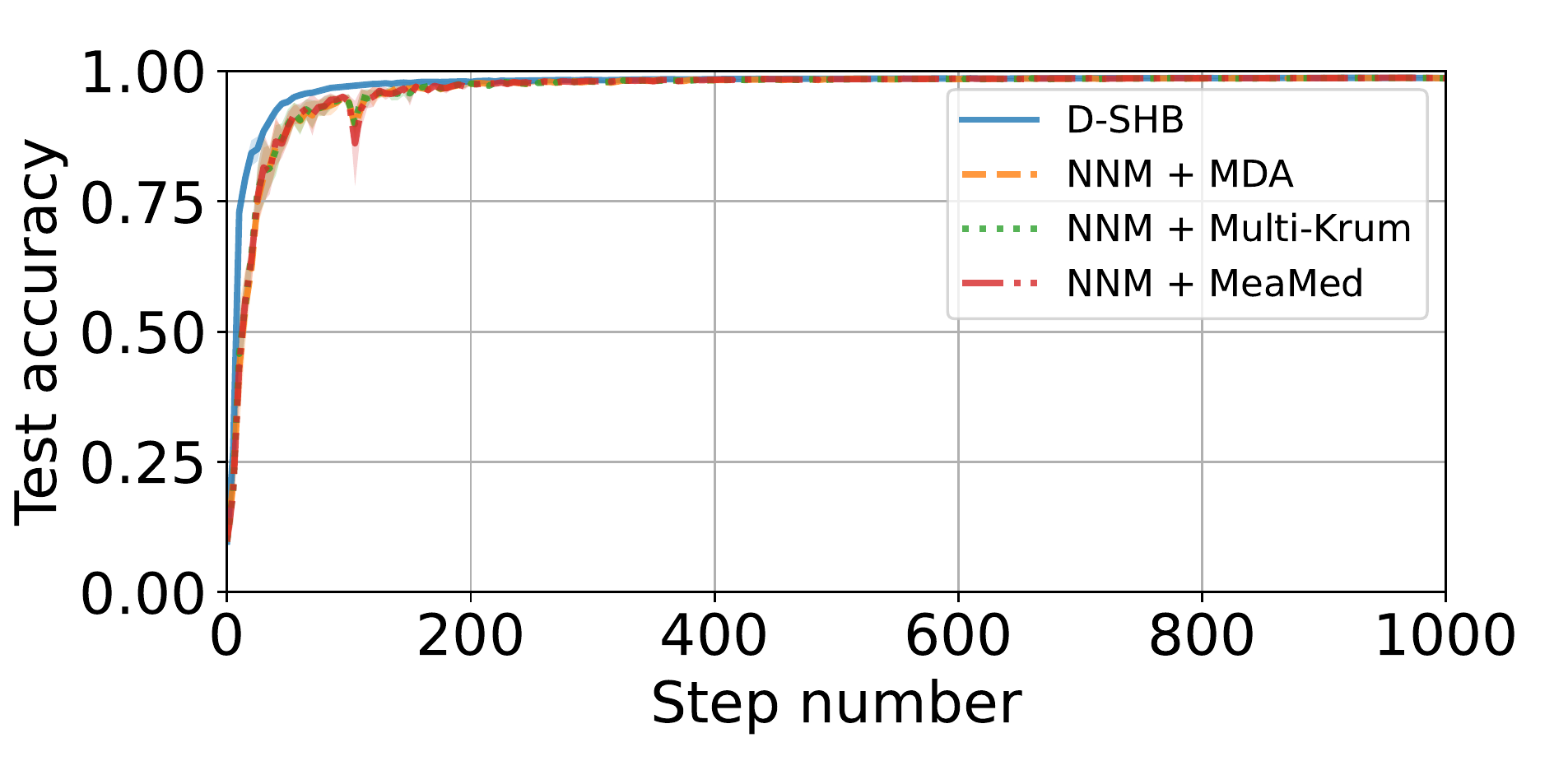}%
    \includegraphics[width=0.5\textwidth]{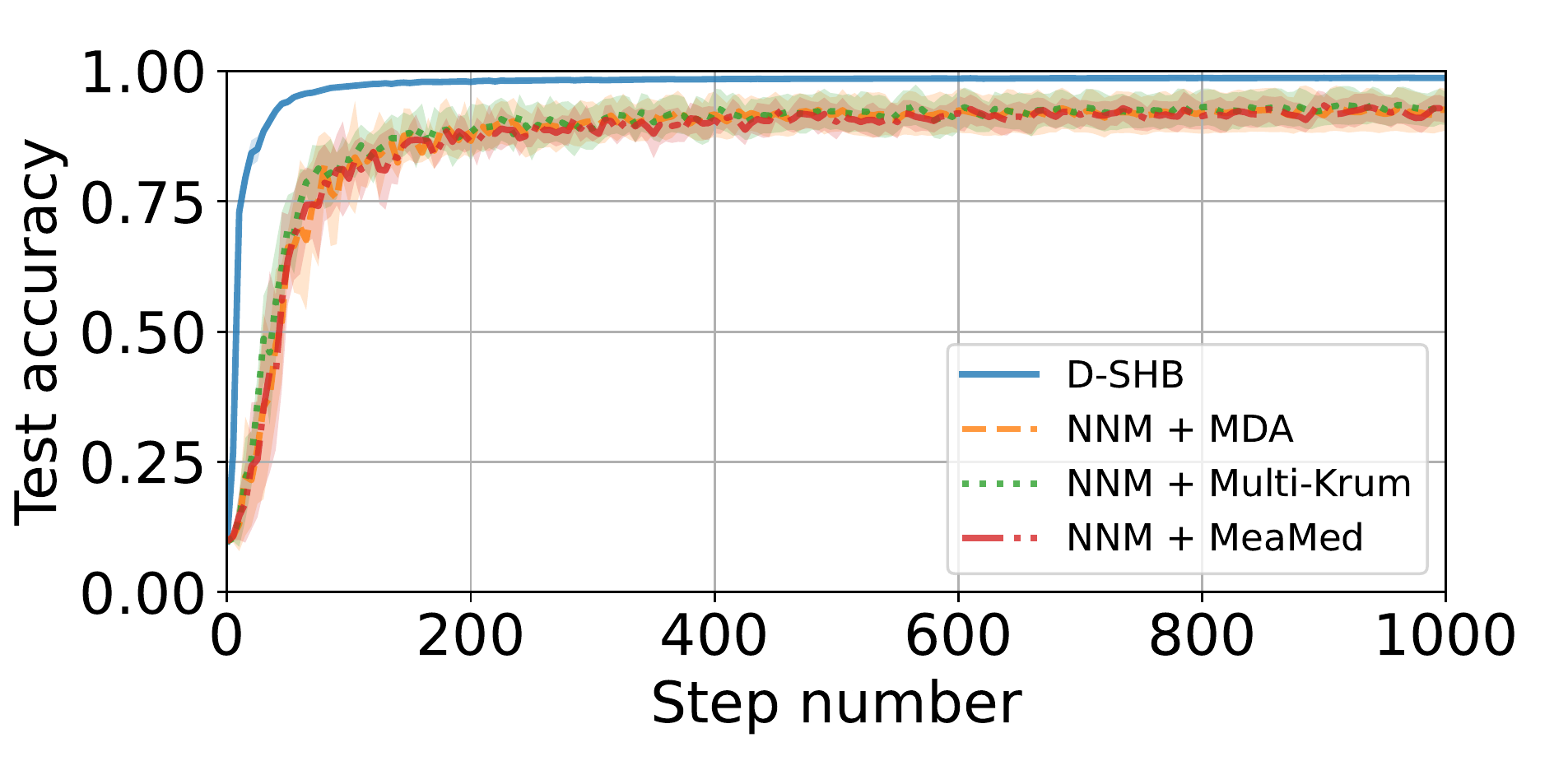}\\%
     \includegraphics[width=0.5\textwidth]{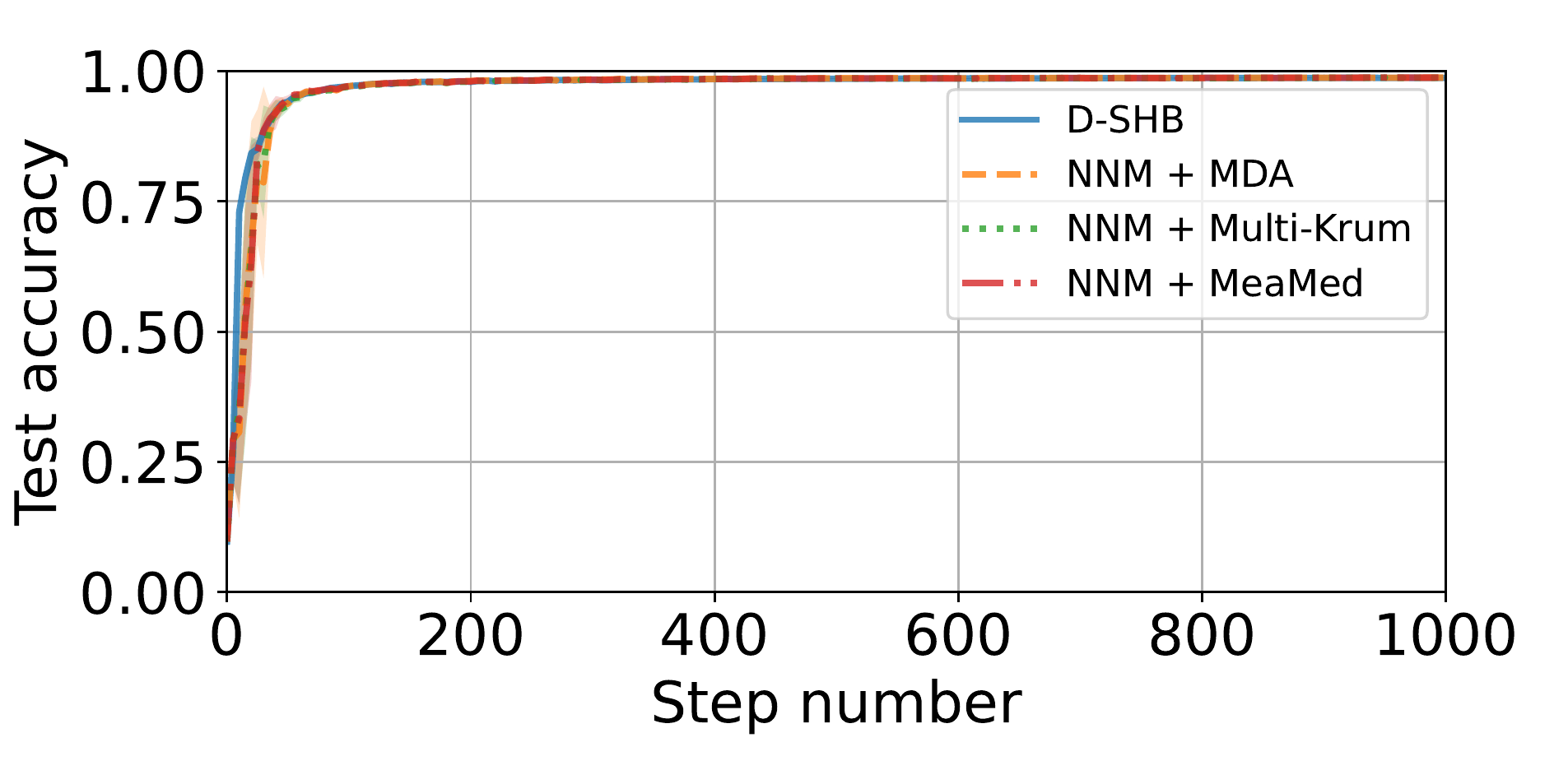}%
    \caption{Experiments on MNIST using robust D-SHB with $f = 6$ Byzantine among $n = 17$ workers, with $\beta = 0.9$ and $\alpha = 1$. The Byzantine workers execute the FOE (\textit{row 1, left}), ALIE (\textit{row 1, right}), Mimic (\textit{row 2, left}), SF (\textit{row 2, right}), and LF (\textit{row 3}) attacks.}
\label{fig:plots_other_4}
\end{figure*}

\begin{figure*}[ht!]
    \centering
    \includegraphics[width=0.5\textwidth]{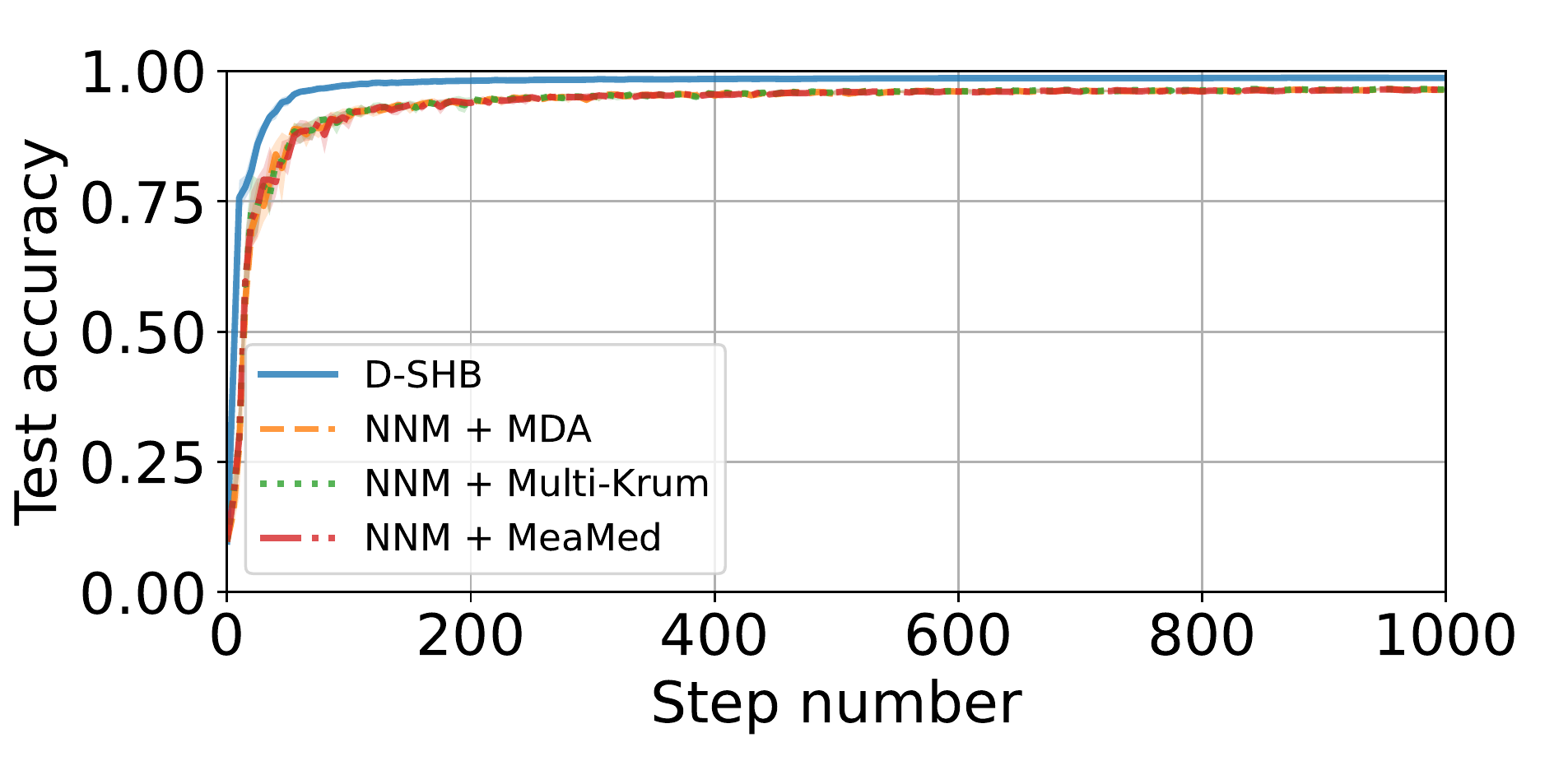}%
    \includegraphics[width=0.5\textwidth]{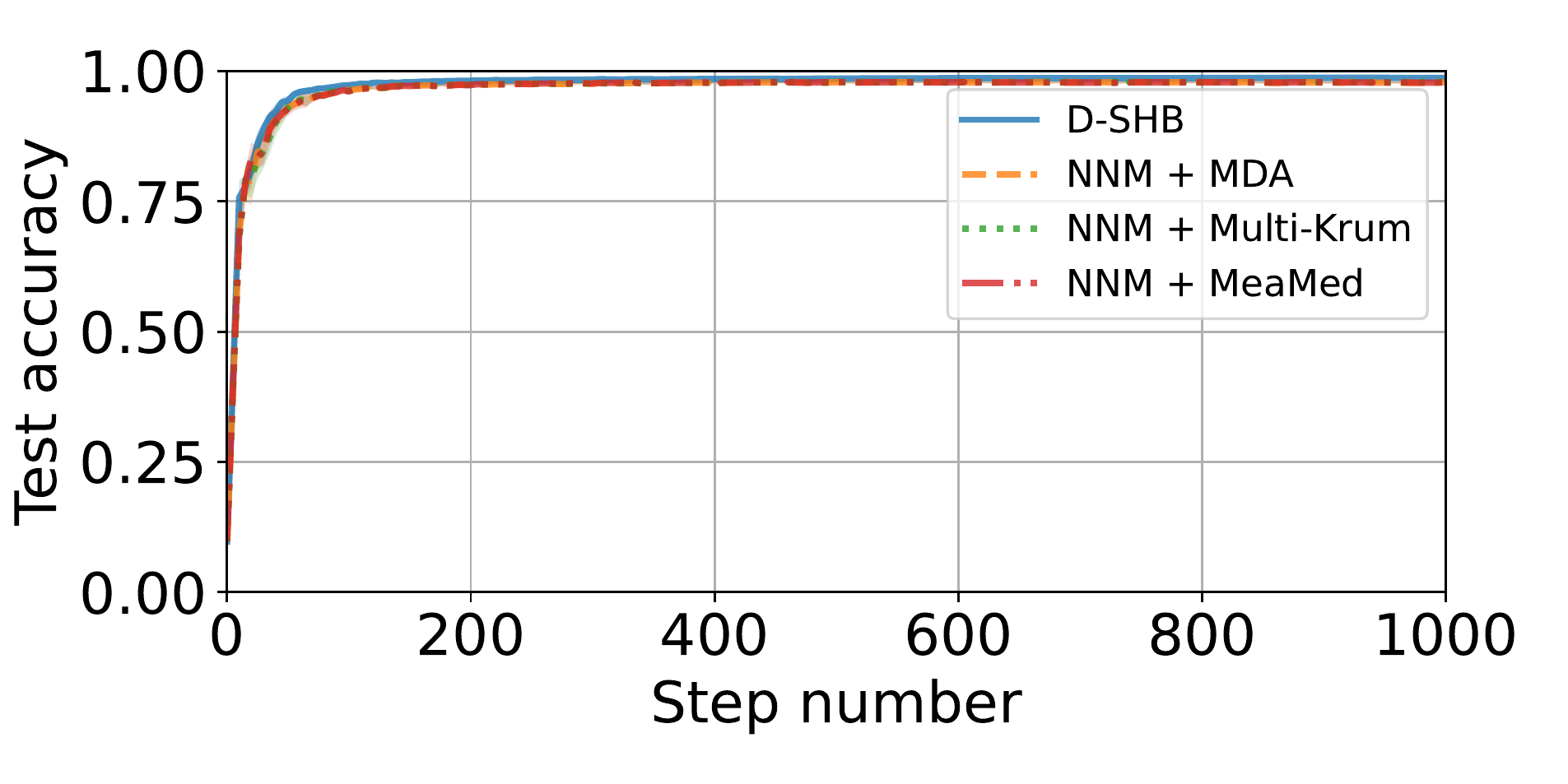}\\%
    \includegraphics[width=0.5\textwidth]{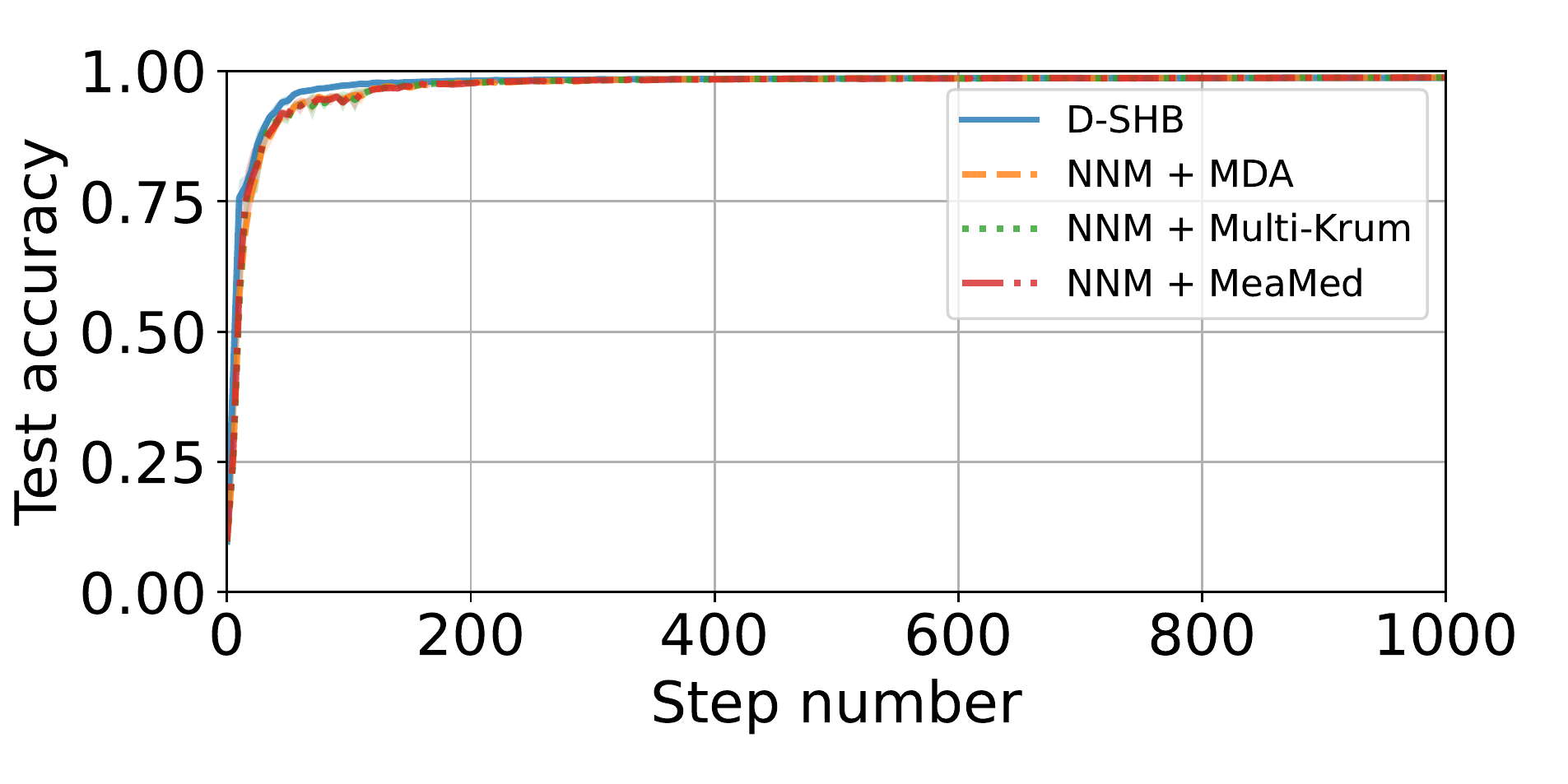}%
    \includegraphics[width=0.5\textwidth]{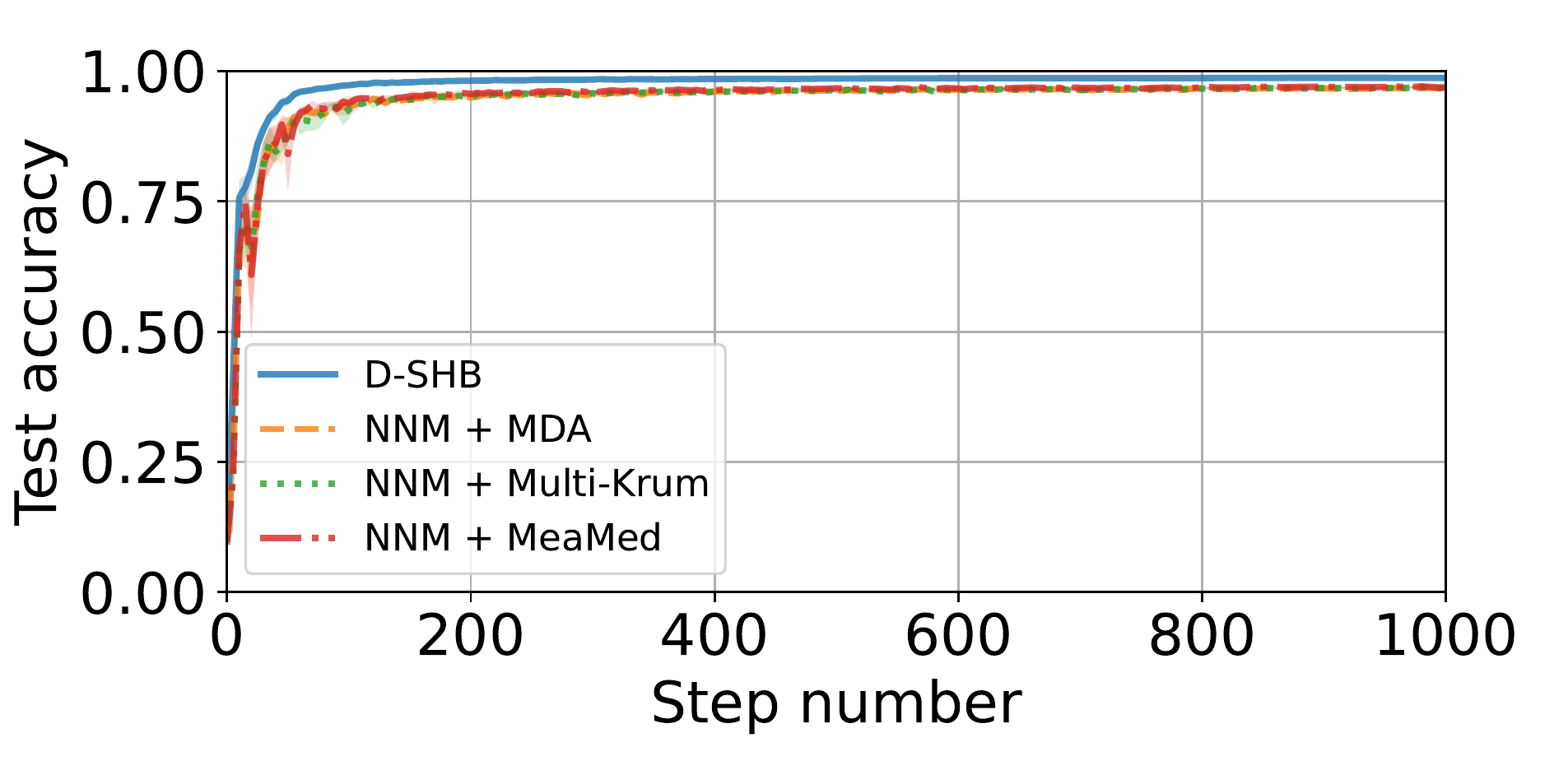}\\%
     \includegraphics[width=0.5\textwidth]{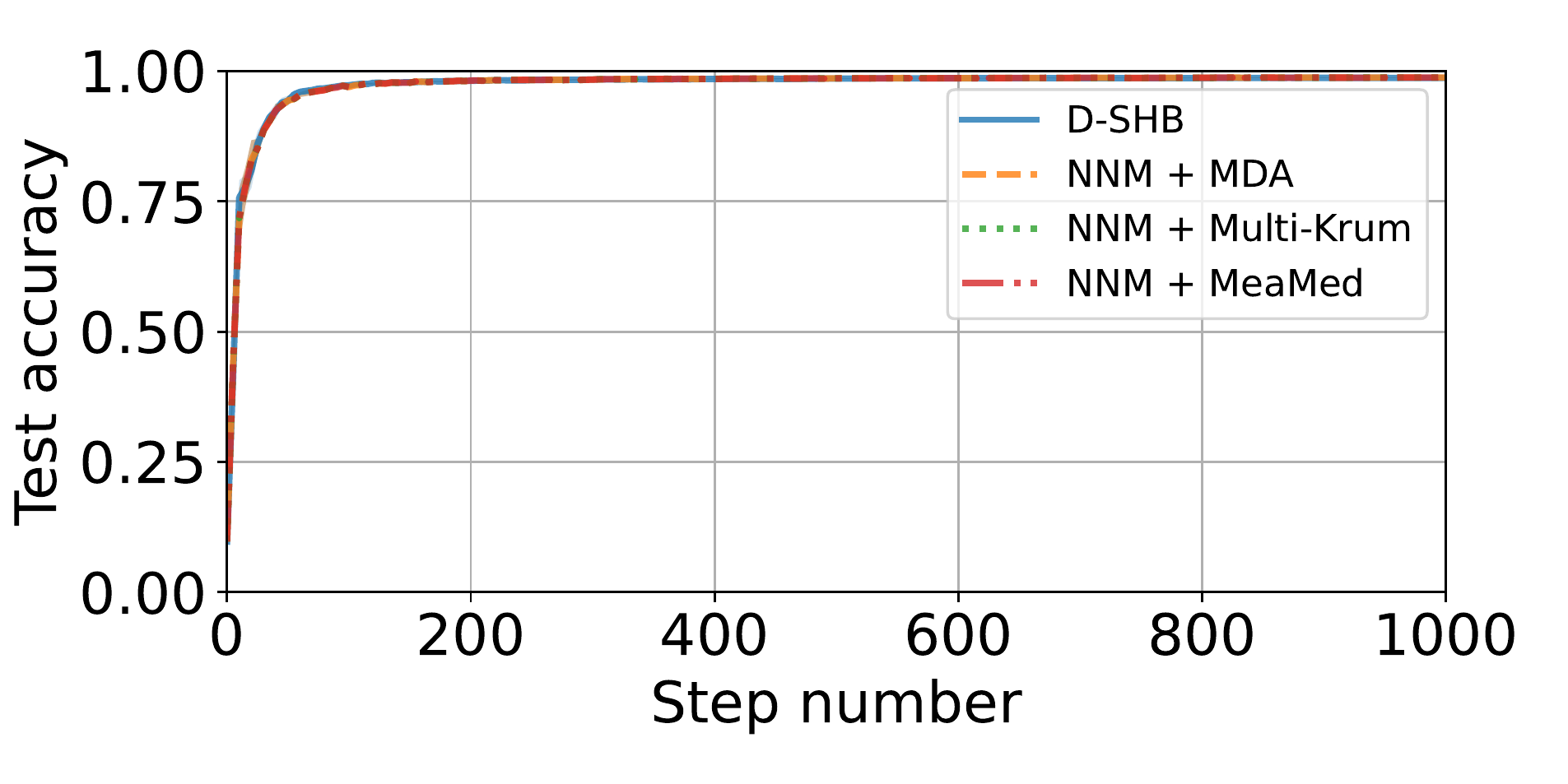}%
    \caption{Experiments on MNIST using robust D-SHB with $f = 6$ Byzantine among $n = 17$ workers, with $\beta = 0.9$ and $\alpha = 10$. The Byzantine workers execute the FOE (\textit{row 1, left}), ALIE (\textit{row 1, right}), Mimic (\textit{row 2, left}), SF (\textit{row 2, right}), and LF (\textit{row 3}) attacks.}
\label{fig:plots_other_5}
\end{figure*}

\begin{figure*}[ht!]
    \centering
    \includegraphics[width=0.5\textwidth]{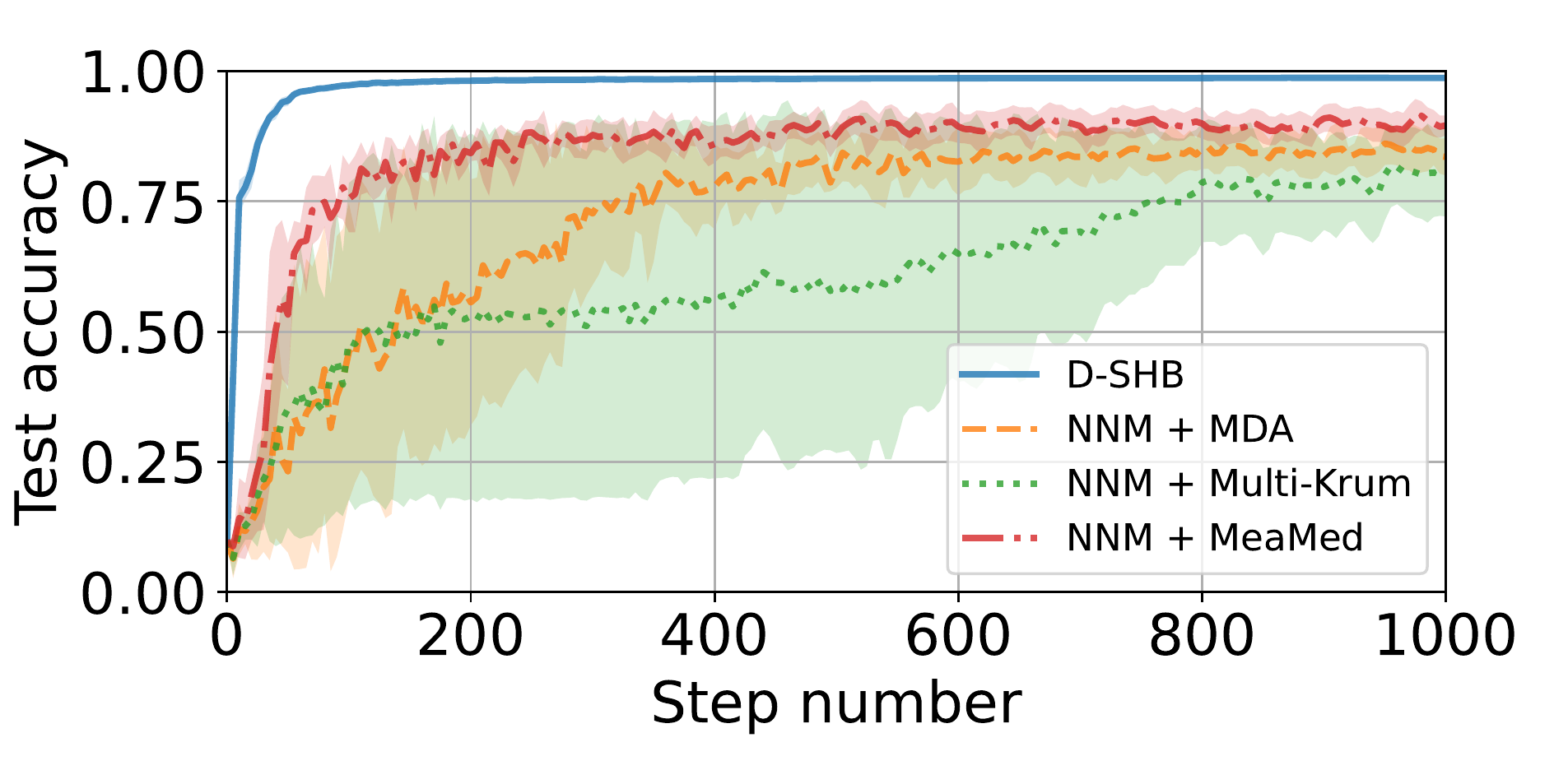}%
    \includegraphics[width=0.5\textwidth]{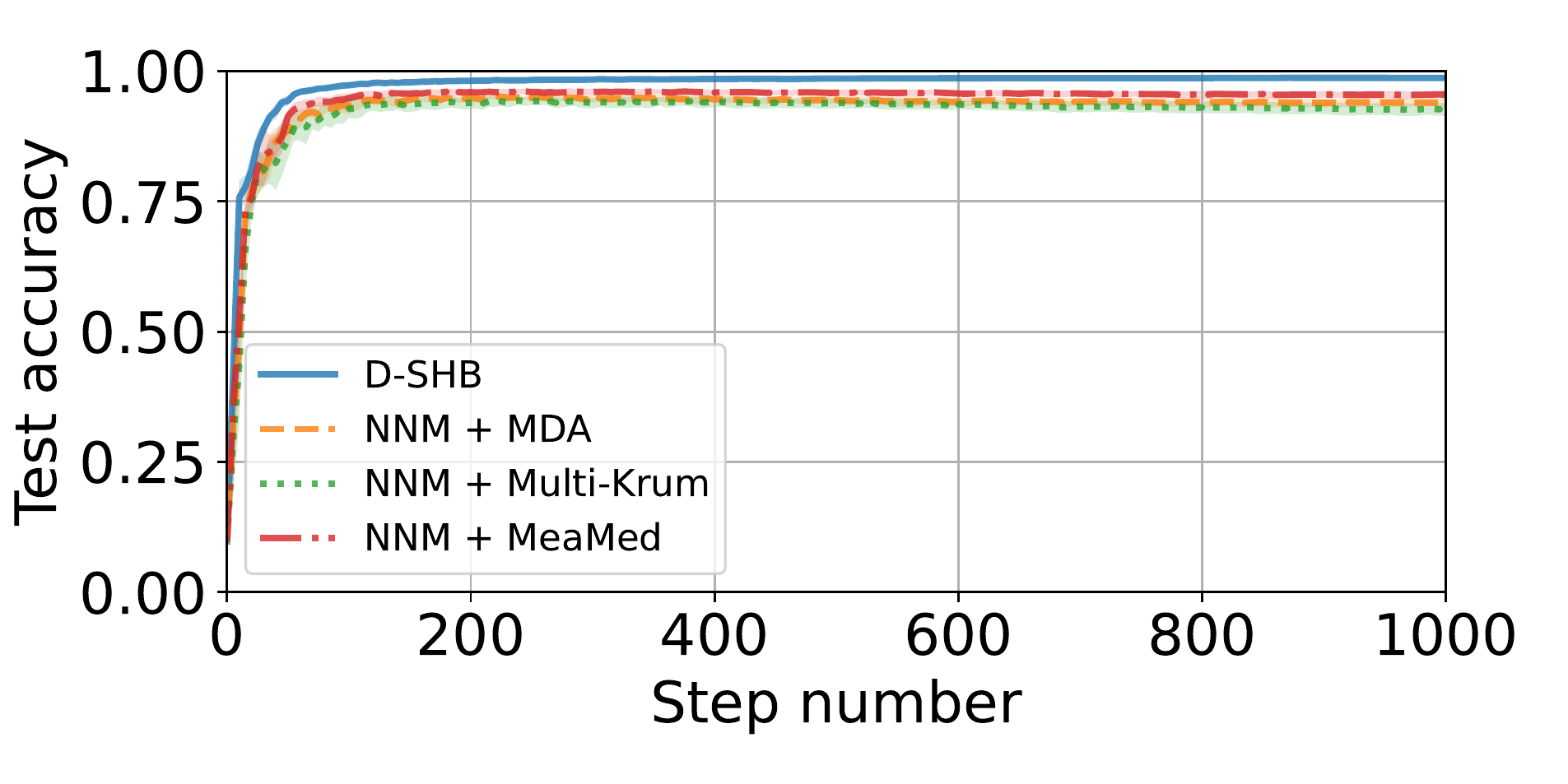}\\%
    \includegraphics[width=0.5\textwidth]{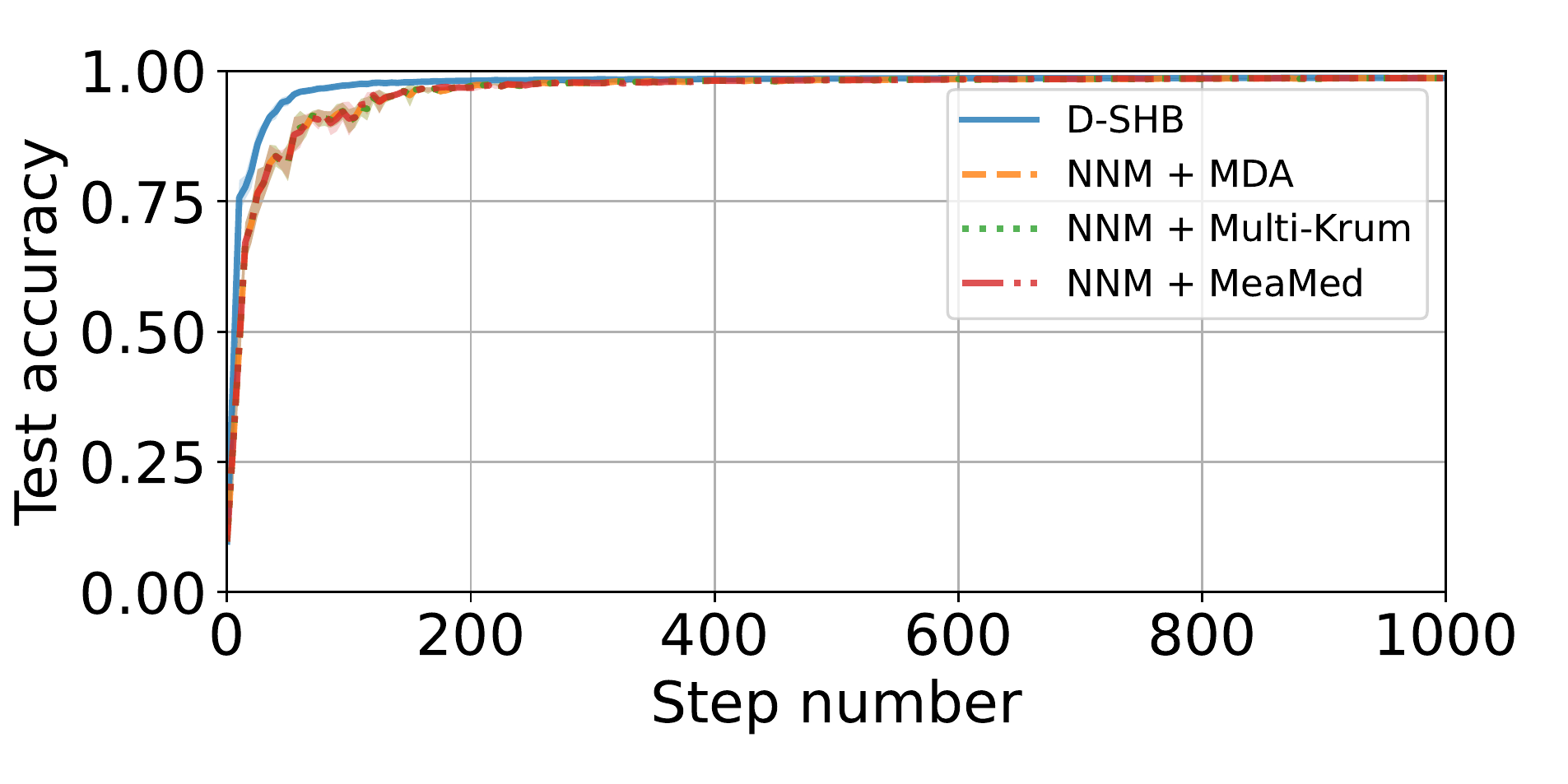}%
    \includegraphics[width=0.5\textwidth]{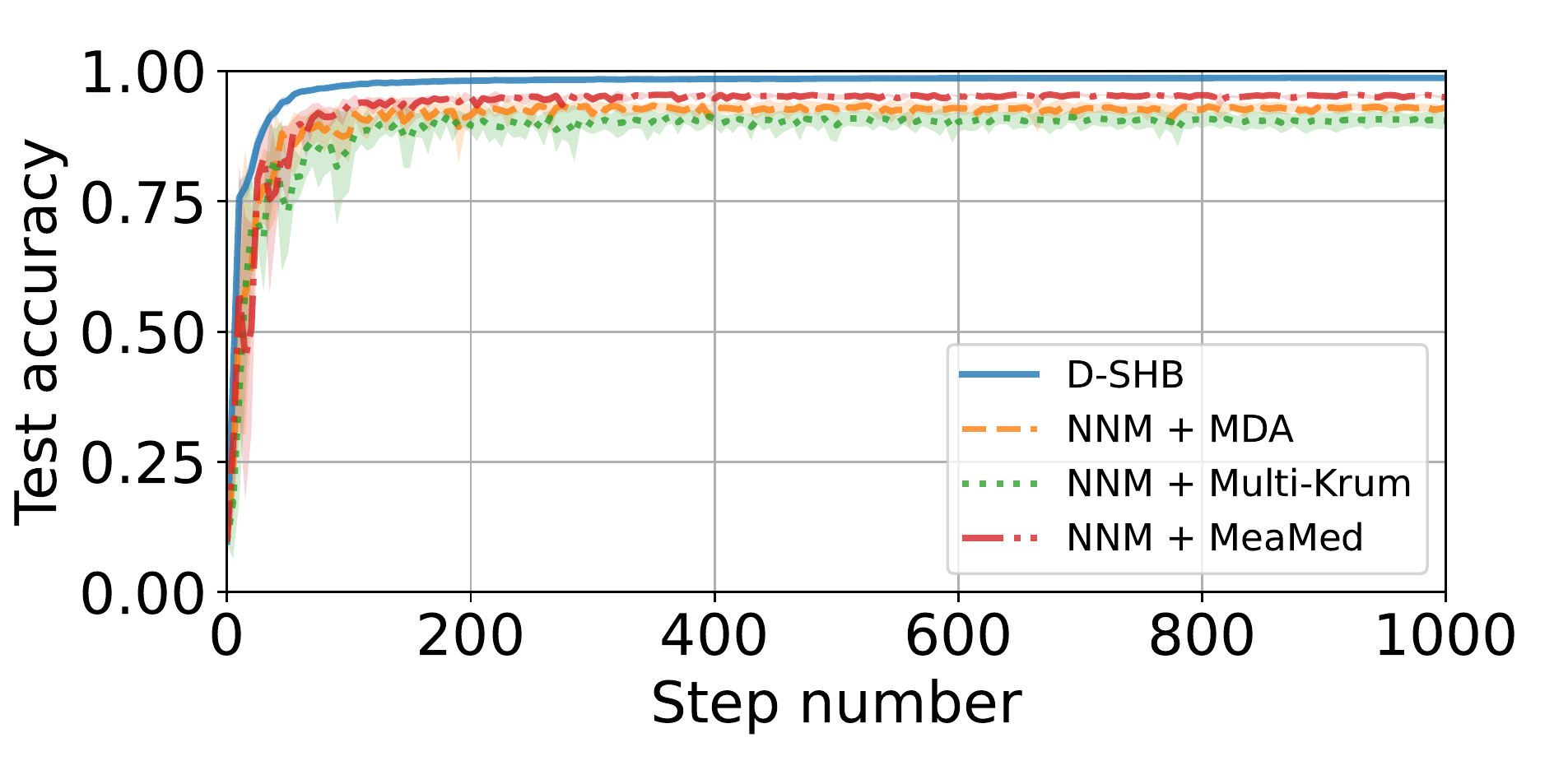}\\%
     \includegraphics[width=0.5\textwidth]{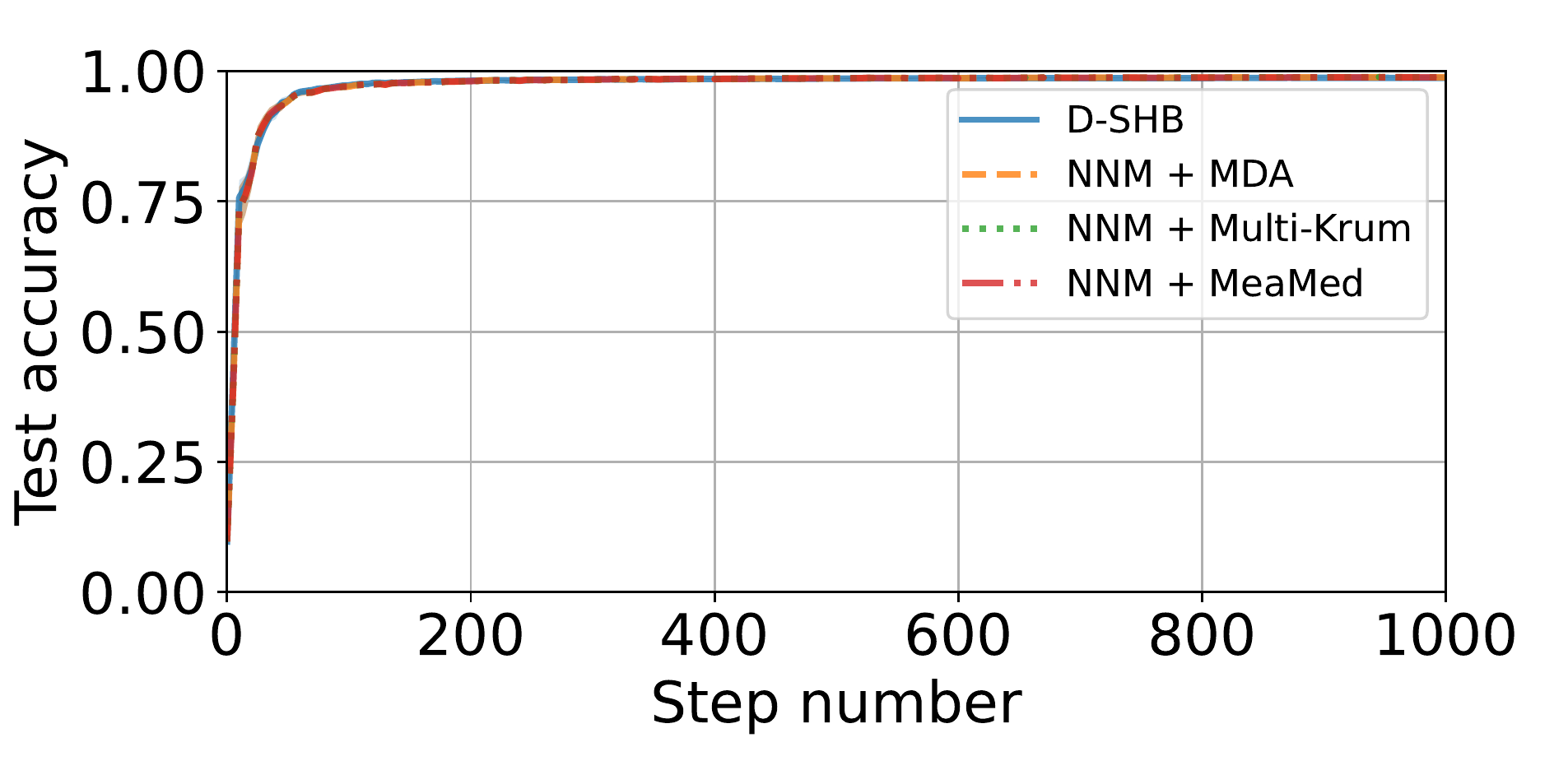}%
    \caption{Experiments on MNIST using robust D-SHB with $f = 8$ Byzantine among $n = 17$ workers, with $\beta = 0.9$ and $\alpha = 10$. The Byzantine workers execute the FOE (\textit{row 1, left}), ALIE (\textit{row 1, right}), Mimic (\textit{row 2, left}), SF (\textit{row 2, right}), and LF (\textit{row 3}) attacks.}
\label{fig:plots_other_6}
\end{figure*}

\clearpage
\subsection{Comprehensive Results on Fashion-MNIST}\label{app:exp_results_fashion}
In this section, we showcase our results on the Fashion-MNIST dataset. We consider two Byzantine regimes: $f = 4$ and $f = 6$ out of $n=17$ workers in total. We also consider three heterogeneity regimes: $\alpha=0.1$ (extreme), $\alpha=1$ (moderate), and $\alpha=10$ (low).
We plots compare the performance of $\cenna{}$ and Bucketing when executed with four aggregation rules namely Krum, GM, CWMed, and CWTM. The plots are presented below.
We note that in the strongest Byzantine setting, i.e., when $f > 4$, Bucketing can only be applied with bucket size $s=1$~\cite{karimireddy2022byzantinerobust}. This means that when $f=6$, executing Bucketing boils down to running the vanilla aggregation rules.

\begin{figure*}[ht!]
    \centering
    \includegraphics[width=0.5\textwidth]{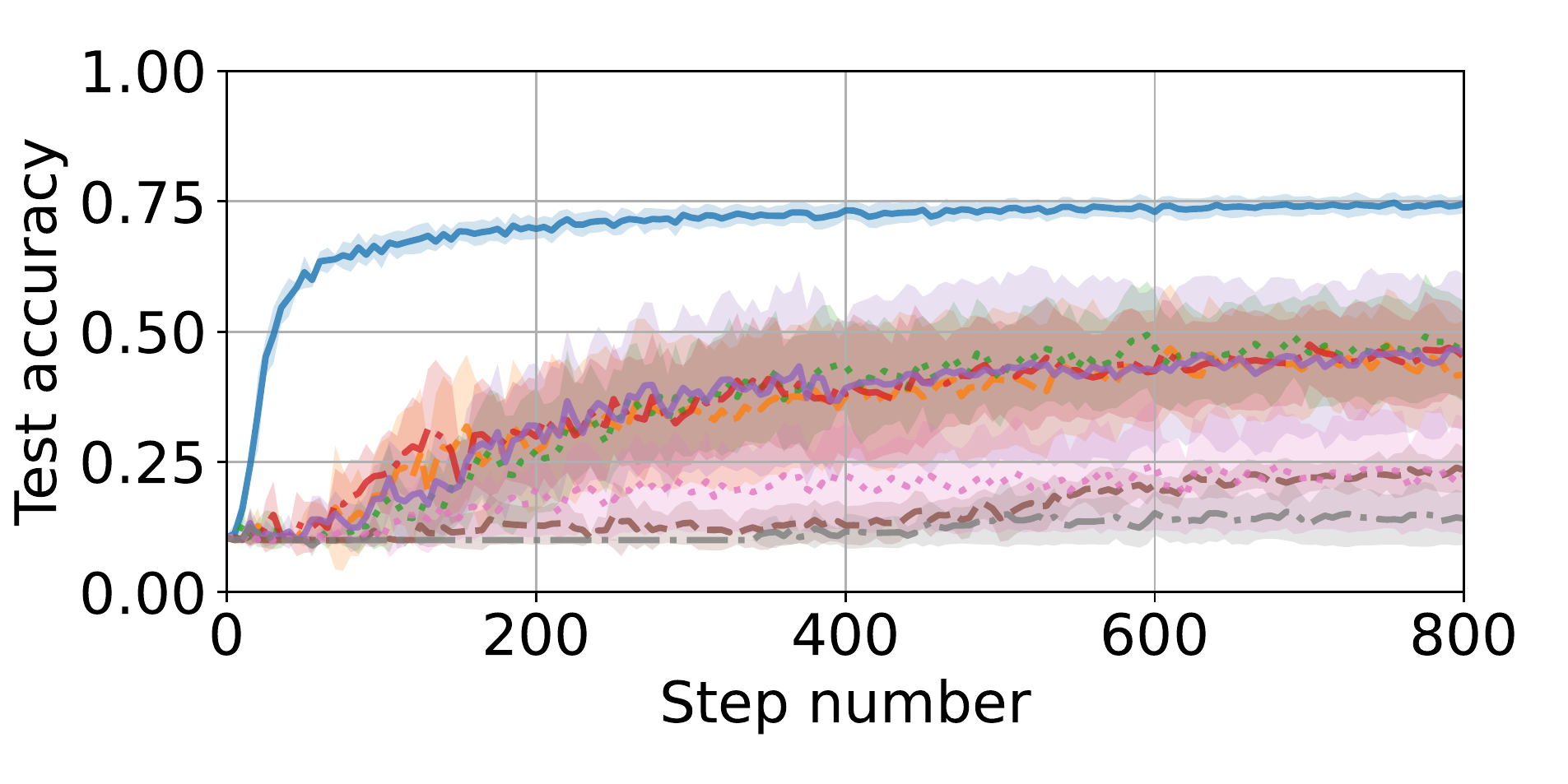}%
    \includegraphics[width=0.5\textwidth]{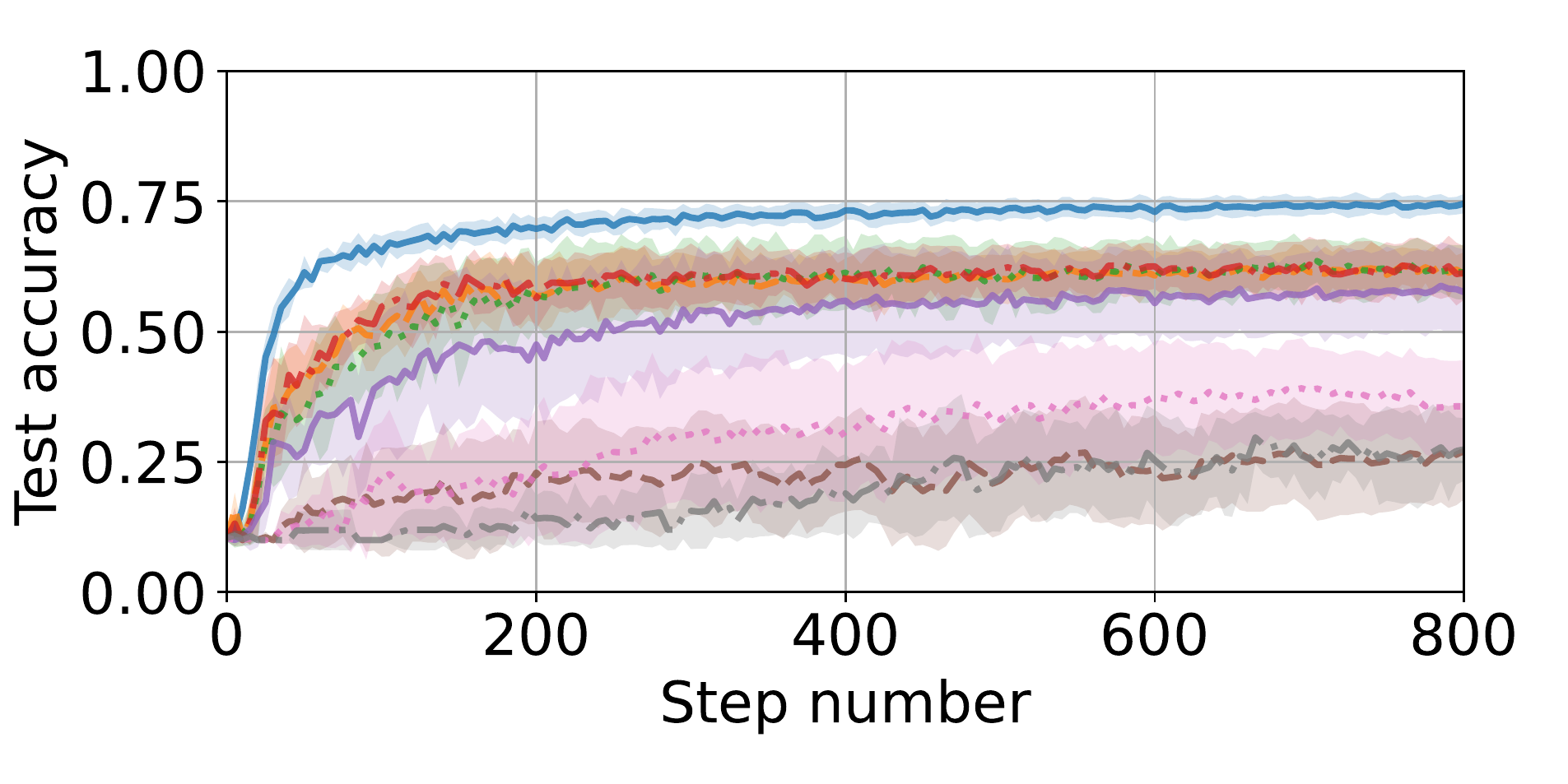}\\%
    \includegraphics[width=0.5\textwidth]{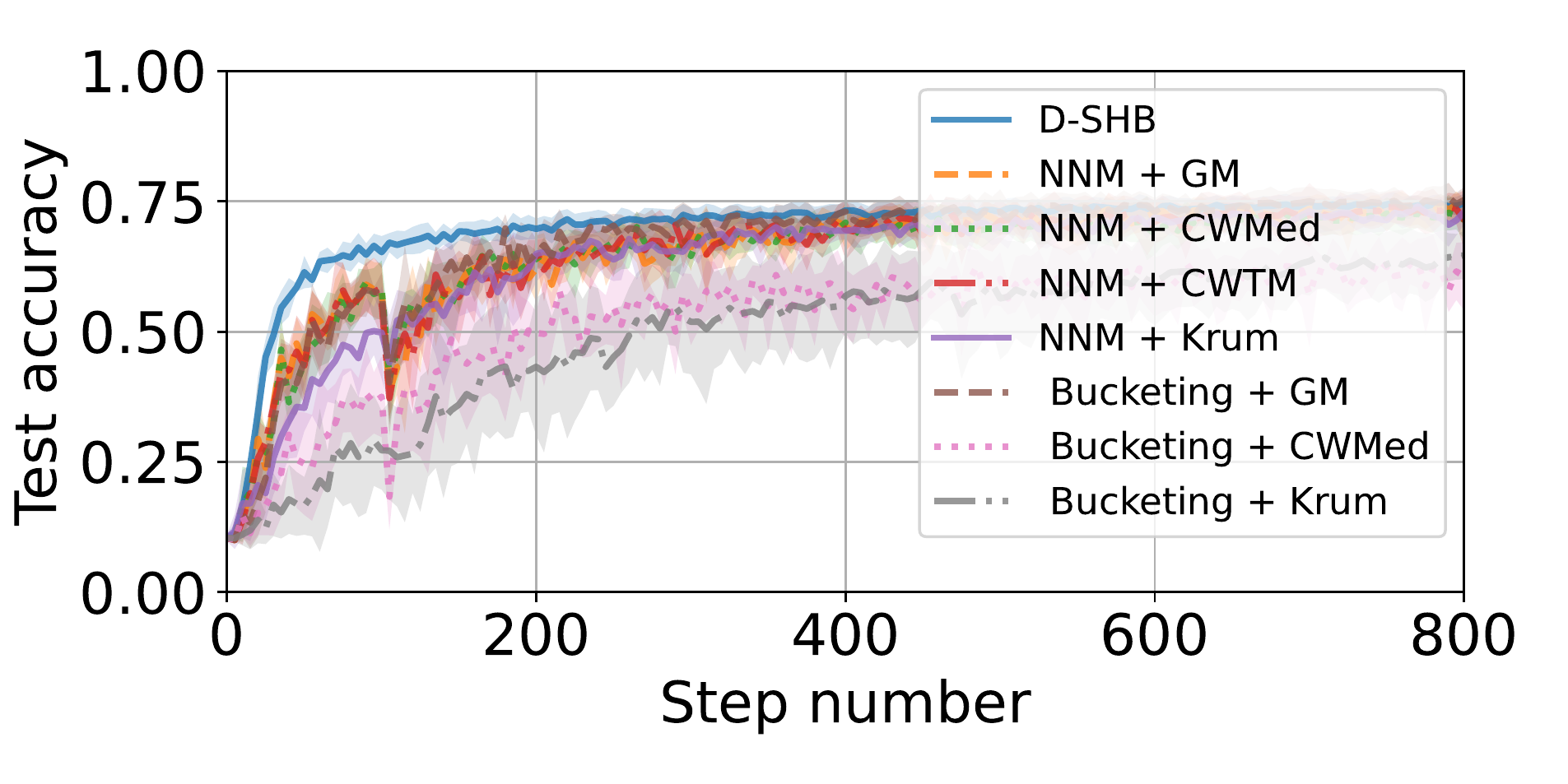}%
    \includegraphics[width=0.5\textwidth]{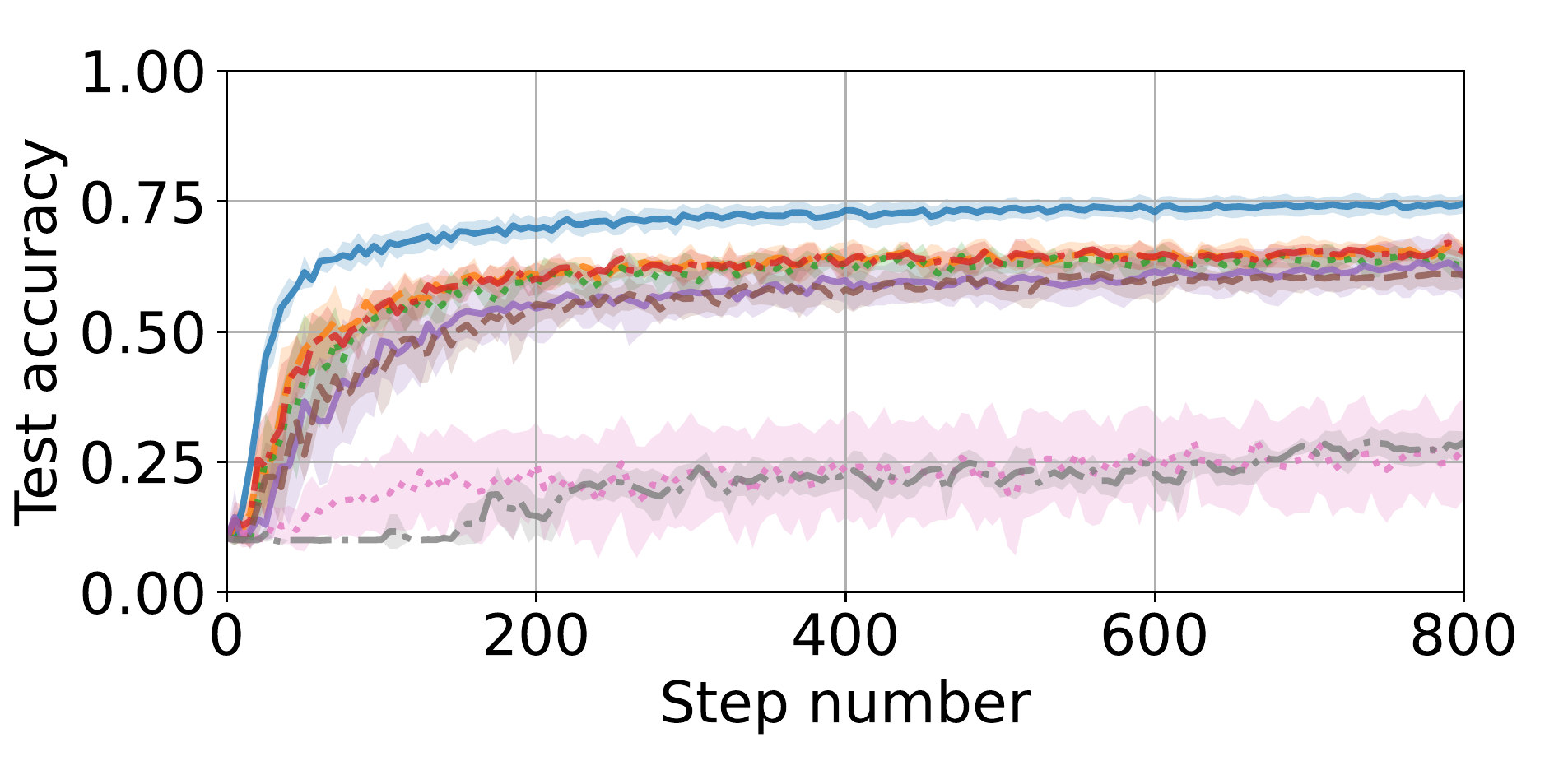}\\%
     \includegraphics[width=0.5\textwidth]{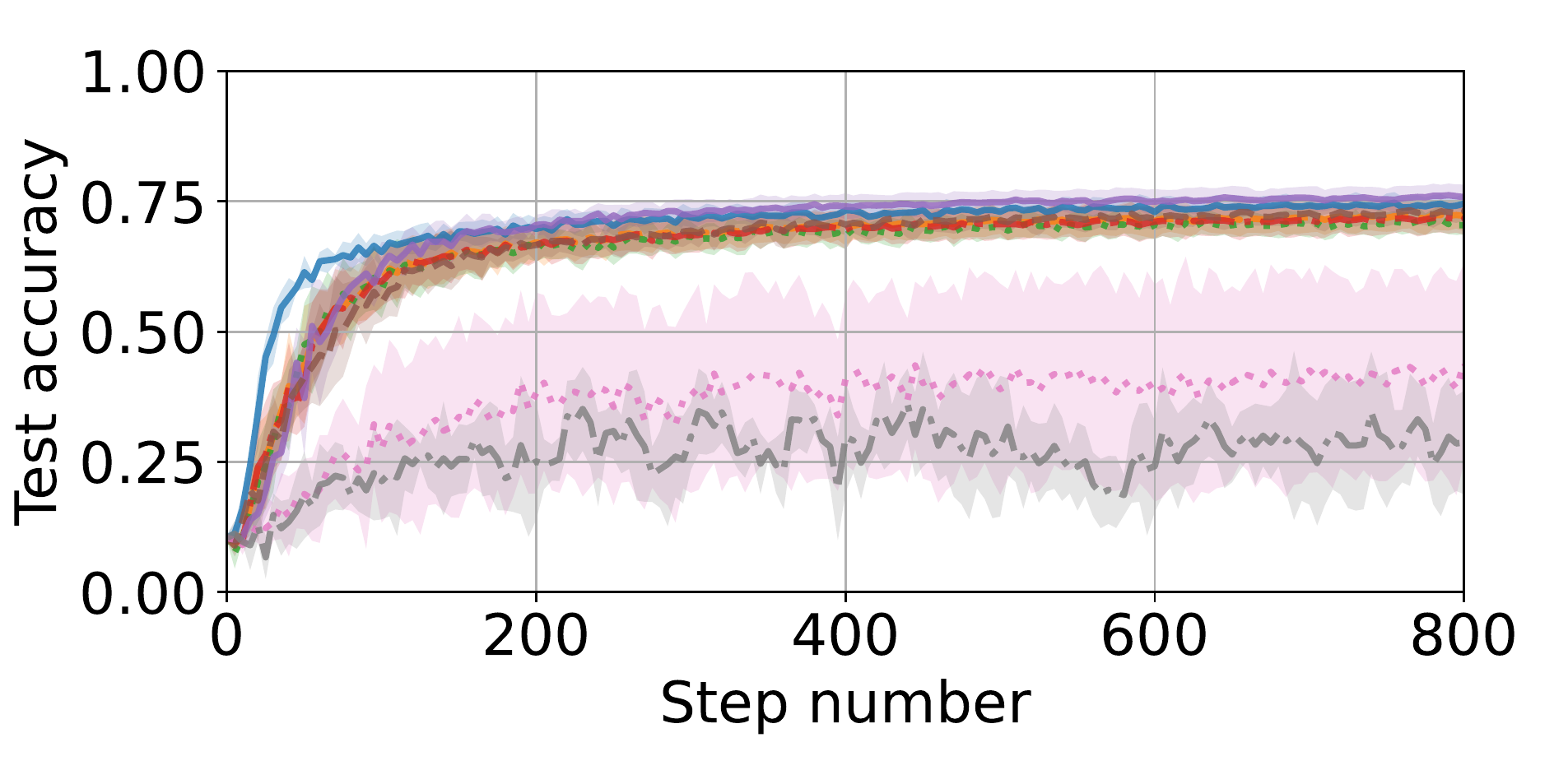}%
    \caption{Experiments on Fashion-MNIST using robust D-SHB with $f = 4$ Byzantine among $n = 17$ workers, with $\beta = 0.9$ and $\alpha = 0.1$. The Byzantine workers execute the FOE (\textit{row 1, left}), ALIE (\textit{row 1, right}), Mimic (\textit{row 2, left}), SF (\textit{row 2, right}), and LF (\textit{row 3}) attacks.}
\label{fig:plots_fashionmnist_1}
\end{figure*}

\begin{figure*}[ht!]
    \centering
    \includegraphics[width=0.5\textwidth]{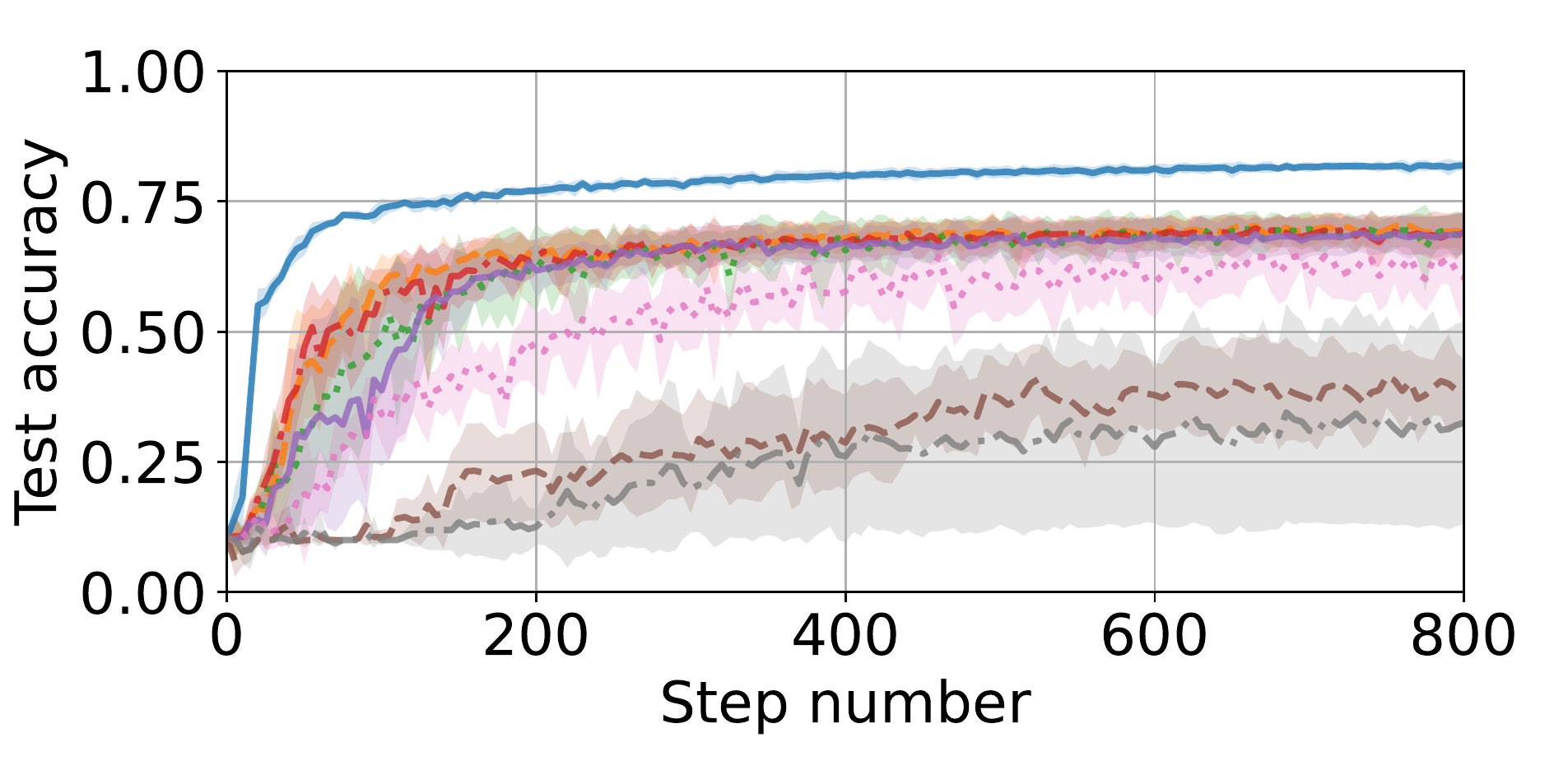}%
    \includegraphics[width=0.5\textwidth]{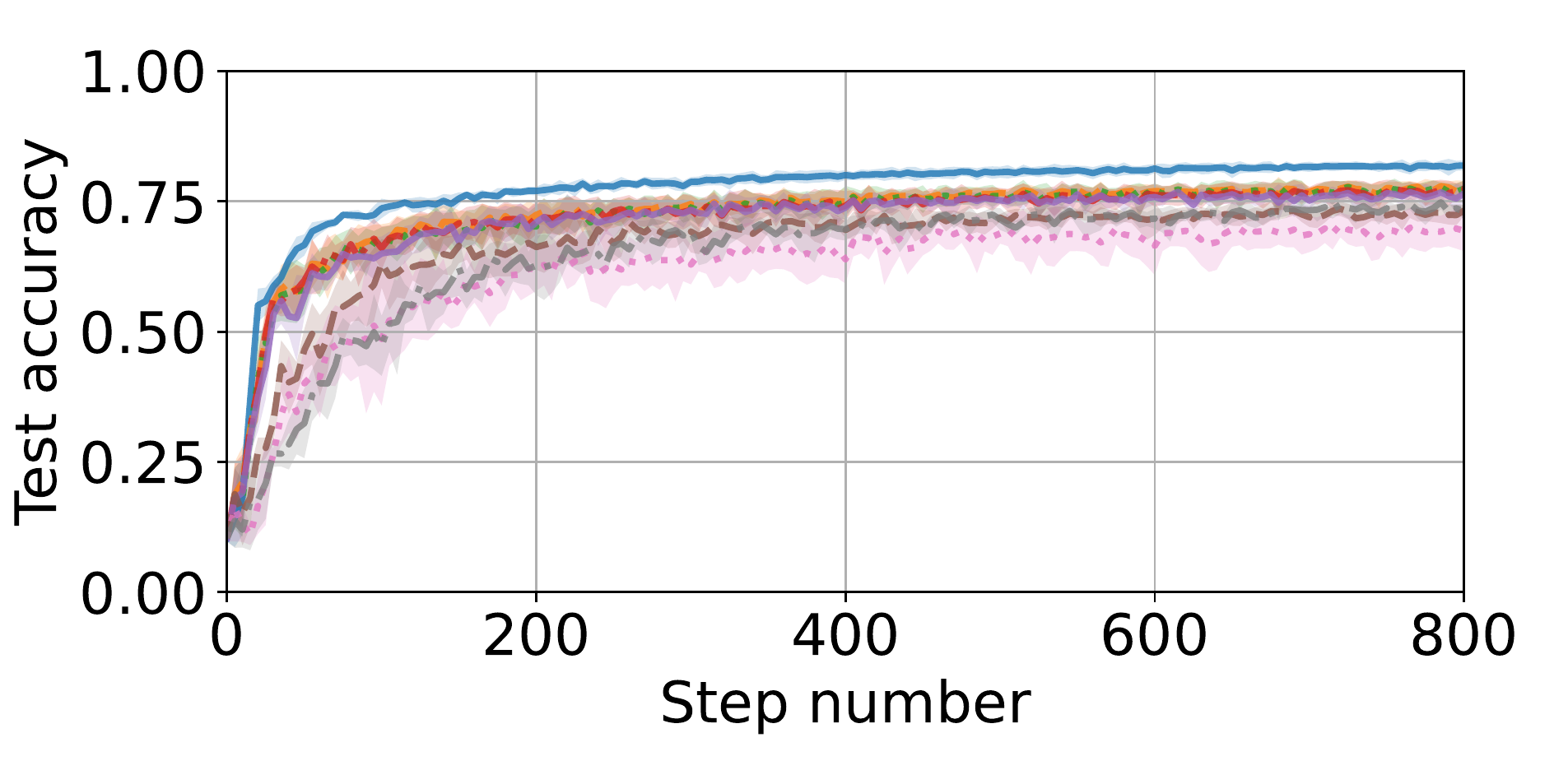}\\%
    \includegraphics[width=0.5\textwidth]{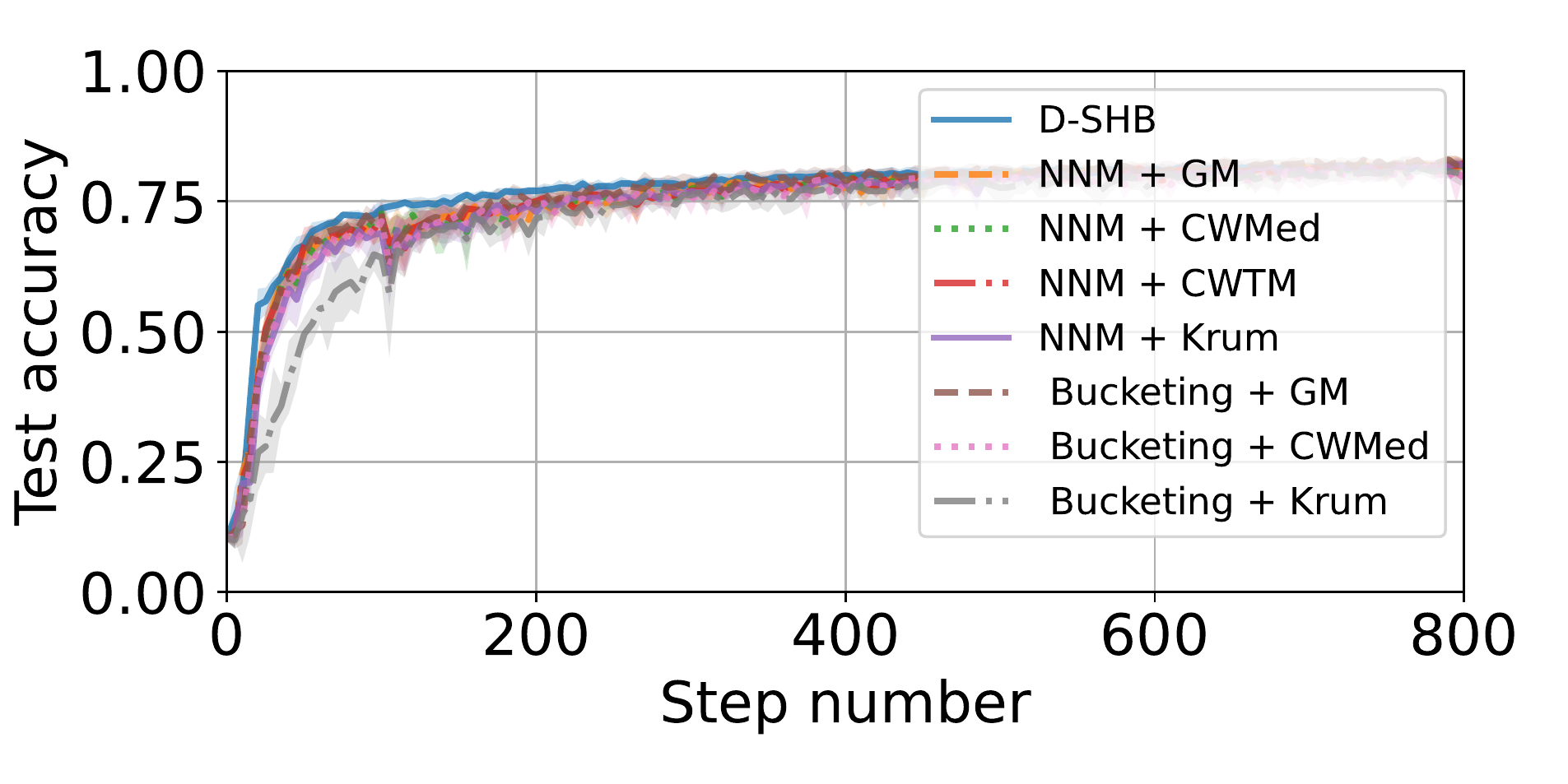}%
    \includegraphics[width=0.5\textwidth]{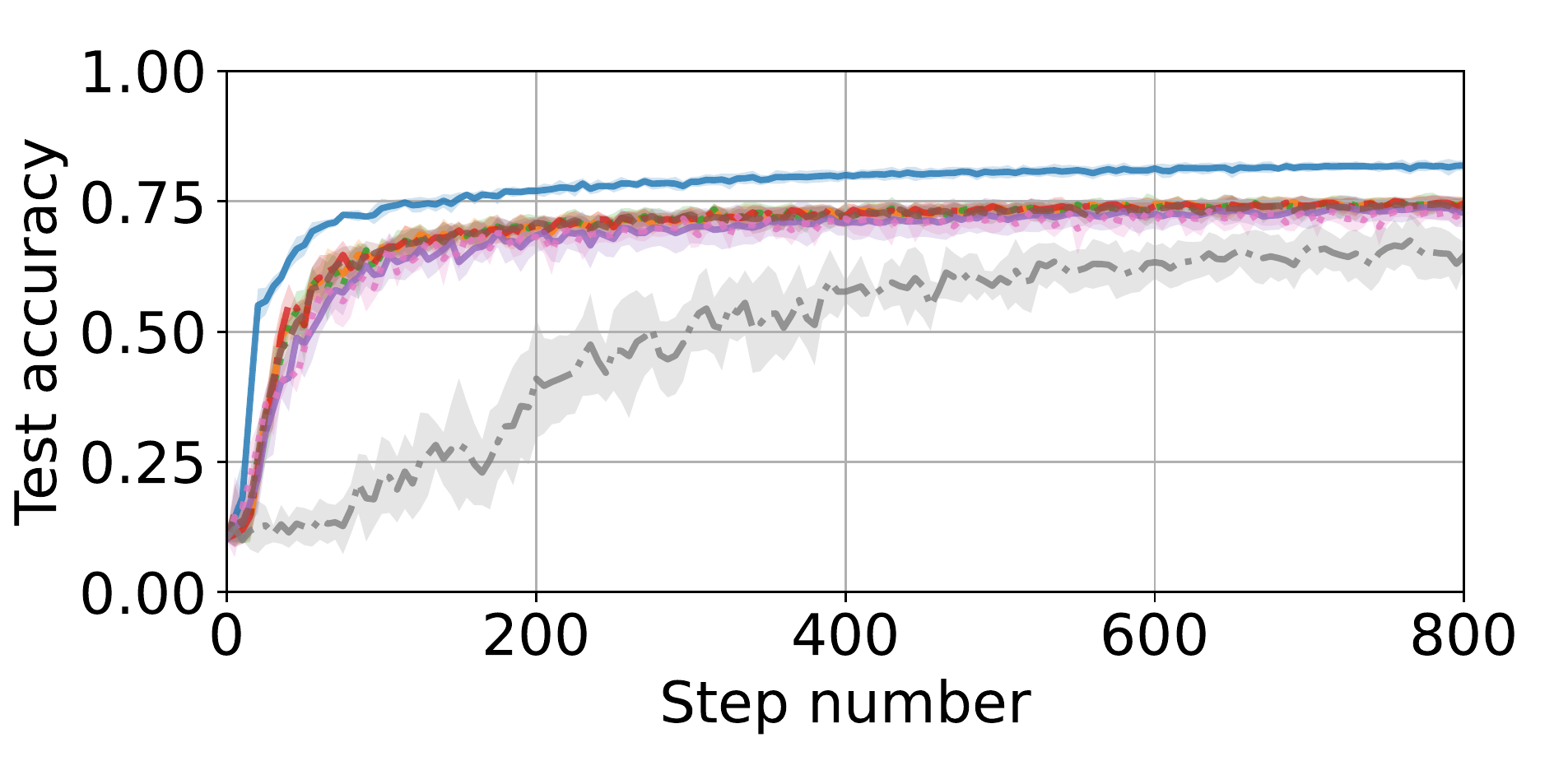}\\%
     \includegraphics[width=0.5\textwidth]{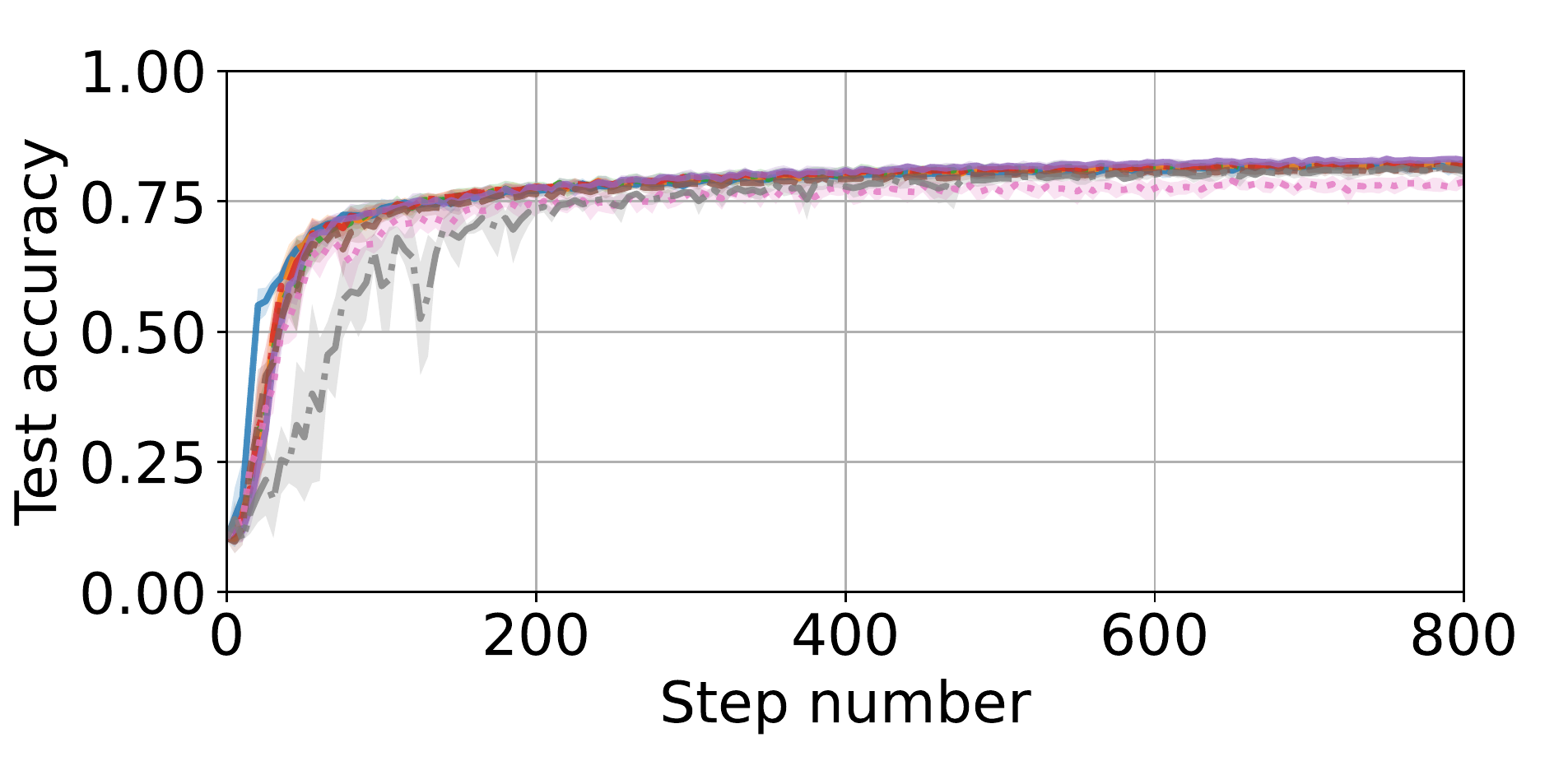}%
    \caption{Experiments on Fashion-MNIST using robust D-SHB with $f = 4$ Byzantine among $n = 17$ workers, with $\beta = 0.9$ and $\alpha = 1$. The Byzantine workers execute the FOE (\textit{row 1, left}), ALIE (\textit{row 1, right}), Mimic (\textit{row 2, left}), SF (\textit{row 2, right}), and LF (\textit{row 3}) attacks.}
\label{fig:plots_fashionmnist_2}
\end{figure*}

\begin{figure*}[ht!]
    \centering
    \includegraphics[width=0.5\textwidth]{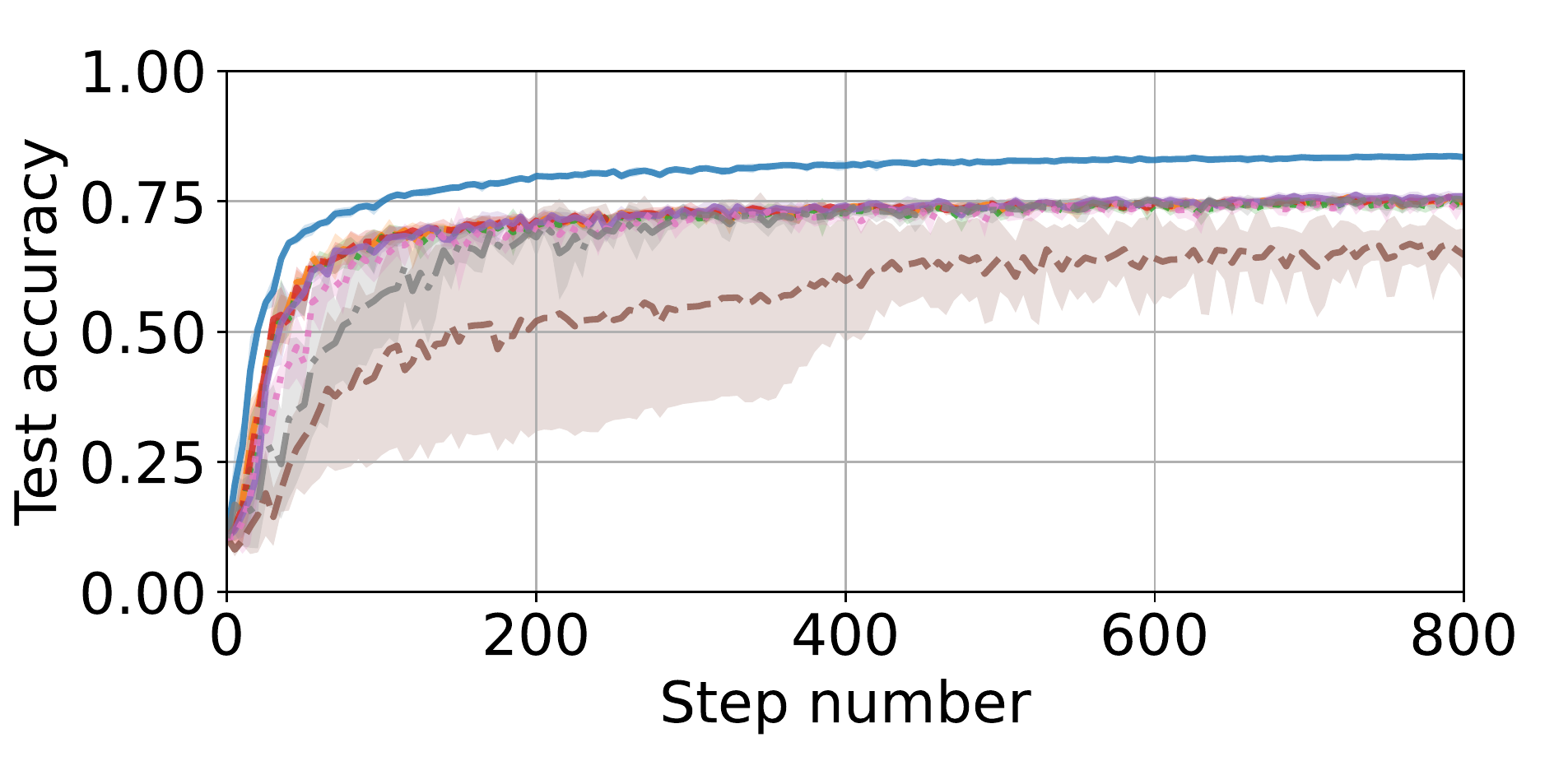}%
    \includegraphics[width=0.5\textwidth]{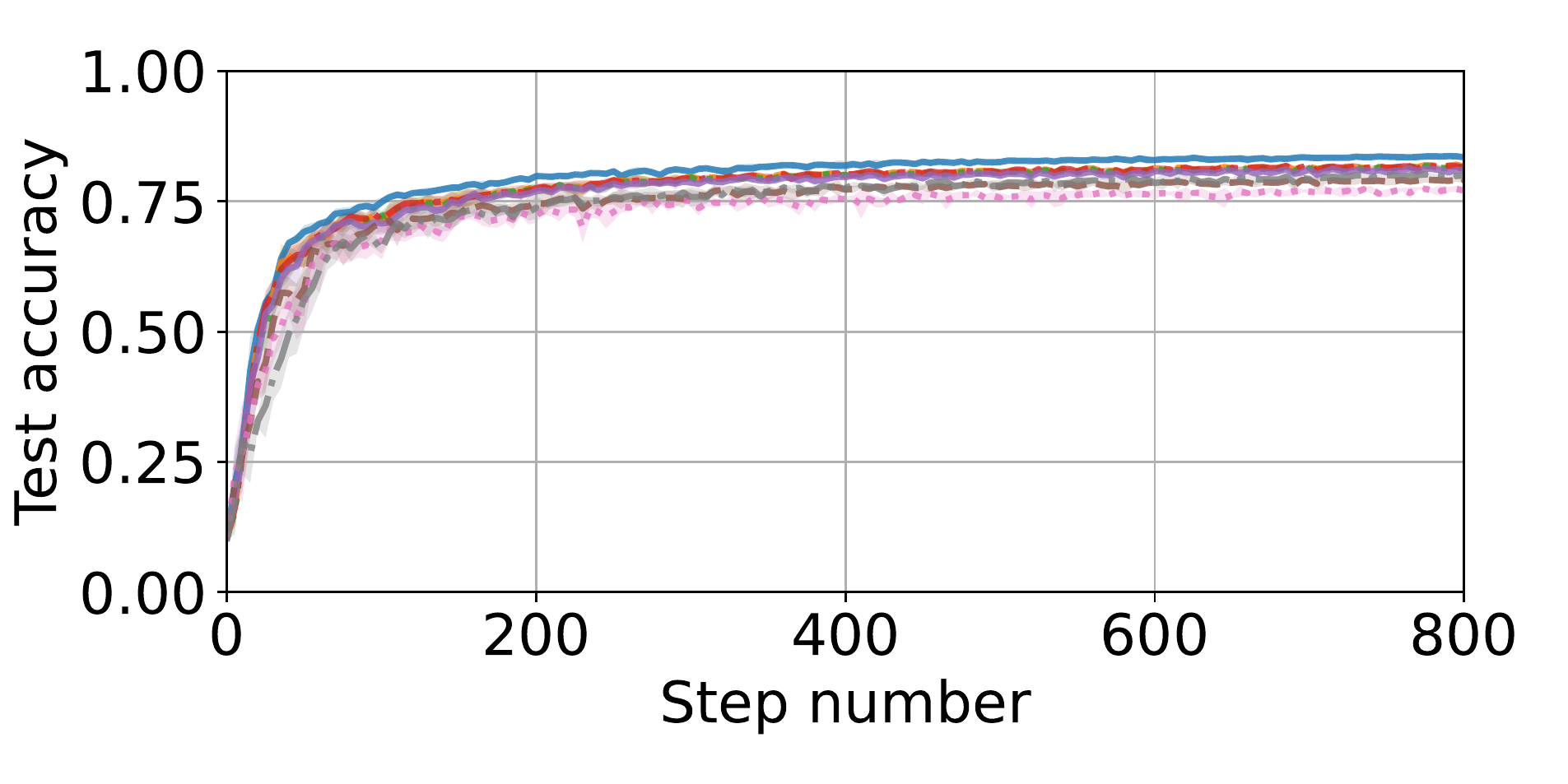}\\%
    \includegraphics[width=0.5\textwidth]{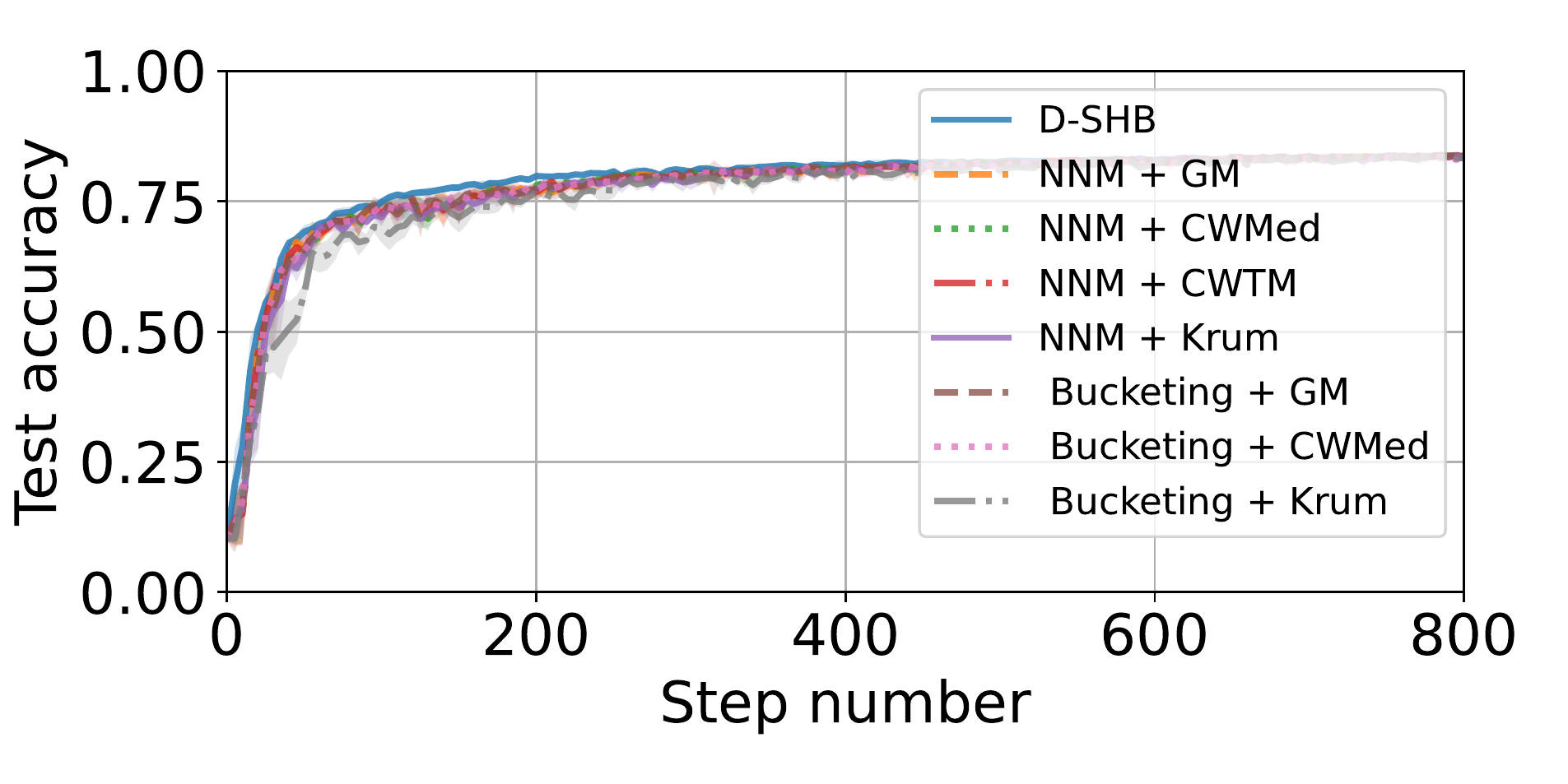}%
    \includegraphics[width=0.5\textwidth]{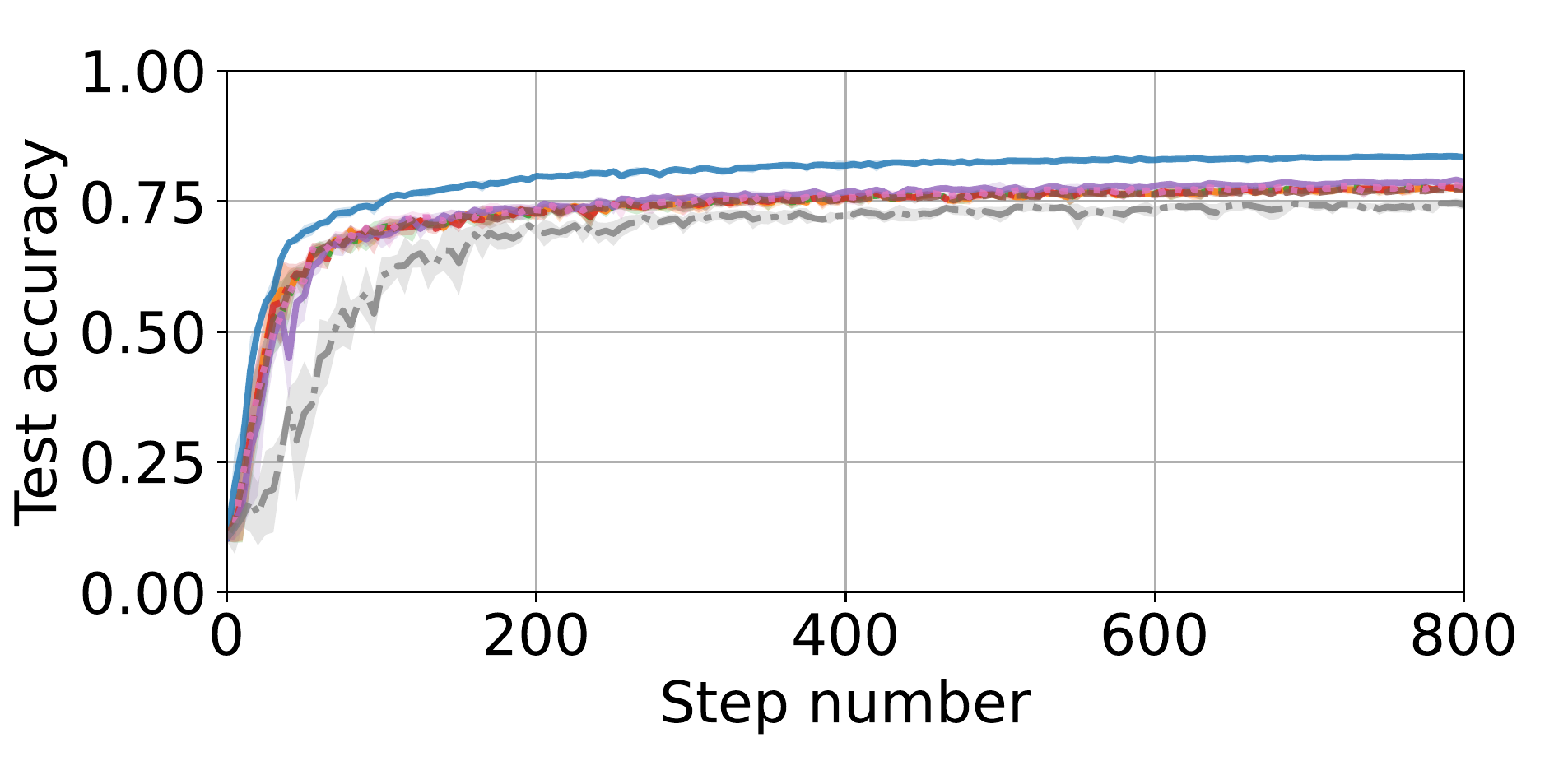}\\%
     \includegraphics[width=0.5\textwidth]{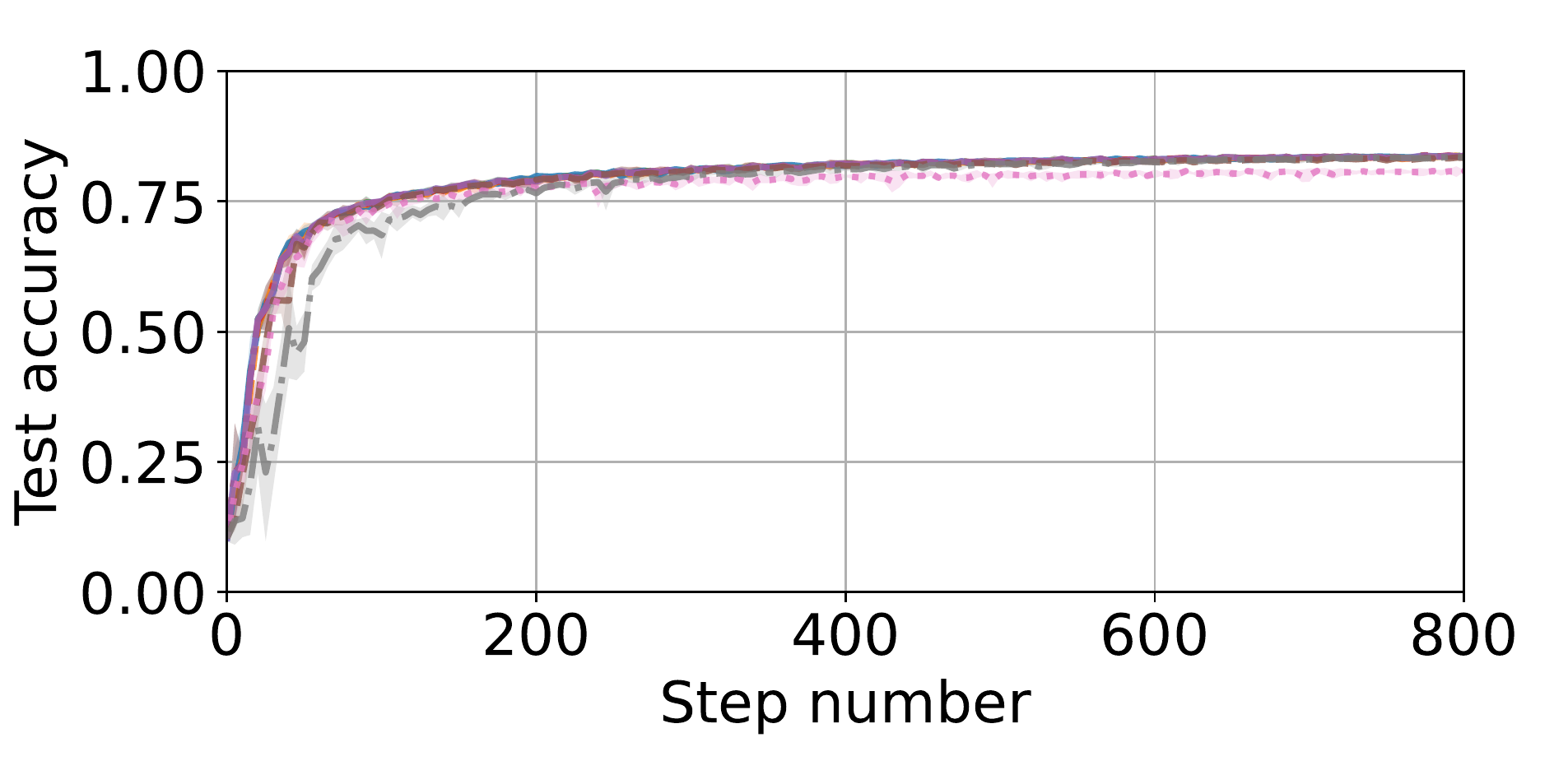}%
    \caption{Experiments on Fashion-MNIST using robust D-SHB with $f = 4$ Byzantine among $n = 17$ workers, with $\beta = 0.9$ and $\alpha = 10$. The Byzantine workers execute the FOE (\textit{row 1, left}), ALIE (\textit{row 1, right}), Mimic (\textit{row 2, left}), SF (\textit{row 2, right}), and LF (\textit{row 3}) attacks.}
\label{fig:plots_fashionmnist_3}
\end{figure*}

\begin{figure*}[ht!]
    \centering
    \includegraphics[width=0.5\textwidth]{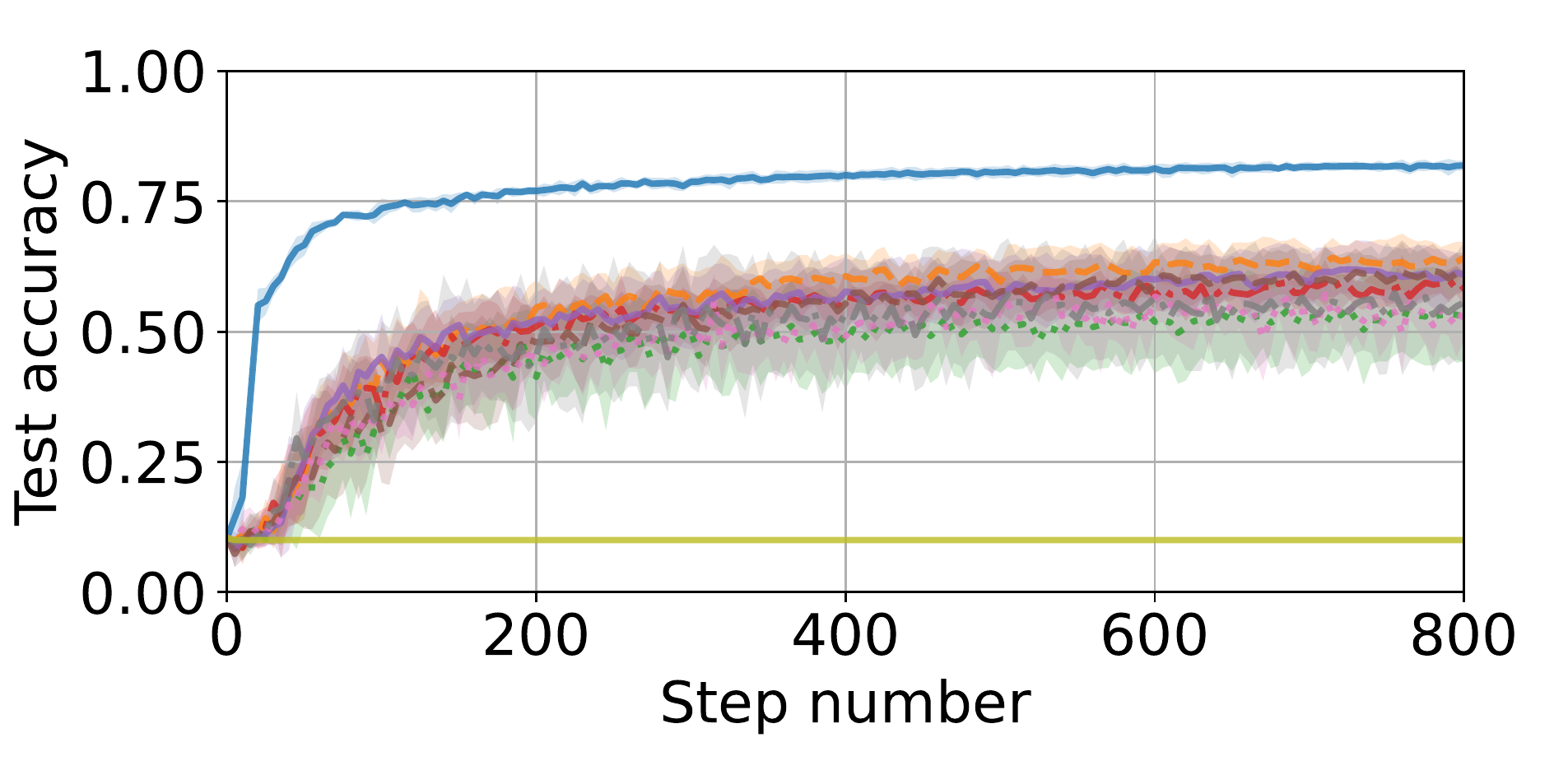}%
    \includegraphics[width=0.5\textwidth]{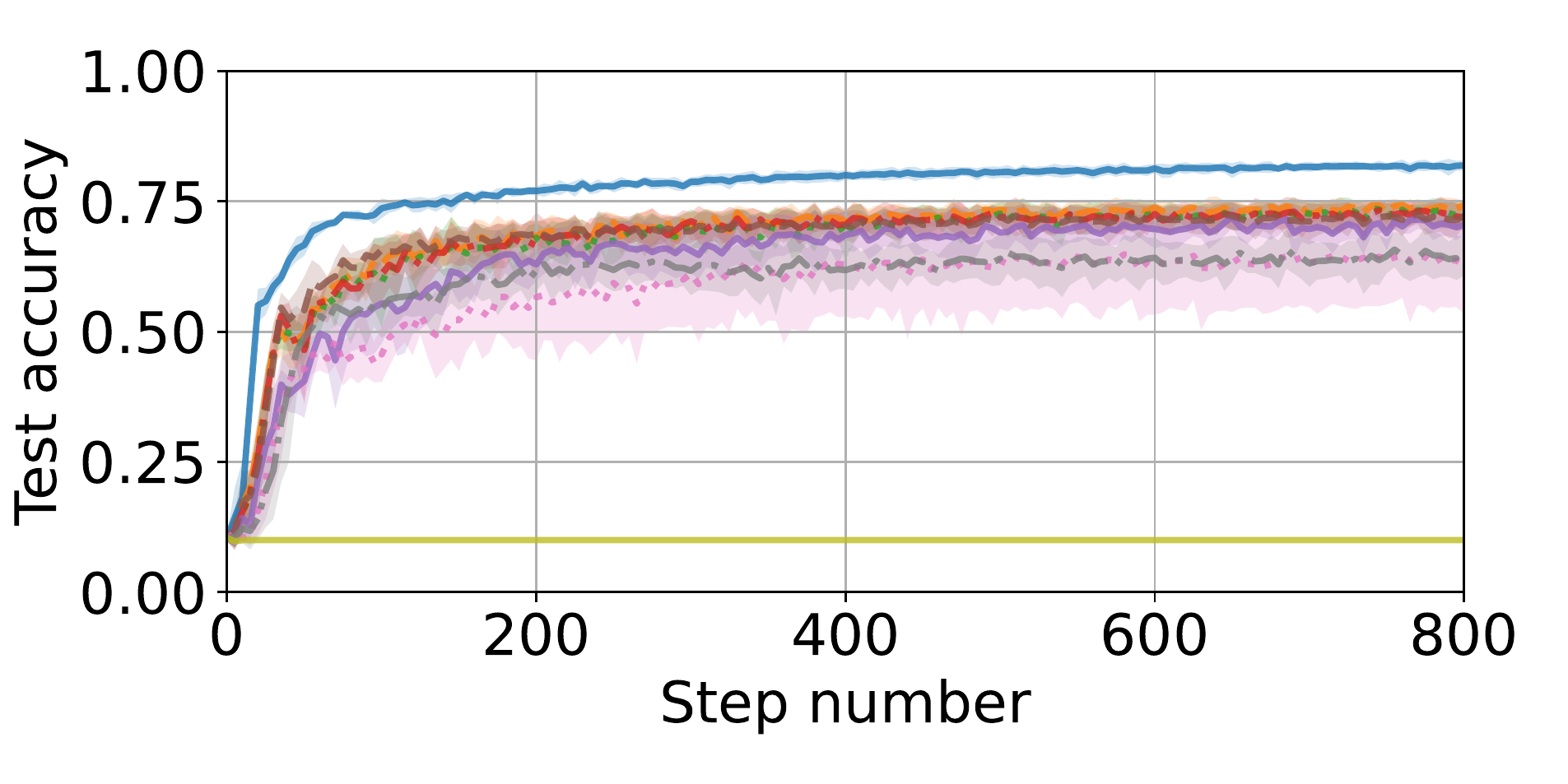}\\%
    \includegraphics[width=0.5\textwidth]{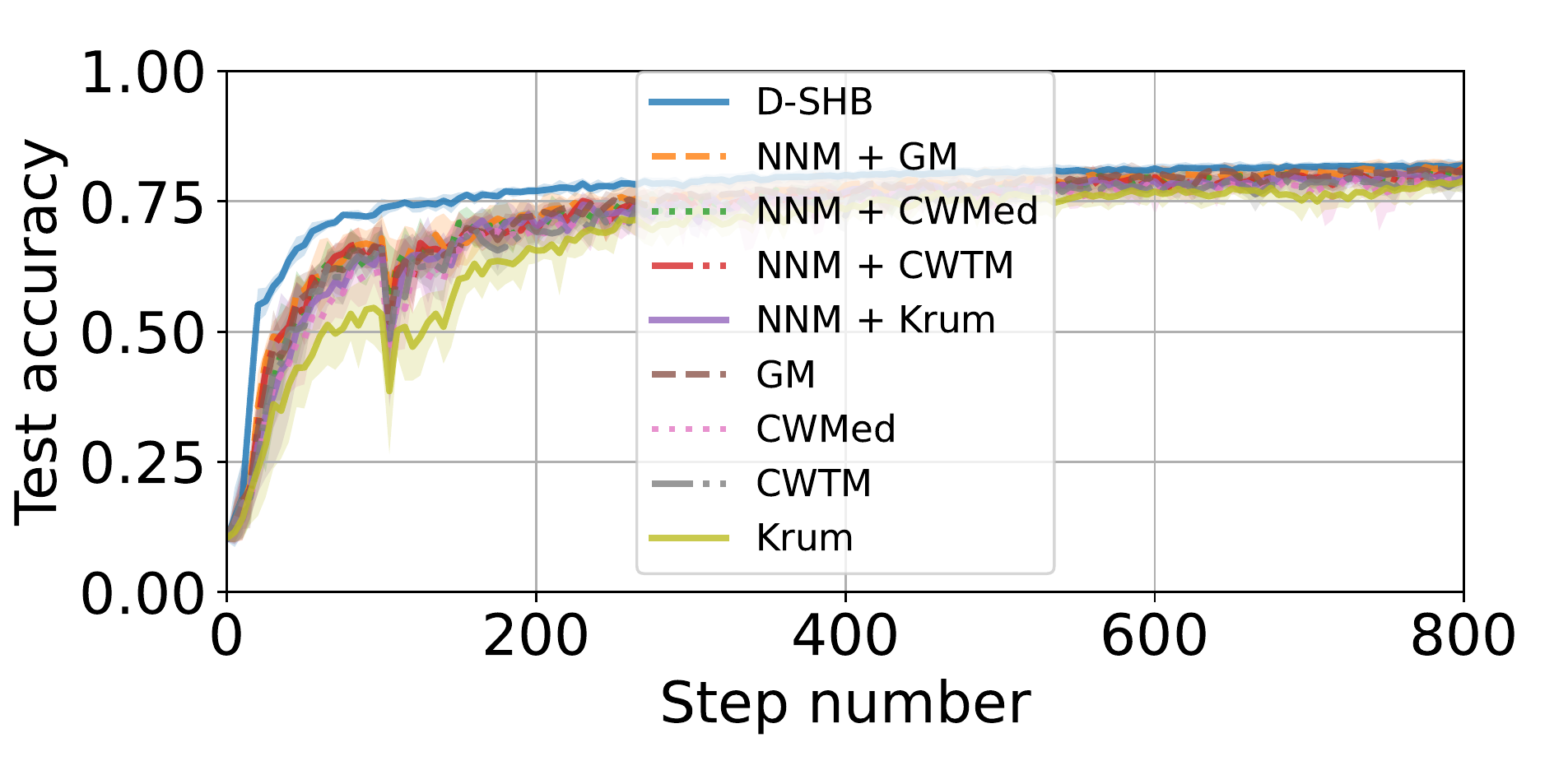}%
    \includegraphics[width=0.5\textwidth]{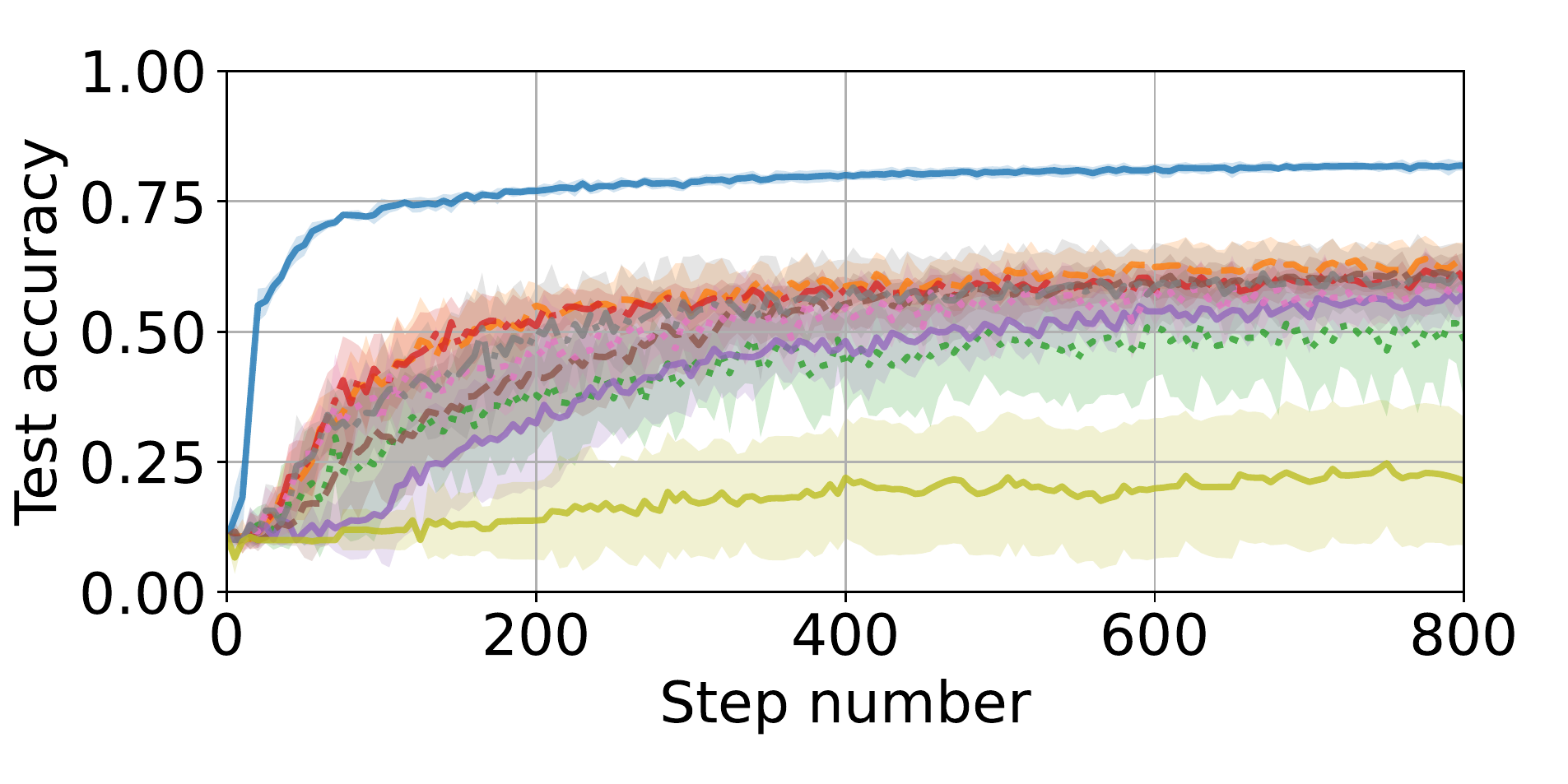}\\%
     \includegraphics[width=0.5\textwidth]{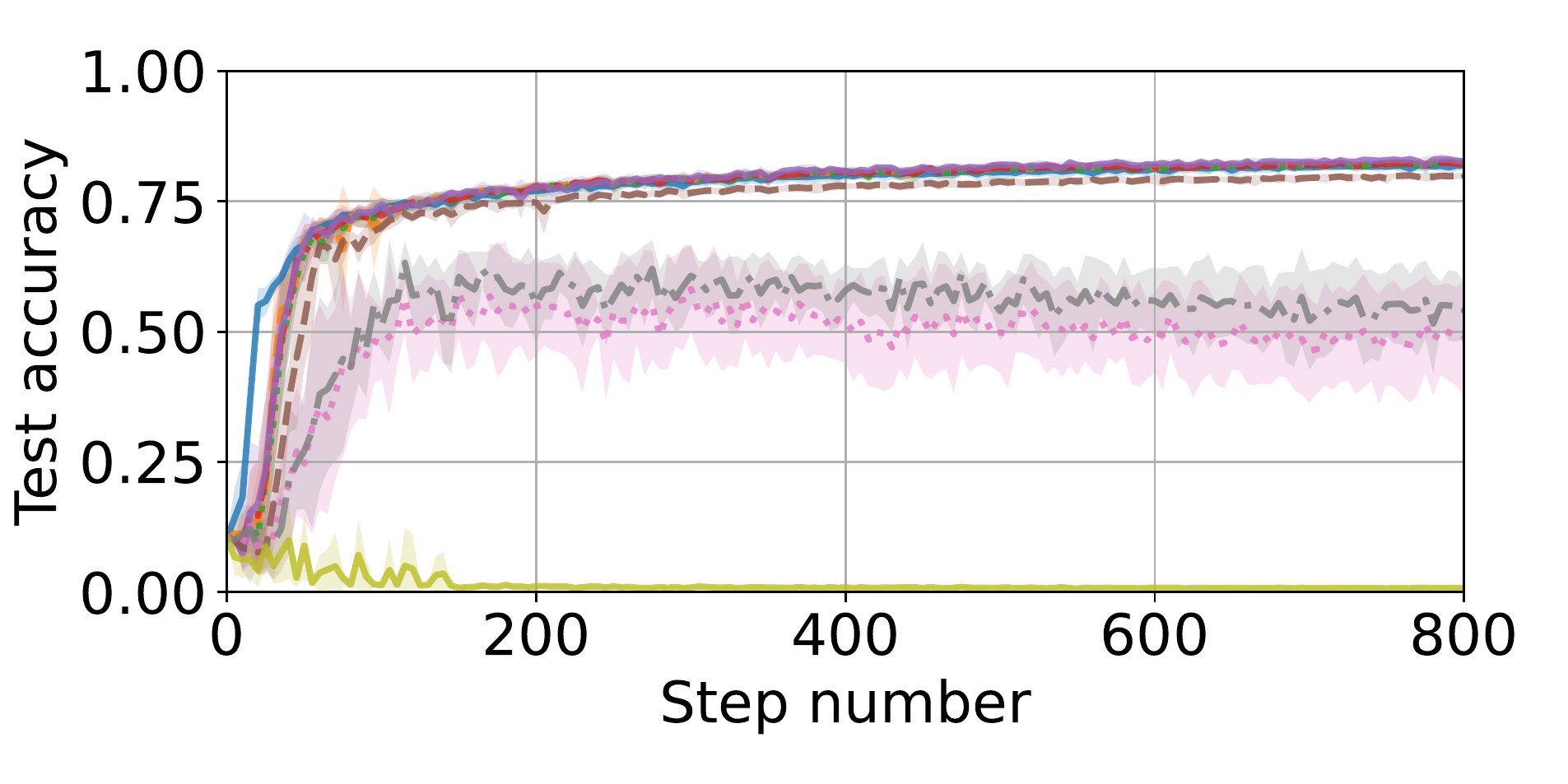}%
    \caption{Experiments on Fashion-MNIST using robust D-SHB with $f = 6$ Byzantine among $n = 17$ workers, with $\beta = 0.9$ and $\alpha = 1$. The Byzantine workers execute the FOE (\textit{row 1, left}), ALIE (\textit{row 1, right}), Mimic (\textit{row 2, left}), SF (\textit{row 2, right}), and LF (\textit{row 3}) attacks.}
\label{fig:plots_fashionmnist_4}
\end{figure*}

\begin{figure*}[ht!]
    \centering
    \includegraphics[width=0.5\textwidth]{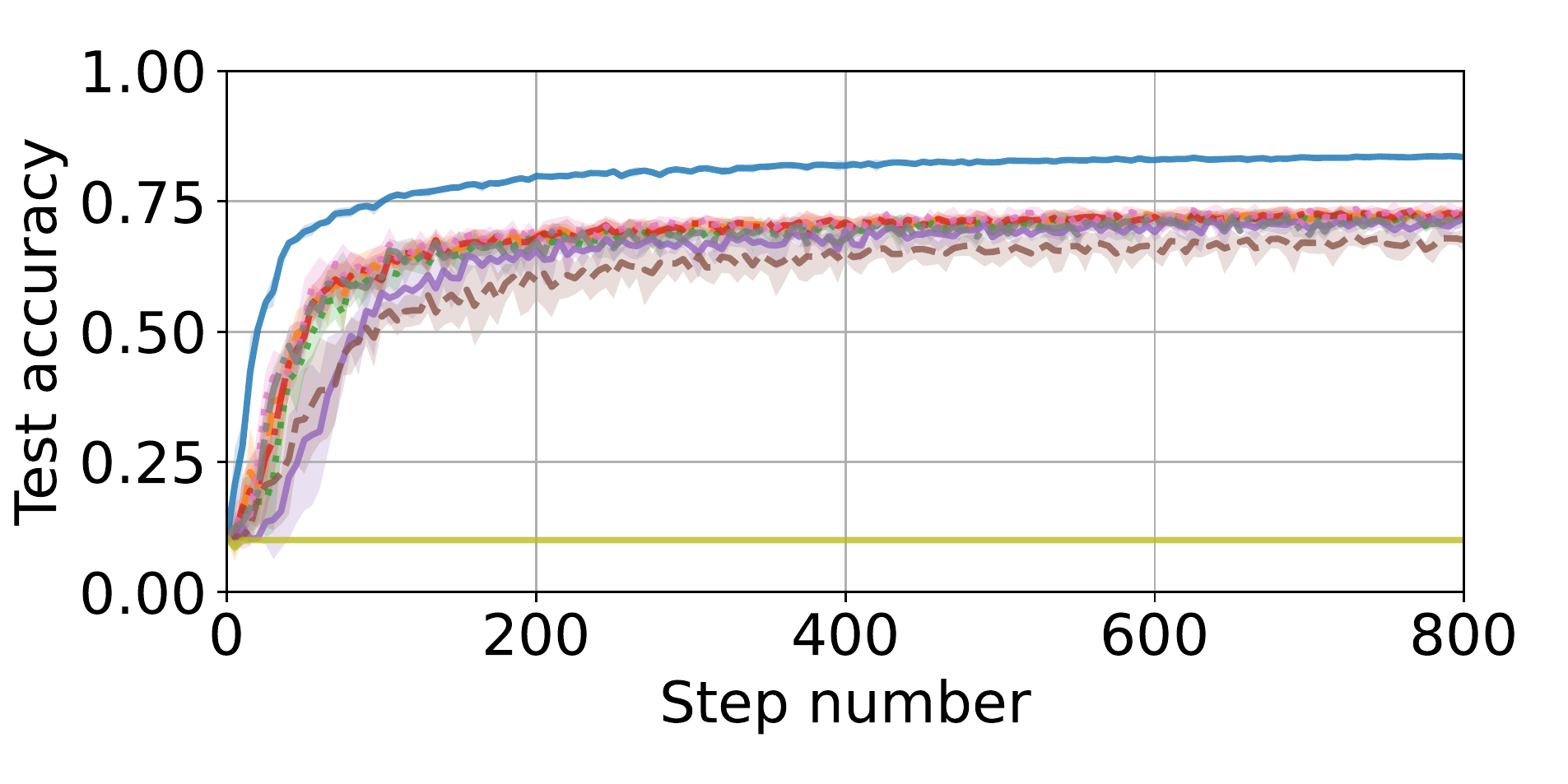}%
    \includegraphics[width=0.5\textwidth]{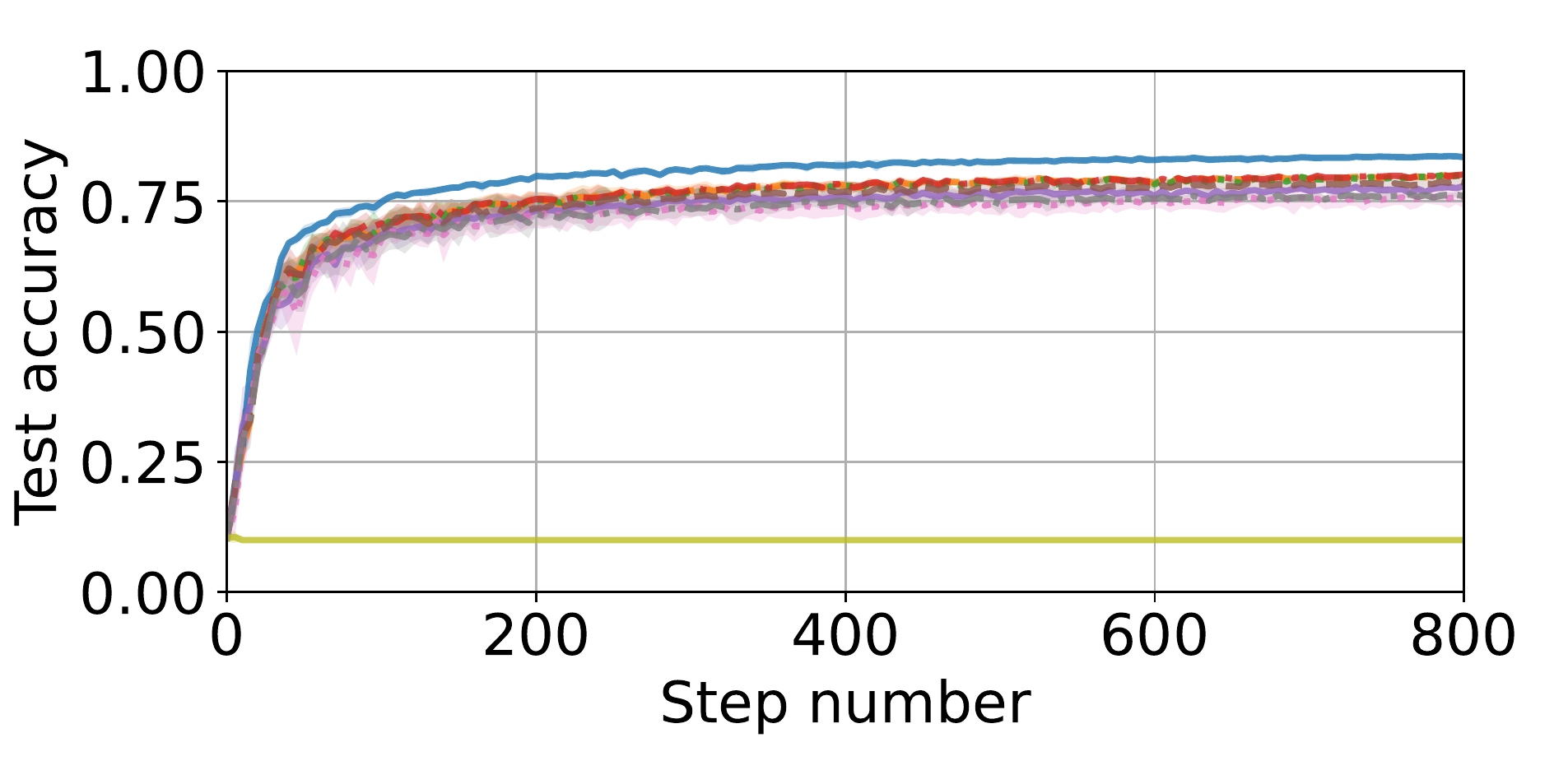}\\%
    \includegraphics[width=0.5\textwidth]{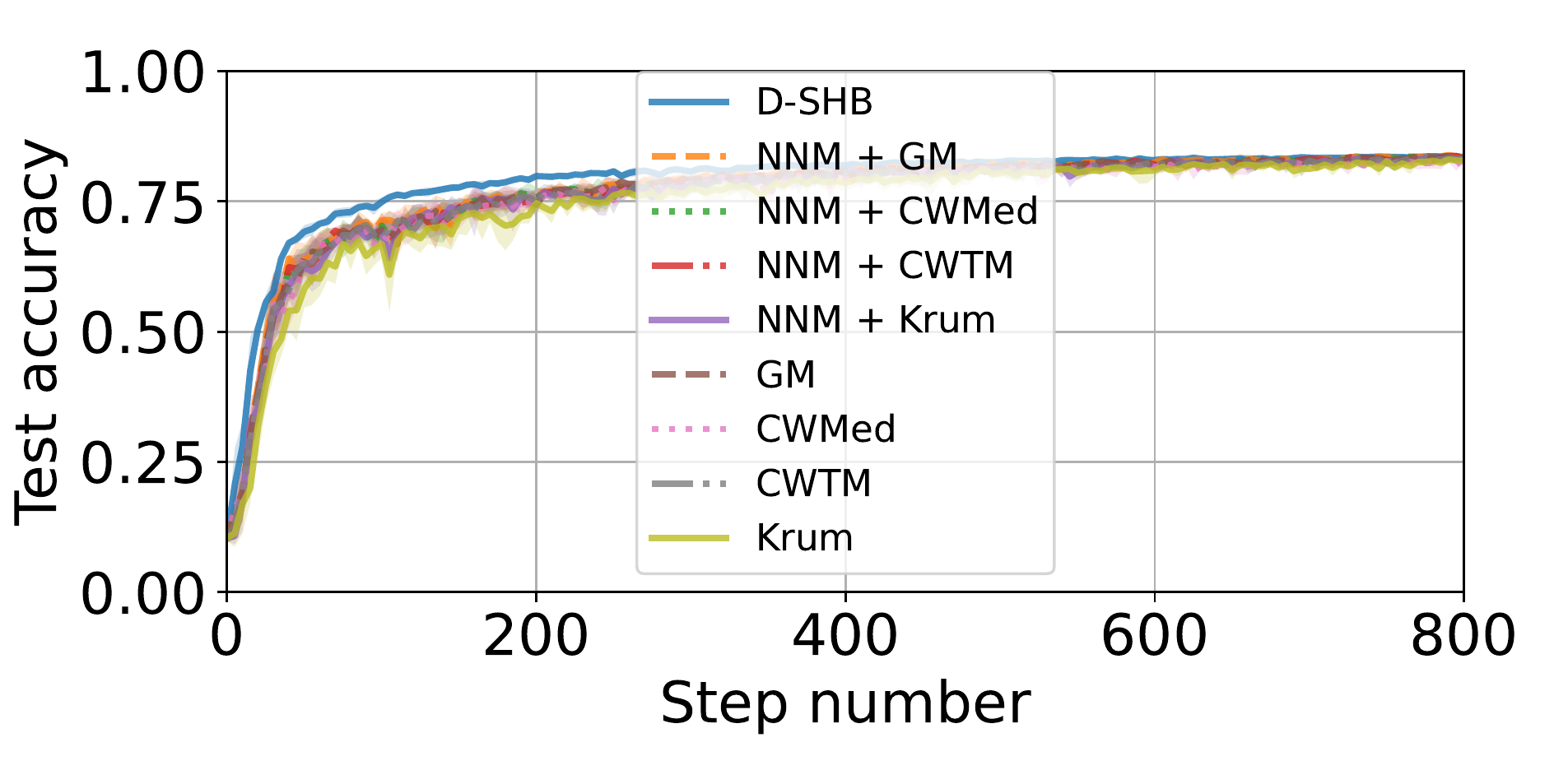}%
    \includegraphics[width=0.5\textwidth]{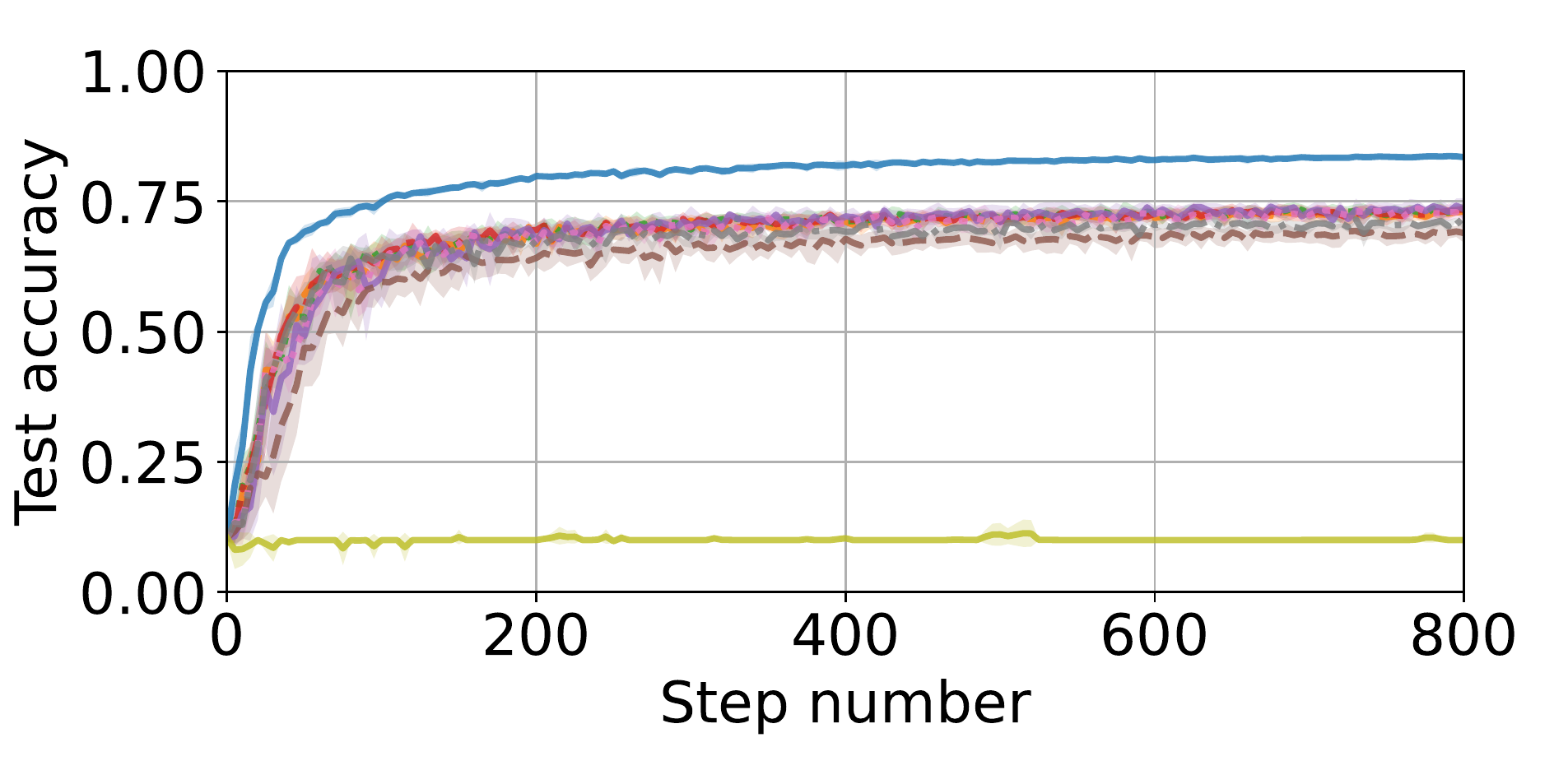}\\%
     \includegraphics[width=0.5\textwidth]{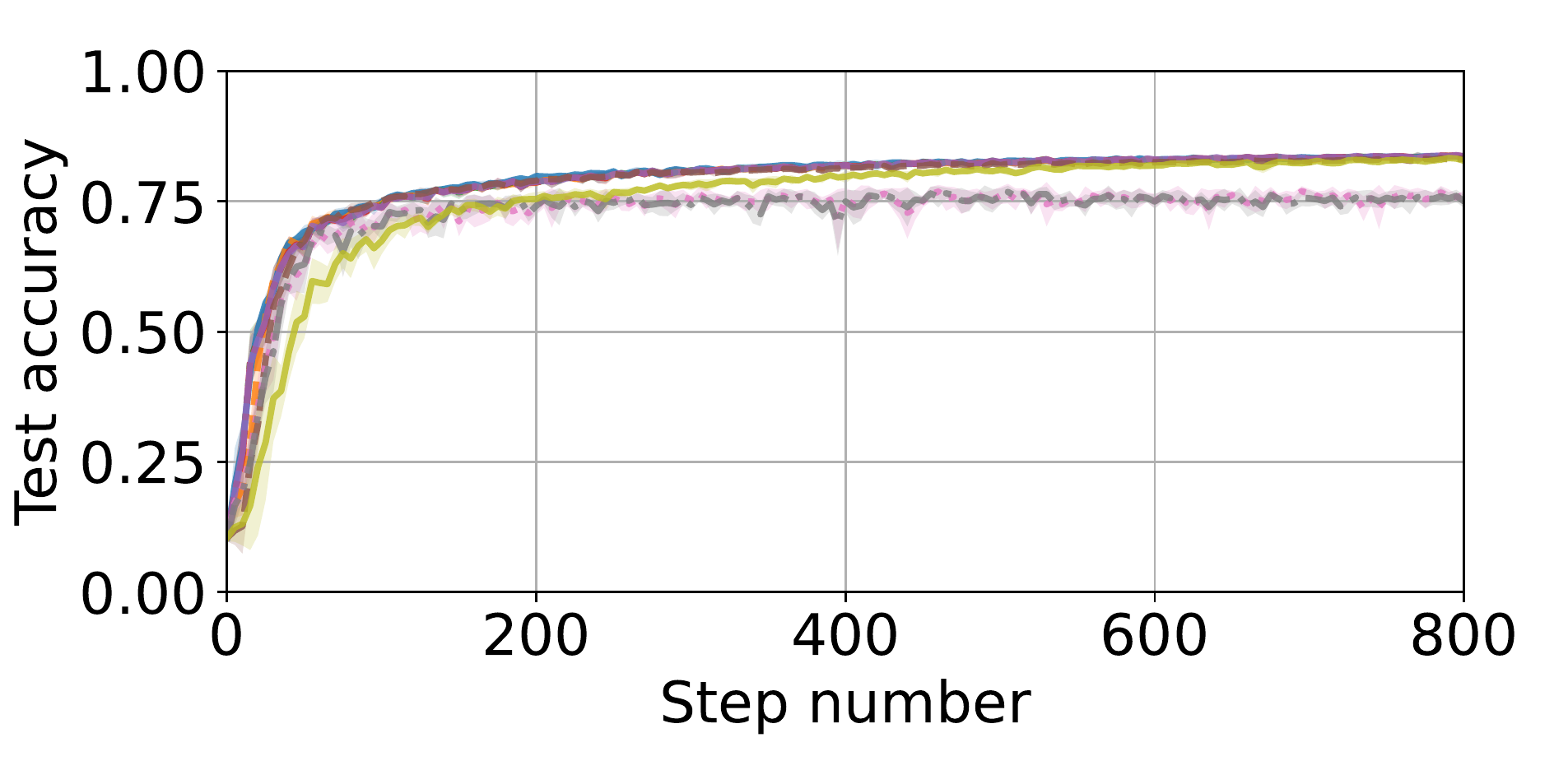}%
    \caption{Experiments on Fashion-MNIST using robust D-SHB with $f = 6$ Byzantine among $n = 17$ workers, with $\beta = 0.9$ and $\alpha = 10$. The Byzantine workers execute the FOE (\textit{row 1, left}), ALIE (\textit{row 1, right}), Mimic (\textit{row 2, left}), SF (\textit{row 2, right}), and LF (\textit{row 3}) attacks.}
\label{fig:plots_fashionmnist_5}
\end{figure*}

\clearpage
\subsection{Comprehensive Results on CIFAR-10}\label{app:exp_results_cifar}
In this section, we present the entirety of our results on CIFAR-10. We consider three Byzantine regimes: $f = 2$, $f = 3$, and $f=4$ out of $n=17$ workers in total. We also consider two heterogeneity regimes: $\alpha=1$ (moderate), and $\alpha=10$ (low).
We compare the performance of $\cenna{}$ and Bucketing when executed with four aggregation rules namely Krum, GM, CWMed, and CWTM.
The plots are presented below and complement our (partial) results in Figure~\ref{fig:plots_cifar_main} in Section~\ref{exp_results_cifar} of the main paper.

\begin{figure*}[ht!]
    \centering
    \includegraphics[width=0.5\textwidth]{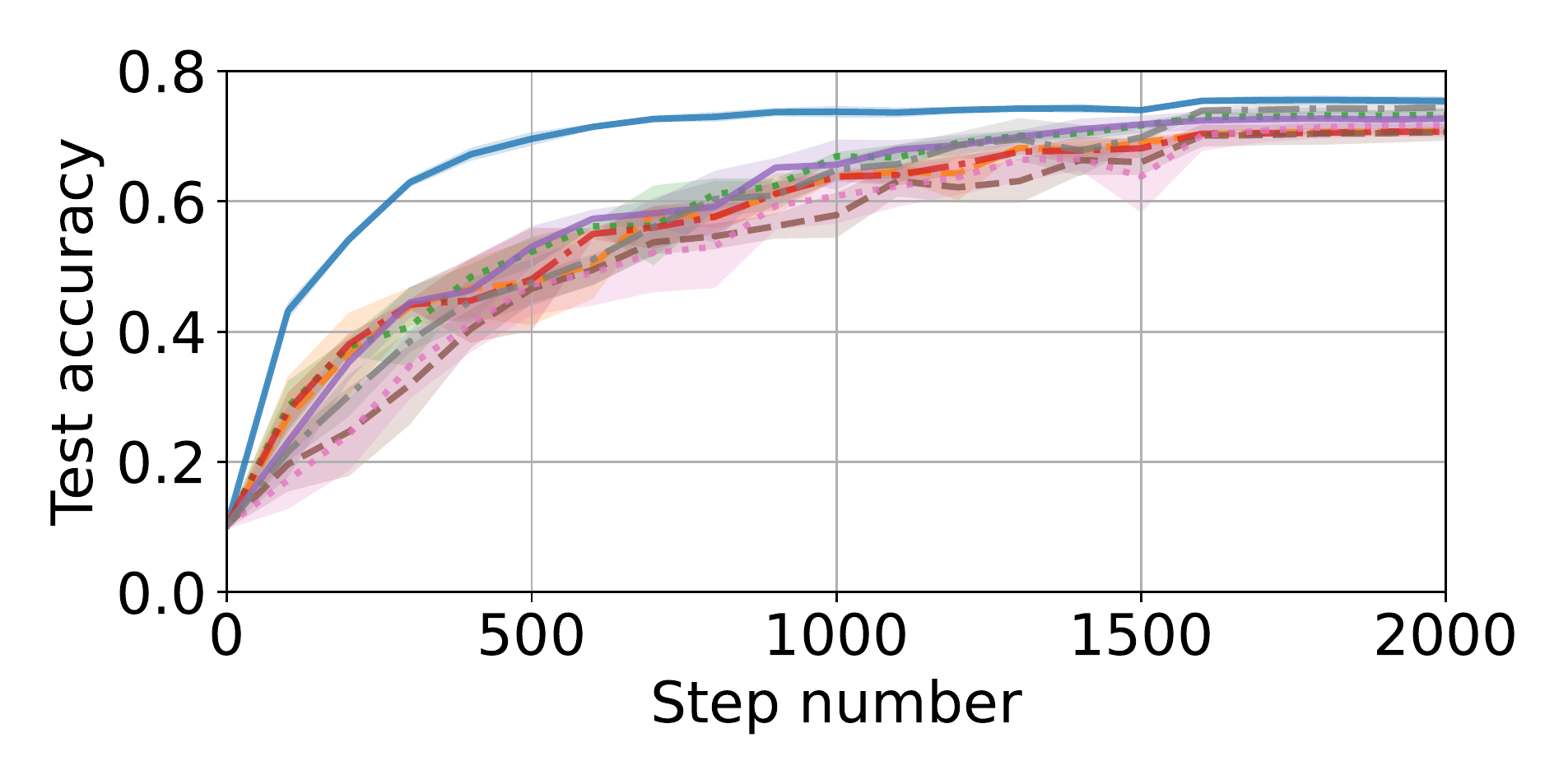}%
    \includegraphics[width=0.5\textwidth]{plots/cifar10_empire-cnn_little_f=2_beta=0.9_alpha=1.pdf}\\%
    \includegraphics[width=0.5\textwidth]{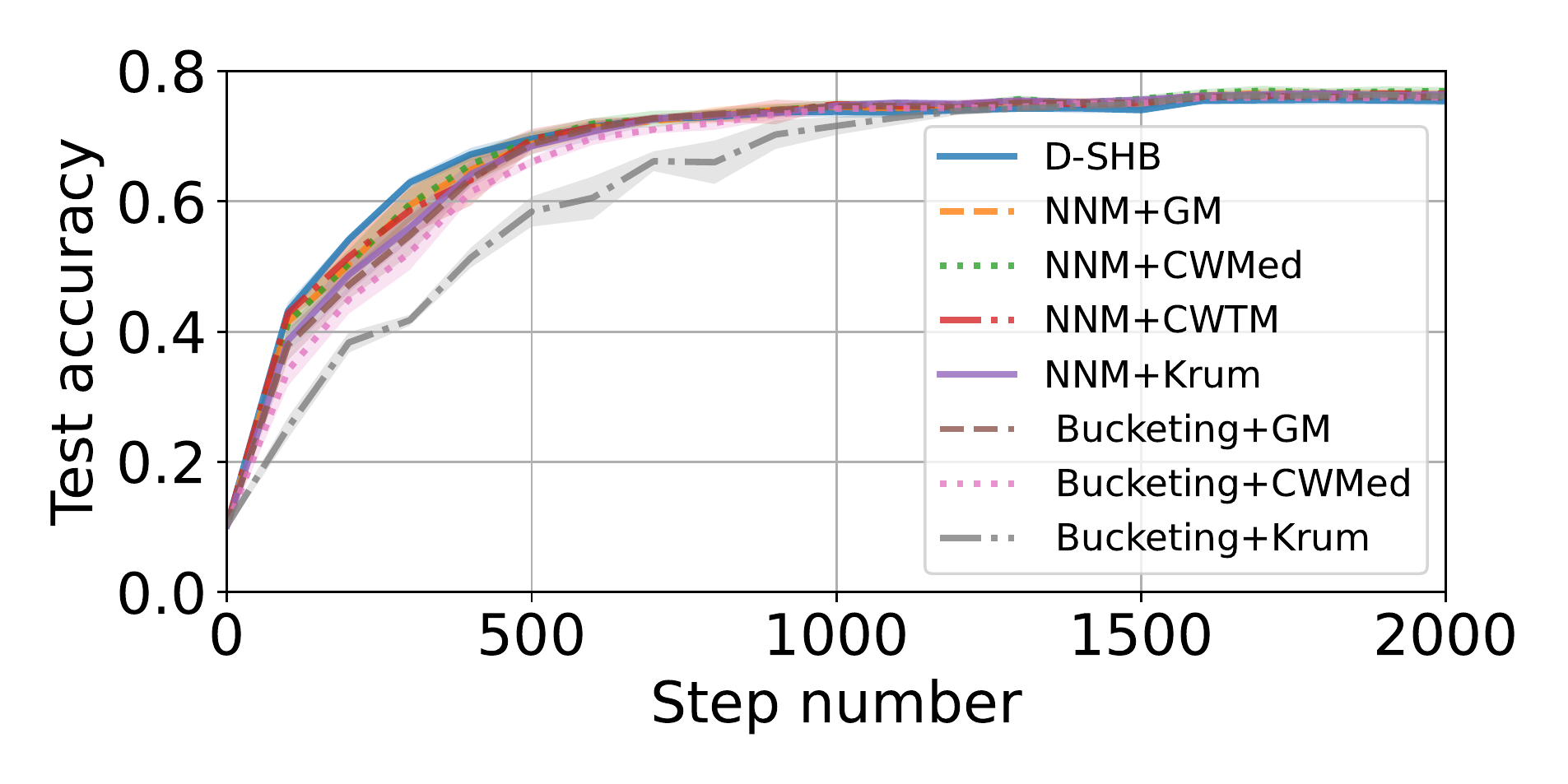}%
    \includegraphics[width=0.5\textwidth]{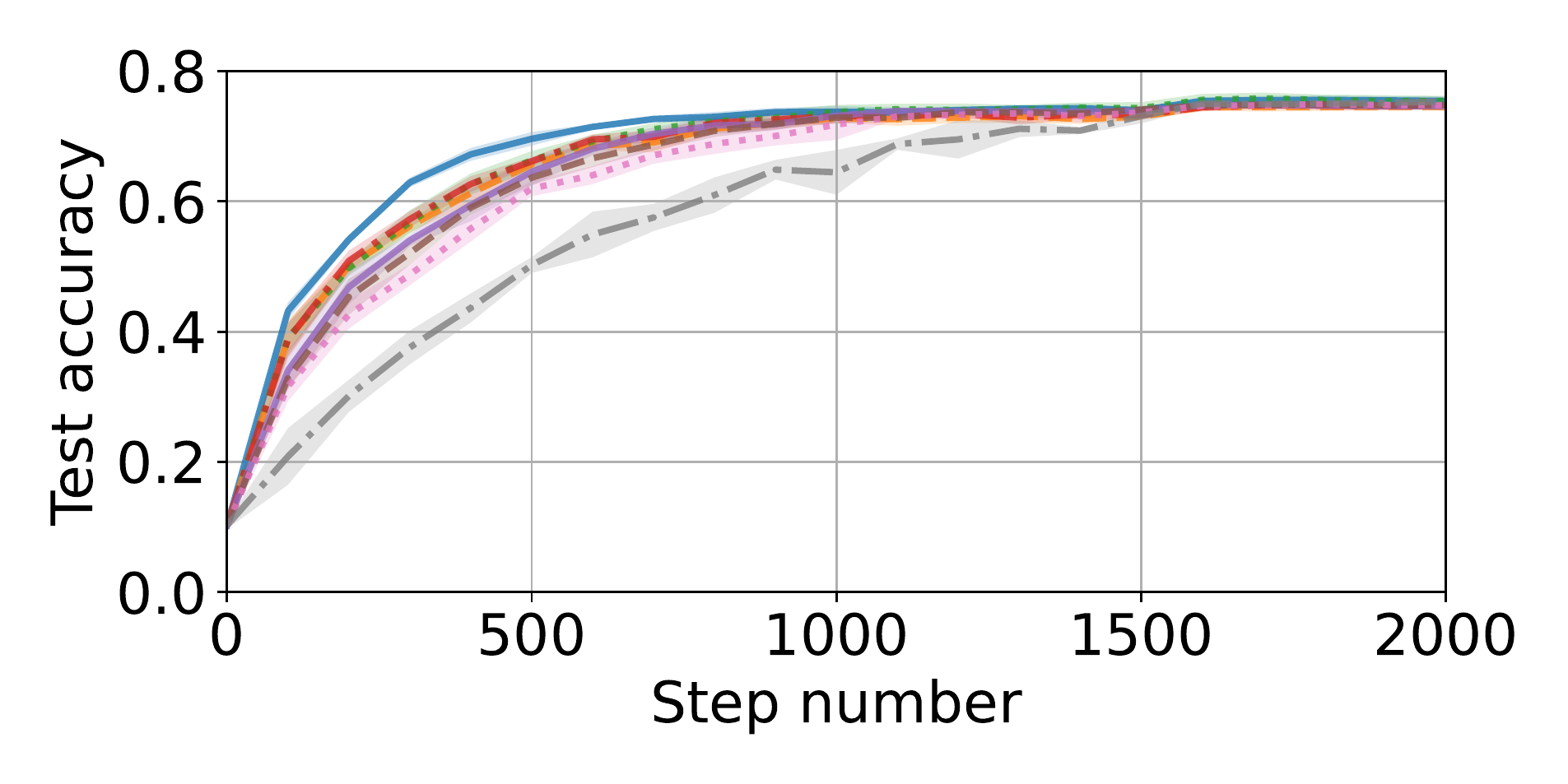}\\%
     \includegraphics[width=0.5\textwidth]{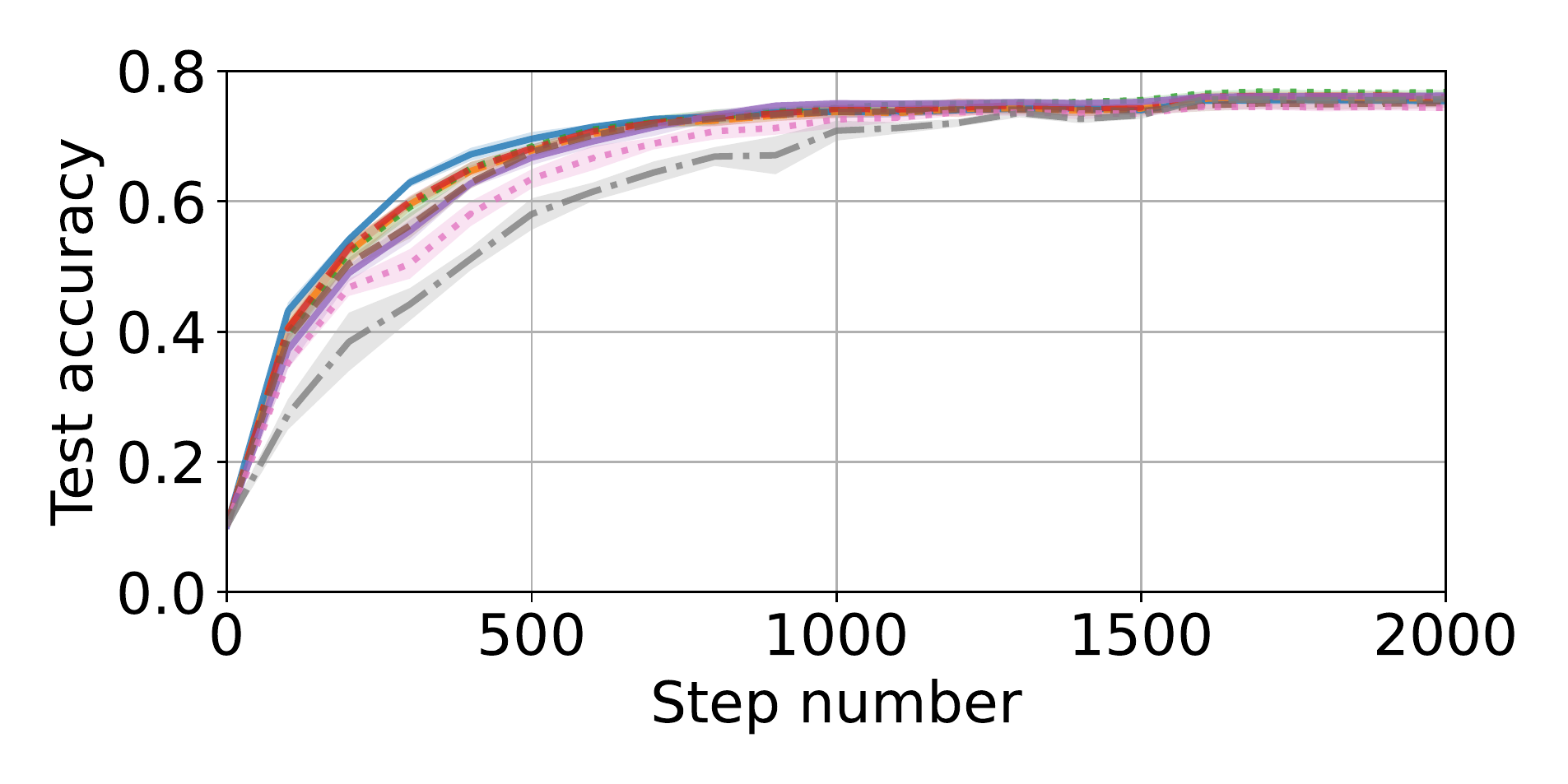}%
    \caption{Experiments on CIFAR-10 using robust D-SHB with $f = 2$ Byzantine among $n = 17$ workers, with $\beta = 0.9$ and $\alpha = 1$. The Byzantine workers execute the FOE (\textit{row 1, left}), ALIE (\textit{row 1, right}), Mimic (\textit{row 2, left}), SF (\textit{row 2, right}), and LF (\textit{row 3}) attacks.}
\label{fig:plots_cifar_1}
\end{figure*}

\begin{figure*}[ht!]
    \centering
    \includegraphics[width=0.5\textwidth]{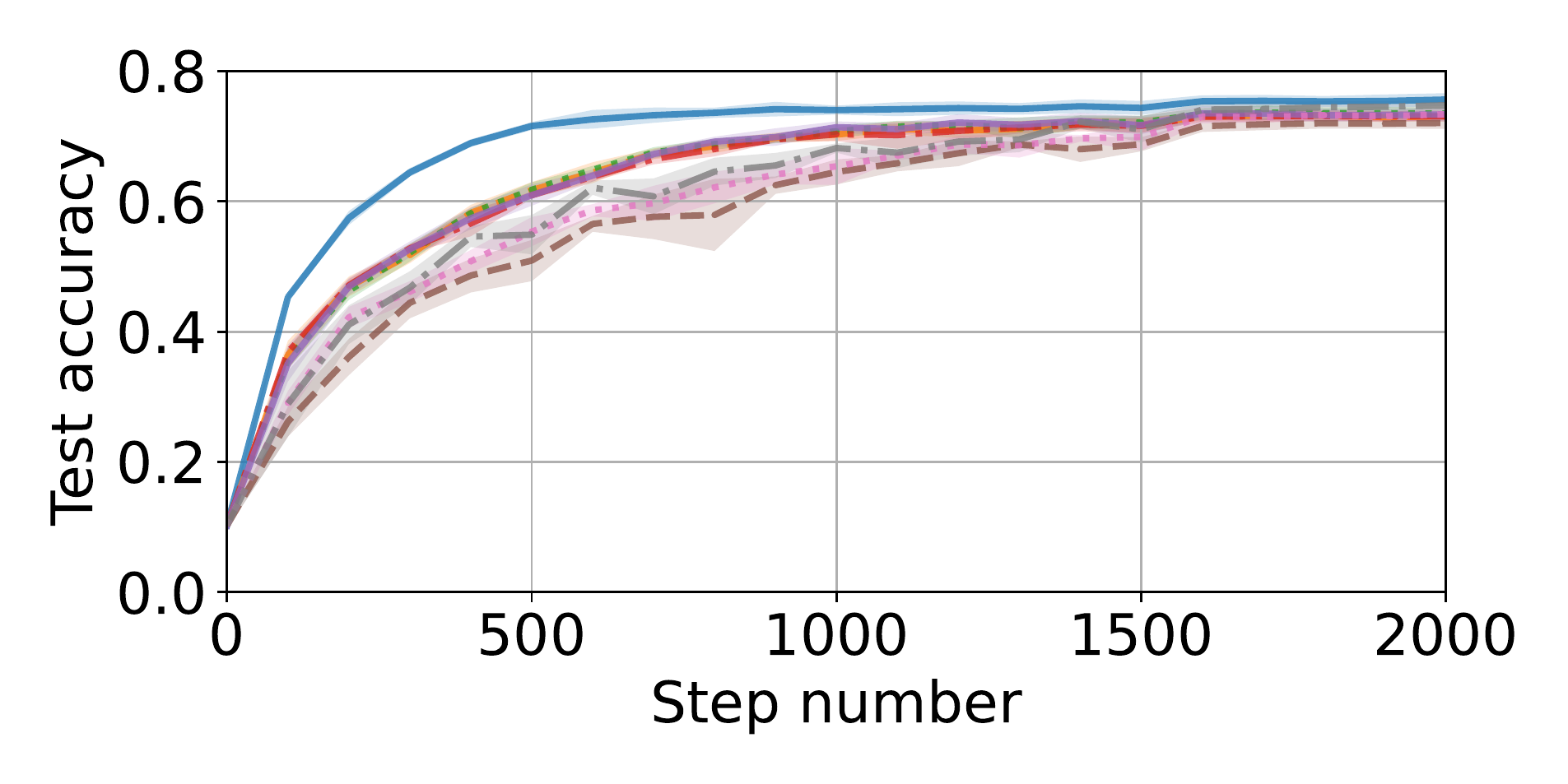}%
    \includegraphics[width=0.5\textwidth]{plots/cifar10_empire-cnn_little_f=2_beta=0.9_alpha=10.pdf}\\%
    \includegraphics[width=0.5\textwidth]{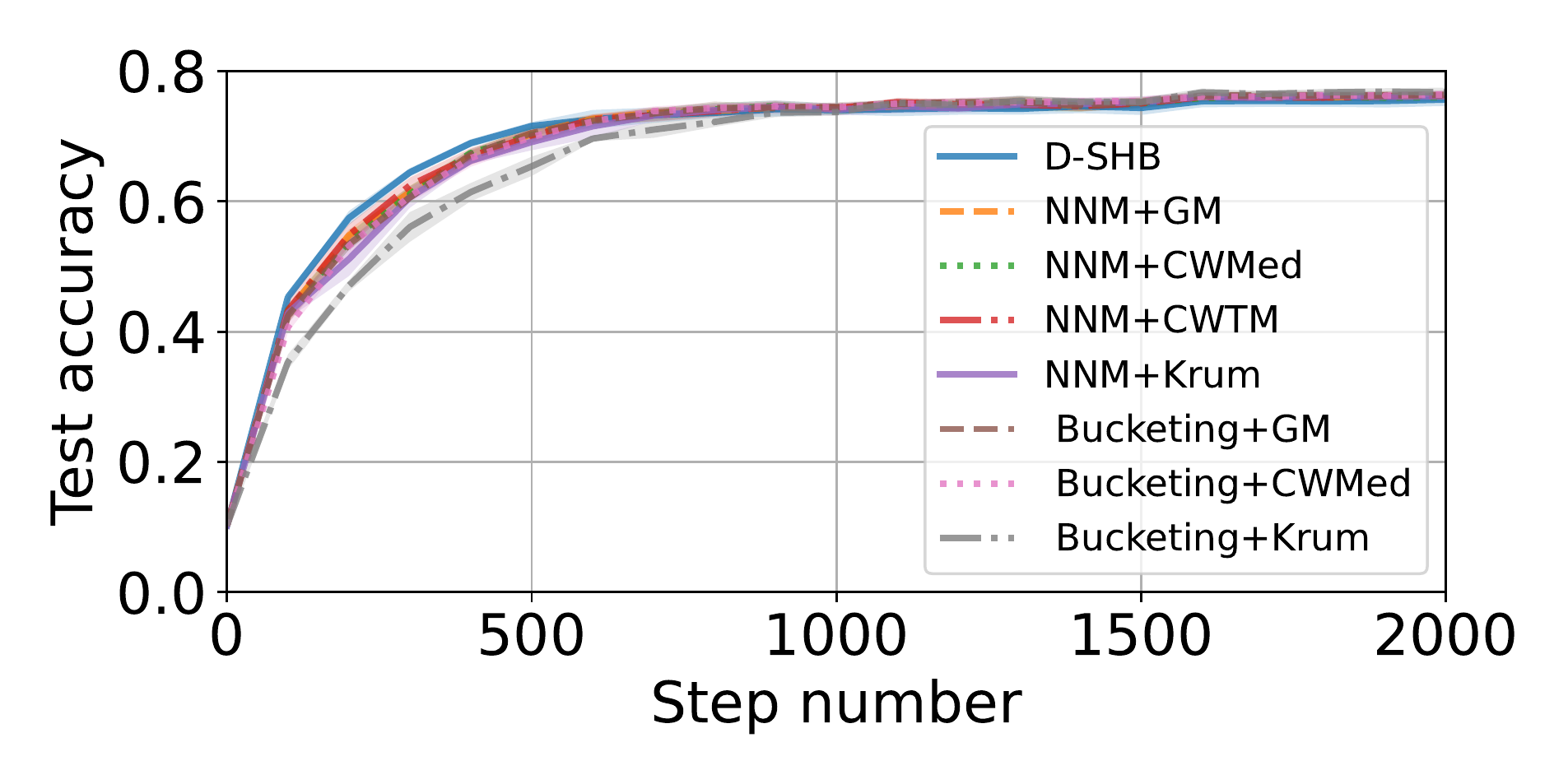}%
    \includegraphics[width=0.5\textwidth]{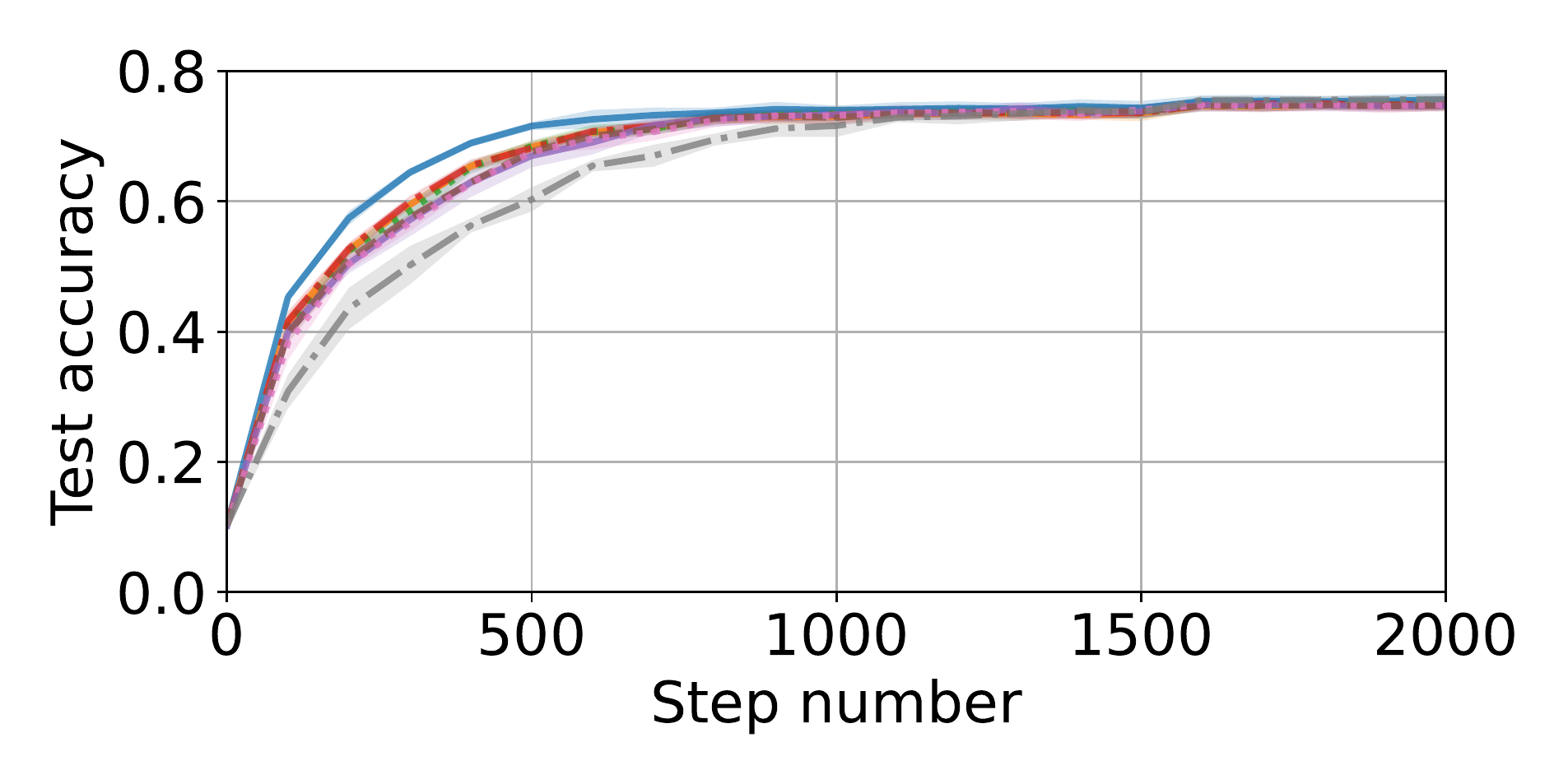}\\%
     \includegraphics[width=0.5\textwidth]{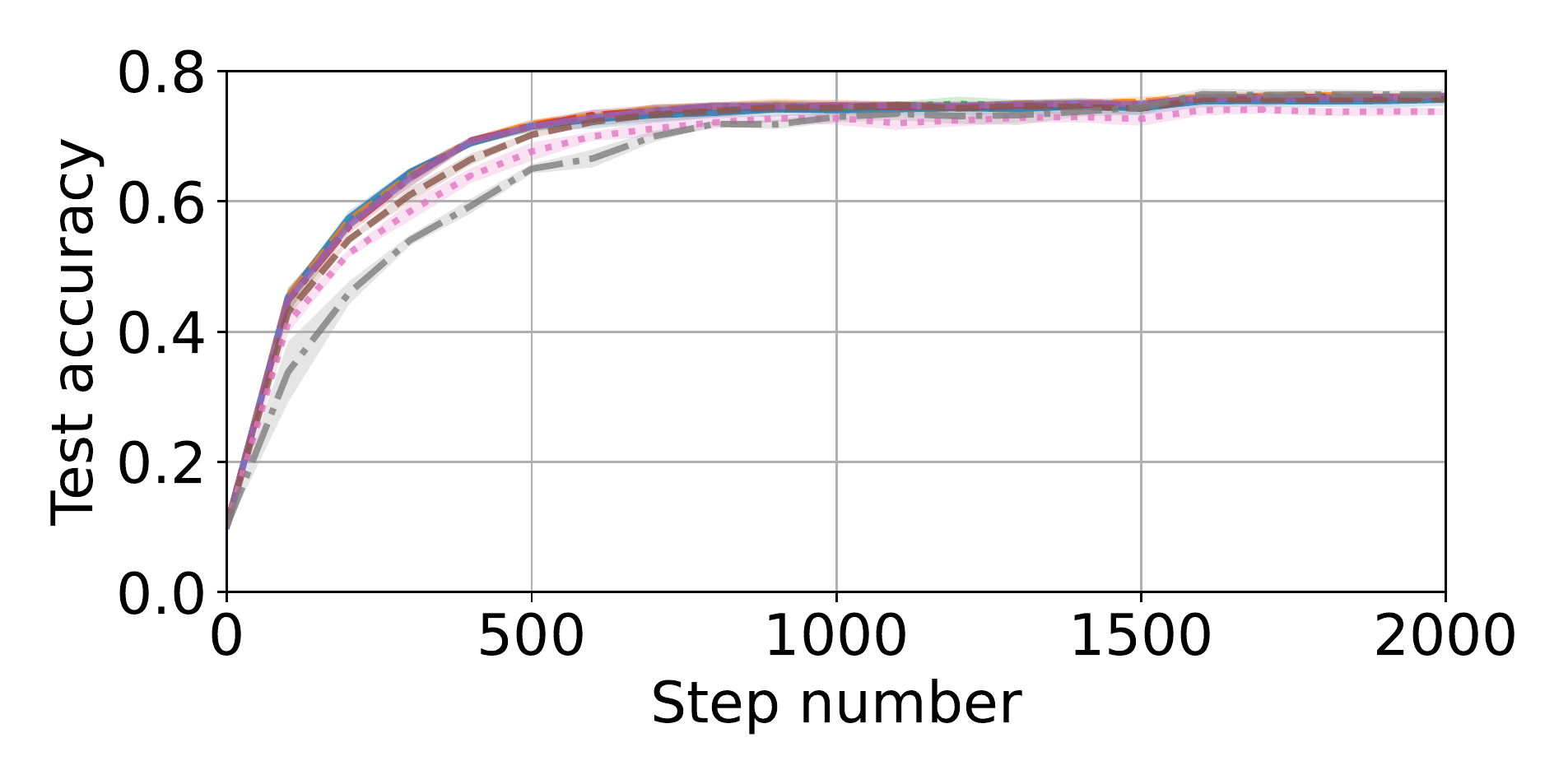}%
    \caption{Experiments on CIFAR-10 using robust D-SHB with $f = 2$ Byzantine among $n = 17$ workers, with $\beta = 0.9$ and $\alpha = 10$. The Byzantine workers execute the FOE (\textit{row 1, left}), ALIE (\textit{row 1, right}), Mimic (\textit{row 2, left}), SF (\textit{row 2, right}), and LF (\textit{row 3}) attacks.}
\label{fig:plots_cifar_2}
\end{figure*}

\begin{figure*}[ht!]
    \centering
    \includegraphics[width=0.5\textwidth]{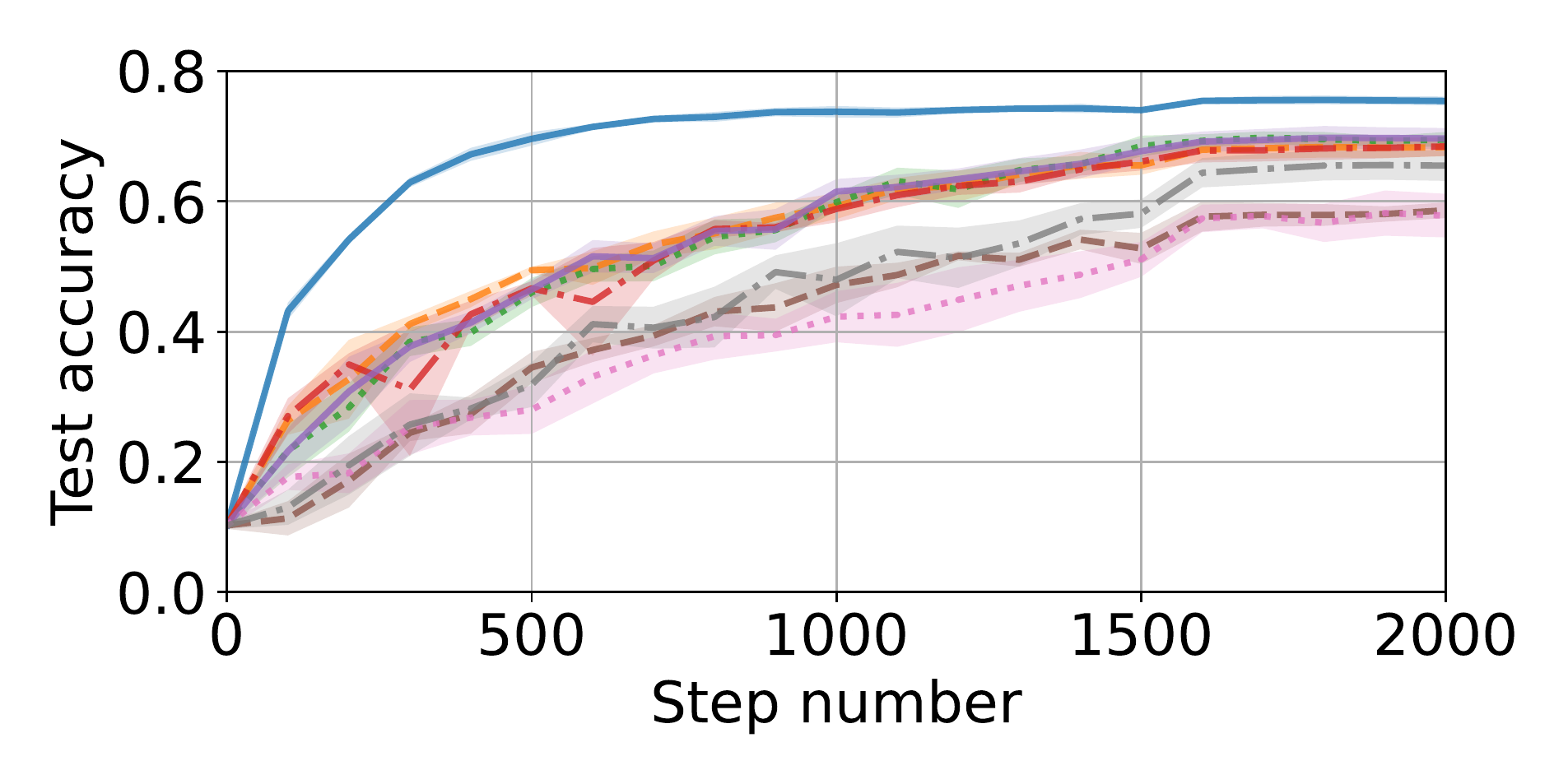}%
    \includegraphics[width=0.5\textwidth]{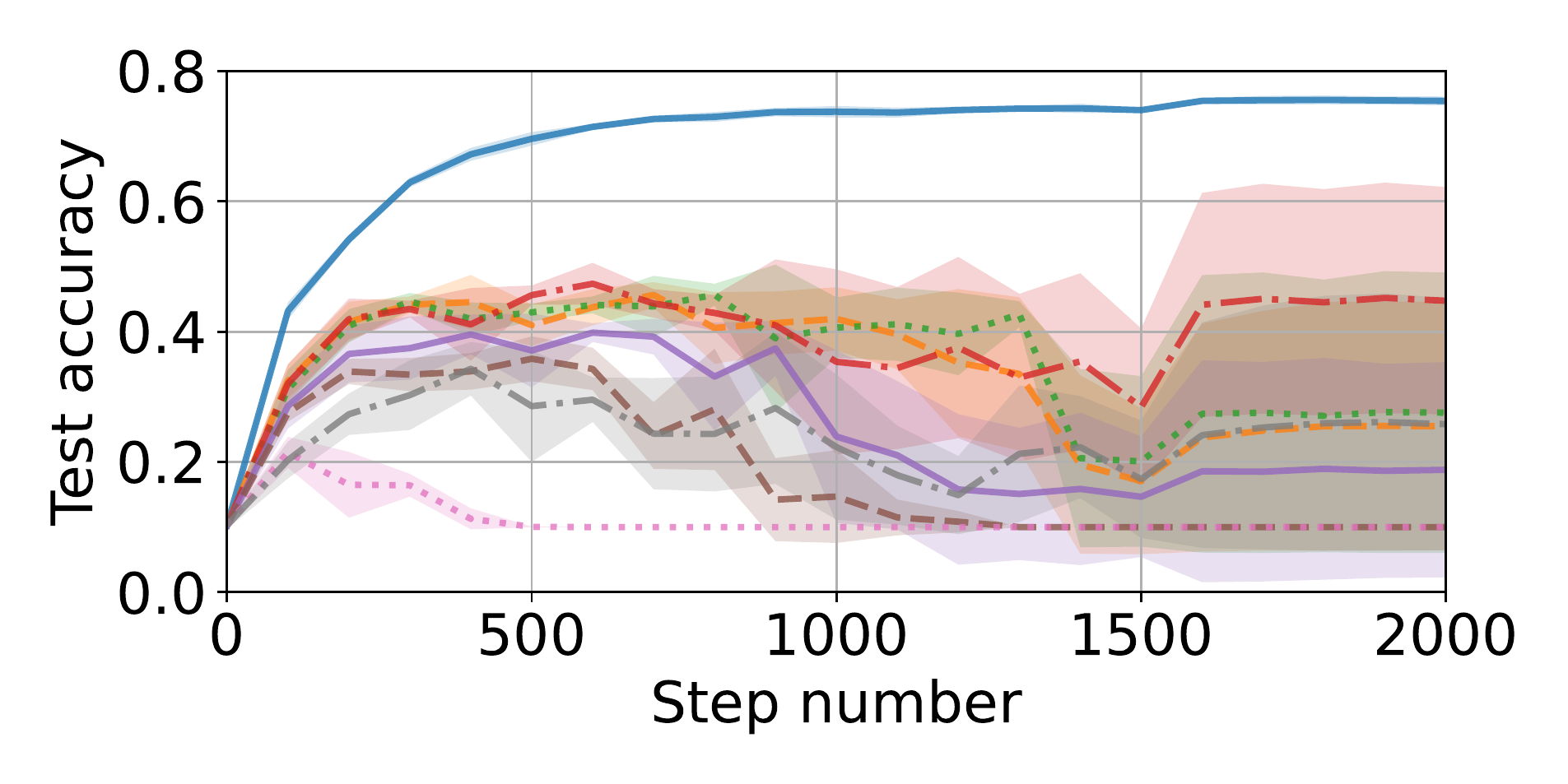}\\%
    \includegraphics[width=0.5\textwidth]{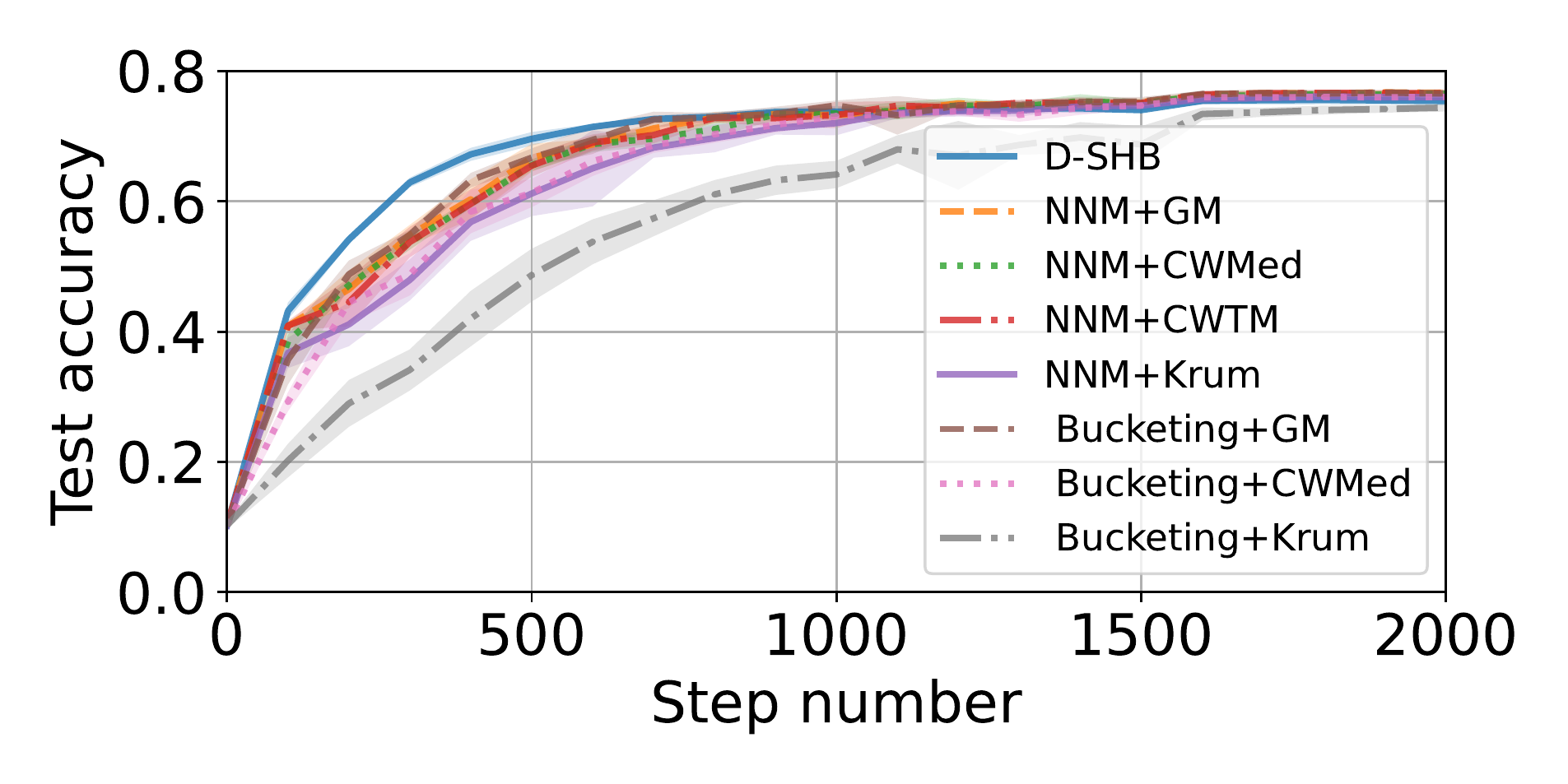}%
    \includegraphics[width=0.5\textwidth]{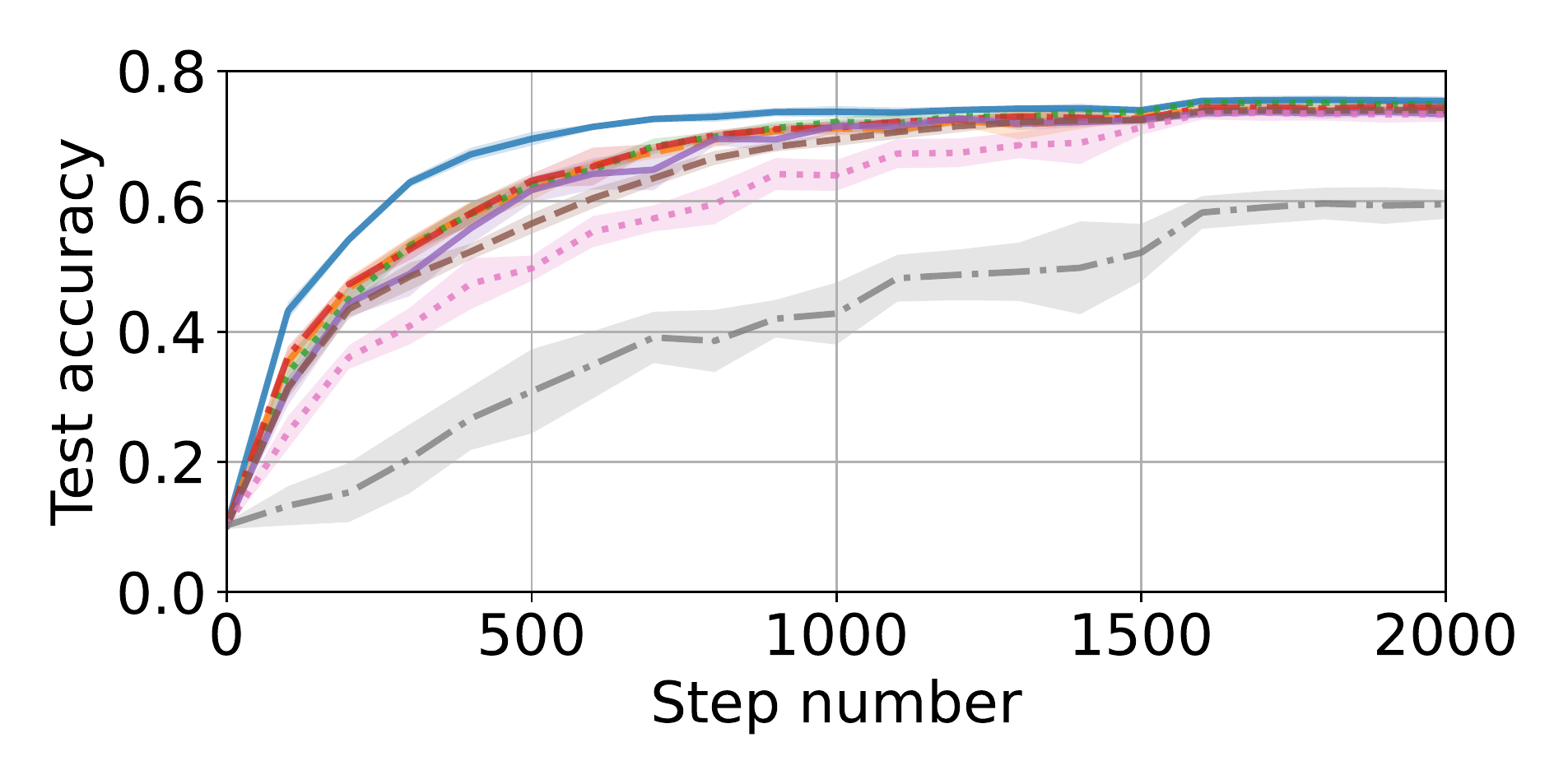}\\%
     \includegraphics[width=0.5\textwidth]{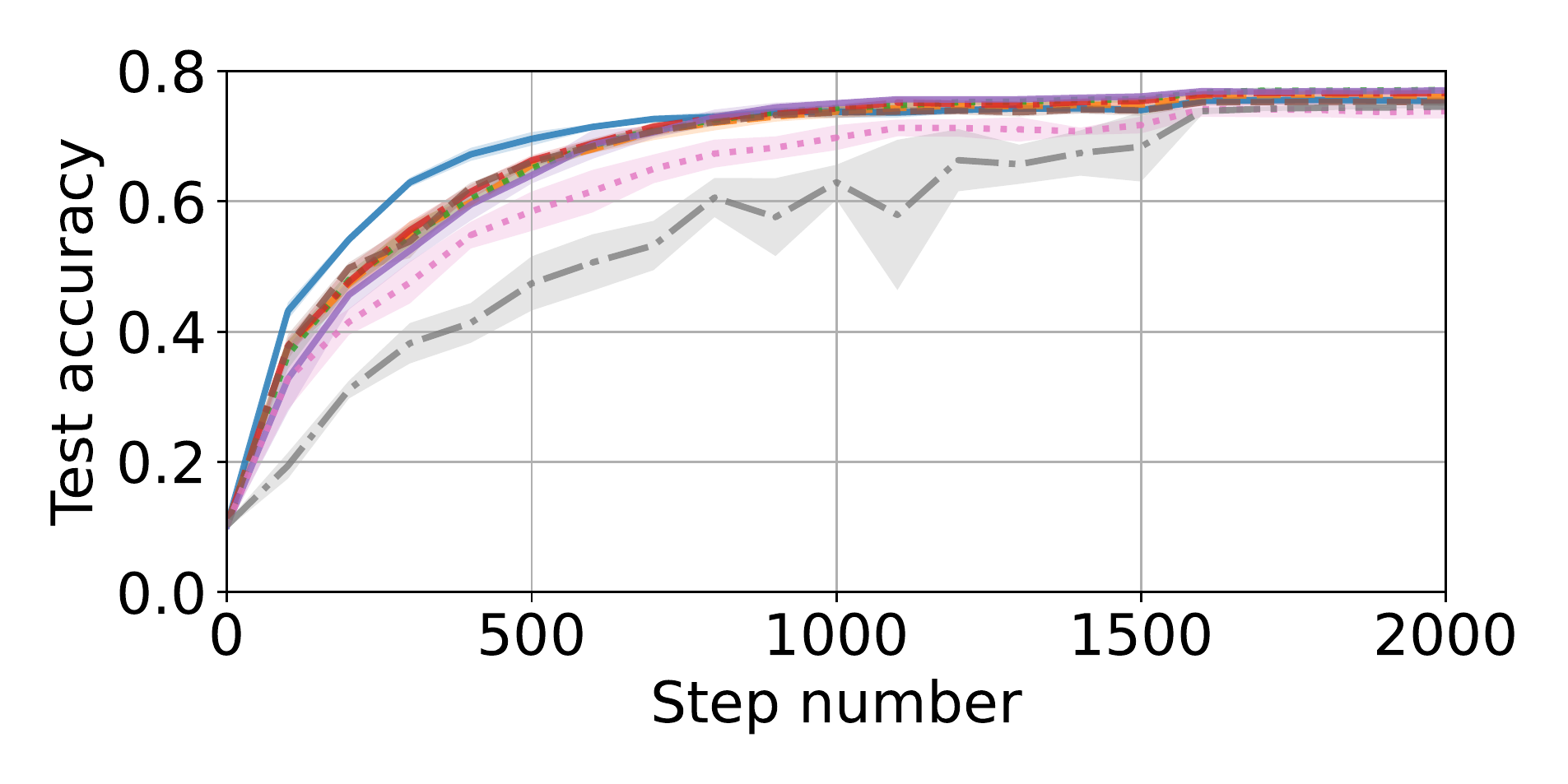}%
    \caption{Experiments on CIFAR-10 using robust D-SHB with $f = 3$ Byzantine among $n = 17$ workers, with $\beta = 0.9$ and $\alpha = 1$. The Byzantine workers execute the FOE (\textit{row 1, left}), ALIE (\textit{row 1, right}), Mimic (\textit{row 2, left}), SF (\textit{row 2, right}), and LF (\textit{row 3}) attacks.}
\label{fig:plots_cifar_3}
\end{figure*}

\begin{figure*}[ht!]
    \centering
    \includegraphics[width=0.5\textwidth]{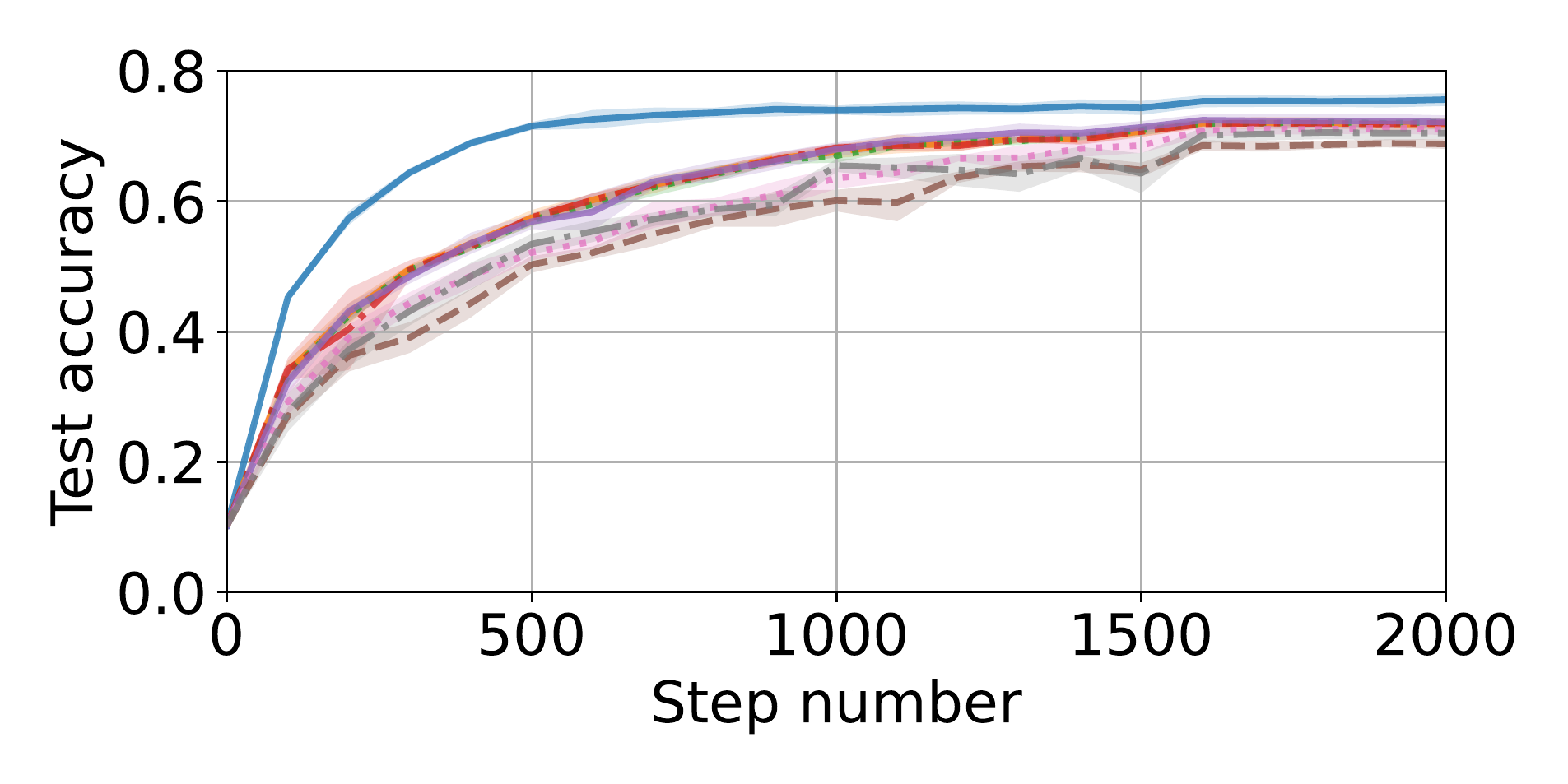}%
    \includegraphics[width=0.5\textwidth]{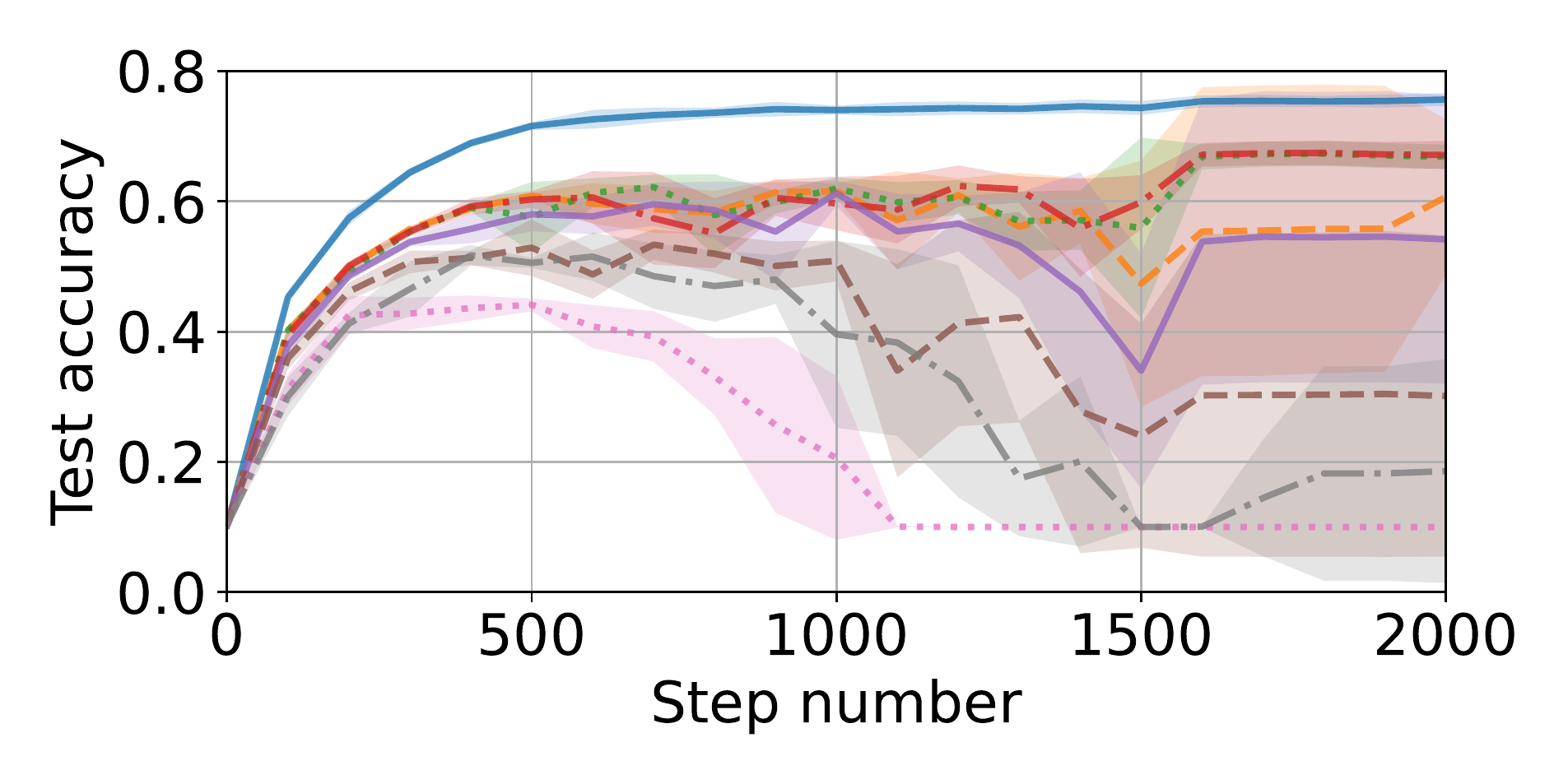}\\%
    \includegraphics[width=0.5\textwidth]{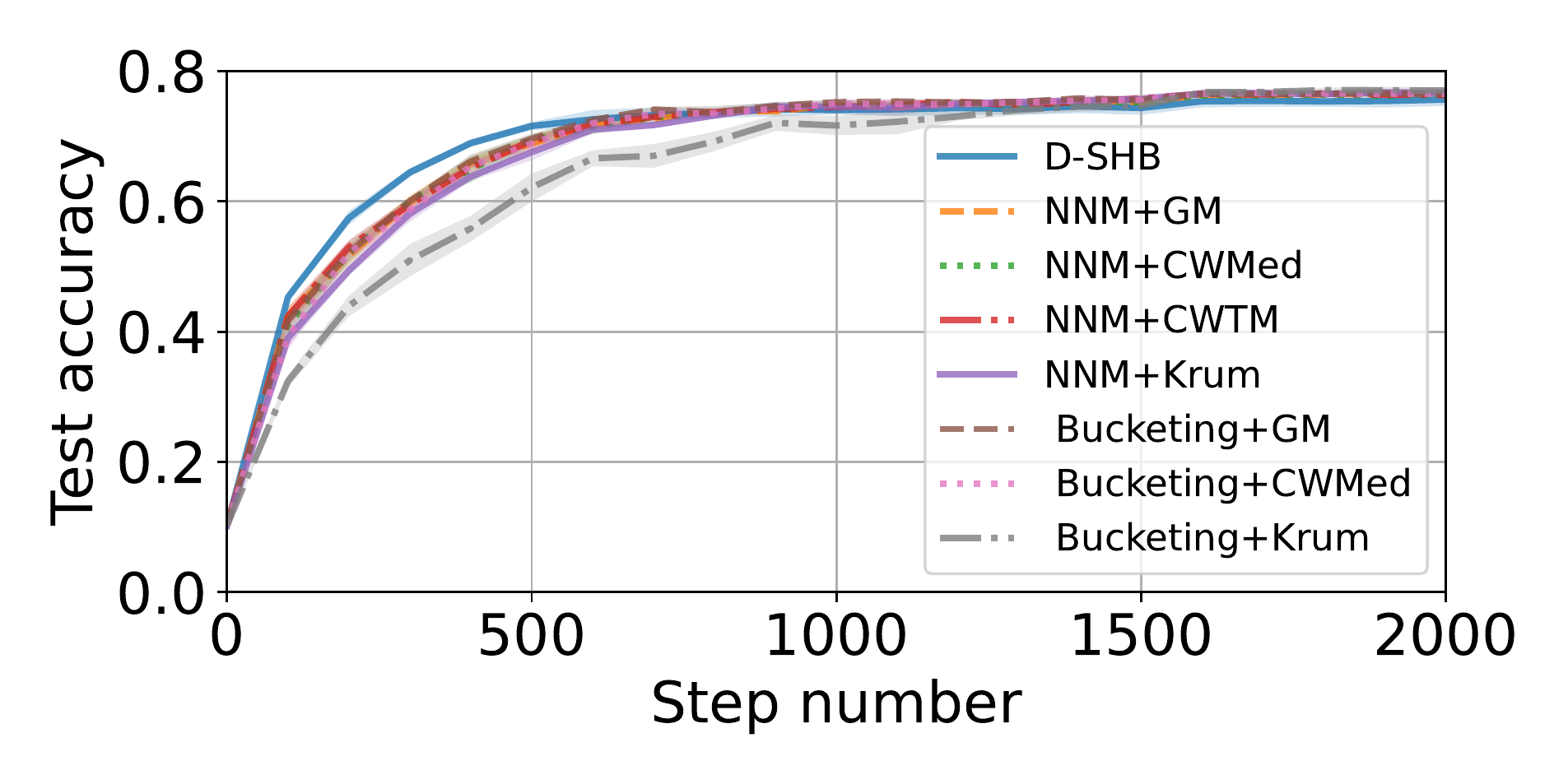}%
    \includegraphics[width=0.5\textwidth]{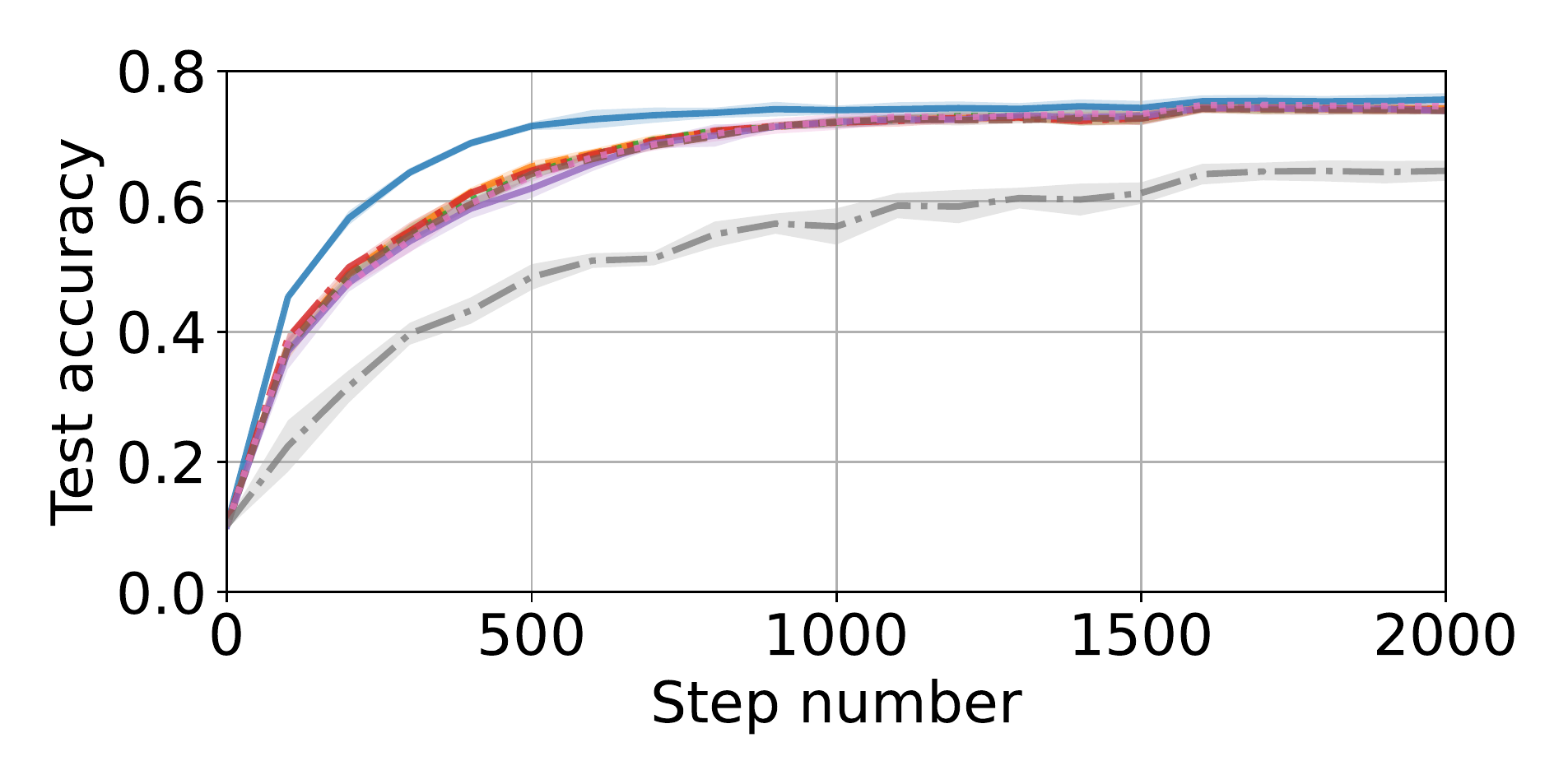}\\%
     \includegraphics[width=0.5\textwidth]{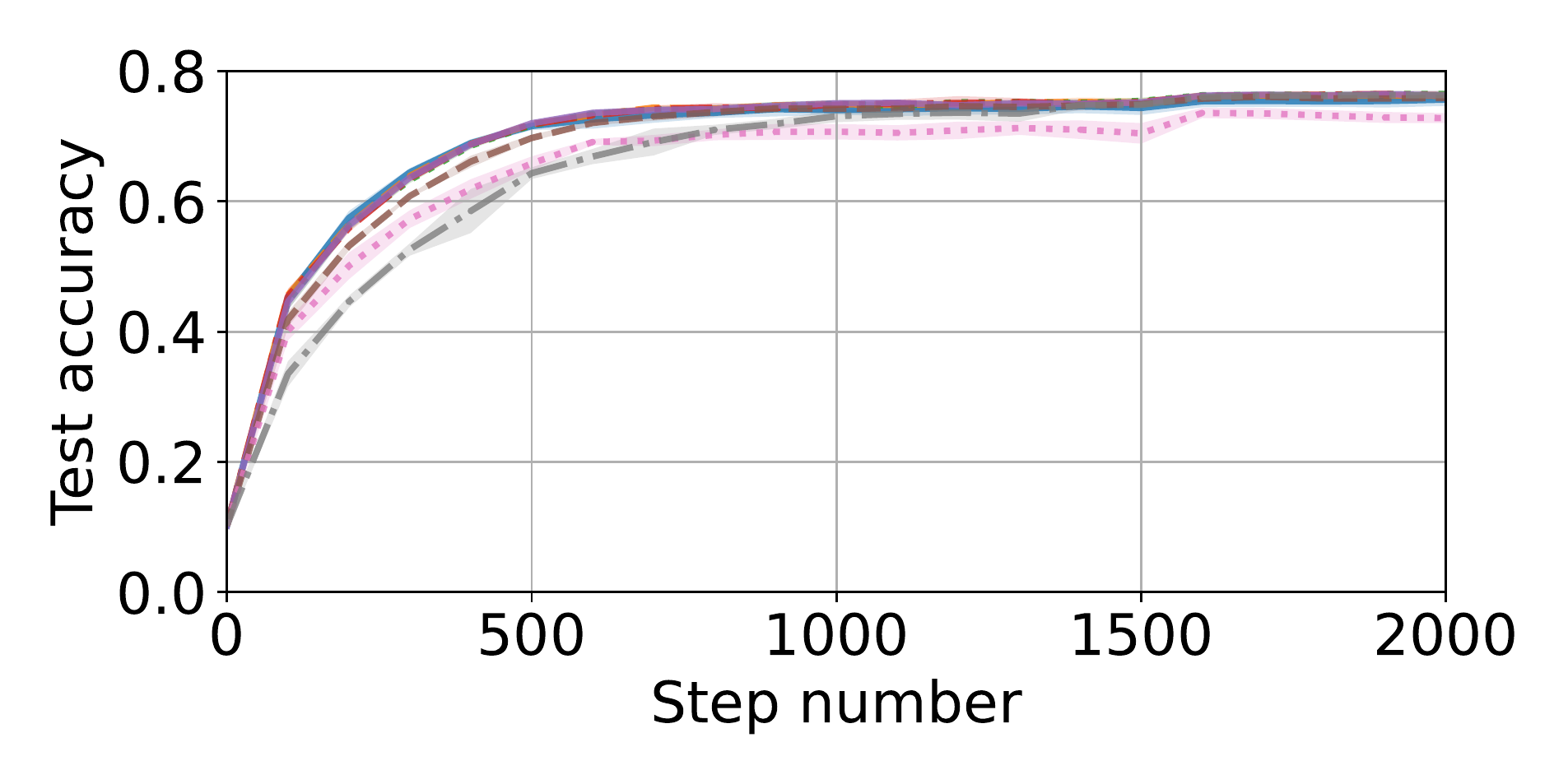}%
    \caption{Experiments on CIFAR-10 using robust D-SHB with $f = 3$ Byzantine among $n = 17$ workers, with $\beta = 0.9$ and $\alpha = 10$. The Byzantine workers execute the FOE (\textit{row 1, left}), ALIE (\textit{row 1, right}), Mimic (\textit{row 2, left}), SF (\textit{row 2, right}), and LF (\textit{row 3}) attacks.}
\label{fig:plots_cifar_4}
\end{figure*}

\begin{figure*}[ht!]
    \centering
    \includegraphics[width=0.5\textwidth]{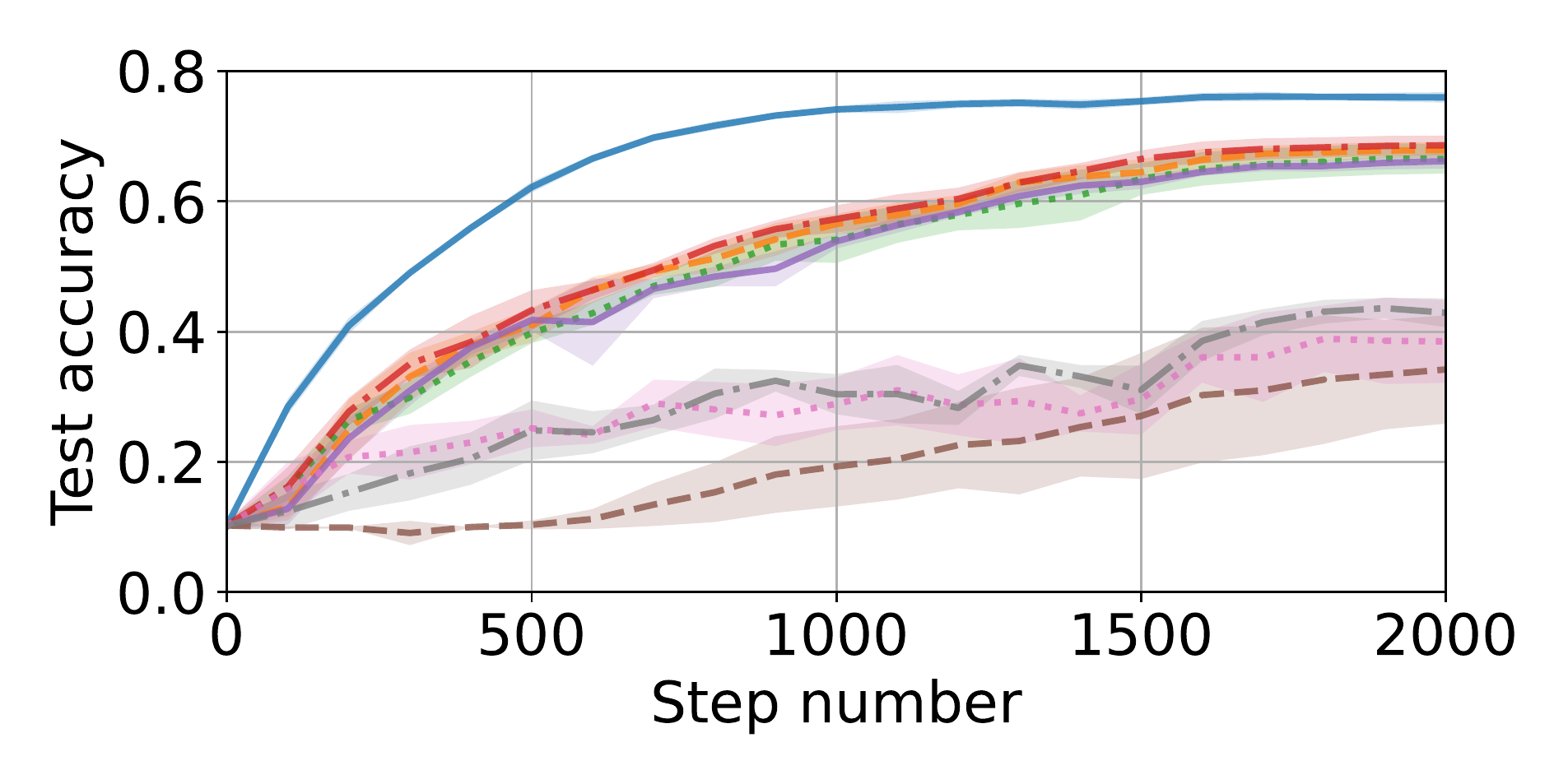}%
    \includegraphics[width=0.5\textwidth]{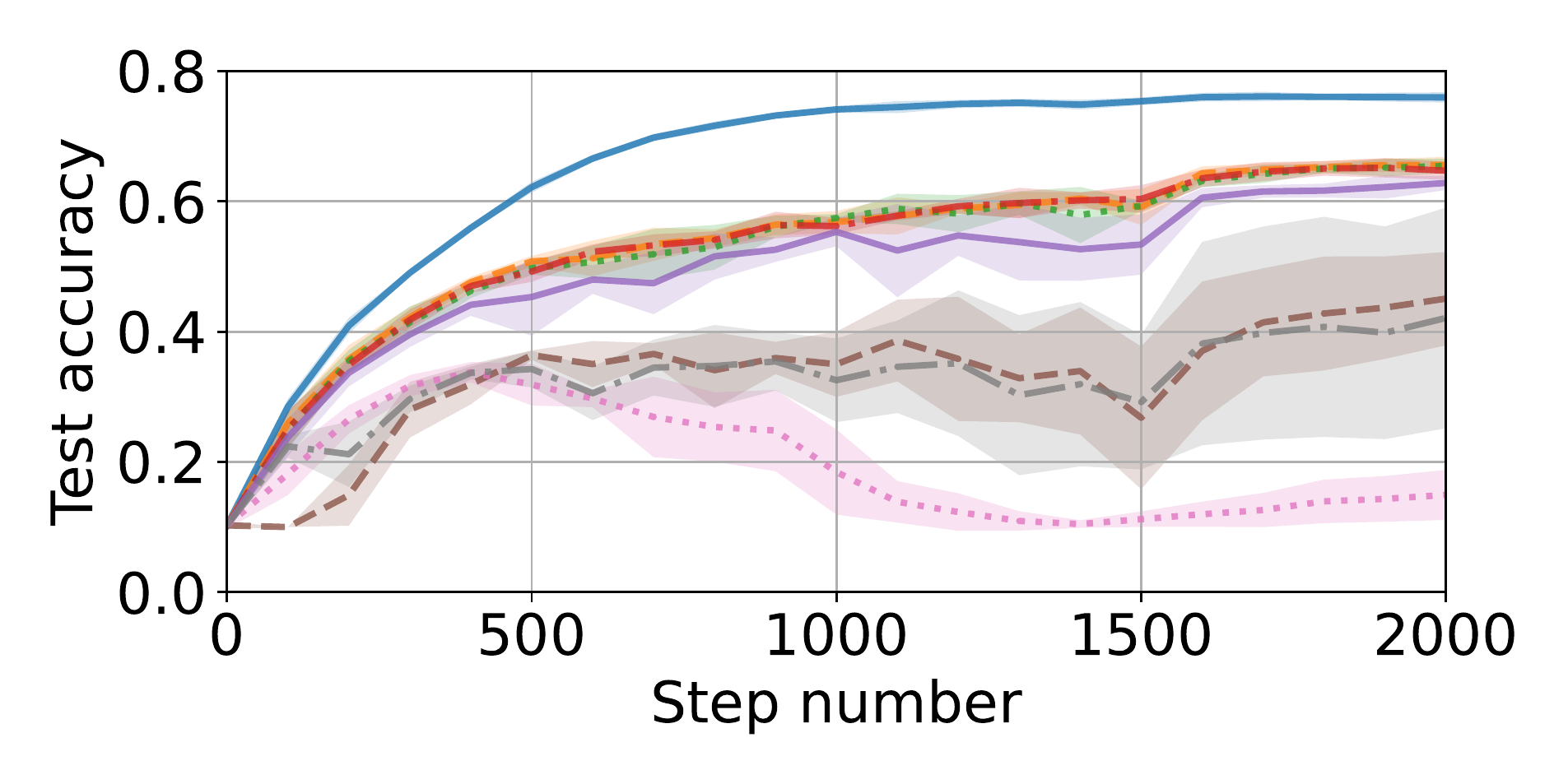}\\%
    \includegraphics[width=0.5\textwidth]{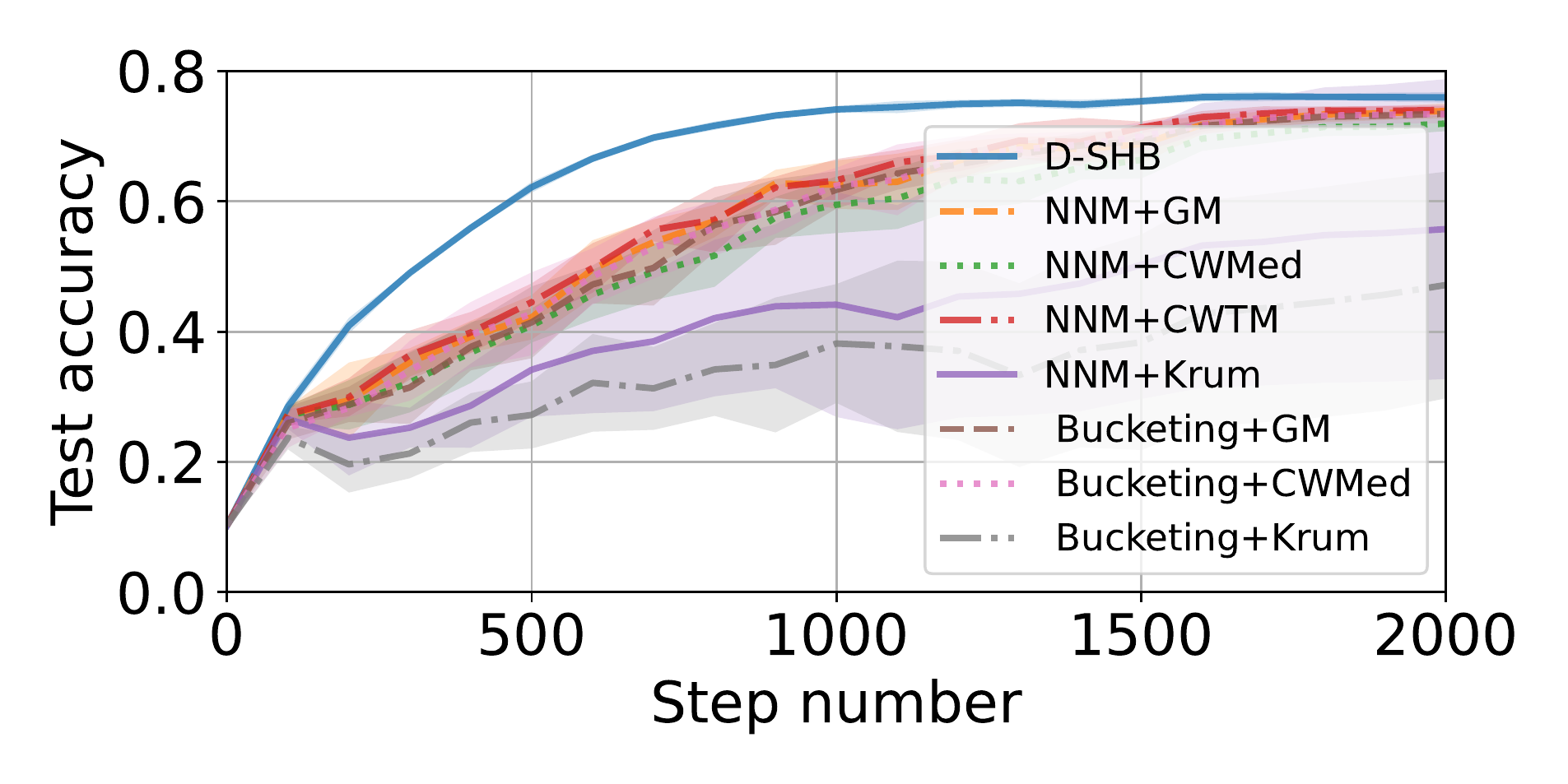}%
    \includegraphics[width=0.5\textwidth]{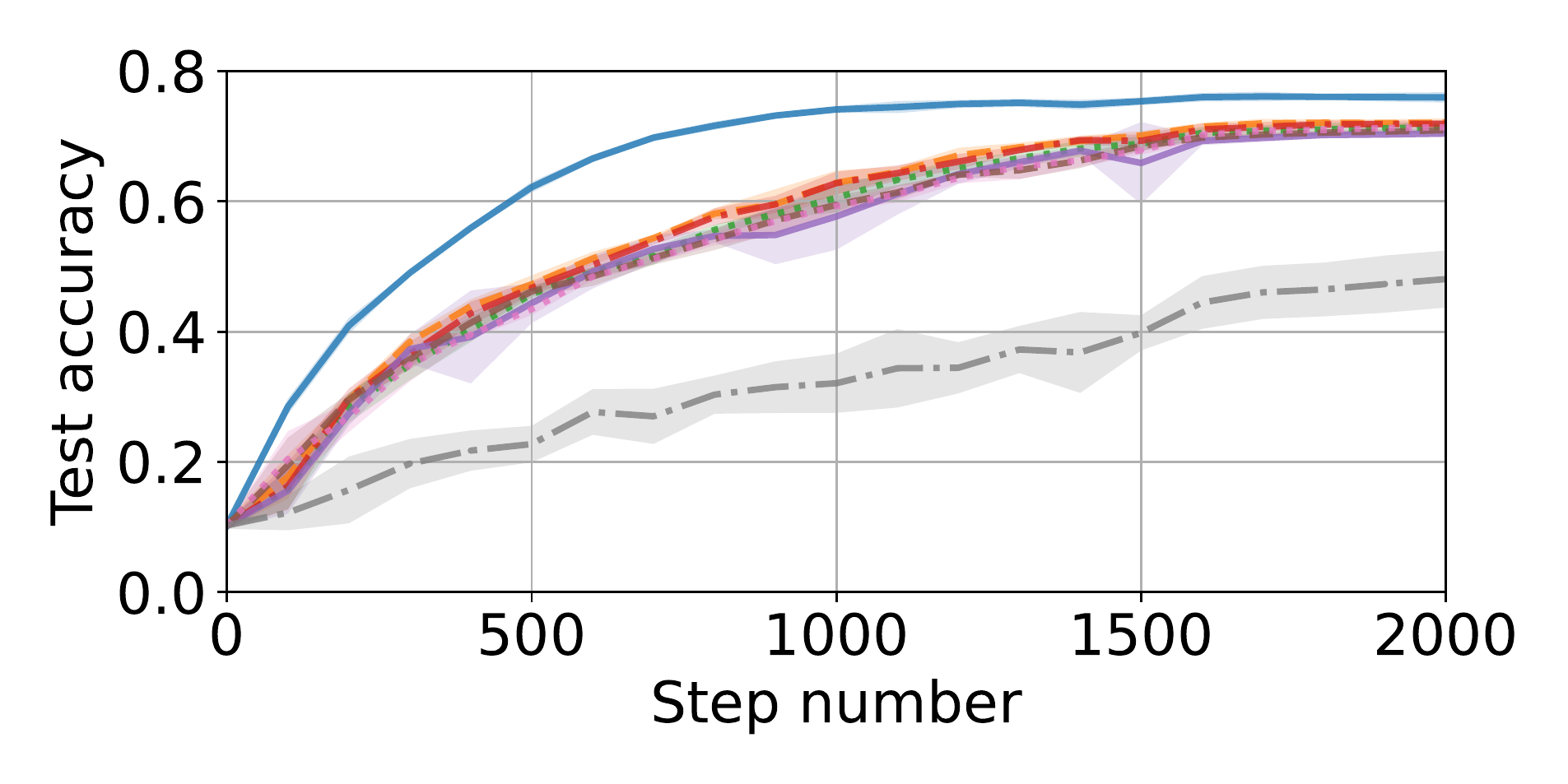}\\%
     \includegraphics[width=0.5\textwidth]{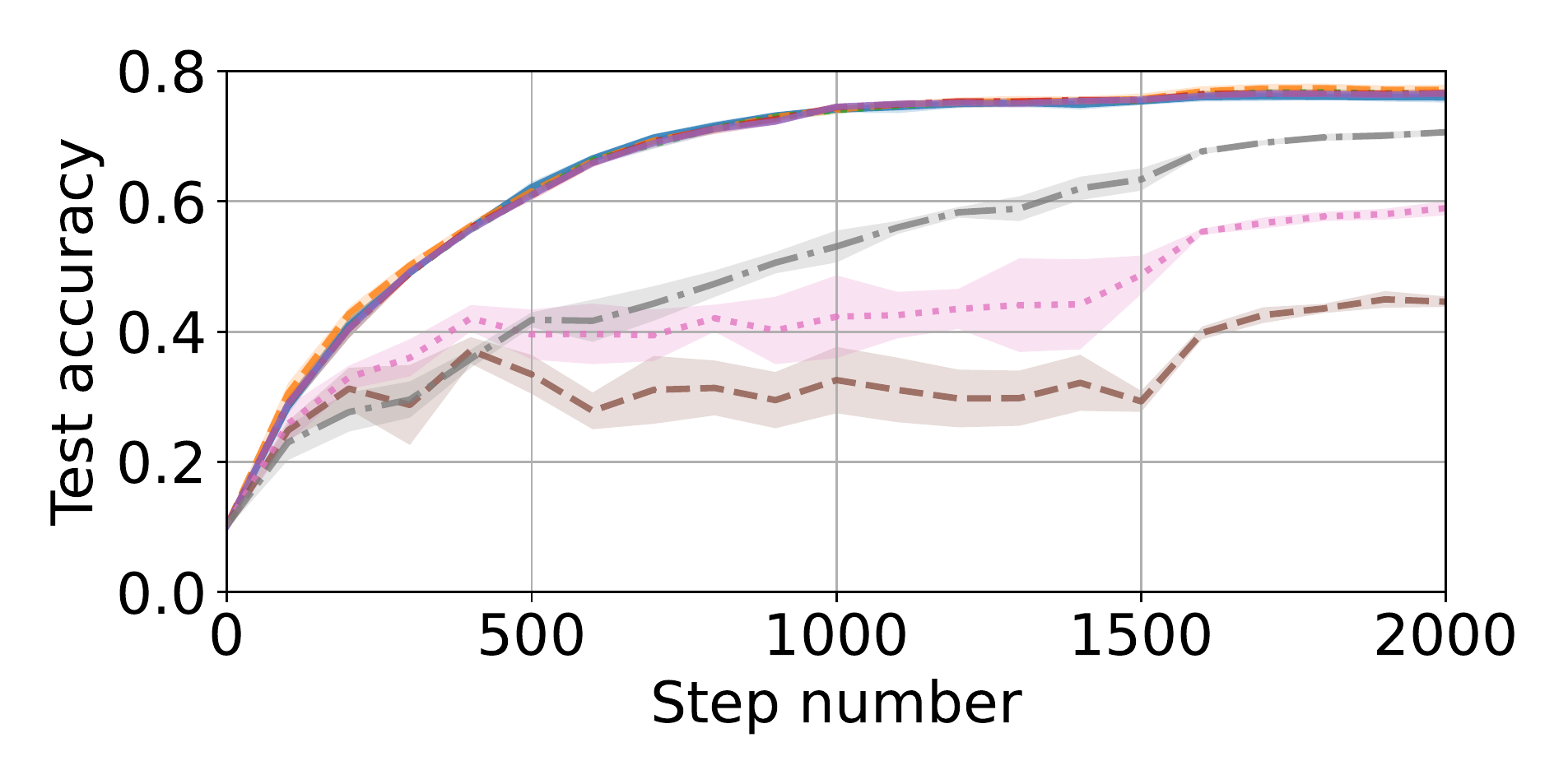}%
    \caption{Experiments on CIFAR-10 using robust D-SHB with $f = 4$ Byzantine among $n = 17$ workers, with $\beta = 0.99$ and $\alpha = 10$. The Byzantine workers execute the FOE (\textit{row 1, left}), ALIE (\textit{row 1, right}), Mimic (\textit{row 2, left}), SF (\textit{row 2, right}), and LF (\textit{row 3}) attacks.}
\label{fig:plots_cifar_5}
\end{figure*}

\end{document}


%

%

\onecolumn
\aistatstitle{Instructions for Paper Submissions to AISTATS 2022: \\
Supplementary Materials}

\section{FORMATTING INSTRUCTIONS}

To prepare a supplementary pdf file, we ask the authors to use \texttt{aistats2022.sty} as a style file and to follow the same formatting instructions as in the main paper.
The only difference is that the supplementary material must be in a \emph{single-column} format.
You can use \texttt{supplement.tex} in our starter pack as a starting point, or append the supplementary content to the main paper and split the final PDF into two separate files.

Note that reviewers are under no obligation to examine your supplementary material.

\section{MISSING PROOFS}

The supplementary materials may contain detailed proofs of the results that are missing in the main paper.

\subsection{Proof of Lemma 3}

\textit{In this section, we present the detailed proof of Lemma 3 and then [ ... ]}

\section{ADDITIONAL EXPERIMENTS}

If you have additional experimental results, you may include them in the supplementary materials.

\subsection{The Effect of Regularization Parameter}

\textit{Our algorithm depends on the regularization parameter $\lambda$. Figure 1 below illustrates the effect of this parameter on the performance of our algorithm. As we can see, [ ... ]}

\section{MISSING PROOFS}

The supplementary materials may contain detailed proofs of the results that are missing in the main paper.

\subsection{Proof of Lemma 3}

\textit{In this section, we present the detailed proof of Lemma 3 and then [ ... ]}

\section{ADDITIONAL EXPERIMENTS}

If you have additional experimental results, you may include them in the supplementary materials.

\subsection{The Effect of Regularization Parameter}

\textit{Our algorithm depends on the regularization parameter $\lambda$. Figure 1 below illustrates the effect of this parameter on the performance of our algorithm. As we can see, [ ... ]}

\section{MISSING PROOFS}

The supplementary materials may contain detailed proofs of the results that are missing in the main paper.

\subsection{Proof of Lemma 3}

\textit{In this section, we present the detailed proof of Lemma 3 and then [ ... ]}

\section{ADDITIONAL EXPERIMENTS}

If you have additional experimental results, you may include them in the supplementary materials.

\subsection{The Effect of Regularization Parameter}

\textit{Our algorithm depends on the regularization parameter $\lambda$. Figure 1 below illustrates the effect of this parameter on the performance of our algorithm. As we can see, [ ... ]}

\vfill